%% file: main.tex
\input{preamble}
\begin{document}
\input{sections/front}

\input{sections/sec1_intro}
\input{sections/sec2_related}

\input{sections/sec3_metrics}

\input{sections/sec4_ooi}

\input{sections/sec5_measure}
\input{sections/sec6_setup}
\input{sections/sec7_results}
\input{sections/sec8_discussion}
\input{sections/sec9_conclusion}
\bibliographystyle{ACM-Reference-Format}
\bibliography{master}
\appendix
\input{sections/appendix}
\end{document}

%% file: preamble.tex
\documentclass[sigconf,nonacm]{acmart}
\setcopyright{none}
\acmDOI{}
\acmISBN{}

\acmYear{2027}
\copyrightyear{2027}

\usepackage{booktabs}
\usepackage{enumitem}
\usepackage{colortbl}
\definecolor{alems}{RGB}{31,119,180}
\definecolor{catrow}{RGB}{240,240,240}
\definecolor{hdrblue}{RGB}{208,228,245}
\definecolor{hdrgreen}{RGB}{208,237,222}
\definecolor{hdrorange}{RGB}{253,225,197}
\definecolor{cx86}{RGB}{33,102,172}
\definecolor{cgn}{RGB}{230,85,13}
\definecolor{cnet}{RGB}{117,107,177}
\definecolor{cgreen}{RGB}{26,150,65}
\usepackage{rotating}
\usepackage{tabularx}
\usepackage{multirow}
\usepackage{listings}
\usepackage{amsmath}
\usepackage{graphicx}
\usepackage{xcolor}
\usepackage{tikz}
\usepackage{pgfplots}
\usepackage{array}
\usepackage{microtype}
\usepackage{balance}
\usepackage{subcaption}
\usepackage[font=small,labelfont=bf]{caption}

\usepackage{amsthm}
\usepackage{dsfont}
\usepackage{threeparttable}
\theoremstyle{plain}

\theoremstyle{definition}

\theoremstyle{remark}

\pgfplotsset{compat=1.18}
\usetikzlibrary{shapes.geometric,arrows.meta,positioning,
  fit,calc,decorations.pathreplacing,backgrounds,patterns,shadows}

\lstdefinestyle{sql}{
  language=SQL,
  basicstyle=\ttfamily\scriptsize,
  keywordstyle=\color{blue!75!black}\bfseries,
  commentstyle=\color{gray!70}\itshape,
  breaklines=true,frame=single,
  numbers=left,numberstyle=\tiny\color{gray},
  xleftmargin=1.2em,framexleftmargin=1.2em,
}

\newcommand{\EpG}{\textsc{EpG}}
\newcommand{\OOI}{\textsc{OOI}}
\newcommand{\ALEMS}{\textsc{A-LEMS}}
\newcommand{\RAPL}{\textsc{RAPL}}

\newcommand{\Epkg}{E_{\mathrm{pkg}}}
\newcommand{\Edyn}{E_{\mathrm{dyn}}}
\newcommand{\tzero}{t_0}
\newcommand{\tone}{t_1}
\newcommand{\ttwo}{t_2}
\newcommand{\fcpu}{f_{\mathrm{cpu}}}

\newcommand{\Hhard}{\mathcal{H}_{\mathrm{hw}}}
\newcommand{\Henv}{\mathcal{H}_{\mathrm{env}}}
\newcommand{\Hrun}{\mathcal{H}_{\mathrm{run}}}

\newcommand{\HEADLINEEPG}{}

\if\relax\HEADLINEEPG\relax
  \renewcommand{\HEADLINEEPG}{6.05\times}
\fi

\makeatletter
\newcommand{\defval}[2]{\@namedef{val@#1}{#2}}
\newcommand{\val}[1]{\@ifundefined{val@#1}{\textcolor{red}{\textbf{TBD}\,\texttt{\detokenize{#1}}}}{\@nameuse{val@#1}}}
\makeatother
\IfFileExists{generated/values.tex}{\input{generated/values.tex}}{}

\newcommand{\ph}[3]{%
\vspace{2pt}%
\begin{center}%
\begin{tikzpicture}%
\draw[thick,gray!40,dashed,rounded corners=3pt]%
  (0,0) rectangle (\dimexpr\linewidth-6pt\relax,#1cm);%
\node[font=\small\bfseries,gray!55,align=center]%
  at (0.5\linewidth-3pt,#1cm*0.68){#2};%
\node[font=\scriptsize,gray!45,align=center]%
  at (0.5\linewidth-3pt,#1cm*0.32){#3};%
\end{tikzpicture}%
\end{center}%
\vspace{2pt}}

\newcommand{\phd}[3]{%
\vspace{2pt}%
\begin{center}%
\begin{tikzpicture}%
\draw[thick,gray!40,dashed,rounded corners=3pt]%
  (0,0) rectangle (\textwidth-6pt,#1cm);%
\node[font=\small\bfseries,gray!55,align=center]%
  at (0.5\textwidth-3pt,#1cm*0.68){#2};%
\node[font=\scriptsize,gray!45,align=center]%
  at (0.5\textwidth-3pt,#1cm*0.32){#3};%
\end{tikzpicture}%
\end{center}%
\vspace{2pt}}

\newcommand{\TBD}[1][]{\textcolor{red}{\textbf{TBD#1}}}

\newcommand{\EpCG}{\textsc{EpCG}}
\newcommand{\EpSG}{\textsc{EpSG}}
\newcommand{\TpG}{\textsc{TpG}}
\newcommand{\TpCG}{\textsc{TpCG}}
\newcommand{\TpSG}{\textsc{TpSG}}
\definecolor{fnavy}{HTML}{1F3A68}
\definecolor{flight}{HTML}{B1BACA}
\definecolor{fgray}{HTML}{A7ADB6}
\definecolor{forange}{HTML}{E69F00}
\definecolor{famber}{HTML}{F5CF6B}
\definecolor{fpale}{HTML}{FBE7C2}
\definecolor{fexcl}{HTML}{E6E8EB}
\definecolor{tNbd}{HTML}{F8EFF5}

\definecolor{fgreen}{HTML}{0A7F52}
\definecolor{fwrong}{HTML}{C26A00}
\definecolor{fna}{HTML}{8A9099}
\definecolor{fzebra}{HTML}{F5F6F8}
\definecolor{tXhd}{HTML}{CFE3F3}
\definecolor{tXbd}{HTML}{EEF5FB}
\definecolor{tNhd}{HTML}{EBD5E4}
\definecolor{tNbd}{HTML}{F8EFF5}
\definecolor{tRhd}{HTML}{F2D9D9}
\definecolor{tRbd}{HTML}{FBF0F0}
\definecolor{tGhd}{HTML}{F8DCC8}
\definecolor{tGbd}{HTML}{FDF1E8}
\definecolor{tMhd}{HTML}{CFE8E8}
\definecolor{tMbd}{HTML}{EDF7F7}
\definecolor{tThd}{HTML}{E3E5E8}
\definecolor{tTbd}{HTML}{F4F5F6}
\definecolor{tQhd}{HTML}{F5EBC4}
\definecolor{tQbd}{HTML}{FCF7E6}
\definecolor{tChd}{HTML}{D5EEDF}
\definecolor{tCbd}{HTML}{EFF8F2}
\definecolor{tBhd}{HTML}{E4DFF2}
\definecolor{tBbd}{HTML}{F4F1FA}
\newcommand{\vabove}[1]{\textcolor{fnavy}{\textbf{#1}}}
\newcommand{\vbelow}[1]{\textcolor{fgreen}{\textbf{#1}}}
\newcommand{\vna}[1]{\textcolor{fna}{\textit{#1}}}
\usetikzlibrary{decorations.pathreplacing}

\usepackage{framed}
\colorlet{shadecolor}{alems!12}
\newcounter{findingctr}
\newenvironment{finding}[1]{\refstepcounter{findingctr}\begin{shaded}\small\noindent\textbf{Implication (#1).}~}{\end{shaded}}

%% file: generated/values.tex
\defval{bfcl.A.agentic.E_wrong_goals_J}{445,622}
\defval{bfcl.A.agentic.EpCG_J}{1,016.8}
\defval{bfcl.A.agentic.EpG_J}{975.8}
\defval{bfcl.A.agentic.EpSG_J}{1,135.0}
\defval{bfcl.A.agentic.N}{4,322}
\defval{bfcl.A.agentic.TpG_s}{18.9}
\defval{bfcl.A.agentic.attempts_per_goal}{1.90}
\defval{bfcl.A.agentic.n_completed}{4,148}
\defval{bfcl.A.agentic.n_correct}{3,716}
\defval{bfcl.A.agentic.p_c}{0.96}
\defval{bfcl.A.agentic.p_s}{0.86}
\defval{bfcl.A.agentic.power_W}{51.5}
\defval{bfcl.B.agentic.E_wrong_goals_J}{159,672}
\defval{bfcl.B.agentic.EpCG_J}{2,321.9}
\defval{bfcl.B.agentic.EpG_J}{2,223.1}
\defval{bfcl.B.agentic.EpSG_J}{2,833.5}
\defval{bfcl.B.agentic.N}{376}
\defval{bfcl.B.agentic.TpG_s}{39.3}
\defval{bfcl.B.agentic.attempts_per_goal}{1.83}
\defval{bfcl.B.agentic.n_completed}{360}
\defval{bfcl.B.agentic.n_correct}{295}
\defval{bfcl.B.agentic.p_c}{0.96}
\defval{bfcl.B.agentic.p_s}{0.78}
\defval{bfcl.B.agentic.power_W}{56.6}
\defval{bfcl.C.agentic.E_wrong_goals_J}{76,213}
\defval{bfcl.C.agentic.EpCG_J}{1,789.0}
\defval{bfcl.C.agentic.EpG_J}{1,747.4}
\defval{bfcl.C.agentic.EpSG_J}{2,126.6}
\defval{bfcl.C.agentic.EpSG_hi_J}{2,127}
\defval{bfcl.C.agentic.EpSG_lo_J}{2,043}
\defval{bfcl.C.agentic.N}{387}
\defval{bfcl.C.agentic.TpG_s}{29.7}
\defval{bfcl.C.agentic.attempts_per_goal}{1.92}
\defval{bfcl.C.agentic.n_completed}{378}
\defval{bfcl.C.agentic.n_correct}{318}
\defval{bfcl.C.agentic.n_unknown}{13}
\defval{bfcl.C.agentic.p_c}{0.98}
\defval{bfcl.C.agentic.p_s}{0.82}
\defval{bfcl.C.agentic.power_W}{58.8}
\defval{bfcl.D.agentic.E_wrong_goals_J}{177,744}
\defval{bfcl.D.agentic.EpCG_J}{3,458.3}
\defval{bfcl.D.agentic.EpG_J}{3,458.3}
\defval{bfcl.D.agentic.EpSG_J}{4,545.7}
\defval{bfcl.D.agentic.N}{255}
\defval{bfcl.D.agentic.TpG_s}{55.2}
\defval{bfcl.D.agentic.attempts_per_goal}{1.75}
\defval{bfcl.D.agentic.n_completed}{255}
\defval{bfcl.D.agentic.n_correct}{194}
\defval{bfcl.D.agentic.p_c}{1.00}
\defval{bfcl.D.agentic.p_s}{0.76}
\defval{bfcl.D.agentic.power_W}{62.7}
\defval{bfcl.E.agentic.E_wrong_goals_J}{111,600}
\defval{bfcl.E.agentic.EpCG_J}{1,860.0}
\defval{bfcl.E.agentic.EpG_J}{1,860.0}
\defval{bfcl.E.agentic.EpSG_J}{undefined}
\defval{bfcl.E.agentic.N}{60}
\defval{bfcl.E.agentic.TpG_s}{27.8}
\defval{bfcl.E.agentic.attempts_per_goal}{1.00}
\defval{bfcl.E.agentic.n_completed}{60}
\defval{bfcl.E.agentic.n_correct}{0}
\defval{bfcl.E.agentic.p_c}{1.00}
\defval{bfcl.E.agentic.p_s}{0.00}
\defval{bfcl.E.agentic.power_W}{66.8}
\defval{bfcl.excluded}{4}
\defval{bfcl.goals}{5,400}
\defval{facts.audit.agree}{28}
\defval{facts.audit.agree_pct}{90}
\defval{facts.audit.fn}{3}
\defval{facts.audit.fp}{0}
\defval{facts.audit.n}{31}
\defval{facts.be.ceil_max_pct}{54.35}
\defval{facts.be.ceil_min_pct}{1.95}
\defval{facts.be.fixed.n}{24}
\defval{facts.be.fixed.n_impossible}{23}
\defval{facts.be.later.n}{17}
\defval{facts.be.later.n_impossible}{17}
\defval{facts.be.n}{41}
\defval{facts.be.n_impossible}{40}
\defval{facts.be.need_min_pct}{70.70}
\defval{facts.be.plin_min_pct}{14.00}
\defval{facts.bfcl.epsg_x}{3.97}
\defval{facts.bfcl.n_conv}{30}
\defval{facts.bfcl.n_models}{4}
\defval{facts.bfcl.same_best}{1}
\defval{facts.bfcl.same_order}{1}
\defval{facts.bfcl.tpsg_x}{3.25}
\defval{facts.calc.bound_max}{0.17}
\defval{facts.calc.bound_min}{0.03}
\defval{facts.calc.ceil_min_pct}{69.9}
\defval{facts.calc.max}{0.17}
\defval{facts.calc.n}{4}
\defval{facts.calc.n_bound}{3}
\defval{facts.calc.n_point}{1}
\defval{facts.calc.ooi_point}{0.04}
\defval{facts.calc.plin_max_pct}{10.0}
\defval{facts.calc.re_max}{1.43}
\defval{facts.calc.re_min}{0.34}
\defval{facts.claim1.A_attempts}{1.0}
\defval{facts.claim1.B_attempts}{1.9}
\defval{facts.claim1.C_attempts}{2.8}
\defval{facts.claim1.calls_A}{5}
\defval{facts.claim1.calls_C}{14}
\defval{facts.claim1.reason_B_call_x_max}{1.07}
\defval{facts.claim1.reason_B_call_x_min}{0.85}
\defval{facts.claim1.reason_B_epg_x_max}{2.03}
\defval{facts.claim1.reason_B_epg_x_min}{1.62}
\defval{facts.claim1.reason_C_call_x_max}{1.05}
\defval{facts.claim1.reason_C_call_x_min}{0.94}
\defval{facts.claim1.reason_C_epg_x_max}{2.94}
\defval{facts.claim1.reason_C_epg_x_min}{2.64}
\defval{facts.claim1.reason_call_dev_pct}{15}
\defval{facts.claim1.reason_fail_x_max}{1.08}
\defval{facts.claim1.reason_fail_x_min}{0.91}
\defval{facts.claim1.tool_B_call_x_max}{2.67}
\defval{facts.claim1.tool_B_call_x_min}{1.28}
\defval{facts.claim1.tool_B_epg_x_max}{5.07}
\defval{facts.claim1.tool_B_epg_x_min}{2.43}
\defval{facts.claim1.tool_C_call_x_max}{3.10}
\defval{facts.claim1.tool_C_call_x_min}{1.26}
\defval{facts.claim1.tool_C_epg_x_max}{8.68}
\defval{facts.claim1.tool_C_epg_x_min}{3.52}
\defval{facts.claim1.tool_fail_x_max}{4.27}
\defval{facts.claim1.tool_fail_x_min}{1.40}
\defval{facts.claim4.gb10.cleanup_pct_ag}{2.8}
\defval{facts.claim4.gb10.cleanup_pct_lin}{36.4}
\defval{facts.claim4.gb10.cleanup_x_max}{0.99}
\defval{facts.claim4.gb10.cleanup_x_min}{0.36}
\defval{facts.claim4.gb10.commonsense_reasoning.cleanup_x}{0.96}
\defval{facts.claim4.gb10.commonsense_reasoning.idle_x}{0.98}
\defval{facts.claim4.gb10.english_to_hindi.cleanup_x}{0.82}
\defval{facts.claim4.gb10.english_to_hindi.idle_x}{0.96}
\defval{facts.claim4.gb10.factual_qa.cleanup_x}{0.94}
\defval{facts.claim4.gb10.factual_qa.idle_x}{0.99}
\defval{facts.claim4.gb10.geography_qa.cleanup_x}{0.39}
\defval{facts.claim4.gb10.geography_qa.idle_x}{0.79}
\defval{facts.claim4.gb10.gsm8k_basic.cleanup_x}{0.93}
\defval{facts.claim4.gb10.gsm8k_basic.idle_x}{0.98}
\defval{facts.claim4.gb10.gsm8k_dist.cleanup_x}{0.91}
\defval{facts.claim4.gb10.gsm8k_dist.idle_x}{0.97}
\defval{facts.claim4.gb10.gsm8k_multi_step.cleanup_x}{0.98}
\defval{facts.claim4.gb10.gsm8k_multi_step.idle_x}{0.99}
\defval{facts.claim4.gb10.humaneval_dist.cleanup_x}{0.91}
\defval{facts.claim4.gb10.humaneval_dist.idle_x}{0.96}
\defval{facts.claim4.gb10.idle_x_max}{0.99}
\defval{facts.claim4.gb10.idle_x_min}{0.77}
\defval{facts.claim4.gb10.keyword_extraction.cleanup_x}{0.53}
\defval{facts.claim4.gb10.keyword_extraction.idle_x}{0.87}
\defval{facts.claim4.gb10.literature_synthesis.cleanup_x}{0.98}
\defval{facts.claim4.gb10.literature_synthesis.idle_x}{0.98}
\defval{facts.claim4.gb10.logical_reasoning.cleanup_x}{0.93}
\defval{facts.claim4.gb10.logical_reasoning.idle_x}{0.98}
\defval{facts.claim4.gb10.methodology_comparison.cleanup_x}{0.98}
\defval{facts.claim4.gb10.methodology_comparison.idle_x}{0.98}
\defval{facts.claim4.gb10.mmlu_dist.cleanup_x}{0.89}
\defval{facts.claim4.gb10.mmlu_dist.idle_x}{0.96}
\defval{facts.claim4.gb10.multilingual_summary.cleanup_x}{0.95}
\defval{facts.claim4.gb10.multilingual_summary.idle_x}{0.98}
\defval{facts.claim4.gb10.news_summary.cleanup_x}{0.65}
\defval{facts.claim4.gb10.news_summary.idle_x}{0.91}
\defval{facts.claim4.gb10.research_abstract.cleanup_x}{0.87}
\defval{facts.claim4.gb10.research_abstract.idle_x}{0.97}
\defval{facts.claim4.gb10.research_summary.cleanup_x}{0.57}
\defval{facts.claim4.gb10.research_summary.idle_x}{0.88}
\defval{facts.claim4.gb10.science_qa.cleanup_x}{0.91}
\defval{facts.claim4.gb10.science_qa.idle_x}{0.98}
\defval{facts.claim4.gb10.sentiment_analysis.cleanup_x}{0.54}
\defval{facts.claim4.gb10.sentiment_analysis.idle_x}{0.86}
\defval{facts.claim4.gb10.t3_creative.cleanup_x}{0.83}
\defval{facts.claim4.gb10.t3_creative.idle_x}{0.95}
\defval{facts.claim4.gb10.t3_planning_notools.cleanup_x}{0.99}
\defval{facts.claim4.gb10.t3_planning_notools.idle_x}{0.98}
\defval{facts.claim4.gb10.t3_reasoning_hard.cleanup_x}{0.96}
\defval{facts.claim4.gb10.t3_reasoning_hard.idle_x}{0.99}
\defval{facts.claim4.gb10.t3_summarization.cleanup_x}{0.78}
\defval{facts.claim4.gb10.t3_summarization.idle_x}{0.95}
\defval{facts.claim4.gb10.t3_translation_hi.cleanup_x}{0.83}
\defval{facts.claim4.gb10.t3_translation_hi.idle_x}{0.96}
\defval{facts.claim4.gb10.technical_blog_post.cleanup_x}{0.89}
\defval{facts.claim4.gb10.technical_blog_post.idle_x}{0.96}
\defval{facts.claim4.gb10.topic_classification.cleanup_x}{0.36}
\defval{facts.claim4.gb10.topic_classification.idle_x}{0.77}
\defval{facts.claim4.x86.cleanup_pct_ag}{0.7}
\defval{facts.claim4.x86.cleanup_pct_lin}{1.1}
\defval{facts.claim4.x86.cleanup_x_max}{1.00}
\defval{facts.claim4.x86.cleanup_x_min}{0.99}
\defval{facts.claim4.x86.factual_qa.cleanup_x}{1.00}
\defval{facts.claim4.x86.factual_qa.idle_x}{1.01}
\defval{facts.claim4.x86.gsm8k_basic.cleanup_x}{1.00}
\defval{facts.claim4.x86.gsm8k_basic.idle_x}{1.01}
\defval{facts.claim4.x86.gsm8k_multi_step.cleanup_x}{0.99}
\defval{facts.claim4.x86.gsm8k_multi_step.idle_x}{0.99}
\defval{facts.claim4.x86.idle_x_max}{1.01}
\defval{facts.claim4.x86.idle_x_min}{0.99}
\defval{facts.claim4.x86.logical_reasoning.cleanup_x}{1.00}
\defval{facts.claim4.x86.logical_reasoning.idle_x}{1.01}
\defval{facts.claim4.x86.science_qa.cleanup_x}{1.00}
\defval{facts.claim4.x86.science_qa.idle_x}{1.01}
\defval{facts.design.n_calc_comp}{4}
\defval{facts.design.n_comp}{61}
\defval{facts.design.n_fixed_comp}{36}
\defval{facts.design.n_goals}{11,621}
\defval{facts.design.n_later_comp}{21}
\defval{facts.design.n_models}{8}
\defval{facts.epg_cells_above1}{35}
\defval{facts.epg_cells_n}{35}
\defval{facts.ex.ag.J}{852.62}
\defval{facts.ex.ag.J_call}{213.15}
\defval{facts.ex.ag.calls}{4}
\defval{facts.ex.ag.s}{14.23}
\defval{facts.ex.ag.tok}{778}
\defval{facts.ex.ag1.J}{290.89}
\defval{facts.ex.ag1.s}{4.85}
\defval{facts.ex.ag1.tok}{269}
\defval{facts.ex.ag2.J}{206.01}
\defval{facts.ex.ag2.s}{3.44}
\defval{facts.ex.ag2.tok}{191}
\defval{facts.ex.ag3.J}{131.37}
\defval{facts.ex.ag3.s}{2.19}
\defval{facts.ex.ag3.tok}{120}
\defval{facts.ex.ag4.J}{224.35}
\defval{facts.ex.ag4.s}{3.74}
\defval{facts.ex.ag4.tok}{198}
\defval{facts.ex.lin.J}{141.55}
\defval{facts.ex.lin.s}{2.39}
\defval{facts.ex.lin.tok}{130}
\defval{facts.ex.percall_ratio}{1.51}
\defval{facts.ex.ratio}{6.02}
\defval{facts.excl.x86.invalid_goals}{2}
\defval{facts.excl.x86_remote_nim.invalid_goals}{18}
\defval{facts.fail.fail_s}{39}
\defval{facts.fail.n_ag}{10}
\defval{facts.fail.n_hit}{9}
\defval{facts.flip.best_same}{14}
\defval{facts.flip.best_same_n}{14}
\defval{facts.flip.changed}{6}
\defval{facts.flip.changed_small}{3}
\defval{facts.flip.holds}{5}
\defval{facts.flip.n}{14}
\defval{facts.flip.none}{1}
\defval{facts.flip.tiny_epg}{8}
\defval{facts.flip.tiny_epsg}{2}
\defval{facts.flip.tiny_n}{12}
\defval{facts.flip.tpsg_n}{5}
\defval{facts.flip.tpsg_x_max}{7.13}
\defval{facts.flip.tpsg_x_min}{1.13}
\defval{facts.g612.a1}{160}
\defval{facts.g612.a2}{167}
\defval{facts.g612.a3}{159}
\defval{facts.g612.attempts}{3}
\defval{facts.g612.call}{23}
\defval{facts.g612.calls}{21}
\defval{facts.g612.goal}{487}
\defval{facts.g612.share}{33}
\defval{facts.gb10.commonsense_reasoning.process_share_pct}{6.1}
\defval{facts.gb10.english_to_hindi.process_share_pct}{5.6}
\defval{facts.gb10.factual_qa.process_share_pct}{5.3}
\defval{facts.gb10.geography_qa.process_share_pct}{5.6}
\defval{facts.gb10.gsm8k_basic.process_share_pct}{5.2}
\defval{facts.gb10.gsm8k_dist.process_share_pct}{4.7}
\defval{facts.gb10.gsm8k_multi_step.process_share_pct}{6.1}
\defval{facts.gb10.humaneval_dist.process_share_pct}{5.0}
\defval{facts.gb10.idle_W}{17.7}
\defval{facts.gb10.keyword_extraction.process_share_pct}{5.6}
\defval{facts.gb10.literature_synthesis.process_share_pct}{5.7}
\defval{facts.gb10.logical_reasoning.process_share_pct}{6.0}
\defval{facts.gb10.methodology_comparison.process_share_pct}{5.6}
\defval{facts.gb10.mmlu_dist.process_share_pct}{5.3}
\defval{facts.gb10.multilingual_summary.process_share_pct}{5.5}
\defval{facts.gb10.news_summary.process_share_pct}{5.7}
\defval{facts.gb10.process_share_max_pct}{6.1}
\defval{facts.gb10.process_share_min_pct}{4.7}
\defval{facts.gb10.process_share_miss_max}{95.3}
\defval{facts.gb10.process_share_miss_min}{93.9}
\defval{facts.gb10.process_share_n}{26}
\defval{facts.gb10.research_abstract.process_share_pct}{5.1}
\defval{facts.gb10.research_summary.process_share_pct}{5.9}
\defval{facts.gb10.science_qa.process_share_pct}{5.7}
\defval{facts.gb10.sentiment_analysis.process_share_pct}{6.0}
\defval{facts.gb10.t3_creative.process_share_pct}{5.6}
\defval{facts.gb10.t3_planning_notools.process_share_pct}{5.7}
\defval{facts.gb10.t3_reasoning_hard.process_share_pct}{5.8}
\defval{facts.gb10.t3_summarization.process_share_pct}{5.6}
\defval{facts.gb10.t3_translation_hi.process_share_pct}{6.0}
\defval{facts.gb10.technical_blog_post.process_share_pct}{5.7}
\defval{facts.gb10.tg_single_calc.process_share_pct}{4.4}
\defval{facts.gb10.topic_classification.process_share_pct}{5.8}
\defval{facts.lenient.flipped}{36}
\defval{facts.lenient.gb10.commonsense_reasoning.OOI}{21.09}
\defval{facts.lenient.gb10.factual_qa.OOI}{7.67}
\defval{facts.lenient.gb10.gsm8k_basic.OOI}{7.30}
\defval{facts.lenient.gb10.gsm8k_multi_step.OOI}{3.84}
\defval{facts.lenient.gb10.logical_reasoning.OOI}{11.78}
\defval{facts.lenient.gb10.science_qa.OOI}{5.03}
\defval{facts.lenient.gb10_qwen3.commonsense_reasoning.OOI}{15.53}
\defval{facts.lenient.gb10_qwen3.factual_qa.OOI}{24.31}
\defval{facts.lenient.gb10_qwen3.gsm8k_basic.OOI}{14.86}
\defval{facts.lenient.gb10_qwen3.gsm8k_multi_step.OOI}{8.24}
\defval{facts.lenient.gb10_qwen3.logical_reasoning.OOI}{21.84}
\defval{facts.lenient.gb10_qwen3.science_qa.OOI}{39.50}
\defval{facts.lenient.gb10_tinyllama.commonsense_reasoning.OOI}{4.19}
\defval{facts.lenient.gb10_tinyllama.factual_qa.OOI}{2.43}
\defval{facts.lenient.gb10_tinyllama.logical_reasoning.OOI}{4.38}
\defval{facts.lenient.mac_tinyllama.factual_qa.OOI}{4.11}
\defval{facts.lenient.n}{24}
\defval{facts.lenient.n_above1}{24}
\defval{facts.lenient.ooi_max}{39.5}
\defval{facts.lenient.ooi_min}{2.4}
\defval{facts.lenient.x86.factual_qa.OOI}{7.58}
\defval{facts.lenient.x86.logical_reasoning.OOI}{3.68}
\defval{facts.lenient.x86_remote_nim.commonsense_reasoning.OOI}{8.31}
\defval{facts.lenient.x86_remote_nim.factual_qa.OOI}{2.47}
\defval{facts.lenient.x86_remote_nim.gsm8k_basic.OOI}{6.52}
\defval{facts.lenient.x86_remote_nim.gsm8k_multi_step.OOI}{7.79}
\defval{facts.lenient.x86_remote_nim.logical_reasoning.OOI}{8.36}
\defval{facts.lenient.x86_remote_nim.science_qa.OOI}{4.40}
\defval{facts.modelsel.n}{6}
\defval{facts.modelsel.n_ci_below1}{6}
\defval{facts.modelsel.n_equal_or_better}{6}
\defval{facts.modelsel.red_max}{88}
\defval{facts.modelsel.red_min}{26}
\defval{facts.n_breadth_reps}{20}
\defval{facts.n_breadth_types}{17}
\defval{facts.n_comparable}{24}
\defval{facts.n_dist_per_class}{50}
\defval{facts.n_dist_questions}{150}
\defval{facts.n_fixed_questions}{23}
\defval{facts.n_goal_types_all}{23}
\defval{facts.n_goal_types_ooi}{22}
\defval{facts.n_too_few}{25}
\defval{facts.newsets.n}{17}
\defval{facts.newsets.n_above1}{17}
\defval{facts.newsets.n_breadth}{13}
\defval{facts.newsets.n_dist_mistral}{1}
\defval{facts.newsets.n_dist_qwen}{3}
\defval{facts.newsets.n_types}{16}
\defval{facts.newsets.ooi_max}{51.21}
\defval{facts.newsets.ooi_min}{1.84}
\defval{facts.nouse.bfcl.agentic_pct}{19}
\defval{facts.nouse.gb10.agentic_pct}{12}
\defval{facts.nouse.gb10.linear_pct}{0}
\defval{facts.nouse.gb10_qwen3.agentic_pct}{0}
\defval{facts.nouse.gb10_qwen3.linear_pct}{0}
\defval{facts.nouse.gb10_tinyllama.agentic_pct}{77}
\defval{facts.nouse.gb10_tinyllama.linear_pct}{89}
\defval{facts.nouse.mac_tinyllama.agentic_pct}{75}
\defval{facts.nouse.mac_tinyllama.linear_pct}{73}
\defval{facts.nouse.x86.agentic_pct}{84}
\defval{facts.nouse.x86.linear_pct}{90}
\defval{facts.nouse.x86_remote_nim.agentic_pct}{5}
\defval{facts.nouse.x86_remote_nim.linear_pct}{7}
\defval{facts.nouse_max_pct}{90}
\defval{facts.nouse_min_pct}{0}
\defval{facts.numloose.flipped}{146}
\defval{facts.numloose.gb10.gsm8k_basic.OOI}{6.54}
\defval{facts.numloose.gb10.gsm8k_multi_step.OOI}{3.84}
\defval{facts.numloose.gb10_qwen3.gsm8k_basic.OOI}{14.86}
\defval{facts.numloose.gb10_qwen3.gsm8k_multi_step.OOI}{8.24}
\defval{facts.numloose.n}{24}
\defval{facts.numloose.n_above1}{24}
\defval{facts.numloose.ooi_max}{39.5}
\defval{facts.numloose.ooi_min}{2.1}
\defval{facts.numloose.x86_remote_nim.gsm8k_basic.OOI}{6.52}
\defval{facts.numloose.x86_remote_nim.gsm8k_multi_step.OOI}{7.79}
\defval{facts.ooi_all}{45}
\defval{facts.ooi_fixed_n}{24}
\defval{facts.ooi_max}{39.5}
\defval{facts.ooi_median}{7.5}
\defval{facts.ooi_min}{2.1}
\defval{facts.ooi_n}{24}
\defval{facts.ooi_n_above1}{24}
\defval{facts.ooi_q1}{3.8}
\defval{facts.ooi_q3}{11.8}
\defval{facts.ooi_total}{41}
\defval{facts.power_ratio_max}{1.13}
\defval{facts.power_ratio_min}{0.85}
\defval{facts.qwen3.ooi_max}{39.5}
\defval{facts.qwen3.ooi_min}{8.2}
\defval{facts.remote.commonsense_reasoning.tok_ratio}{9.2}
\defval{facts.remote.factual_qa.tok_ratio}{4.1}
\defval{facts.remote.gsm8k_basic.tok_ratio}{10.6}
\defval{facts.remote.gsm8k_multi_step.tok_ratio}{8.7}
\defval{facts.remote.logical_reasoning.tok_ratio}{10.9}
\defval{facts.remote.n_bound_holds}{6}
\defval{facts.remote.n_types}{6}
\defval{facts.remote.science_qa.tok_ratio}{23.8}
\defval{facts.remote.tok_ratio_max}{24}
\defval{facts.remote.tok_ratio_min}{4}
\defval{facts.share_energy_max}{1.91}
\defval{facts.share_energy_min}{0.91}
\defval{facts.success_ratio_max}{1.30}
\defval{facts.success_ratio_min}{0.48}
\defval{facts.thresh.n1}{9}
\defval{facts.thresh.n3}{5}
\defval{facts.thresh.n5}{2}
\defval{facts.thresh.n_below1}{0}
\defval{facts.thresh.ooi_min}{1.02}
\defval{facts.thresh.sci_limit_pct}{2.9}
\defval{facts.time_ratio_max}{46.27}
\defval{facts.time_ratio_min}{2.50}
\defval{facts.tool.gsm8k.ag_ok}{45}
\defval{facts.tool.gsm8k.lin_n}{50}
\defval{facts.tool.gsm8k.lin_ok}{48}
\defval{facts.tool.gsm8k.n_ag}{50}
\defval{facts.tool.gsm8k.n_tool}{50}
\defval{facts.tool.gsm8k.ooi}{4.80}
\defval{facts.tool.gsm8k.ooi_hi}{8.14}
\defval{facts.tool.gsm8k.ooi_lo}{2.83}
\defval{facts.tool.humaneval.ag_ok}{36}
\defval{facts.tool.humaneval.lin_n}{50}
\defval{facts.tool.humaneval.lin_ok}{48}
\defval{facts.tool.humaneval.n_ag}{50}
\defval{facts.tool.humaneval.n_tool}{31}
\defval{facts.tool.humaneval.ooi}{27.11}
\defval{facts.tool.humaneval.ooi_hi}{42.07}
\defval{facts.tool.humaneval.ooi_lo}{16.35}
\defval{facts.tool.mistral.ag_ok}{10}
\defval{facts.tool.mistral.lin_n}{50}
\defval{facts.tool.mistral.lin_ok}{21}
\defval{facts.tool.mistral.n_ag}{50}
\defval{facts.tool.mistral.n_tool}{43}
\defval{facts.tool.mistral.ooi}{41.05}
\defval{facts.tool.mistral.ooi_hi}{90.80}
\defval{facts.tool.mistral.ooi_lo}{20.89}
\defval{facts.tool.planning.ag_ok}{20}
\defval{facts.tool.planning.lin_n}{20}
\defval{facts.tool.planning.lin_ok}{20}
\defval{facts.tool.planning.n_ag}{20}
\defval{facts.tool.planning.n_tool}{20}
\defval{facts.tool.planning.ooi}{7.38}
\defval{facts.tool.planning.ooi_hi}{7.74}
\defval{facts.tool.planning.ooi_lo}{7.04}
\defval{facts.tool.reasoning.ag_ok}{20}
\defval{facts.tool.reasoning.lin_n}{20}
\defval{facts.tool.reasoning.lin_ok}{20}
\defval{facts.tool.reasoning.n_ag}{20}
\defval{facts.tool.reasoning.n_tool}{20}
\defval{facts.tool.reasoning.ooi}{1.84}
\defval{facts.tool.reasoning.ooi_hi}{1.92}
\defval{facts.tool.reasoning.ooi_lo}{1.78}
\defval{facts.trace.attempt1_J}{2,256}
\defval{facts.trace.attempt2_J}{1,358}
\defval{facts.trace.failed_pct}{62}
\defval{facts.trace.goal_J}{3,614}
\defval{facts.unacc.cleanup_pct}{136}
\defval{facts.unacc.completed_pct}{76.08}
\defval{facts.unacc.idle_pct}{132}
\defval{facts.unacc.percall_pct}{60.32}
\defval{facts.unacc.process_pct}{6.10}
\defval{facts.valid.bfcl.bg_cpu_max}{48.5}
\defval{facts.valid.bfcl.bg_cpu_n90}{0}
\defval{facts.valid.bfcl.bg_cpu_p50}{1.0}
\defval{facts.valid.bfcl.bg_cpu_p95}{5.0}
\defval{facts.valid.bfcl.cov95_pct}{97}
\defval{facts.valid.bfcl.cov_mean}{98.1}
\defval{facts.valid.bfcl.cov_min}{0.0}
\defval{facts.valid.bfcl.idle_W}{32.3}
\defval{facts.valid.bfcl.idle_cv_pct}{3.5}
\defval{facts.valid.bfcl.runs}{10,155}
\defval{facts.valid.bfcl.throttle_runs}{0}
\defval{facts.valid.gb10.bg_cpu_max}{15.5}
\defval{facts.valid.gb10.bg_cpu_n90}{0}
\defval{facts.valid.gb10.bg_cpu_p50}{0.5}
\defval{facts.valid.gb10.bg_cpu_p95}{1.5}
\defval{facts.valid.gb10.cov95_pct}{100}
\defval{facts.valid.gb10.cov_mean}{99.4}
\defval{facts.valid.gb10.cov_min}{94.0}
\defval{facts.valid.gb10.idle_W}{17.7}
\defval{facts.valid.gb10.idle_cv_pct}{0.2}
\defval{facts.valid.gb10.runs}{600}
\defval{facts.valid.gb10.throttle_runs}{0}
\defval{facts.valid.gb10_mistral_A.bg_cpu_max}{12.0}
\defval{facts.valid.gb10_mistral_A.bg_cpu_n90}{0}
\defval{facts.valid.gb10_mistral_A.bg_cpu_p50}{1.0}
\defval{facts.valid.gb10_mistral_A.bg_cpu_p95}{5.5}
\defval{facts.valid.gb10_mistral_A.cov95_pct}{99}
\defval{facts.valid.gb10_mistral_A.cov_mean}{99.3}
\defval{facts.valid.gb10_mistral_A.cov_min}{91.7}
\defval{facts.valid.gb10_mistral_A.idle_W}{19.2}
\defval{facts.valid.gb10_mistral_A.idle_cv_pct}{0.0}
\defval{facts.valid.gb10_mistral_A.runs}{180}
\defval{facts.valid.gb10_mistral_A.throttle_runs}{0}
\defval{facts.valid.gb10_mistral_B.bg_cpu_max}{9.5}
\defval{facts.valid.gb10_mistral_B.bg_cpu_n90}{0}
\defval{facts.valid.gb10_mistral_B.bg_cpu_p50}{1.5}
\defval{facts.valid.gb10_mistral_B.bg_cpu_p95}{6.0}
\defval{facts.valid.gb10_mistral_B.cov95_pct}{100}
\defval{facts.valid.gb10_mistral_B.cov_mean}{99.5}
\defval{facts.valid.gb10_mistral_B.cov_min}{92.6}
\defval{facts.valid.gb10_mistral_B.idle_W}{19.2}
\defval{facts.valid.gb10_mistral_B.idle_cv_pct}{0.0}
\defval{facts.valid.gb10_mistral_B.runs}{261}
\defval{facts.valid.gb10_mistral_B.throttle_runs}{0}
\defval{facts.valid.gb10_mistral_C.bg_cpu_max}{7.9}
\defval{facts.valid.gb10_mistral_C.bg_cpu_n90}{0}
\defval{facts.valid.gb10_mistral_C.bg_cpu_p50}{0.5}
\defval{facts.valid.gb10_mistral_C.bg_cpu_p95}{1.0}
\defval{facts.valid.gb10_mistral_C.cov95_pct}{100}
\defval{facts.valid.gb10_mistral_C.cov_mean}{99.6}
\defval{facts.valid.gb10_mistral_C.cov_min}{96.1}
\defval{facts.valid.gb10_mistral_C.idle_W}{19.2}
\defval{facts.valid.gb10_mistral_C.idle_cv_pct}{0.0}
\defval{facts.valid.gb10_mistral_C.runs}{342}
\defval{facts.valid.gb10_mistral_C.throttle_runs}{0}
\defval{facts.valid.gb10_mistral_tools.bg_cpu_max}{12.5}
\defval{facts.valid.gb10_mistral_tools.bg_cpu_n90}{0}
\defval{facts.valid.gb10_mistral_tools.bg_cpu_p50}{1.5}
\defval{facts.valid.gb10_mistral_tools.bg_cpu_p95}{5.5}
\defval{facts.valid.gb10_mistral_tools.cov95_pct}{100}
\defval{facts.valid.gb10_mistral_tools.cov_mean}{99.7}
\defval{facts.valid.gb10_mistral_tools.cov_min}{97.7}
\defval{facts.valid.gb10_mistral_tools.idle_W}{19.2}
\defval{facts.valid.gb10_mistral_tools.idle_cv_pct}{0.0}
\defval{facts.valid.gb10_mistral_tools.runs}{330}
\defval{facts.valid.gb10_mistral_tools.throttle_runs}{0}
\defval{facts.valid.gb10_qwen3.bg_cpu_max}{21.0}
\defval{facts.valid.gb10_qwen3.bg_cpu_n90}{0}
\defval{facts.valid.gb10_qwen3.bg_cpu_p50}{0.5}
\defval{facts.valid.gb10_qwen3.bg_cpu_p95}{4.0}
\defval{facts.valid.gb10_qwen3.cov95_pct}{93}
\defval{facts.valid.gb10_qwen3.cov_mean}{97.4}
\defval{facts.valid.gb10_qwen3.cov_min}{68.2}
\defval{facts.valid.gb10_qwen3.idle_W}{19.2}
\defval{facts.valid.gb10_qwen3.idle_cv_pct}{0.0}
\defval{facts.valid.gb10_qwen3.runs}{540}
\defval{facts.valid.gb10_qwen3.throttle_runs}{0}
\defval{facts.valid.gb10_tinyllama.bg_cpu_max}{38.2}
\defval{facts.valid.gb10_tinyllama.bg_cpu_n90}{0}
\defval{facts.valid.gb10_tinyllama.bg_cpu_p50}{1.5}
\defval{facts.valid.gb10_tinyllama.bg_cpu_p95}{7.0}
\defval{facts.valid.gb10_tinyllama.cov95_pct}{57}
\defval{facts.valid.gb10_tinyllama.cov_mean}{79.4}
\defval{facts.valid.gb10_tinyllama.cov_min}{0.0}
\defval{facts.valid.gb10_tinyllama.idle_W}{19.2}
\defval{facts.valid.gb10_tinyllama.idle_cv_pct}{0.0}
\defval{facts.valid.gb10_tinyllama.runs}{1,080}
\defval{facts.valid.gb10_tinyllama.throttle_runs}{0}
\defval{facts.valid.mac_tinyllama.bg_cap}{50}
\defval{facts.valid.mac_tinyllama.bg_cpu_max}{46.9}
\defval{facts.valid.mac_tinyllama.bg_cpu_n90}{0}
\defval{facts.valid.mac_tinyllama.bg_cpu_p50}{7.3}
\defval{facts.valid.mac_tinyllama.bg_cpu_p95}{22.7}
\defval{facts.valid.mac_tinyllama.bg_flagged}{399}
\defval{facts.valid.mac_tinyllama.bg_flagged_pct}{38}
\defval{facts.valid.mac_tinyllama.bg_mid}{6}
\defval{facts.valid.mac_tinyllama.cov95_pct}{50}
\defval{facts.valid.mac_tinyllama.cov_mean}{62.5}
\defval{facts.valid.mac_tinyllama.cov_min}{0.0}
\defval{facts.valid.mac_tinyllama.idle_W}{1.6}
\defval{facts.valid.mac_tinyllama.idle_cv_pct}{undefined}
\defval{facts.valid.mac_tinyllama.runs}{647}
\defval{facts.valid.mac_tinyllama.runs_all}{1,046}
\defval{facts.valid.mac_tinyllama.throttle_runs}{0}
\defval{facts.valid.x86.bg_cpu_max}{100.0}
\defval{facts.valid.x86.bg_cpu_n90}{3}
\defval{facts.valid.x86.bg_cpu_p50}{2.5}
\defval{facts.valid.x86.bg_cpu_p95}{19.5}
\defval{facts.valid.x86.cov95_pct}{90}
\defval{facts.valid.x86.cov_mean}{97.8}
\defval{facts.valid.x86.cov_min}{43.4}
\defval{facts.valid.x86.idle_W}{2.4}
\defval{facts.valid.x86.idle_cv_pct}{9.9}
\defval{facts.valid.x86.runs}{1,440}
\defval{facts.valid.x86.throttle_runs}{0}
\defval{facts.valid.x86_remote_nim.bg_cpu_max}{100.0}
\defval{facts.valid.x86_remote_nim.bg_cpu_n90}{1}
\defval{facts.valid.x86_remote_nim.bg_cpu_p50}{0.0}
\defval{facts.valid.x86_remote_nim.bg_cpu_p95}{12.5}
\defval{facts.valid.x86_remote_nim.cov95_pct}{100}
\defval{facts.valid.x86_remote_nim.cov_mean}{99.4}
\defval{facts.valid.x86_remote_nim.cov_min}{98.0}
\defval{facts.valid.x86_remote_nim.idle_W}{4.5}
\defval{facts.valid.x86_remote_nim.idle_cv_pct}{11.5}
\defval{facts.valid.x86_remote_nim.runs}{608}
\defval{facts.valid.x86_remote_nim.throttle_runs}{0}
\defval{facts.x86.idle_W}{2.4}
\defval{facts.x86.reps_max}{175}
\defval{facts.x86.reps_min}{51}
\defval{facts.x86.share_ag_pct}{84}
\defval{facts.x86.share_factor}{0.97}
\defval{facts.x86.share_lin_pct}{81}
\defval{facts.zero.epsg_lin_lower}{1,140}
\defval{facts.zero.n_fail_ag}{46}
\defval{facts.zero.ooi_upper}{2.09}
\defval{facts.zero.p_ag_pct}{23.3}
\defval{facts.zero.p_upper_pct}{6.0}
\defval{facts.zero.waste_x}{6.5}
\defval{gb10.commonsense_reasoning.OOI_EpCG}{16.17}
\defval{gb10.commonsense_reasoning.OOI_EpCG_hi}{18.42}
\defval{gb10.commonsense_reasoning.OOI_EpCG_lo}{13.90}
\defval{gb10.commonsense_reasoning.OOI_EpG}{16.17}
\defval{gb10.commonsense_reasoning.OOI_EpG_hi}{18.42}
\defval{gb10.commonsense_reasoning.OOI_EpG_lo}{13.90}
\defval{gb10.commonsense_reasoning.OOI_EpSG}{21.09}
\defval{gb10.commonsense_reasoning.OOI_EpSG_hi}{27.87}
\defval{gb10.commonsense_reasoning.OOI_EpSG_lo}{16.81}
\defval{gb10.commonsense_reasoning.agentic.E_wrong_goals_J}{49,441}
\defval{gb10.commonsense_reasoning.agentic.EpCG_J}{6,953.3}
\defval{gb10.commonsense_reasoning.agentic.EpG_J}{6,953.3}
\defval{gb10.commonsense_reasoning.agentic.EpSG_J}{9,069.5}
\defval{gb10.commonsense_reasoning.agentic.N}{30}
\defval{gb10.commonsense_reasoning.agentic.TpG_s}{121.2}
\defval{gb10.commonsense_reasoning.agentic.attempts_per_goal}{1.00}
\defval{gb10.commonsense_reasoning.agentic.n_completed}{30}
\defval{gb10.commonsense_reasoning.agentic.n_correct}{23}
\defval{gb10.commonsense_reasoning.agentic.p_c}{1.00}
\defval{gb10.commonsense_reasoning.agentic.p_s}{0.77}
\defval{gb10.commonsense_reasoning.agentic.power_W}{57.4}
\defval{gb10.commonsense_reasoning.energy_ratio}{16.17}
\defval{gb10.commonsense_reasoning.linear.E_wrong_goals_J}{0}
\defval{gb10.commonsense_reasoning.linear.EpCG_J}{430.0}
\defval{gb10.commonsense_reasoning.linear.EpG_J}{430.0}
\defval{gb10.commonsense_reasoning.linear.EpSG_J}{430.0}
\defval{gb10.commonsense_reasoning.linear.N}{30}
\defval{gb10.commonsense_reasoning.linear.TpG_s}{7.8}
\defval{gb10.commonsense_reasoning.linear.attempts_per_goal}{1.00}
\defval{gb10.commonsense_reasoning.linear.n_completed}{30}
\defval{gb10.commonsense_reasoning.linear.n_correct}{30}
\defval{gb10.commonsense_reasoning.linear.p_c}{1.00}
\defval{gb10.commonsense_reasoning.linear.p_s}{1.00}
\defval{gb10.commonsense_reasoning.linear.power_W}{54.9}
\defval{gb10.commonsense_reasoning.min_successes}{23}
\defval{gb10.commonsense_reasoning.power_ratio}{1.05}
\defval{gb10.commonsense_reasoning.power_ratio_hi}{1.06}
\defval{gb10.commonsense_reasoning.power_ratio_lo}{1.04}
\defval{gb10.commonsense_reasoning.share_energy}{0.91}
\defval{gb10.commonsense_reasoning.success_ratio}{1.30}
\defval{gb10.commonsense_reasoning.time_ratio}{15.47}
\defval{gb10.commonsense_reasoning.time_ratio_hi}{17.53}
\defval{gb10.commonsense_reasoning.time_ratio_lo}{13.35}
\defval{gb10.excluded}{180}
\defval{gb10.factual_qa.OOI_EpCG}{7.67}
\defval{gb10.factual_qa.OOI_EpCG_hi}{8.43}
\defval{gb10.factual_qa.OOI_EpCG_lo}{7.00}
\defval{gb10.factual_qa.OOI_EpG}{7.67}
\defval{gb10.factual_qa.OOI_EpG_hi}{8.43}
\defval{gb10.factual_qa.OOI_EpG_lo}{7.00}
\defval{gb10.factual_qa.OOI_EpSG}{7.67}
\defval{gb10.factual_qa.OOI_EpSG_hi}{8.43}
\defval{gb10.factual_qa.OOI_EpSG_lo}{7.00}
\defval{gb10.factual_qa.agentic.E_wrong_goals_J}{0}
\defval{gb10.factual_qa.agentic.EpCG_J}{2,226.2}
\defval{gb10.factual_qa.agentic.EpG_J}{2,226.2}
\defval{gb10.factual_qa.agentic.EpSG_J}{2,226.2}
\defval{gb10.factual_qa.agentic.N}{30}
\defval{gb10.factual_qa.agentic.TpG_s}{40.8}
\defval{gb10.factual_qa.agentic.attempts_per_goal}{1.00}
\defval{gb10.factual_qa.agentic.n_completed}{30}
\defval{gb10.factual_qa.agentic.n_correct}{30}
\defval{gb10.factual_qa.agentic.p_c}{1.00}
\defval{gb10.factual_qa.agentic.p_s}{1.00}
\defval{gb10.factual_qa.agentic.power_W}{54.6}
\defval{gb10.factual_qa.energy_ratio}{7.67}
\defval{gb10.factual_qa.linear.E_wrong_goals_J}{0}
\defval{gb10.factual_qa.linear.EpCG_J}{290.1}
\defval{gb10.factual_qa.linear.EpG_J}{290.1}
\defval{gb10.factual_qa.linear.EpSG_J}{290.1}
\defval{gb10.factual_qa.linear.N}{30}
\defval{gb10.factual_qa.linear.TpG_s}{5.4}
\defval{gb10.factual_qa.linear.attempts_per_goal}{1.00}
\defval{gb10.factual_qa.linear.n_completed}{30}
\defval{gb10.factual_qa.linear.n_correct}{30}
\defval{gb10.factual_qa.linear.p_c}{1.00}
\defval{gb10.factual_qa.linear.p_s}{1.00}
\defval{gb10.factual_qa.linear.power_W}{53.7}
\defval{gb10.factual_qa.min_successes}{30}
\defval{gb10.factual_qa.power_ratio}{1.02}
\defval{gb10.factual_qa.power_ratio_hi}{1.04}
\defval{gb10.factual_qa.power_ratio_lo}{0.99}
\defval{gb10.factual_qa.share_energy}{1.00}
\defval{gb10.factual_qa.success_ratio}{1.00}
\defval{gb10.factual_qa.time_ratio}{7.54}
\defval{gb10.factual_qa.time_ratio_hi}{8.25}
\defval{gb10.factual_qa.time_ratio_lo}{6.91}
\defval{gb10.goals}{420}
\defval{gb10.gsm8k_basic.OOI_EpCG}{6.54}
\defval{gb10.gsm8k_basic.OOI_EpCG_hi}{7.40}
\defval{gb10.gsm8k_basic.OOI_EpCG_lo}{5.81}
\defval{gb10.gsm8k_basic.OOI_EpG}{6.54}
\defval{gb10.gsm8k_basic.OOI_EpG_hi}{7.40}
\defval{gb10.gsm8k_basic.OOI_EpG_lo}{5.81}
\defval{gb10.gsm8k_basic.OOI_EpSG}{7.30}
\defval{gb10.gsm8k_basic.OOI_EpSG_hi}{8.95}
\defval{gb10.gsm8k_basic.OOI_EpSG_lo}{6.07}
\defval{gb10.gsm8k_basic.agentic.E_wrong_goals_J}{5,706}
\defval{gb10.gsm8k_basic.agentic.EpCG_J}{1,481.5}
\defval{gb10.gsm8k_basic.agentic.EpG_J}{1,481.5}
\defval{gb10.gsm8k_basic.agentic.EpSG_J}{1,709.4}
\defval{gb10.gsm8k_basic.agentic.N}{30}
\defval{gb10.gsm8k_basic.agentic.TpG_s}{27.0}
\defval{gb10.gsm8k_basic.agentic.attempts_per_goal}{1.00}
\defval{gb10.gsm8k_basic.agentic.n_completed}{30}
\defval{gb10.gsm8k_basic.agentic.n_correct}{26}
\defval{gb10.gsm8k_basic.agentic.p_c}{1.00}
\defval{gb10.gsm8k_basic.agentic.p_s}{0.87}
\defval{gb10.gsm8k_basic.agentic.power_W}{54.8}
\defval{gb10.gsm8k_basic.energy_ratio}{6.54}
\defval{gb10.gsm8k_basic.linear.E_wrong_goals_J}{289}
\defval{gb10.gsm8k_basic.linear.EpCG_J}{226.5}
\defval{gb10.gsm8k_basic.linear.EpG_J}{226.5}
\defval{gb10.gsm8k_basic.linear.EpSG_J}{234.3}
\defval{gb10.gsm8k_basic.linear.N}{30}
\defval{gb10.gsm8k_basic.linear.TpG_s}{4.1}
\defval{gb10.gsm8k_basic.linear.attempts_per_goal}{1.03}
\defval{gb10.gsm8k_basic.linear.n_completed}{30}
\defval{gb10.gsm8k_basic.linear.n_correct}{29}
\defval{gb10.gsm8k_basic.linear.p_c}{1.00}
\defval{gb10.gsm8k_basic.linear.p_s}{0.97}
\defval{gb10.gsm8k_basic.linear.power_W}{54.9}
\defval{gb10.gsm8k_basic.min_successes}{26}
\defval{gb10.gsm8k_basic.power_ratio}{1.00}
\defval{gb10.gsm8k_basic.power_ratio_hi}{1.02}
\defval{gb10.gsm8k_basic.power_ratio_lo}{0.97}
\defval{gb10.gsm8k_basic.share_energy}{0.95}
\defval{gb10.gsm8k_basic.success_ratio}{1.12}
\defval{gb10.gsm8k_basic.time_ratio}{6.56}
\defval{gb10.gsm8k_basic.time_ratio_hi}{7.47}
\defval{gb10.gsm8k_basic.time_ratio_lo}{5.81}
\defval{gb10.gsm8k_multi_step.OOI_EpCG}{3.58}
\defval{gb10.gsm8k_multi_step.OOI_EpCG_hi}{3.84}
\defval{gb10.gsm8k_multi_step.OOI_EpCG_lo}{3.35}
\defval{gb10.gsm8k_multi_step.OOI_EpG}{3.58}
\defval{gb10.gsm8k_multi_step.OOI_EpG_hi}{3.84}
\defval{gb10.gsm8k_multi_step.OOI_EpG_lo}{3.35}
\defval{gb10.gsm8k_multi_step.OOI_EpSG}{3.84}
\defval{gb10.gsm8k_multi_step.OOI_EpSG_hi}{4.40}
\defval{gb10.gsm8k_multi_step.OOI_EpSG_lo}{3.44}
\defval{gb10.gsm8k_multi_step.agentic.E_wrong_goals_J}{4,838}
\defval{gb10.gsm8k_multi_step.agentic.EpCG_J}{2,294.6}
\defval{gb10.gsm8k_multi_step.agentic.EpG_J}{2,294.6}
\defval{gb10.gsm8k_multi_step.agentic.EpSG_J}{2,458.5}
\defval{gb10.gsm8k_multi_step.agentic.N}{30}
\defval{gb10.gsm8k_multi_step.agentic.TpG_s}{39.3}
\defval{gb10.gsm8k_multi_step.agentic.attempts_per_goal}{1.00}
\defval{gb10.gsm8k_multi_step.agentic.n_completed}{30}
\defval{gb10.gsm8k_multi_step.agentic.n_correct}{28}
\defval{gb10.gsm8k_multi_step.agentic.p_c}{1.00}
\defval{gb10.gsm8k_multi_step.agentic.p_s}{0.93}
\defval{gb10.gsm8k_multi_step.agentic.power_W}{58.4}
\defval{gb10.gsm8k_multi_step.energy_ratio}{3.58}
\defval{gb10.gsm8k_multi_step.linear.E_wrong_goals_J}{0}
\defval{gb10.gsm8k_multi_step.linear.EpCG_J}{640.3}
\defval{gb10.gsm8k_multi_step.linear.EpG_J}{640.3}
\defval{gb10.gsm8k_multi_step.linear.EpSG_J}{640.3}
\defval{gb10.gsm8k_multi_step.linear.N}{30}
\defval{gb10.gsm8k_multi_step.linear.TpG_s}{11.3}
\defval{gb10.gsm8k_multi_step.linear.attempts_per_goal}{1.00}
\defval{gb10.gsm8k_multi_step.linear.n_completed}{30}
\defval{gb10.gsm8k_multi_step.linear.n_correct}{30}
\defval{gb10.gsm8k_multi_step.linear.p_c}{1.00}
\defval{gb10.gsm8k_multi_step.linear.p_s}{1.00}
\defval{gb10.gsm8k_multi_step.linear.power_W}{56.5}
\defval{gb10.gsm8k_multi_step.min_successes}{28}
\defval{gb10.gsm8k_multi_step.power_ratio}{1.03}
\defval{gb10.gsm8k_multi_step.power_ratio_hi}{1.04}
\defval{gb10.gsm8k_multi_step.power_ratio_lo}{1.02}
\defval{gb10.gsm8k_multi_step.share_energy}{0.95}
\defval{gb10.gsm8k_multi_step.success_ratio}{1.07}
\defval{gb10.gsm8k_multi_step.time_ratio}{3.47}
\defval{gb10.gsm8k_multi_step.time_ratio_hi}{3.70}
\defval{gb10.gsm8k_multi_step.time_ratio_lo}{3.25}
\defval{gb10.logical_reasoning.OOI_EpCG}{11.39}
\defval{gb10.logical_reasoning.OOI_EpCG_hi}{12.31}
\defval{gb10.logical_reasoning.OOI_EpCG_lo}{10.57}
\defval{gb10.logical_reasoning.OOI_EpG}{11.39}
\defval{gb10.logical_reasoning.OOI_EpG_hi}{12.31}
\defval{gb10.logical_reasoning.OOI_EpG_lo}{10.57}
\defval{gb10.logical_reasoning.OOI_EpSG}{11.78}
\defval{gb10.logical_reasoning.OOI_EpSG_hi}{13.09}
\defval{gb10.logical_reasoning.OOI_EpSG_lo}{10.73}
\defval{gb10.logical_reasoning.agentic.E_wrong_goals_J}{2,883}
\defval{gb10.logical_reasoning.agentic.EpCG_J}{3,110.7}
\defval{gb10.logical_reasoning.agentic.EpG_J}{3,110.7}
\defval{gb10.logical_reasoning.agentic.EpSG_J}{3,217.9}
\defval{gb10.logical_reasoning.agentic.N}{30}
\defval{gb10.logical_reasoning.agentic.TpG_s}{53.0}
\defval{gb10.logical_reasoning.agentic.attempts_per_goal}{1.00}
\defval{gb10.logical_reasoning.agentic.n_completed}{30}
\defval{gb10.logical_reasoning.agentic.n_correct}{29}
\defval{gb10.logical_reasoning.agentic.p_c}{1.00}
\defval{gb10.logical_reasoning.agentic.p_s}{0.97}
\defval{gb10.logical_reasoning.agentic.power_W}{58.7}
\defval{gb10.logical_reasoning.energy_ratio}{11.39}
\defval{gb10.logical_reasoning.linear.E_wrong_goals_J}{0}
\defval{gb10.logical_reasoning.linear.EpCG_J}{273.1}
\defval{gb10.logical_reasoning.linear.EpG_J}{273.1}
\defval{gb10.logical_reasoning.linear.EpSG_J}{273.1}
\defval{gb10.logical_reasoning.linear.N}{30}
\defval{gb10.logical_reasoning.linear.TpG_s}{4.8}
\defval{gb10.logical_reasoning.linear.attempts_per_goal}{1.00}
\defval{gb10.logical_reasoning.linear.n_completed}{30}
\defval{gb10.logical_reasoning.linear.n_correct}{30}
\defval{gb10.logical_reasoning.linear.p_c}{1.00}
\defval{gb10.logical_reasoning.linear.p_s}{1.00}
\defval{gb10.logical_reasoning.linear.power_W}{57.0}
\defval{gb10.logical_reasoning.min_successes}{29}
\defval{gb10.logical_reasoning.power_ratio}{1.03}
\defval{gb10.logical_reasoning.power_ratio_hi}{1.04}
\defval{gb10.logical_reasoning.power_ratio_lo}{1.02}
\defval{gb10.logical_reasoning.share_energy}{0.99}
\defval{gb10.logical_reasoning.success_ratio}{1.03}
\defval{gb10.logical_reasoning.time_ratio}{11.05}
\defval{gb10.logical_reasoning.time_ratio_hi}{11.98}
\defval{gb10.logical_reasoning.time_ratio_lo}{10.25}
\defval{gb10.science_qa.OOI_EpCG}{5.03}
\defval{gb10.science_qa.OOI_EpCG_hi}{5.63}
\defval{gb10.science_qa.OOI_EpCG_lo}{4.53}
\defval{gb10.science_qa.OOI_EpG}{5.03}
\defval{gb10.science_qa.OOI_EpG_hi}{5.63}
\defval{gb10.science_qa.OOI_EpG_lo}{4.53}
\defval{gb10.science_qa.OOI_EpSG}{5.03}
\defval{gb10.science_qa.OOI_EpSG_hi}{5.63}
\defval{gb10.science_qa.OOI_EpSG_lo}{4.53}
\defval{gb10.science_qa.agentic.E_wrong_goals_J}{0}
\defval{gb10.science_qa.agentic.EpCG_J}{843.9}
\defval{gb10.science_qa.agentic.EpG_J}{843.9}
\defval{gb10.science_qa.agentic.EpSG_J}{843.9}
\defval{gb10.science_qa.agentic.N}{30}
\defval{gb10.science_qa.agentic.TpG_s}{16.1}
\defval{gb10.science_qa.agentic.attempts_per_goal}{1.00}
\defval{gb10.science_qa.agentic.n_completed}{30}
\defval{gb10.science_qa.agentic.n_correct}{30}
\defval{gb10.science_qa.agentic.p_c}{1.00}
\defval{gb10.science_qa.agentic.p_s}{1.00}
\defval{gb10.science_qa.agentic.power_W}{52.3}
\defval{gb10.science_qa.energy_ratio}{5.03}
\defval{gb10.science_qa.linear.E_wrong_goals_J}{0}
\defval{gb10.science_qa.linear.EpCG_J}{167.6}
\defval{gb10.science_qa.linear.EpG_J}{167.6}
\defval{gb10.science_qa.linear.EpSG_J}{167.6}
\defval{gb10.science_qa.linear.N}{30}
\defval{gb10.science_qa.linear.TpG_s}{3.2}
\defval{gb10.science_qa.linear.attempts_per_goal}{1.00}
\defval{gb10.science_qa.linear.n_completed}{30}
\defval{gb10.science_qa.linear.n_correct}{30}
\defval{gb10.science_qa.linear.p_c}{1.00}
\defval{gb10.science_qa.linear.p_s}{1.00}
\defval{gb10.science_qa.linear.power_W}{52.1}
\defval{gb10.science_qa.min_successes}{30}
\defval{gb10.science_qa.power_ratio}{1.00}
\defval{gb10.science_qa.power_ratio_hi}{1.01}
\defval{gb10.science_qa.power_ratio_lo}{1.00}
\defval{gb10.science_qa.share_energy}{1.00}
\defval{gb10.science_qa.success_ratio}{1.00}
\defval{gb10.science_qa.time_ratio}{5.01}
\defval{gb10.science_qa.time_ratio_hi}{5.62}
\defval{gb10.science_qa.time_ratio_lo}{4.51}
\defval{gb10.tg_single_calc.OOI_EpCG}{0.34}
\defval{gb10.tg_single_calc.OOI_EpCG_hi}{0.41}
\defval{gb10.tg_single_calc.OOI_EpCG_lo}{0.27}
\defval{gb10.tg_single_calc.OOI_EpG}{0.34}
\defval{gb10.tg_single_calc.OOI_EpG_hi}{0.41}
\defval{gb10.tg_single_calc.OOI_EpG_lo}{0.27}
\defval{gb10.tg_single_calc.OOI_EpSG}{0.04}
\defval{gb10.tg_single_calc.OOI_EpSG_hi}{0.08}
\defval{gb10.tg_single_calc.OOI_EpSG_lo}{0.01}
\defval{gb10.tg_single_calc.agentic.E_wrong_goals_J}{1,517}
\defval{gb10.tg_single_calc.agentic.EpCG_J}{399.9}
\defval{gb10.tg_single_calc.agentic.EpG_J}{399.9}
\defval{gb10.tg_single_calc.agentic.EpSG_J}{444.3}
\defval{gb10.tg_single_calc.agentic.N}{30}
\defval{gb10.tg_single_calc.agentic.TpG_s}{5.6}
\defval{gb10.tg_single_calc.agentic.attempts_per_goal}{1.00}
\defval{gb10.tg_single_calc.agentic.n_completed}{30}
\defval{gb10.tg_single_calc.agentic.n_correct}{27}
\defval{gb10.tg_single_calc.agentic.p_c}{1.00}
\defval{gb10.tg_single_calc.agentic.p_s}{0.90}
\defval{gb10.tg_single_calc.agentic.power_W}{71.0}
\defval{gb10.tg_single_calc.energy_ratio}{0.34}
\defval{gb10.tg_single_calc.linear.E_wrong_goals_J}{32,103}
\defval{gb10.tg_single_calc.linear.EpCG_J}{1,187.9}
\defval{gb10.tg_single_calc.linear.EpG_J}{1,187.9}
\defval{gb10.tg_single_calc.linear.EpSG_J}{11,879.4}
\defval{gb10.tg_single_calc.linear.N}{30}
\defval{gb10.tg_single_calc.linear.TpG_s}{17.4}
\defval{gb10.tg_single_calc.linear.attempts_per_goal}{1.00}
\defval{gb10.tg_single_calc.linear.n_completed}{30}
\defval{gb10.tg_single_calc.linear.n_correct}{3}
\defval{gb10.tg_single_calc.linear.p_c}{1.00}
\defval{gb10.tg_single_calc.linear.p_s}{0.10}
\defval{gb10.tg_single_calc.linear.power_W}{68.3}
\defval{gb10.tg_single_calc.min_successes}{3}
\defval{gb10.tg_single_calc.power_ratio}{1.04}
\defval{gb10.tg_single_calc.power_ratio_hi}{1.05}
\defval{gb10.tg_single_calc.power_ratio_lo}{1.03}
\defval{gb10.tg_single_calc.share_energy}{0.33}
\defval{gb10.tg_single_calc.success_ratio}{0.11}
\defval{gb10.tg_single_calc.time_ratio}{0.32}
\defval{gb10.tg_single_calc.time_ratio_hi}{0.40}
\defval{gb10.tg_single_calc.time_ratio_lo}{0.26}
\defval{gb10_mistral_A.commonsense_reasoning.OOI_EpCG}{18.68}
\defval{gb10_mistral_A.commonsense_reasoning.OOI_EpCG_hi}{22.31}
\defval{gb10_mistral_A.commonsense_reasoning.OOI_EpCG_lo}{15.60}
\defval{gb10_mistral_A.commonsense_reasoning.OOI_EpG}{18.68}
\defval{gb10_mistral_A.commonsense_reasoning.OOI_EpG_hi}{22.31}
\defval{gb10_mistral_A.commonsense_reasoning.OOI_EpG_lo}{15.60}
\defval{gb10_mistral_A.commonsense_reasoning.OOI_EpSG}{31.14}
\defval{gb10_mistral_A.commonsense_reasoning.OOI_EpSG_hi}{69.00}
\defval{gb10_mistral_A.commonsense_reasoning.OOI_EpSG_lo}{19.97}
\defval{gb10_mistral_A.commonsense_reasoning.agentic.E_wrong_goals_J}{40,655}
\defval{gb10_mistral_A.commonsense_reasoning.agentic.EpCG_J}{9,657.1}
\defval{gb10_mistral_A.commonsense_reasoning.agentic.EpCall_J}{1,379.6}
\defval{gb10_mistral_A.commonsense_reasoning.agentic.EpG_J}{9,657.1}
\defval{gb10_mistral_A.commonsense_reasoning.agentic.EpSG_J}{16,095.2}
\defval{gb10_mistral_A.commonsense_reasoning.agentic.N}{10}
\defval{gb10_mistral_A.commonsense_reasoning.agentic.TpG_s}{128.4}
\defval{gb10_mistral_A.commonsense_reasoning.agentic.attempts_per_goal}{1.00}
\defval{gb10_mistral_A.commonsense_reasoning.agentic.calls_per_goal}{7.00}
\defval{gb10_mistral_A.commonsense_reasoning.agentic.n_completed}{10}
\defval{gb10_mistral_A.commonsense_reasoning.agentic.n_correct}{6}
\defval{gb10_mistral_A.commonsense_reasoning.agentic.p_c}{1.00}
\defval{gb10_mistral_A.commonsense_reasoning.agentic.p_s}{0.60}
\defval{gb10_mistral_A.commonsense_reasoning.agentic.power_W}{75.2}
\defval{gb10_mistral_A.commonsense_reasoning.energy_ratio}{18.68}
\defval{gb10_mistral_A.commonsense_reasoning.linear.E_wrong_goals_J}{0}
\defval{gb10_mistral_A.commonsense_reasoning.linear.EpCG_J}{516.9}
\defval{gb10_mistral_A.commonsense_reasoning.linear.EpCall_J}{516.9}
\defval{gb10_mistral_A.commonsense_reasoning.linear.EpG_J}{516.9}
\defval{gb10_mistral_A.commonsense_reasoning.linear.EpSG_J}{516.9}
\defval{gb10_mistral_A.commonsense_reasoning.linear.N}{10}
\defval{gb10_mistral_A.commonsense_reasoning.linear.TpG_s}{7.2}
\defval{gb10_mistral_A.commonsense_reasoning.linear.attempts_per_goal}{1.00}
\defval{gb10_mistral_A.commonsense_reasoning.linear.calls_per_goal}{1.00}
\defval{gb10_mistral_A.commonsense_reasoning.linear.n_completed}{10}
\defval{gb10_mistral_A.commonsense_reasoning.linear.n_correct}{10}
\defval{gb10_mistral_A.commonsense_reasoning.linear.p_c}{1.00}
\defval{gb10_mistral_A.commonsense_reasoning.linear.p_s}{1.00}
\defval{gb10_mistral_A.commonsense_reasoning.linear.power_W}{72.1}
\defval{gb10_mistral_A.commonsense_reasoning.min_successes}{6}
\defval{gb10_mistral_A.commonsense_reasoning.power_ratio}{1.04}
\defval{gb10_mistral_A.commonsense_reasoning.power_ratio_hi}{1.06}
\defval{gb10_mistral_A.commonsense_reasoning.power_ratio_lo}{1.03}
\defval{gb10_mistral_A.commonsense_reasoning.share_energy}{0.85}
\defval{gb10_mistral_A.commonsense_reasoning.success_ratio}{1.67}
\defval{gb10_mistral_A.commonsense_reasoning.time_ratio}{17.90}
\defval{gb10_mistral_A.commonsense_reasoning.time_ratio_hi}{21.20}
\defval{gb10_mistral_A.commonsense_reasoning.time_ratio_lo}{15.07}
\defval{gb10_mistral_A.excluded}{0}
\defval{gb10_mistral_A.factual_qa.OOI_EpCG}{7.31}
\defval{gb10_mistral_A.factual_qa.OOI_EpCG_hi}{8.54}
\defval{gb10_mistral_A.factual_qa.OOI_EpCG_lo}{6.21}
\defval{gb10_mistral_A.factual_qa.OOI_EpG}{7.31}
\defval{gb10_mistral_A.factual_qa.OOI_EpG_hi}{8.54}
\defval{gb10_mistral_A.factual_qa.OOI_EpG_lo}{6.21}
\defval{gb10_mistral_A.factual_qa.OOI_EpSG}{7.31}
\defval{gb10_mistral_A.factual_qa.OOI_EpSG_hi}{8.54}
\defval{gb10_mistral_A.factual_qa.OOI_EpSG_lo}{6.21}
\defval{gb10_mistral_A.factual_qa.agentic.E_wrong_goals_J}{0}
\defval{gb10_mistral_A.factual_qa.agentic.EpCG_J}{2,910.6}
\defval{gb10_mistral_A.factual_qa.agentic.EpCall_J}{485.1}
\defval{gb10_mistral_A.factual_qa.agentic.EpG_J}{2,910.6}
\defval{gb10_mistral_A.factual_qa.agentic.EpSG_J}{2,910.6}
\defval{gb10_mistral_A.factual_qa.agentic.N}{10}
\defval{gb10_mistral_A.factual_qa.agentic.TpG_s}{39.6}
\defval{gb10_mistral_A.factual_qa.agentic.attempts_per_goal}{1.00}
\defval{gb10_mistral_A.factual_qa.agentic.calls_per_goal}{6.00}
\defval{gb10_mistral_A.factual_qa.agentic.n_completed}{10}
\defval{gb10_mistral_A.factual_qa.agentic.n_correct}{10}
\defval{gb10_mistral_A.factual_qa.agentic.p_c}{1.00}
\defval{gb10_mistral_A.factual_qa.agentic.p_s}{1.00}
\defval{gb10_mistral_A.factual_qa.agentic.power_W}{73.6}
\defval{gb10_mistral_A.factual_qa.energy_ratio}{7.31}
\defval{gb10_mistral_A.factual_qa.linear.E_wrong_goals_J}{0}
\defval{gb10_mistral_A.factual_qa.linear.EpCG_J}{398.1}
\defval{gb10_mistral_A.factual_qa.linear.EpCall_J}{398.1}
\defval{gb10_mistral_A.factual_qa.linear.EpG_J}{398.1}
\defval{gb10_mistral_A.factual_qa.linear.EpSG_J}{398.1}
\defval{gb10_mistral_A.factual_qa.linear.N}{10}
\defval{gb10_mistral_A.factual_qa.linear.TpG_s}{5.6}
\defval{gb10_mistral_A.factual_qa.linear.attempts_per_goal}{1.00}
\defval{gb10_mistral_A.factual_qa.linear.calls_per_goal}{1.00}
\defval{gb10_mistral_A.factual_qa.linear.n_completed}{10}
\defval{gb10_mistral_A.factual_qa.linear.n_correct}{10}
\defval{gb10_mistral_A.factual_qa.linear.p_c}{1.00}
\defval{gb10_mistral_A.factual_qa.linear.p_s}{1.00}
\defval{gb10_mistral_A.factual_qa.linear.power_W}{71.3}
\defval{gb10_mistral_A.factual_qa.min_successes}{10}
\defval{gb10_mistral_A.factual_qa.power_ratio}{1.03}
\defval{gb10_mistral_A.factual_qa.power_ratio_hi}{1.04}
\defval{gb10_mistral_A.factual_qa.power_ratio_lo}{1.02}
\defval{gb10_mistral_A.factual_qa.share_energy}{1.00}
\defval{gb10_mistral_A.factual_qa.success_ratio}{1.00}
\defval{gb10_mistral_A.factual_qa.time_ratio}{7.09}
\defval{gb10_mistral_A.factual_qa.time_ratio_hi}{8.25}
\defval{gb10_mistral_A.factual_qa.time_ratio_lo}{6.04}
\defval{gb10_mistral_A.goals}{180}
\defval{gb10_mistral_A.gsm8k_basic.OOI_EpCG}{7.32}
\defval{gb10_mistral_A.gsm8k_basic.OOI_EpCG_hi}{9.39}
\defval{gb10_mistral_A.gsm8k_basic.OOI_EpCG_lo}{5.95}
\defval{gb10_mistral_A.gsm8k_basic.OOI_EpG}{7.32}
\defval{gb10_mistral_A.gsm8k_basic.OOI_EpG_hi}{9.39}
\defval{gb10_mistral_A.gsm8k_basic.OOI_EpG_lo}{5.95}
\defval{gb10_mistral_A.gsm8k_basic.OOI_EpSG}{9.41}
\defval{gb10_mistral_A.gsm8k_basic.OOI_EpSG_hi}{17.02}
\defval{gb10_mistral_A.gsm8k_basic.OOI_EpSG_lo}{5.79}
\defval{gb10_mistral_A.gsm8k_basic.agentic.E_wrong_goals_J}{5,921}
\defval{gb10_mistral_A.gsm8k_basic.agentic.EpCG_J}{2,071.8}
\defval{gb10_mistral_A.gsm8k_basic.agentic.EpCall_J}{414.4}
\defval{gb10_mistral_A.gsm8k_basic.agentic.EpG_J}{2,071.8}
\defval{gb10_mistral_A.gsm8k_basic.agentic.EpSG_J}{2,959.7}
\defval{gb10_mistral_A.gsm8k_basic.agentic.N}{10}
\defval{gb10_mistral_A.gsm8k_basic.agentic.TpG_s}{27.8}
\defval{gb10_mistral_A.gsm8k_basic.agentic.attempts_per_goal}{1.00}
\defval{gb10_mistral_A.gsm8k_basic.agentic.calls_per_goal}{5.00}
\defval{gb10_mistral_A.gsm8k_basic.agentic.n_completed}{10}
\defval{gb10_mistral_A.gsm8k_basic.agentic.n_correct}{7}
\defval{gb10_mistral_A.gsm8k_basic.agentic.p_c}{1.00}
\defval{gb10_mistral_A.gsm8k_basic.agentic.p_s}{0.70}
\defval{gb10_mistral_A.gsm8k_basic.agentic.power_W}{74.5}
\defval{gb10_mistral_A.gsm8k_basic.energy_ratio}{7.32}
\defval{gb10_mistral_A.gsm8k_basic.linear.E_wrong_goals_J}{293}
\defval{gb10_mistral_A.gsm8k_basic.linear.EpCG_J}{283.0}
\defval{gb10_mistral_A.gsm8k_basic.linear.EpCall_J}{283.0}
\defval{gb10_mistral_A.gsm8k_basic.linear.EpG_J}{283.0}
\defval{gb10_mistral_A.gsm8k_basic.linear.EpSG_J}{314.5}
\defval{gb10_mistral_A.gsm8k_basic.linear.N}{10}
\defval{gb10_mistral_A.gsm8k_basic.linear.TpG_s}{3.9}
\defval{gb10_mistral_A.gsm8k_basic.linear.attempts_per_goal}{1.00}
\defval{gb10_mistral_A.gsm8k_basic.linear.calls_per_goal}{1.00}
\defval{gb10_mistral_A.gsm8k_basic.linear.n_completed}{10}
\defval{gb10_mistral_A.gsm8k_basic.linear.n_correct}{9}
\defval{gb10_mistral_A.gsm8k_basic.linear.p_c}{1.00}
\defval{gb10_mistral_A.gsm8k_basic.linear.p_s}{0.90}
\defval{gb10_mistral_A.gsm8k_basic.linear.power_W}{72.9}
\defval{gb10_mistral_A.gsm8k_basic.min_successes}{7}
\defval{gb10_mistral_A.gsm8k_basic.power_ratio}{1.02}
\defval{gb10_mistral_A.gsm8k_basic.power_ratio_hi}{1.03}
\defval{gb10_mistral_A.gsm8k_basic.power_ratio_lo}{1.00}
\defval{gb10_mistral_A.gsm8k_basic.share_energy}{0.89}
\defval{gb10_mistral_A.gsm8k_basic.success_ratio}{1.29}
\defval{gb10_mistral_A.gsm8k_basic.time_ratio}{7.17}
\defval{gb10_mistral_A.gsm8k_basic.time_ratio_hi}{9.31}
\defval{gb10_mistral_A.gsm8k_basic.time_ratio_lo}{5.79}
\defval{gb10_mistral_A.gsm8k_multi_step.OOI_EpCG}{3.41}
\defval{gb10_mistral_A.gsm8k_multi_step.OOI_EpCG_hi}{3.69}
\defval{gb10_mistral_A.gsm8k_multi_step.OOI_EpCG_lo}{3.18}
\defval{gb10_mistral_A.gsm8k_multi_step.OOI_EpG}{3.41}
\defval{gb10_mistral_A.gsm8k_multi_step.OOI_EpG_hi}{3.69}
\defval{gb10_mistral_A.gsm8k_multi_step.OOI_EpG_lo}{3.18}
\defval{gb10_mistral_A.gsm8k_multi_step.OOI_EpSG}{3.79}
\defval{gb10_mistral_A.gsm8k_multi_step.OOI_EpSG_hi}{5.05}
\defval{gb10_mistral_A.gsm8k_multi_step.OOI_EpSG_lo}{3.25}
\defval{gb10_mistral_A.gsm8k_multi_step.agentic.E_wrong_goals_J}{2,766}
\defval{gb10_mistral_A.gsm8k_multi_step.agentic.EpCG_J}{2,965.7}
\defval{gb10_mistral_A.gsm8k_multi_step.agentic.EpCall_J}{593.1}
\defval{gb10_mistral_A.gsm8k_multi_step.agentic.EpG_J}{2,965.7}
\defval{gb10_mistral_A.gsm8k_multi_step.agentic.EpSG_J}{3,295.2}
\defval{gb10_mistral_A.gsm8k_multi_step.agentic.N}{10}
\defval{gb10_mistral_A.gsm8k_multi_step.agentic.TpG_s}{39.6}
\defval{gb10_mistral_A.gsm8k_multi_step.agentic.attempts_per_goal}{1.00}
\defval{gb10_mistral_A.gsm8k_multi_step.agentic.calls_per_goal}{5.00}
\defval{gb10_mistral_A.gsm8k_multi_step.agentic.n_completed}{10}
\defval{gb10_mistral_A.gsm8k_multi_step.agentic.n_correct}{9}
\defval{gb10_mistral_A.gsm8k_multi_step.agentic.p_c}{1.00}
\defval{gb10_mistral_A.gsm8k_multi_step.agentic.p_s}{0.90}
\defval{gb10_mistral_A.gsm8k_multi_step.agentic.power_W}{74.9}
\defval{gb10_mistral_A.gsm8k_multi_step.energy_ratio}{3.41}
\defval{gb10_mistral_A.gsm8k_multi_step.linear.E_wrong_goals_J}{0}
\defval{gb10_mistral_A.gsm8k_multi_step.linear.EpCG_J}{868.5}
\defval{gb10_mistral_A.gsm8k_multi_step.linear.EpCall_J}{868.5}
\defval{gb10_mistral_A.gsm8k_multi_step.linear.EpG_J}{868.5}
\defval{gb10_mistral_A.gsm8k_multi_step.linear.EpSG_J}{868.5}
\defval{gb10_mistral_A.gsm8k_multi_step.linear.N}{10}
\defval{gb10_mistral_A.gsm8k_multi_step.linear.TpG_s}{12.0}
\defval{gb10_mistral_A.gsm8k_multi_step.linear.attempts_per_goal}{1.00}
\defval{gb10_mistral_A.gsm8k_multi_step.linear.calls_per_goal}{1.00}
\defval{gb10_mistral_A.gsm8k_multi_step.linear.n_completed}{10}
\defval{gb10_mistral_A.gsm8k_multi_step.linear.n_correct}{10}
\defval{gb10_mistral_A.gsm8k_multi_step.linear.p_c}{1.00}
\defval{gb10_mistral_A.gsm8k_multi_step.linear.p_s}{1.00}
\defval{gb10_mistral_A.gsm8k_multi_step.linear.power_W}{72.3}
\defval{gb10_mistral_A.gsm8k_multi_step.min_successes}{9}
\defval{gb10_mistral_A.gsm8k_multi_step.power_ratio}{1.04}
\defval{gb10_mistral_A.gsm8k_multi_step.power_ratio_hi}{1.05}
\defval{gb10_mistral_A.gsm8k_multi_step.power_ratio_lo}{1.03}
\defval{gb10_mistral_A.gsm8k_multi_step.share_energy}{0.92}
\defval{gb10_mistral_A.gsm8k_multi_step.success_ratio}{1.11}
\defval{gb10_mistral_A.gsm8k_multi_step.time_ratio}{3.29}
\defval{gb10_mistral_A.gsm8k_multi_step.time_ratio_hi}{3.57}
\defval{gb10_mistral_A.gsm8k_multi_step.time_ratio_lo}{3.07}
\defval{gb10_mistral_A.logical_reasoning.OOI_EpCG}{11.25}
\defval{gb10_mistral_A.logical_reasoning.OOI_EpCG_hi}{13.04}
\defval{gb10_mistral_A.logical_reasoning.OOI_EpCG_lo}{9.68}
\defval{gb10_mistral_A.logical_reasoning.OOI_EpG}{11.25}
\defval{gb10_mistral_A.logical_reasoning.OOI_EpG_hi}{13.04}
\defval{gb10_mistral_A.logical_reasoning.OOI_EpG_lo}{9.68}
\defval{gb10_mistral_A.logical_reasoning.OOI_EpSG}{11.25}
\defval{gb10_mistral_A.logical_reasoning.OOI_EpSG_hi}{13.04}
\defval{gb10_mistral_A.logical_reasoning.OOI_EpSG_lo}{9.68}
\defval{gb10_mistral_A.logical_reasoning.agentic.E_wrong_goals_J}{0}
\defval{gb10_mistral_A.logical_reasoning.agentic.EpCG_J}{3,895.0}
\defval{gb10_mistral_A.logical_reasoning.agentic.EpCall_J}{556.4}
\defval{gb10_mistral_A.logical_reasoning.agentic.EpG_J}{3,895.0}
\defval{gb10_mistral_A.logical_reasoning.agentic.EpSG_J}{3,895.0}
\defval{gb10_mistral_A.logical_reasoning.agentic.N}{10}
\defval{gb10_mistral_A.logical_reasoning.agentic.TpG_s}{52.7}
\defval{gb10_mistral_A.logical_reasoning.agentic.attempts_per_goal}{1.00}
\defval{gb10_mistral_A.logical_reasoning.agentic.calls_per_goal}{7.00}
\defval{gb10_mistral_A.logical_reasoning.agentic.n_completed}{10}
\defval{gb10_mistral_A.logical_reasoning.agentic.n_correct}{10}
\defval{gb10_mistral_A.logical_reasoning.agentic.p_c}{1.00}
\defval{gb10_mistral_A.logical_reasoning.agentic.p_s}{1.00}
\defval{gb10_mistral_A.logical_reasoning.agentic.power_W}{73.9}
\defval{gb10_mistral_A.logical_reasoning.energy_ratio}{11.25}
\defval{gb10_mistral_A.logical_reasoning.linear.E_wrong_goals_J}{0}
\defval{gb10_mistral_A.logical_reasoning.linear.EpCG_J}{346.2}
\defval{gb10_mistral_A.logical_reasoning.linear.EpCall_J}{346.2}
\defval{gb10_mistral_A.logical_reasoning.linear.EpG_J}{346.2}
\defval{gb10_mistral_A.logical_reasoning.linear.EpSG_J}{346.2}
\defval{gb10_mistral_A.logical_reasoning.linear.N}{10}
\defval{gb10_mistral_A.logical_reasoning.linear.TpG_s}{4.8}
\defval{gb10_mistral_A.logical_reasoning.linear.attempts_per_goal}{1.00}
\defval{gb10_mistral_A.logical_reasoning.linear.calls_per_goal}{1.00}
\defval{gb10_mistral_A.logical_reasoning.linear.n_completed}{10}
\defval{gb10_mistral_A.logical_reasoning.linear.n_correct}{10}
\defval{gb10_mistral_A.logical_reasoning.linear.p_c}{1.00}
\defval{gb10_mistral_A.logical_reasoning.linear.p_s}{1.00}
\defval{gb10_mistral_A.logical_reasoning.linear.power_W}{71.6}
\defval{gb10_mistral_A.logical_reasoning.min_successes}{10}
\defval{gb10_mistral_A.logical_reasoning.power_ratio}{1.03}
\defval{gb10_mistral_A.logical_reasoning.power_ratio_hi}{1.04}
\defval{gb10_mistral_A.logical_reasoning.power_ratio_lo}{1.03}
\defval{gb10_mistral_A.logical_reasoning.share_energy}{1.00}
\defval{gb10_mistral_A.logical_reasoning.success_ratio}{1.00}
\defval{gb10_mistral_A.logical_reasoning.time_ratio}{10.89}
\defval{gb10_mistral_A.logical_reasoning.time_ratio_hi}{12.61}
\defval{gb10_mistral_A.logical_reasoning.time_ratio_lo}{9.39}
\defval{gb10_mistral_A.science_qa.OOI_EpCG}{5.98}
\defval{gb10_mistral_A.science_qa.OOI_EpCG_hi}{7.13}
\defval{gb10_mistral_A.science_qa.OOI_EpCG_lo}{5.07}
\defval{gb10_mistral_A.science_qa.OOI_EpG}{5.98}
\defval{gb10_mistral_A.science_qa.OOI_EpG_hi}{7.13}
\defval{gb10_mistral_A.science_qa.OOI_EpG_lo}{5.07}
\defval{gb10_mistral_A.science_qa.OOI_EpSG}{5.98}
\defval{gb10_mistral_A.science_qa.OOI_EpSG_hi}{7.13}
\defval{gb10_mistral_A.science_qa.OOI_EpSG_lo}{5.07}
\defval{gb10_mistral_A.science_qa.agentic.E_wrong_goals_J}{0}
\defval{gb10_mistral_A.science_qa.agentic.EpCG_J}{1,144.4}
\defval{gb10_mistral_A.science_qa.agentic.EpCall_J}{381.5}
\defval{gb10_mistral_A.science_qa.agentic.EpG_J}{1,144.4}
\defval{gb10_mistral_A.science_qa.agentic.EpSG_J}{1,144.4}
\defval{gb10_mistral_A.science_qa.agentic.N}{10}
\defval{gb10_mistral_A.science_qa.agentic.TpG_s}{16.6}
\defval{gb10_mistral_A.science_qa.agentic.attempts_per_goal}{1.00}
\defval{gb10_mistral_A.science_qa.agentic.calls_per_goal}{3.00}
\defval{gb10_mistral_A.science_qa.agentic.n_completed}{10}
\defval{gb10_mistral_A.science_qa.agentic.n_correct}{10}
\defval{gb10_mistral_A.science_qa.agentic.p_c}{1.00}
\defval{gb10_mistral_A.science_qa.agentic.p_s}{1.00}
\defval{gb10_mistral_A.science_qa.agentic.power_W}{68.7}
\defval{gb10_mistral_A.science_qa.energy_ratio}{5.98}
\defval{gb10_mistral_A.science_qa.linear.E_wrong_goals_J}{0}
\defval{gb10_mistral_A.science_qa.linear.EpCG_J}{191.3}
\defval{gb10_mistral_A.science_qa.linear.EpCall_J}{191.3}
\defval{gb10_mistral_A.science_qa.linear.EpG_J}{191.3}
\defval{gb10_mistral_A.science_qa.linear.EpSG_J}{191.3}
\defval{gb10_mistral_A.science_qa.linear.N}{10}
\defval{gb10_mistral_A.science_qa.linear.TpG_s}{2.8}
\defval{gb10_mistral_A.science_qa.linear.attempts_per_goal}{1.00}
\defval{gb10_mistral_A.science_qa.linear.calls_per_goal}{1.00}
\defval{gb10_mistral_A.science_qa.linear.n_completed}{10}
\defval{gb10_mistral_A.science_qa.linear.n_correct}{10}
\defval{gb10_mistral_A.science_qa.linear.p_c}{1.00}
\defval{gb10_mistral_A.science_qa.linear.p_s}{1.00}
\defval{gb10_mistral_A.science_qa.linear.power_W}{69.1}
\defval{gb10_mistral_A.science_qa.min_successes}{10}
\defval{gb10_mistral_A.science_qa.power_ratio}{1.00}
\defval{gb10_mistral_A.science_qa.power_ratio_hi}{1.00}
\defval{gb10_mistral_A.science_qa.power_ratio_lo}{0.99}
\defval{gb10_mistral_A.science_qa.share_energy}{1.00}
\defval{gb10_mistral_A.science_qa.success_ratio}{1.00}
\defval{gb10_mistral_A.science_qa.time_ratio}{6.01}
\defval{gb10_mistral_A.science_qa.time_ratio_hi}{7.18}
\defval{gb10_mistral_A.science_qa.time_ratio_lo}{5.08}
\defval{gb10_mistral_A.tg_sequential_2.OOI_EpCG}{0.57}
\defval{gb10_mistral_A.tg_sequential_2.OOI_EpCG_hi}{0.66}
\defval{gb10_mistral_A.tg_sequential_2.OOI_EpCG_lo}{0.49}
\defval{gb10_mistral_A.tg_sequential_2.OOI_EpG}{0.57}
\defval{gb10_mistral_A.tg_sequential_2.OOI_EpG_hi}{0.66}
\defval{gb10_mistral_A.tg_sequential_2.OOI_EpG_lo}{0.49}
\defval{gb10_mistral_A.tg_sequential_2.OOI_EpSG}{undefined}
\defval{gb10_mistral_A.tg_sequential_2.OOI_EpSG_hi}{undefined}
\defval{gb10_mistral_A.tg_sequential_2.OOI_EpSG_lo}{undefined}
\defval{gb10_mistral_A.tg_sequential_2.agentic.E_wrong_goals_J}{0}
\defval{gb10_mistral_A.tg_sequential_2.agentic.EpCG_J}{1,029.8}
\defval{gb10_mistral_A.tg_sequential_2.agentic.EpCall_J}{1,029.8}
\defval{gb10_mistral_A.tg_sequential_2.agentic.EpG_J}{1,029.8}
\defval{gb10_mistral_A.tg_sequential_2.agentic.EpSG_J}{1,029.8}
\defval{gb10_mistral_A.tg_sequential_2.agentic.N}{10}
\defval{gb10_mistral_A.tg_sequential_2.agentic.TpG_s}{14.1}
\defval{gb10_mistral_A.tg_sequential_2.agentic.attempts_per_goal}{1.00}
\defval{gb10_mistral_A.tg_sequential_2.agentic.calls_per_goal}{1.00}
\defval{gb10_mistral_A.tg_sequential_2.agentic.n_completed}{10}
\defval{gb10_mistral_A.tg_sequential_2.agentic.n_correct}{10}
\defval{gb10_mistral_A.tg_sequential_2.agentic.p_c}{1.00}
\defval{gb10_mistral_A.tg_sequential_2.agentic.p_s}{1.00}
\defval{gb10_mistral_A.tg_sequential_2.agentic.power_W}{72.9}
\defval{gb10_mistral_A.tg_sequential_2.energy_ratio}{0.57}
\defval{gb10_mistral_A.tg_sequential_2.linear.E_wrong_goals_J}{18,002}
\defval{gb10_mistral_A.tg_sequential_2.linear.EpCG_J}{1,800.2}
\defval{gb10_mistral_A.tg_sequential_2.linear.EpCall_J}{1,800.2}
\defval{gb10_mistral_A.tg_sequential_2.linear.EpG_J}{1,800.2}
\defval{gb10_mistral_A.tg_sequential_2.linear.EpSG_J}{undefined}
\defval{gb10_mistral_A.tg_sequential_2.linear.N}{10}
\defval{gb10_mistral_A.tg_sequential_2.linear.TpG_s}{25.8}
\defval{gb10_mistral_A.tg_sequential_2.linear.attempts_per_goal}{1.00}
\defval{gb10_mistral_A.tg_sequential_2.linear.calls_per_goal}{1.00}
\defval{gb10_mistral_A.tg_sequential_2.linear.n_completed}{10}
\defval{gb10_mistral_A.tg_sequential_2.linear.n_correct}{0}
\defval{gb10_mistral_A.tg_sequential_2.linear.p_c}{1.00}
\defval{gb10_mistral_A.tg_sequential_2.linear.p_s}{0.00}
\defval{gb10_mistral_A.tg_sequential_2.linear.power_W}{69.9}
\defval{gb10_mistral_A.tg_sequential_2.min_successes}{0}
\defval{gb10_mistral_A.tg_sequential_2.power_ratio}{1.04}
\defval{gb10_mistral_A.tg_sequential_2.power_ratio_hi}{1.06}
\defval{gb10_mistral_A.tg_sequential_2.power_ratio_lo}{1.03}
\defval{gb10_mistral_A.tg_sequential_2.share_energy}{undefined}
\defval{gb10_mistral_A.tg_sequential_2.success_ratio}{0.00}
\defval{gb10_mistral_A.tg_sequential_2.time_ratio}{0.55}
\defval{gb10_mistral_A.tg_sequential_2.time_ratio_hi}{0.63}
\defval{gb10_mistral_A.tg_sequential_2.time_ratio_lo}{0.47}
\defval{gb10_mistral_A.tg_single_calc.OOI_EpCG}{0.23}
\defval{gb10_mistral_A.tg_single_calc.OOI_EpCG_hi}{0.28}
\defval{gb10_mistral_A.tg_single_calc.OOI_EpCG_lo}{0.18}
\defval{gb10_mistral_A.tg_single_calc.OOI_EpG}{0.23}
\defval{gb10_mistral_A.tg_single_calc.OOI_EpG_hi}{0.28}
\defval{gb10_mistral_A.tg_single_calc.OOI_EpG_lo}{0.18}
\defval{gb10_mistral_A.tg_single_calc.OOI_EpSG}{0.02}
\defval{gb10_mistral_A.tg_single_calc.OOI_EpSG_hi}{undefined}
\defval{gb10_mistral_A.tg_single_calc.OOI_EpSG_lo}{undefined}
\defval{gb10_mistral_A.tg_single_calc.agentic.E_wrong_goals_J}{0}
\defval{gb10_mistral_A.tg_single_calc.agentic.EpCG_J}{260.8}
\defval{gb10_mistral_A.tg_single_calc.agentic.EpCall_J}{260.8}
\defval{gb10_mistral_A.tg_single_calc.agentic.EpG_J}{260.8}
\defval{gb10_mistral_A.tg_single_calc.agentic.EpSG_J}{260.8}
\defval{gb10_mistral_A.tg_single_calc.agentic.N}{10}
\defval{gb10_mistral_A.tg_single_calc.agentic.TpG_s}{3.7}
\defval{gb10_mistral_A.tg_single_calc.agentic.attempts_per_goal}{1.00}
\defval{gb10_mistral_A.tg_single_calc.agentic.calls_per_goal}{1.00}
\defval{gb10_mistral_A.tg_single_calc.agentic.n_completed}{10}
\defval{gb10_mistral_A.tg_single_calc.agentic.n_correct}{10}
\defval{gb10_mistral_A.tg_single_calc.agentic.p_c}{1.00}
\defval{gb10_mistral_A.tg_single_calc.agentic.p_s}{1.00}
\defval{gb10_mistral_A.tg_single_calc.agentic.power_W}{70.4}
\defval{gb10_mistral_A.tg_single_calc.energy_ratio}{0.23}
\defval{gb10_mistral_A.tg_single_calc.linear.E_wrong_goals_J}{10,548}
\defval{gb10_mistral_A.tg_single_calc.linear.EpCG_J}{1,155.1}
\defval{gb10_mistral_A.tg_single_calc.linear.EpCall_J}{1,155.1}
\defval{gb10_mistral_A.tg_single_calc.linear.EpG_J}{1,155.1}
\defval{gb10_mistral_A.tg_single_calc.linear.EpSG_J}{11,550.8}
\defval{gb10_mistral_A.tg_single_calc.linear.N}{10}
\defval{gb10_mistral_A.tg_single_calc.linear.TpG_s}{17.2}
\defval{gb10_mistral_A.tg_single_calc.linear.attempts_per_goal}{1.00}
\defval{gb10_mistral_A.tg_single_calc.linear.calls_per_goal}{1.00}
\defval{gb10_mistral_A.tg_single_calc.linear.n_completed}{10}
\defval{gb10_mistral_A.tg_single_calc.linear.n_correct}{1}
\defval{gb10_mistral_A.tg_single_calc.linear.p_c}{1.00}
\defval{gb10_mistral_A.tg_single_calc.linear.p_s}{0.10}
\defval{gb10_mistral_A.tg_single_calc.linear.power_W}{67.3}
\defval{gb10_mistral_A.tg_single_calc.min_successes}{1}
\defval{gb10_mistral_A.tg_single_calc.power_ratio}{1.05}
\defval{gb10_mistral_A.tg_single_calc.power_ratio_hi}{1.05}
\defval{gb10_mistral_A.tg_single_calc.power_ratio_lo}{1.04}
\defval{gb10_mistral_A.tg_single_calc.share_energy}{0.39}
\defval{gb10_mistral_A.tg_single_calc.success_ratio}{0.10}
\defval{gb10_mistral_A.tg_single_calc.time_ratio}{0.22}
\defval{gb10_mistral_A.tg_single_calc.time_ratio_hi}{0.27}
\defval{gb10_mistral_A.tg_single_calc.time_ratio_lo}{0.17}
\defval{gb10_mistral_A.tg_single_db.OOI_EpCG}{0.45}
\defval{gb10_mistral_A.tg_single_db.OOI_EpCG_hi}{0.71}
\defval{gb10_mistral_A.tg_single_db.OOI_EpCG_lo}{0.30}
\defval{gb10_mistral_A.tg_single_db.OOI_EpG}{0.45}
\defval{gb10_mistral_A.tg_single_db.OOI_EpG_hi}{0.71}
\defval{gb10_mistral_A.tg_single_db.OOI_EpG_lo}{0.30}
\defval{gb10_mistral_A.tg_single_db.OOI_EpSG}{undefined}
\defval{gb10_mistral_A.tg_single_db.OOI_EpSG_hi}{undefined}
\defval{gb10_mistral_A.tg_single_db.OOI_EpSG_lo}{undefined}
\defval{gb10_mistral_A.tg_single_db.agentic.E_wrong_goals_J}{0}
\defval{gb10_mistral_A.tg_single_db.agentic.EpCG_J}{441.5}
\defval{gb10_mistral_A.tg_single_db.agentic.EpCall_J}{441.5}
\defval{gb10_mistral_A.tg_single_db.agentic.EpG_J}{441.5}
\defval{gb10_mistral_A.tg_single_db.agentic.EpSG_J}{441.5}
\defval{gb10_mistral_A.tg_single_db.agentic.N}{10}
\defval{gb10_mistral_A.tg_single_db.agentic.TpG_s}{6.4}
\defval{gb10_mistral_A.tg_single_db.agentic.attempts_per_goal}{1.00}
\defval{gb10_mistral_A.tg_single_db.agentic.calls_per_goal}{1.00}
\defval{gb10_mistral_A.tg_single_db.agentic.n_completed}{10}
\defval{gb10_mistral_A.tg_single_db.agentic.n_correct}{10}
\defval{gb10_mistral_A.tg_single_db.agentic.p_c}{1.00}
\defval{gb10_mistral_A.tg_single_db.agentic.p_s}{1.00}
\defval{gb10_mistral_A.tg_single_db.agentic.power_W}{68.9}
\defval{gb10_mistral_A.tg_single_db.energy_ratio}{0.45}
\defval{gb10_mistral_A.tg_single_db.linear.E_wrong_goals_J}{9,864}
\defval{gb10_mistral_A.tg_single_db.linear.EpCG_J}{986.4}
\defval{gb10_mistral_A.tg_single_db.linear.EpCall_J}{986.4}
\defval{gb10_mistral_A.tg_single_db.linear.EpG_J}{986.4}
\defval{gb10_mistral_A.tg_single_db.linear.EpSG_J}{undefined}
\defval{gb10_mistral_A.tg_single_db.linear.N}{10}
\defval{gb10_mistral_A.tg_single_db.linear.TpG_s}{14.8}
\defval{gb10_mistral_A.tg_single_db.linear.attempts_per_goal}{1.00}
\defval{gb10_mistral_A.tg_single_db.linear.calls_per_goal}{1.00}
\defval{gb10_mistral_A.tg_single_db.linear.n_completed}{10}
\defval{gb10_mistral_A.tg_single_db.linear.n_correct}{0}
\defval{gb10_mistral_A.tg_single_db.linear.p_c}{1.00}
\defval{gb10_mistral_A.tg_single_db.linear.p_s}{0.00}
\defval{gb10_mistral_A.tg_single_db.linear.power_W}{66.8}
\defval{gb10_mistral_A.tg_single_db.min_successes}{0}
\defval{gb10_mistral_A.tg_single_db.power_ratio}{1.03}
\defval{gb10_mistral_A.tg_single_db.power_ratio_hi}{1.05}
\defval{gb10_mistral_A.tg_single_db.power_ratio_lo}{1.02}
\defval{gb10_mistral_A.tg_single_db.share_energy}{undefined}
\defval{gb10_mistral_A.tg_single_db.success_ratio}{0.00}
\defval{gb10_mistral_A.tg_single_db.time_ratio}{0.43}
\defval{gb10_mistral_A.tg_single_db.time_ratio_hi}{0.70}
\defval{gb10_mistral_A.tg_single_db.time_ratio_lo}{0.29}
\defval{gb10_mistral_B.commonsense_reasoning.OOI_EpCG}{31.48}
\defval{gb10_mistral_B.commonsense_reasoning.OOI_EpCG_hi}{36.09}
\defval{gb10_mistral_B.commonsense_reasoning.OOI_EpCG_lo}{26.39}
\defval{gb10_mistral_B.commonsense_reasoning.OOI_EpG}{31.48}
\defval{gb10_mistral_B.commonsense_reasoning.OOI_EpG_hi}{36.09}
\defval{gb10_mistral_B.commonsense_reasoning.OOI_EpG_lo}{26.39}
\defval{gb10_mistral_B.commonsense_reasoning.OOI_EpSG}{34.98}
\defval{gb10_mistral_B.commonsense_reasoning.OOI_EpSG_hi}{47.80}
\defval{gb10_mistral_B.commonsense_reasoning.OOI_EpSG_lo}{27.60}
\defval{gb10_mistral_B.commonsense_reasoning.agentic.E_wrong_goals_J}{18,545}
\defval{gb10_mistral_B.commonsense_reasoning.agentic.EpCG_J}{18,356.3}
\defval{gb10_mistral_B.commonsense_reasoning.agentic.EpCall_J}{1,380.2}
\defval{gb10_mistral_B.commonsense_reasoning.agentic.EpG_J}{18,356.3}
\defval{gb10_mistral_B.commonsense_reasoning.agentic.EpSG_J}{20,395.8}
\defval{gb10_mistral_B.commonsense_reasoning.agentic.N}{10}
\defval{gb10_mistral_B.commonsense_reasoning.agentic.TpG_s}{242.1}
\defval{gb10_mistral_B.commonsense_reasoning.agentic.attempts_per_goal}{1.90}
\defval{gb10_mistral_B.commonsense_reasoning.agentic.calls_per_goal}{13.30}
\defval{gb10_mistral_B.commonsense_reasoning.agentic.n_completed}{10}
\defval{gb10_mistral_B.commonsense_reasoning.agentic.n_correct}{9}
\defval{gb10_mistral_B.commonsense_reasoning.agentic.p_c}{1.00}
\defval{gb10_mistral_B.commonsense_reasoning.agentic.p_s}{0.90}
\defval{gb10_mistral_B.commonsense_reasoning.agentic.power_W}{75.8}
\defval{gb10_mistral_B.commonsense_reasoning.energy_ratio}{31.48}
\defval{gb10_mistral_B.commonsense_reasoning.linear.E_wrong_goals_J}{0}
\defval{gb10_mistral_B.commonsense_reasoning.linear.EpCG_J}{583.1}
\defval{gb10_mistral_B.commonsense_reasoning.linear.EpCall_J}{583.1}
\defval{gb10_mistral_B.commonsense_reasoning.linear.EpG_J}{583.1}
\defval{gb10_mistral_B.commonsense_reasoning.linear.EpSG_J}{583.1}
\defval{gb10_mistral_B.commonsense_reasoning.linear.N}{10}
\defval{gb10_mistral_B.commonsense_reasoning.linear.TpG_s}{8.1}
\defval{gb10_mistral_B.commonsense_reasoning.linear.attempts_per_goal}{1.00}
\defval{gb10_mistral_B.commonsense_reasoning.linear.calls_per_goal}{1.00}
\defval{gb10_mistral_B.commonsense_reasoning.linear.n_completed}{10}
\defval{gb10_mistral_B.commonsense_reasoning.linear.n_correct}{10}
\defval{gb10_mistral_B.commonsense_reasoning.linear.p_c}{1.00}
\defval{gb10_mistral_B.commonsense_reasoning.linear.p_s}{1.00}
\defval{gb10_mistral_B.commonsense_reasoning.linear.power_W}{71.7}
\defval{gb10_mistral_B.commonsense_reasoning.min_successes}{9}
\defval{gb10_mistral_B.commonsense_reasoning.power_ratio}{1.06}
\defval{gb10_mistral_B.commonsense_reasoning.power_ratio_hi}{1.07}
\defval{gb10_mistral_B.commonsense_reasoning.power_ratio_lo}{1.05}
\defval{gb10_mistral_B.commonsense_reasoning.share_energy}{0.97}
\defval{gb10_mistral_B.commonsense_reasoning.success_ratio}{1.11}
\defval{gb10_mistral_B.commonsense_reasoning.time_ratio}{29.75}
\defval{gb10_mistral_B.commonsense_reasoning.time_ratio_hi}{34.07}
\defval{gb10_mistral_B.commonsense_reasoning.time_ratio_lo}{25.04}
\defval{gb10_mistral_B.excluded}{0}
\defval{gb10_mistral_B.factual_qa.OOI_EpCG}{14.59}
\defval{gb10_mistral_B.factual_qa.OOI_EpCG_hi}{17.04}
\defval{gb10_mistral_B.factual_qa.OOI_EpCG_lo}{11.83}
\defval{gb10_mistral_B.factual_qa.OOI_EpG}{14.59}
\defval{gb10_mistral_B.factual_qa.OOI_EpG_hi}{17.04}
\defval{gb10_mistral_B.factual_qa.OOI_EpG_lo}{11.83}
\defval{gb10_mistral_B.factual_qa.OOI_EpSG}{14.59}
\defval{gb10_mistral_B.factual_qa.OOI_EpSG_hi}{17.04}
\defval{gb10_mistral_B.factual_qa.OOI_EpSG_lo}{11.83}
\defval{gb10_mistral_B.factual_qa.agentic.E_wrong_goals_J}{0}
\defval{gb10_mistral_B.factual_qa.agentic.EpCG_J}{5,916.3}
\defval{gb10_mistral_B.factual_qa.agentic.EpCall_J}{519.0}
\defval{gb10_mistral_B.factual_qa.agentic.EpG_J}{5,916.3}
\defval{gb10_mistral_B.factual_qa.agentic.EpSG_J}{5,916.3}
\defval{gb10_mistral_B.factual_qa.agentic.N}{10}
\defval{gb10_mistral_B.factual_qa.agentic.TpG_s}{81.3}
\defval{gb10_mistral_B.factual_qa.agentic.attempts_per_goal}{1.90}
\defval{gb10_mistral_B.factual_qa.agentic.calls_per_goal}{11.40}
\defval{gb10_mistral_B.factual_qa.agentic.n_completed}{10}
\defval{gb10_mistral_B.factual_qa.agentic.n_correct}{10}
\defval{gb10_mistral_B.factual_qa.agentic.p_c}{1.00}
\defval{gb10_mistral_B.factual_qa.agentic.p_s}{1.00}
\defval{gb10_mistral_B.factual_qa.agentic.power_W}{72.8}
\defval{gb10_mistral_B.factual_qa.energy_ratio}{14.59}
\defval{gb10_mistral_B.factual_qa.linear.E_wrong_goals_J}{0}
\defval{gb10_mistral_B.factual_qa.linear.EpCG_J}{405.5}
\defval{gb10_mistral_B.factual_qa.linear.EpCall_J}{405.5}
\defval{gb10_mistral_B.factual_qa.linear.EpG_J}{405.5}
\defval{gb10_mistral_B.factual_qa.linear.EpSG_J}{405.5}
\defval{gb10_mistral_B.factual_qa.linear.N}{10}
\defval{gb10_mistral_B.factual_qa.linear.TpG_s}{5.9}
\defval{gb10_mistral_B.factual_qa.linear.attempts_per_goal}{1.00}
\defval{gb10_mistral_B.factual_qa.linear.calls_per_goal}{1.00}
\defval{gb10_mistral_B.factual_qa.linear.n_completed}{10}
\defval{gb10_mistral_B.factual_qa.linear.n_correct}{10}
\defval{gb10_mistral_B.factual_qa.linear.p_c}{1.00}
\defval{gb10_mistral_B.factual_qa.linear.p_s}{1.00}
\defval{gb10_mistral_B.factual_qa.linear.power_W}{68.5}
\defval{gb10_mistral_B.factual_qa.min_successes}{10}
\defval{gb10_mistral_B.factual_qa.power_ratio}{1.06}
\defval{gb10_mistral_B.factual_qa.power_ratio_hi}{1.13}
\defval{gb10_mistral_B.factual_qa.power_ratio_lo}{1.00}
\defval{gb10_mistral_B.factual_qa.share_energy}{1.00}
\defval{gb10_mistral_B.factual_qa.success_ratio}{1.00}
\defval{gb10_mistral_B.factual_qa.time_ratio}{13.72}
\defval{gb10_mistral_B.factual_qa.time_ratio_hi}{15.92}
\defval{gb10_mistral_B.factual_qa.time_ratio_lo}{11.44}
\defval{gb10_mistral_B.goals}{180}
\defval{gb10_mistral_B.gsm8k_basic.OOI_EpCG}{12.42}
\defval{gb10_mistral_B.gsm8k_basic.OOI_EpCG_hi}{16.63}
\defval{gb10_mistral_B.gsm8k_basic.OOI_EpCG_lo}{9.75}
\defval{gb10_mistral_B.gsm8k_basic.OOI_EpG}{12.42}
\defval{gb10_mistral_B.gsm8k_basic.OOI_EpG_hi}{16.63}
\defval{gb10_mistral_B.gsm8k_basic.OOI_EpG_lo}{9.75}
\defval{gb10_mistral_B.gsm8k_basic.OOI_EpSG}{13.97}
\defval{gb10_mistral_B.gsm8k_basic.OOI_EpSG_hi}{24.05}
\defval{gb10_mistral_B.gsm8k_basic.OOI_EpSG_lo}{8.59}
\defval{gb10_mistral_B.gsm8k_basic.agentic.E_wrong_goals_J}{6,858}
\defval{gb10_mistral_B.gsm8k_basic.agentic.EpCG_J}{3,434.6}
\defval{gb10_mistral_B.gsm8k_basic.agentic.EpCall_J}{361.5}
\defval{gb10_mistral_B.gsm8k_basic.agentic.EpG_J}{3,434.6}
\defval{gb10_mistral_B.gsm8k_basic.agentic.EpSG_J}{4,293.2}
\defval{gb10_mistral_B.gsm8k_basic.agentic.N}{10}
\defval{gb10_mistral_B.gsm8k_basic.agentic.TpG_s}{51.4}
\defval{gb10_mistral_B.gsm8k_basic.agentic.attempts_per_goal}{1.90}
\defval{gb10_mistral_B.gsm8k_basic.agentic.calls_per_goal}{9.50}
\defval{gb10_mistral_B.gsm8k_basic.agentic.n_completed}{10}
\defval{gb10_mistral_B.gsm8k_basic.agentic.n_correct}{8}
\defval{gb10_mistral_B.gsm8k_basic.agentic.p_c}{1.00}
\defval{gb10_mistral_B.gsm8k_basic.agentic.p_s}{0.80}
\defval{gb10_mistral_B.gsm8k_basic.agentic.power_W}{66.9}
\defval{gb10_mistral_B.gsm8k_basic.energy_ratio}{12.42}
\defval{gb10_mistral_B.gsm8k_basic.linear.E_wrong_goals_J}{397}
\defval{gb10_mistral_B.gsm8k_basic.linear.EpCG_J}{276.6}
\defval{gb10_mistral_B.gsm8k_basic.linear.EpCall_J}{276.6}
\defval{gb10_mistral_B.gsm8k_basic.linear.EpG_J}{276.6}
\defval{gb10_mistral_B.gsm8k_basic.linear.EpSG_J}{307.3}
\defval{gb10_mistral_B.gsm8k_basic.linear.N}{10}
\defval{gb10_mistral_B.gsm8k_basic.linear.TpG_s}{4.3}
\defval{gb10_mistral_B.gsm8k_basic.linear.attempts_per_goal}{1.00}
\defval{gb10_mistral_B.gsm8k_basic.linear.calls_per_goal}{1.00}
\defval{gb10_mistral_B.gsm8k_basic.linear.n_completed}{10}
\defval{gb10_mistral_B.gsm8k_basic.linear.n_correct}{9}
\defval{gb10_mistral_B.gsm8k_basic.linear.p_c}{1.00}
\defval{gb10_mistral_B.gsm8k_basic.linear.p_s}{0.90}
\defval{gb10_mistral_B.gsm8k_basic.linear.power_W}{64.4}
\defval{gb10_mistral_B.gsm8k_basic.min_successes}{8}
\defval{gb10_mistral_B.gsm8k_basic.power_ratio}{1.04}
\defval{gb10_mistral_B.gsm8k_basic.power_ratio_hi}{1.13}
\defval{gb10_mistral_B.gsm8k_basic.power_ratio_lo}{0.95}
\defval{gb10_mistral_B.gsm8k_basic.share_energy}{0.96}
\defval{gb10_mistral_B.gsm8k_basic.success_ratio}{1.12}
\defval{gb10_mistral_B.gsm8k_basic.time_ratio}{11.95}
\defval{gb10_mistral_B.gsm8k_basic.time_ratio_hi}{16.28}
\defval{gb10_mistral_B.gsm8k_basic.time_ratio_lo}{9.52}
\defval{gb10_mistral_B.gsm8k_multi_step.OOI_EpCG}{7.17}
\defval{gb10_mistral_B.gsm8k_multi_step.OOI_EpCG_hi}{8.63}
\defval{gb10_mistral_B.gsm8k_multi_step.OOI_EpCG_lo}{5.85}
\defval{gb10_mistral_B.gsm8k_multi_step.OOI_EpG}{7.17}
\defval{gb10_mistral_B.gsm8k_multi_step.OOI_EpG_hi}{8.63}
\defval{gb10_mistral_B.gsm8k_multi_step.OOI_EpG_lo}{5.85}
\defval{gb10_mistral_B.gsm8k_multi_step.OOI_EpSG}{7.17}
\defval{gb10_mistral_B.gsm8k_multi_step.OOI_EpSG_hi}{8.63}
\defval{gb10_mistral_B.gsm8k_multi_step.OOI_EpSG_lo}{5.85}
\defval{gb10_mistral_B.gsm8k_multi_step.agentic.E_wrong_goals_J}{0}
\defval{gb10_mistral_B.gsm8k_multi_step.agentic.EpCG_J}{5,542.5}
\defval{gb10_mistral_B.gsm8k_multi_step.agentic.EpCall_J}{583.4}
\defval{gb10_mistral_B.gsm8k_multi_step.agentic.EpG_J}{5,542.5}
\defval{gb10_mistral_B.gsm8k_multi_step.agentic.EpSG_J}{5,542.5}
\defval{gb10_mistral_B.gsm8k_multi_step.agentic.N}{10}
\defval{gb10_mistral_B.gsm8k_multi_step.agentic.TpG_s}{76.3}
\defval{gb10_mistral_B.gsm8k_multi_step.agentic.attempts_per_goal}{1.90}
\defval{gb10_mistral_B.gsm8k_multi_step.agentic.calls_per_goal}{9.50}
\defval{gb10_mistral_B.gsm8k_multi_step.agentic.n_completed}{10}
\defval{gb10_mistral_B.gsm8k_multi_step.agentic.n_correct}{10}
\defval{gb10_mistral_B.gsm8k_multi_step.agentic.p_c}{1.00}
\defval{gb10_mistral_B.gsm8k_multi_step.agentic.p_s}{1.00}
\defval{gb10_mistral_B.gsm8k_multi_step.agentic.power_W}{72.6}
\defval{gb10_mistral_B.gsm8k_multi_step.energy_ratio}{7.17}
\defval{gb10_mistral_B.gsm8k_multi_step.linear.E_wrong_goals_J}{0}
\defval{gb10_mistral_B.gsm8k_multi_step.linear.EpCG_J}{773.1}
\defval{gb10_mistral_B.gsm8k_multi_step.linear.EpCall_J}{773.1}
\defval{gb10_mistral_B.gsm8k_multi_step.linear.EpG_J}{773.1}
\defval{gb10_mistral_B.gsm8k_multi_step.linear.EpSG_J}{773.1}
\defval{gb10_mistral_B.gsm8k_multi_step.linear.N}{10}
\defval{gb10_mistral_B.gsm8k_multi_step.linear.TpG_s}{11.1}
\defval{gb10_mistral_B.gsm8k_multi_step.linear.attempts_per_goal}{1.00}
\defval{gb10_mistral_B.gsm8k_multi_step.linear.calls_per_goal}{1.00}
\defval{gb10_mistral_B.gsm8k_multi_step.linear.n_completed}{10}
\defval{gb10_mistral_B.gsm8k_multi_step.linear.n_correct}{10}
\defval{gb10_mistral_B.gsm8k_multi_step.linear.p_c}{1.00}
\defval{gb10_mistral_B.gsm8k_multi_step.linear.p_s}{1.00}
\defval{gb10_mistral_B.gsm8k_multi_step.linear.power_W}{69.9}
\defval{gb10_mistral_B.gsm8k_multi_step.min_successes}{10}
\defval{gb10_mistral_B.gsm8k_multi_step.power_ratio}{1.04}
\defval{gb10_mistral_B.gsm8k_multi_step.power_ratio_hi}{1.10}
\defval{gb10_mistral_B.gsm8k_multi_step.power_ratio_lo}{0.98}
\defval{gb10_mistral_B.gsm8k_multi_step.share_energy}{1.00}
\defval{gb10_mistral_B.gsm8k_multi_step.success_ratio}{1.00}
\defval{gb10_mistral_B.gsm8k_multi_step.time_ratio}{6.90}
\defval{gb10_mistral_B.gsm8k_multi_step.time_ratio_hi}{8.01}
\defval{gb10_mistral_B.gsm8k_multi_step.time_ratio_lo}{5.74}
\defval{gb10_mistral_B.logical_reasoning.OOI_EpCG}{20.70}
\defval{gb10_mistral_B.logical_reasoning.OOI_EpCG_hi}{23.74}
\defval{gb10_mistral_B.logical_reasoning.OOI_EpCG_lo}{17.68}
\defval{gb10_mistral_B.logical_reasoning.OOI_EpG}{20.70}
\defval{gb10_mistral_B.logical_reasoning.OOI_EpG_hi}{23.74}
\defval{gb10_mistral_B.logical_reasoning.OOI_EpG_lo}{17.68}
\defval{gb10_mistral_B.logical_reasoning.OOI_EpSG}{20.70}
\defval{gb10_mistral_B.logical_reasoning.OOI_EpSG_hi}{23.74}
\defval{gb10_mistral_B.logical_reasoning.OOI_EpSG_lo}{17.68}
\defval{gb10_mistral_B.logical_reasoning.agentic.E_wrong_goals_J}{0}
\defval{gb10_mistral_B.logical_reasoning.agentic.EpCG_J}{7,687.2}
\defval{gb10_mistral_B.logical_reasoning.agentic.EpCall_J}{578.0}
\defval{gb10_mistral_B.logical_reasoning.agentic.EpG_J}{7,687.2}
\defval{gb10_mistral_B.logical_reasoning.agentic.EpSG_J}{7,687.2}
\defval{gb10_mistral_B.logical_reasoning.agentic.N}{10}
\defval{gb10_mistral_B.logical_reasoning.agentic.TpG_s}{103.2}
\defval{gb10_mistral_B.logical_reasoning.agentic.attempts_per_goal}{1.90}
\defval{gb10_mistral_B.logical_reasoning.agentic.calls_per_goal}{13.30}
\defval{gb10_mistral_B.logical_reasoning.agentic.n_completed}{10}
\defval{gb10_mistral_B.logical_reasoning.agentic.n_correct}{10}
\defval{gb10_mistral_B.logical_reasoning.agentic.p_c}{1.00}
\defval{gb10_mistral_B.logical_reasoning.agentic.p_s}{1.00}
\defval{gb10_mistral_B.logical_reasoning.agentic.power_W}{74.5}
\defval{gb10_mistral_B.logical_reasoning.energy_ratio}{20.70}
\defval{gb10_mistral_B.logical_reasoning.linear.E_wrong_goals_J}{0}
\defval{gb10_mistral_B.logical_reasoning.linear.EpCG_J}{371.3}
\defval{gb10_mistral_B.logical_reasoning.linear.EpCall_J}{371.3}
\defval{gb10_mistral_B.logical_reasoning.linear.EpG_J}{371.3}
\defval{gb10_mistral_B.logical_reasoning.linear.EpSG_J}{371.3}
\defval{gb10_mistral_B.logical_reasoning.linear.N}{10}
\defval{gb10_mistral_B.logical_reasoning.linear.TpG_s}{5.2}
\defval{gb10_mistral_B.logical_reasoning.linear.attempts_per_goal}{1.00}
\defval{gb10_mistral_B.logical_reasoning.linear.calls_per_goal}{1.00}
\defval{gb10_mistral_B.logical_reasoning.linear.n_completed}{10}
\defval{gb10_mistral_B.logical_reasoning.linear.n_correct}{10}
\defval{gb10_mistral_B.logical_reasoning.linear.p_c}{1.00}
\defval{gb10_mistral_B.logical_reasoning.linear.p_s}{1.00}
\defval{gb10_mistral_B.logical_reasoning.linear.power_W}{70.9}
\defval{gb10_mistral_B.logical_reasoning.min_successes}{10}
\defval{gb10_mistral_B.logical_reasoning.power_ratio}{1.05}
\defval{gb10_mistral_B.logical_reasoning.power_ratio_hi}{1.06}
\defval{gb10_mistral_B.logical_reasoning.power_ratio_lo}{1.04}
\defval{gb10_mistral_B.logical_reasoning.share_energy}{1.00}
\defval{gb10_mistral_B.logical_reasoning.success_ratio}{1.00}
\defval{gb10_mistral_B.logical_reasoning.time_ratio}{19.71}
\defval{gb10_mistral_B.logical_reasoning.time_ratio_hi}{22.58}
\defval{gb10_mistral_B.logical_reasoning.time_ratio_lo}{16.83}
\defval{gb10_mistral_B.science_qa.OOI_EpCG}{10.33}
\defval{gb10_mistral_B.science_qa.OOI_EpCG_hi}{12.37}
\defval{gb10_mistral_B.science_qa.OOI_EpCG_lo}{8.47}
\defval{gb10_mistral_B.science_qa.OOI_EpG}{10.33}
\defval{gb10_mistral_B.science_qa.OOI_EpG_hi}{12.37}
\defval{gb10_mistral_B.science_qa.OOI_EpG_lo}{8.47}
\defval{gb10_mistral_B.science_qa.OOI_EpSG}{10.33}
\defval{gb10_mistral_B.science_qa.OOI_EpSG_hi}{12.37}
\defval{gb10_mistral_B.science_qa.OOI_EpSG_lo}{8.47}
\defval{gb10_mistral_B.science_qa.agentic.E_wrong_goals_J}{0}
\defval{gb10_mistral_B.science_qa.agentic.EpCG_J}{1,854.0}
\defval{gb10_mistral_B.science_qa.agentic.EpCall_J}{325.3}
\defval{gb10_mistral_B.science_qa.agentic.EpG_J}{1,854.0}
\defval{gb10_mistral_B.science_qa.agentic.EpSG_J}{1,854.0}
\defval{gb10_mistral_B.science_qa.agentic.N}{10}
\defval{gb10_mistral_B.science_qa.agentic.TpG_s}{31.5}
\defval{gb10_mistral_B.science_qa.agentic.attempts_per_goal}{1.90}
\defval{gb10_mistral_B.science_qa.agentic.calls_per_goal}{5.70}
\defval{gb10_mistral_B.science_qa.agentic.n_completed}{10}
\defval{gb10_mistral_B.science_qa.agentic.n_correct}{10}
\defval{gb10_mistral_B.science_qa.agentic.p_c}{1.00}
\defval{gb10_mistral_B.science_qa.agentic.p_s}{1.00}
\defval{gb10_mistral_B.science_qa.agentic.power_W}{58.9}
\defval{gb10_mistral_B.science_qa.energy_ratio}{10.33}
\defval{gb10_mistral_B.science_qa.linear.E_wrong_goals_J}{0}
\defval{gb10_mistral_B.science_qa.linear.EpCG_J}{179.5}
\defval{gb10_mistral_B.science_qa.linear.EpCall_J}{179.5}
\defval{gb10_mistral_B.science_qa.linear.EpG_J}{179.5}
\defval{gb10_mistral_B.science_qa.linear.EpSG_J}{179.5}
\defval{gb10_mistral_B.science_qa.linear.N}{10}
\defval{gb10_mistral_B.science_qa.linear.TpG_s}{3.0}
\defval{gb10_mistral_B.science_qa.linear.attempts_per_goal}{1.00}
\defval{gb10_mistral_B.science_qa.linear.calls_per_goal}{1.00}
\defval{gb10_mistral_B.science_qa.linear.n_completed}{10}
\defval{gb10_mistral_B.science_qa.linear.n_correct}{10}
\defval{gb10_mistral_B.science_qa.linear.p_c}{1.00}
\defval{gb10_mistral_B.science_qa.linear.p_s}{1.00}
\defval{gb10_mistral_B.science_qa.linear.power_W}{59.0}
\defval{gb10_mistral_B.science_qa.min_successes}{10}
\defval{gb10_mistral_B.science_qa.power_ratio}{1.00}
\defval{gb10_mistral_B.science_qa.power_ratio_hi}{1.01}
\defval{gb10_mistral_B.science_qa.power_ratio_lo}{0.99}
\defval{gb10_mistral_B.science_qa.share_energy}{1.00}
\defval{gb10_mistral_B.science_qa.success_ratio}{1.00}
\defval{gb10_mistral_B.science_qa.time_ratio}{10.34}
\defval{gb10_mistral_B.science_qa.time_ratio_hi}{12.39}
\defval{gb10_mistral_B.science_qa.time_ratio_lo}{8.48}
\defval{gb10_mistral_B.tg_sequential_2.OOI_EpCG}{1.24}
\defval{gb10_mistral_B.tg_sequential_2.OOI_EpCG_hi}{1.52}
\defval{gb10_mistral_B.tg_sequential_2.OOI_EpCG_lo}{0.99}
\defval{gb10_mistral_B.tg_sequential_2.OOI_EpG}{1.24}
\defval{gb10_mistral_B.tg_sequential_2.OOI_EpG_hi}{1.52}
\defval{gb10_mistral_B.tg_sequential_2.OOI_EpG_lo}{0.99}
\defval{gb10_mistral_B.tg_sequential_2.OOI_EpSG}{undefined}
\defval{gb10_mistral_B.tg_sequential_2.OOI_EpSG_hi}{undefined}
\defval{gb10_mistral_B.tg_sequential_2.OOI_EpSG_lo}{undefined}
\defval{gb10_mistral_B.tg_sequential_2.agentic.E_wrong_goals_J}{0}
\defval{gb10_mistral_B.tg_sequential_2.agentic.EpCG_J}{2,499.4}
\defval{gb10_mistral_B.tg_sequential_2.agentic.EpCall_J}{1,315.5}
\defval{gb10_mistral_B.tg_sequential_2.agentic.EpG_J}{2,499.4}
\defval{gb10_mistral_B.tg_sequential_2.agentic.EpSG_J}{2,499.4}
\defval{gb10_mistral_B.tg_sequential_2.agentic.N}{10}
\defval{gb10_mistral_B.tg_sequential_2.agentic.TpG_s}{32.8}
\defval{gb10_mistral_B.tg_sequential_2.agentic.attempts_per_goal}{1.90}
\defval{gb10_mistral_B.tg_sequential_2.agentic.calls_per_goal}{1.90}
\defval{gb10_mistral_B.tg_sequential_2.agentic.n_completed}{10}
\defval{gb10_mistral_B.tg_sequential_2.agentic.n_correct}{10}
\defval{gb10_mistral_B.tg_sequential_2.agentic.p_c}{1.00}
\defval{gb10_mistral_B.tg_sequential_2.agentic.p_s}{1.00}
\defval{gb10_mistral_B.tg_sequential_2.agentic.power_W}{76.1}
\defval{gb10_mistral_B.tg_sequential_2.energy_ratio}{1.24}
\defval{gb10_mistral_B.tg_sequential_2.linear.E_wrong_goals_J}{20,132}
\defval{gb10_mistral_B.tg_sequential_2.linear.EpCG_J}{2,013.2}
\defval{gb10_mistral_B.tg_sequential_2.linear.EpCall_J}{2,013.2}
\defval{gb10_mistral_B.tg_sequential_2.linear.EpG_J}{2,013.2}
\defval{gb10_mistral_B.tg_sequential_2.linear.EpSG_J}{undefined}
\defval{gb10_mistral_B.tg_sequential_2.linear.N}{10}
\defval{gb10_mistral_B.tg_sequential_2.linear.TpG_s}{27.5}
\defval{gb10_mistral_B.tg_sequential_2.linear.attempts_per_goal}{1.00}
\defval{gb10_mistral_B.tg_sequential_2.linear.calls_per_goal}{1.00}
\defval{gb10_mistral_B.tg_sequential_2.linear.n_completed}{10}
\defval{gb10_mistral_B.tg_sequential_2.linear.n_correct}{0}
\defval{gb10_mistral_B.tg_sequential_2.linear.p_c}{1.00}
\defval{gb10_mistral_B.tg_sequential_2.linear.p_s}{0.00}
\defval{gb10_mistral_B.tg_sequential_2.linear.power_W}{73.1}
\defval{gb10_mistral_B.tg_sequential_2.min_successes}{0}
\defval{gb10_mistral_B.tg_sequential_2.power_ratio}{1.04}
\defval{gb10_mistral_B.tg_sequential_2.power_ratio_hi}{1.06}
\defval{gb10_mistral_B.tg_sequential_2.power_ratio_lo}{1.03}
\defval{gb10_mistral_B.tg_sequential_2.share_energy}{undefined}
\defval{gb10_mistral_B.tg_sequential_2.success_ratio}{0.00}
\defval{gb10_mistral_B.tg_sequential_2.time_ratio}{1.19}
\defval{gb10_mistral_B.tg_sequential_2.time_ratio_hi}{1.46}
\defval{gb10_mistral_B.tg_sequential_2.time_ratio_lo}{0.95}
\defval{gb10_mistral_B.tg_single_calc.OOI_EpCG}{1.01}
\defval{gb10_mistral_B.tg_single_calc.OOI_EpCG_hi}{1.31}
\defval{gb10_mistral_B.tg_single_calc.OOI_EpCG_lo}{0.76}
\defval{gb10_mistral_B.tg_single_calc.OOI_EpG}{1.01}
\defval{gb10_mistral_B.tg_single_calc.OOI_EpG_hi}{1.31}
\defval{gb10_mistral_B.tg_single_calc.OOI_EpG_lo}{0.76}
\defval{gb10_mistral_B.tg_single_calc.OOI_EpSG}{0.11}
\defval{gb10_mistral_B.tg_single_calc.OOI_EpSG_hi}{undefined}
\defval{gb10_mistral_B.tg_single_calc.OOI_EpSG_lo}{undefined}
\defval{gb10_mistral_B.tg_single_calc.agentic.E_wrong_goals_J}{1,375}
\defval{gb10_mistral_B.tg_single_calc.agentic.EpCG_J}{1,323.5}
\defval{gb10_mistral_B.tg_single_calc.agentic.EpCall_J}{696.6}
\defval{gb10_mistral_B.tg_single_calc.agentic.EpG_J}{1,323.5}
\defval{gb10_mistral_B.tg_single_calc.agentic.EpSG_J}{1,470.5}
\defval{gb10_mistral_B.tg_single_calc.agentic.N}{10}
\defval{gb10_mistral_B.tg_single_calc.agentic.TpG_s}{17.6}
\defval{gb10_mistral_B.tg_single_calc.agentic.attempts_per_goal}{1.90}
\defval{gb10_mistral_B.tg_single_calc.agentic.calls_per_goal}{1.90}
\defval{gb10_mistral_B.tg_single_calc.agentic.n_completed}{10}
\defval{gb10_mistral_B.tg_single_calc.agentic.n_correct}{9}
\defval{gb10_mistral_B.tg_single_calc.agentic.p_c}{1.00}
\defval{gb10_mistral_B.tg_single_calc.agentic.p_s}{0.90}
\defval{gb10_mistral_B.tg_single_calc.agentic.power_W}{75.2}
\defval{gb10_mistral_B.tg_single_calc.energy_ratio}{1.01}
\defval{gb10_mistral_B.tg_single_calc.linear.E_wrong_goals_J}{11,703}
\defval{gb10_mistral_B.tg_single_calc.linear.EpCG_J}{1,307.2}
\defval{gb10_mistral_B.tg_single_calc.linear.EpCall_J}{1,307.2}
\defval{gb10_mistral_B.tg_single_calc.linear.EpG_J}{1,307.2}
\defval{gb10_mistral_B.tg_single_calc.linear.EpSG_J}{13,071.8}
\defval{gb10_mistral_B.tg_single_calc.linear.N}{10}
\defval{gb10_mistral_B.tg_single_calc.linear.TpG_s}{18.0}
\defval{gb10_mistral_B.tg_single_calc.linear.attempts_per_goal}{1.00}
\defval{gb10_mistral_B.tg_single_calc.linear.calls_per_goal}{1.00}
\defval{gb10_mistral_B.tg_single_calc.linear.n_completed}{10}
\defval{gb10_mistral_B.tg_single_calc.linear.n_correct}{1}
\defval{gb10_mistral_B.tg_single_calc.linear.p_c}{1.00}
\defval{gb10_mistral_B.tg_single_calc.linear.p_s}{0.10}
\defval{gb10_mistral_B.tg_single_calc.linear.power_W}{72.7}
\defval{gb10_mistral_B.tg_single_calc.min_successes}{1}
\defval{gb10_mistral_B.tg_single_calc.power_ratio}{1.03}
\defval{gb10_mistral_B.tg_single_calc.power_ratio_hi}{1.05}
\defval{gb10_mistral_B.tg_single_calc.power_ratio_lo}{1.02}
\defval{gb10_mistral_B.tg_single_calc.share_energy}{-0.01}
\defval{gb10_mistral_B.tg_single_calc.success_ratio}{0.11}
\defval{gb10_mistral_B.tg_single_calc.time_ratio}{0.98}
\defval{gb10_mistral_B.tg_single_calc.time_ratio_hi}{1.26}
\defval{gb10_mistral_B.tg_single_calc.time_ratio_lo}{0.74}
\defval{gb10_mistral_B.tg_single_db.OOI_EpCG}{1.33}
\defval{gb10_mistral_B.tg_single_db.OOI_EpCG_hi}{1.75}
\defval{gb10_mistral_B.tg_single_db.OOI_EpCG_lo}{0.96}
\defval{gb10_mistral_B.tg_single_db.OOI_EpG}{1.33}
\defval{gb10_mistral_B.tg_single_db.OOI_EpG_hi}{1.75}
\defval{gb10_mistral_B.tg_single_db.OOI_EpG_lo}{0.96}
\defval{gb10_mistral_B.tg_single_db.OOI_EpSG}{undefined}
\defval{gb10_mistral_B.tg_single_db.OOI_EpSG_hi}{undefined}
\defval{gb10_mistral_B.tg_single_db.OOI_EpSG_lo}{undefined}
\defval{gb10_mistral_B.tg_single_db.agentic.E_wrong_goals_J}{0}
\defval{gb10_mistral_B.tg_single_db.agentic.EpCG_J}{1,535.3}
\defval{gb10_mistral_B.tg_single_db.agentic.EpCall_J}{808.1}
\defval{gb10_mistral_B.tg_single_db.agentic.EpG_J}{1,535.3}
\defval{gb10_mistral_B.tg_single_db.agentic.EpSG_J}{1,535.3}
\defval{gb10_mistral_B.tg_single_db.agentic.N}{10}
\defval{gb10_mistral_B.tg_single_db.agentic.TpG_s}{21.0}
\defval{gb10_mistral_B.tg_single_db.agentic.attempts_per_goal}{1.90}
\defval{gb10_mistral_B.tg_single_db.agentic.calls_per_goal}{1.90}
\defval{gb10_mistral_B.tg_single_db.agentic.n_completed}{10}
\defval{gb10_mistral_B.tg_single_db.agentic.n_correct}{10}
\defval{gb10_mistral_B.tg_single_db.agentic.p_c}{1.00}
\defval{gb10_mistral_B.tg_single_db.agentic.p_s}{1.00}
\defval{gb10_mistral_B.tg_single_db.agentic.power_W}{73.1}
\defval{gb10_mistral_B.tg_single_db.energy_ratio}{1.33}
\defval{gb10_mistral_B.tg_single_db.linear.E_wrong_goals_J}{11,550}
\defval{gb10_mistral_B.tg_single_db.linear.EpCG_J}{1,155.0}
\defval{gb10_mistral_B.tg_single_db.linear.EpCall_J}{1,155.0}
\defval{gb10_mistral_B.tg_single_db.linear.EpG_J}{1,155.0}
\defval{gb10_mistral_B.tg_single_db.linear.EpSG_J}{undefined}
\defval{gb10_mistral_B.tg_single_db.linear.N}{10}
\defval{gb10_mistral_B.tg_single_db.linear.TpG_s}{16.6}
\defval{gb10_mistral_B.tg_single_db.linear.attempts_per_goal}{1.00}
\defval{gb10_mistral_B.tg_single_db.linear.calls_per_goal}{1.00}
\defval{gb10_mistral_B.tg_single_db.linear.n_completed}{10}
\defval{gb10_mistral_B.tg_single_db.linear.n_correct}{0}
\defval{gb10_mistral_B.tg_single_db.linear.p_c}{1.00}
\defval{gb10_mistral_B.tg_single_db.linear.p_s}{0.00}
\defval{gb10_mistral_B.tg_single_db.linear.power_W}{69.8}
\defval{gb10_mistral_B.tg_single_db.min_successes}{0}
\defval{gb10_mistral_B.tg_single_db.power_ratio}{1.05}
\defval{gb10_mistral_B.tg_single_db.power_ratio_hi}{1.12}
\defval{gb10_mistral_B.tg_single_db.power_ratio_lo}{0.99}
\defval{gb10_mistral_B.tg_single_db.share_energy}{undefined}
\defval{gb10_mistral_B.tg_single_db.success_ratio}{0.00}
\defval{gb10_mistral_B.tg_single_db.time_ratio}{1.27}
\defval{gb10_mistral_B.tg_single_db.time_ratio_hi}{1.67}
\defval{gb10_mistral_B.tg_single_db.time_ratio_lo}{0.91}
\defval{gb10_mistral_C.commonsense_reasoning.OOI_EpCG}{43.70}
\defval{gb10_mistral_C.commonsense_reasoning.OOI_EpCG_hi}{57.02}
\defval{gb10_mistral_C.commonsense_reasoning.OOI_EpCG_lo}{33.13}
\defval{gb10_mistral_C.commonsense_reasoning.OOI_EpG}{43.70}
\defval{gb10_mistral_C.commonsense_reasoning.OOI_EpG_hi}{57.02}
\defval{gb10_mistral_C.commonsense_reasoning.OOI_EpG_lo}{33.13}
\defval{gb10_mistral_C.commonsense_reasoning.OOI_EpSG}{62.42}
\defval{gb10_mistral_C.commonsense_reasoning.OOI_EpSG_hi}{106.56}
\defval{gb10_mistral_C.commonsense_reasoning.OOI_EpSG_lo}{44.43}
\defval{gb10_mistral_C.commonsense_reasoning.agentic.E_wrong_goals_J}{60,315}
\defval{gb10_mistral_C.commonsense_reasoning.agentic.EpCG_J}{25,587.0}
\defval{gb10_mistral_C.commonsense_reasoning.agentic.EpCall_J}{1,305.5}
\defval{gb10_mistral_C.commonsense_reasoning.agentic.EpG_J}{25,587.0}
\defval{gb10_mistral_C.commonsense_reasoning.agentic.EpSG_J}{36,552.8}
\defval{gb10_mistral_C.commonsense_reasoning.agentic.N}{10}
\defval{gb10_mistral_C.commonsense_reasoning.agentic.TpG_s}{343.7}
\defval{gb10_mistral_C.commonsense_reasoning.agentic.attempts_per_goal}{2.80}
\defval{gb10_mistral_C.commonsense_reasoning.agentic.calls_per_goal}{19.60}
\defval{gb10_mistral_C.commonsense_reasoning.agentic.failure_J}{8,850}
\defval{gb10_mistral_C.commonsense_reasoning.agentic.n_completed}{10}
\defval{gb10_mistral_C.commonsense_reasoning.agentic.n_correct}{7}
\defval{gb10_mistral_C.commonsense_reasoning.agentic.p_c}{1.00}
\defval{gb10_mistral_C.commonsense_reasoning.agentic.p_s}{0.70}
\defval{gb10_mistral_C.commonsense_reasoning.agentic.power_W}{74.4}
\defval{gb10_mistral_C.commonsense_reasoning.energy_ratio}{43.70}
\defval{gb10_mistral_C.commonsense_reasoning.linear.E_wrong_goals_J}{0}
\defval{gb10_mistral_C.commonsense_reasoning.linear.EpCG_J}{585.6}
\defval{gb10_mistral_C.commonsense_reasoning.linear.EpCall_J}{585.6}
\defval{gb10_mistral_C.commonsense_reasoning.linear.EpG_J}{585.6}
\defval{gb10_mistral_C.commonsense_reasoning.linear.EpSG_J}{585.6}
\defval{gb10_mistral_C.commonsense_reasoning.linear.N}{10}
\defval{gb10_mistral_C.commonsense_reasoning.linear.TpG_s}{8.4}
\defval{gb10_mistral_C.commonsense_reasoning.linear.attempts_per_goal}{1.00}
\defval{gb10_mistral_C.commonsense_reasoning.linear.calls_per_goal}{1.00}
\defval{gb10_mistral_C.commonsense_reasoning.linear.n_completed}{10}
\defval{gb10_mistral_C.commonsense_reasoning.linear.n_correct}{10}
\defval{gb10_mistral_C.commonsense_reasoning.linear.p_c}{1.00}
\defval{gb10_mistral_C.commonsense_reasoning.linear.p_s}{1.00}
\defval{gb10_mistral_C.commonsense_reasoning.linear.power_W}{69.6}
\defval{gb10_mistral_C.commonsense_reasoning.min_successes}{7}
\defval{gb10_mistral_C.commonsense_reasoning.power_ratio}{1.07}
\defval{gb10_mistral_C.commonsense_reasoning.power_ratio_hi}{1.08}
\defval{gb10_mistral_C.commonsense_reasoning.power_ratio_lo}{1.06}
\defval{gb10_mistral_C.commonsense_reasoning.share_energy}{0.91}
\defval{gb10_mistral_C.commonsense_reasoning.success_ratio}{1.43}
\defval{gb10_mistral_C.commonsense_reasoning.time_ratio}{40.87}
\defval{gb10_mistral_C.commonsense_reasoning.time_ratio_hi}{53.35}
\defval{gb10_mistral_C.commonsense_reasoning.time_ratio_lo}{31.02}
\defval{gb10_mistral_C.excluded}{0}
\defval{gb10_mistral_C.factual_qa.OOI_EpCG}{21.67}
\defval{gb10_mistral_C.factual_qa.OOI_EpCG_hi}{26.48}
\defval{gb10_mistral_C.factual_qa.OOI_EpCG_lo}{16.71}
\defval{gb10_mistral_C.factual_qa.OOI_EpG}{21.67}
\defval{gb10_mistral_C.factual_qa.OOI_EpG_hi}{26.48}
\defval{gb10_mistral_C.factual_qa.OOI_EpG_lo}{16.71}
\defval{gb10_mistral_C.factual_qa.OOI_EpSG}{21.67}
\defval{gb10_mistral_C.factual_qa.OOI_EpSG_hi}{26.48}
\defval{gb10_mistral_C.factual_qa.OOI_EpSG_lo}{16.71}
\defval{gb10_mistral_C.factual_qa.agentic.E_wrong_goals_J}{0}
\defval{gb10_mistral_C.factual_qa.agentic.EpCG_J}{8,520.1}
\defval{gb10_mistral_C.factual_qa.agentic.EpCall_J}{507.2}
\defval{gb10_mistral_C.factual_qa.agentic.EpG_J}{8,520.1}
\defval{gb10_mistral_C.factual_qa.agentic.EpSG_J}{8,520.1}
\defval{gb10_mistral_C.factual_qa.agentic.N}{10}
\defval{gb10_mistral_C.factual_qa.agentic.TpG_s}{114.8}
\defval{gb10_mistral_C.factual_qa.agentic.attempts_per_goal}{2.80}
\defval{gb10_mistral_C.factual_qa.agentic.calls_per_goal}{16.80}
\defval{gb10_mistral_C.factual_qa.agentic.failure_J}{3,116}
\defval{gb10_mistral_C.factual_qa.agentic.n_completed}{10}
\defval{gb10_mistral_C.factual_qa.agentic.n_correct}{10}
\defval{gb10_mistral_C.factual_qa.agentic.p_c}{1.00}
\defval{gb10_mistral_C.factual_qa.agentic.p_s}{1.00}
\defval{gb10_mistral_C.factual_qa.agentic.power_W}{74.2}
\defval{gb10_mistral_C.factual_qa.energy_ratio}{21.67}
\defval{gb10_mistral_C.factual_qa.linear.E_wrong_goals_J}{0}
\defval{gb10_mistral_C.factual_qa.linear.EpCG_J}{393.1}
\defval{gb10_mistral_C.factual_qa.linear.EpCall_J}{393.1}
\defval{gb10_mistral_C.factual_qa.linear.EpG_J}{393.1}
\defval{gb10_mistral_C.factual_qa.linear.EpSG_J}{393.1}
\defval{gb10_mistral_C.factual_qa.linear.N}{10}
\defval{gb10_mistral_C.factual_qa.linear.TpG_s}{5.6}
\defval{gb10_mistral_C.factual_qa.linear.attempts_per_goal}{1.00}
\defval{gb10_mistral_C.factual_qa.linear.calls_per_goal}{1.00}
\defval{gb10_mistral_C.factual_qa.linear.n_completed}{10}
\defval{gb10_mistral_C.factual_qa.linear.n_correct}{10}
\defval{gb10_mistral_C.factual_qa.linear.p_c}{1.00}
\defval{gb10_mistral_C.factual_qa.linear.p_s}{1.00}
\defval{gb10_mistral_C.factual_qa.linear.power_W}{70.0}
\defval{gb10_mistral_C.factual_qa.min_successes}{10}
\defval{gb10_mistral_C.factual_qa.power_ratio}{1.06}
\defval{gb10_mistral_C.factual_qa.power_ratio_hi}{1.07}
\defval{gb10_mistral_C.factual_qa.power_ratio_lo}{1.05}
\defval{gb10_mistral_C.factual_qa.share_energy}{1.00}
\defval{gb10_mistral_C.factual_qa.success_ratio}{1.00}
\defval{gb10_mistral_C.factual_qa.time_ratio}{20.45}
\defval{gb10_mistral_C.factual_qa.time_ratio_hi}{24.99}
\defval{gb10_mistral_C.factual_qa.time_ratio_lo}{15.71}
\defval{gb10_mistral_C.goals}{180}
\defval{gb10_mistral_C.gsm8k_basic.OOI_EpCG}{21.08}
\defval{gb10_mistral_C.gsm8k_basic.OOI_EpCG_hi}{25.39}
\defval{gb10_mistral_C.gsm8k_basic.OOI_EpCG_lo}{16.60}
\defval{gb10_mistral_C.gsm8k_basic.OOI_EpG}{21.08}
\defval{gb10_mistral_C.gsm8k_basic.OOI_EpG_hi}{25.39}
\defval{gb10_mistral_C.gsm8k_basic.OOI_EpG_lo}{16.60}
\defval{gb10_mistral_C.gsm8k_basic.OOI_EpSG}{35.13}
\defval{gb10_mistral_C.gsm8k_basic.OOI_EpSG_hi}{69.33}
\defval{gb10_mistral_C.gsm8k_basic.OOI_EpSG_lo}{23.98}
\defval{gb10_mistral_C.gsm8k_basic.agentic.E_wrong_goals_J}{18,810}
\defval{gb10_mistral_C.gsm8k_basic.agentic.EpCG_J}{5,467.7}
\defval{gb10_mistral_C.gsm8k_basic.agentic.EpCall_J}{390.6}
\defval{gb10_mistral_C.gsm8k_basic.agentic.EpG_J}{5,467.7}
\defval{gb10_mistral_C.gsm8k_basic.agentic.EpSG_J}{9,112.8}
\defval{gb10_mistral_C.gsm8k_basic.agentic.N}{10}
\defval{gb10_mistral_C.gsm8k_basic.agentic.TpG_s}{73.9}
\defval{gb10_mistral_C.gsm8k_basic.agentic.attempts_per_goal}{2.80}
\defval{gb10_mistral_C.gsm8k_basic.agentic.calls_per_goal}{14.00}
\defval{gb10_mistral_C.gsm8k_basic.agentic.failure_J}{1,887}
\defval{gb10_mistral_C.gsm8k_basic.agentic.n_completed}{10}
\defval{gb10_mistral_C.gsm8k_basic.agentic.n_correct}{6}
\defval{gb10_mistral_C.gsm8k_basic.agentic.p_c}{1.00}
\defval{gb10_mistral_C.gsm8k_basic.agentic.p_s}{0.60}
\defval{gb10_mistral_C.gsm8k_basic.agentic.power_W}{74.0}
\defval{gb10_mistral_C.gsm8k_basic.energy_ratio}{21.08}
\defval{gb10_mistral_C.gsm8k_basic.linear.E_wrong_goals_J}{0}
\defval{gb10_mistral_C.gsm8k_basic.linear.EpCG_J}{259.4}
\defval{gb10_mistral_C.gsm8k_basic.linear.EpCall_J}{259.4}
\defval{gb10_mistral_C.gsm8k_basic.linear.EpG_J}{259.4}
\defval{gb10_mistral_C.gsm8k_basic.linear.EpSG_J}{259.4}
\defval{gb10_mistral_C.gsm8k_basic.linear.N}{10}
\defval{gb10_mistral_C.gsm8k_basic.linear.TpG_s}{3.6}
\defval{gb10_mistral_C.gsm8k_basic.linear.attempts_per_goal}{1.00}
\defval{gb10_mistral_C.gsm8k_basic.linear.calls_per_goal}{1.00}
\defval{gb10_mistral_C.gsm8k_basic.linear.n_completed}{10}
\defval{gb10_mistral_C.gsm8k_basic.linear.n_correct}{10}
\defval{gb10_mistral_C.gsm8k_basic.linear.p_c}{1.00}
\defval{gb10_mistral_C.gsm8k_basic.linear.p_s}{1.00}
\defval{gb10_mistral_C.gsm8k_basic.linear.power_W}{71.7}
\defval{gb10_mistral_C.gsm8k_basic.min_successes}{6}
\defval{gb10_mistral_C.gsm8k_basic.power_ratio}{1.03}
\defval{gb10_mistral_C.gsm8k_basic.power_ratio_hi}{1.05}
\defval{gb10_mistral_C.gsm8k_basic.power_ratio_lo}{1.01}
\defval{gb10_mistral_C.gsm8k_basic.share_energy}{0.86}
\defval{gb10_mistral_C.gsm8k_basic.success_ratio}{1.67}
\defval{gb10_mistral_C.gsm8k_basic.time_ratio}{20.43}
\defval{gb10_mistral_C.gsm8k_basic.time_ratio_hi}{24.43}
\defval{gb10_mistral_C.gsm8k_basic.time_ratio_lo}{16.17}
\defval{gb10_mistral_C.gsm8k_multi_step.OOI_EpCG}{10.39}
\defval{gb10_mistral_C.gsm8k_multi_step.OOI_EpCG_hi}{11.77}
\defval{gb10_mistral_C.gsm8k_multi_step.OOI_EpCG_lo}{8.58}
\defval{gb10_mistral_C.gsm8k_multi_step.OOI_EpG}{10.39}
\defval{gb10_mistral_C.gsm8k_multi_step.OOI_EpG_hi}{11.77}
\defval{gb10_mistral_C.gsm8k_multi_step.OOI_EpG_lo}{8.58}
\defval{gb10_mistral_C.gsm8k_multi_step.OOI_EpSG}{10.39}
\defval{gb10_mistral_C.gsm8k_multi_step.OOI_EpSG_hi}{11.77}
\defval{gb10_mistral_C.gsm8k_multi_step.OOI_EpSG_lo}{8.58}
\defval{gb10_mistral_C.gsm8k_multi_step.agentic.E_wrong_goals_J}{0}
\defval{gb10_mistral_C.gsm8k_multi_step.agentic.EpCG_J}{8,018.4}
\defval{gb10_mistral_C.gsm8k_multi_step.agentic.EpCall_J}{572.7}
\defval{gb10_mistral_C.gsm8k_multi_step.agentic.EpG_J}{8,018.4}
\defval{gb10_mistral_C.gsm8k_multi_step.agentic.EpSG_J}{8,018.4}
\defval{gb10_mistral_C.gsm8k_multi_step.agentic.N}{10}
\defval{gb10_mistral_C.gsm8k_multi_step.agentic.TpG_s}{107.8}
\defval{gb10_mistral_C.gsm8k_multi_step.agentic.attempts_per_goal}{2.80}
\defval{gb10_mistral_C.gsm8k_multi_step.agentic.calls_per_goal}{14.00}
\defval{gb10_mistral_C.gsm8k_multi_step.agentic.failure_J}{2,807}
\defval{gb10_mistral_C.gsm8k_multi_step.agentic.n_completed}{10}
\defval{gb10_mistral_C.gsm8k_multi_step.agentic.n_correct}{10}
\defval{gb10_mistral_C.gsm8k_multi_step.agentic.p_c}{1.00}
\defval{gb10_mistral_C.gsm8k_multi_step.agentic.p_s}{1.00}
\defval{gb10_mistral_C.gsm8k_multi_step.agentic.power_W}{74.4}
\defval{gb10_mistral_C.gsm8k_multi_step.energy_ratio}{10.39}
\defval{gb10_mistral_C.gsm8k_multi_step.linear.E_wrong_goals_J}{0}
\defval{gb10_mistral_C.gsm8k_multi_step.linear.EpCG_J}{771.8}
\defval{gb10_mistral_C.gsm8k_multi_step.linear.EpCall_J}{771.8}
\defval{gb10_mistral_C.gsm8k_multi_step.linear.EpG_J}{771.8}
\defval{gb10_mistral_C.gsm8k_multi_step.linear.EpSG_J}{771.8}
\defval{gb10_mistral_C.gsm8k_multi_step.linear.N}{10}
\defval{gb10_mistral_C.gsm8k_multi_step.linear.TpG_s}{11.0}
\defval{gb10_mistral_C.gsm8k_multi_step.linear.attempts_per_goal}{1.00}
\defval{gb10_mistral_C.gsm8k_multi_step.linear.calls_per_goal}{1.00}
\defval{gb10_mistral_C.gsm8k_multi_step.linear.n_completed}{10}
\defval{gb10_mistral_C.gsm8k_multi_step.linear.n_correct}{10}
\defval{gb10_mistral_C.gsm8k_multi_step.linear.p_c}{1.00}
\defval{gb10_mistral_C.gsm8k_multi_step.linear.p_s}{1.00}
\defval{gb10_mistral_C.gsm8k_multi_step.linear.power_W}{70.0}
\defval{gb10_mistral_C.gsm8k_multi_step.min_successes}{10}
\defval{gb10_mistral_C.gsm8k_multi_step.power_ratio}{1.06}
\defval{gb10_mistral_C.gsm8k_multi_step.power_ratio_hi}{1.08}
\defval{gb10_mistral_C.gsm8k_multi_step.power_ratio_lo}{1.05}
\defval{gb10_mistral_C.gsm8k_multi_step.share_energy}{1.00}
\defval{gb10_mistral_C.gsm8k_multi_step.success_ratio}{1.00}
\defval{gb10_mistral_C.gsm8k_multi_step.time_ratio}{9.77}
\defval{gb10_mistral_C.gsm8k_multi_step.time_ratio_hi}{11.15}
\defval{gb10_mistral_C.gsm8k_multi_step.time_ratio_lo}{8.06}
\defval{gb10_mistral_C.logical_reasoning.OOI_EpCG}{37.32}
\defval{gb10_mistral_C.logical_reasoning.OOI_EpCG_hi}{45.42}
\defval{gb10_mistral_C.logical_reasoning.OOI_EpCG_lo}{28.61}
\defval{gb10_mistral_C.logical_reasoning.OOI_EpG}{37.32}
\defval{gb10_mistral_C.logical_reasoning.OOI_EpG_hi}{45.42}
\defval{gb10_mistral_C.logical_reasoning.OOI_EpG_lo}{28.61}
\defval{gb10_mistral_C.logical_reasoning.OOI_EpSG}{37.32}
\defval{gb10_mistral_C.logical_reasoning.OOI_EpSG_hi}{45.42}
\defval{gb10_mistral_C.logical_reasoning.OOI_EpSG_lo}{28.61}
\defval{gb10_mistral_C.logical_reasoning.agentic.E_wrong_goals_J}{0}
\defval{gb10_mistral_C.logical_reasoning.agentic.EpCG_J}{11,469.2}
\defval{gb10_mistral_C.logical_reasoning.agentic.EpCall_J}{585.2}
\defval{gb10_mistral_C.logical_reasoning.agentic.EpG_J}{11,469.2}
\defval{gb10_mistral_C.logical_reasoning.agentic.EpSG_J}{11,469.2}
\defval{gb10_mistral_C.logical_reasoning.agentic.N}{10}
\defval{gb10_mistral_C.logical_reasoning.agentic.TpG_s}{151.1}
\defval{gb10_mistral_C.logical_reasoning.agentic.attempts_per_goal}{2.80}
\defval{gb10_mistral_C.logical_reasoning.agentic.calls_per_goal}{19.60}
\defval{gb10_mistral_C.logical_reasoning.agentic.failure_J}{4,208}
\defval{gb10_mistral_C.logical_reasoning.agentic.n_completed}{10}
\defval{gb10_mistral_C.logical_reasoning.agentic.n_correct}{10}
\defval{gb10_mistral_C.logical_reasoning.agentic.p_c}{1.00}
\defval{gb10_mistral_C.logical_reasoning.agentic.p_s}{1.00}
\defval{gb10_mistral_C.logical_reasoning.agentic.power_W}{75.9}
\defval{gb10_mistral_C.logical_reasoning.energy_ratio}{37.32}
\defval{gb10_mistral_C.logical_reasoning.linear.E_wrong_goals_J}{0}
\defval{gb10_mistral_C.logical_reasoning.linear.EpCG_J}{307.4}
\defval{gb10_mistral_C.logical_reasoning.linear.EpCall_J}{307.4}
\defval{gb10_mistral_C.logical_reasoning.linear.EpG_J}{307.4}
\defval{gb10_mistral_C.logical_reasoning.linear.EpSG_J}{307.4}
\defval{gb10_mistral_C.logical_reasoning.linear.N}{10}
\defval{gb10_mistral_C.logical_reasoning.linear.TpG_s}{4.3}
\defval{gb10_mistral_C.logical_reasoning.linear.attempts_per_goal}{1.00}
\defval{gb10_mistral_C.logical_reasoning.linear.calls_per_goal}{1.00}
\defval{gb10_mistral_C.logical_reasoning.linear.n_completed}{10}
\defval{gb10_mistral_C.logical_reasoning.linear.n_correct}{10}
\defval{gb10_mistral_C.logical_reasoning.linear.p_c}{1.00}
\defval{gb10_mistral_C.logical_reasoning.linear.p_s}{1.00}
\defval{gb10_mistral_C.logical_reasoning.linear.power_W}{71.0}
\defval{gb10_mistral_C.logical_reasoning.min_successes}{10}
\defval{gb10_mistral_C.logical_reasoning.power_ratio}{1.07}
\defval{gb10_mistral_C.logical_reasoning.power_ratio_hi}{1.09}
\defval{gb10_mistral_C.logical_reasoning.power_ratio_lo}{1.05}
\defval{gb10_mistral_C.logical_reasoning.share_energy}{1.00}
\defval{gb10_mistral_C.logical_reasoning.success_ratio}{1.00}
\defval{gb10_mistral_C.logical_reasoning.time_ratio}{34.88}
\defval{gb10_mistral_C.logical_reasoning.time_ratio_hi}{42.40}
\defval{gb10_mistral_C.logical_reasoning.time_ratio_lo}{26.87}
\defval{gb10_mistral_C.science_qa.OOI_EpCG}{15.80}
\defval{gb10_mistral_C.science_qa.OOI_EpCG_hi}{18.99}
\defval{gb10_mistral_C.science_qa.OOI_EpCG_lo}{12.61}
\defval{gb10_mistral_C.science_qa.OOI_EpG}{15.80}
\defval{gb10_mistral_C.science_qa.OOI_EpG_hi}{18.99}
\defval{gb10_mistral_C.science_qa.OOI_EpG_lo}{12.61}
\defval{gb10_mistral_C.science_qa.OOI_EpSG}{15.80}
\defval{gb10_mistral_C.science_qa.OOI_EpSG_hi}{18.99}
\defval{gb10_mistral_C.science_qa.OOI_EpSG_lo}{12.61}
\defval{gb10_mistral_C.science_qa.agentic.E_wrong_goals_J}{0}
\defval{gb10_mistral_C.science_qa.agentic.EpCG_J}{3,224.2}
\defval{gb10_mistral_C.science_qa.agentic.EpCall_J}{383.8}
\defval{gb10_mistral_C.science_qa.agentic.EpG_J}{3,224.2}
\defval{gb10_mistral_C.science_qa.agentic.EpSG_J}{3,224.2}
\defval{gb10_mistral_C.science_qa.agentic.N}{10}
\defval{gb10_mistral_C.science_qa.agentic.TpG_s}{44.7}
\defval{gb10_mistral_C.science_qa.agentic.attempts_per_goal}{2.80}
\defval{gb10_mistral_C.science_qa.agentic.calls_per_goal}{8.40}
\defval{gb10_mistral_C.science_qa.agentic.failure_J}{1,155}
\defval{gb10_mistral_C.science_qa.agentic.n_completed}{10}
\defval{gb10_mistral_C.science_qa.agentic.n_correct}{10}
\defval{gb10_mistral_C.science_qa.agentic.p_c}{1.00}
\defval{gb10_mistral_C.science_qa.agentic.p_s}{1.00}
\defval{gb10_mistral_C.science_qa.agentic.power_W}{72.2}
\defval{gb10_mistral_C.science_qa.energy_ratio}{15.80}
\defval{gb10_mistral_C.science_qa.linear.E_wrong_goals_J}{0}
\defval{gb10_mistral_C.science_qa.linear.EpCG_J}{204.1}
\defval{gb10_mistral_C.science_qa.linear.EpCall_J}{204.1}
\defval{gb10_mistral_C.science_qa.linear.EpG_J}{204.1}
\defval{gb10_mistral_C.science_qa.linear.EpSG_J}{204.1}
\defval{gb10_mistral_C.science_qa.linear.N}{10}
\defval{gb10_mistral_C.science_qa.linear.TpG_s}{2.9}
\defval{gb10_mistral_C.science_qa.linear.attempts_per_goal}{1.00}
\defval{gb10_mistral_C.science_qa.linear.calls_per_goal}{1.00}
\defval{gb10_mistral_C.science_qa.linear.n_completed}{10}
\defval{gb10_mistral_C.science_qa.linear.n_correct}{10}
\defval{gb10_mistral_C.science_qa.linear.p_c}{1.00}
\defval{gb10_mistral_C.science_qa.linear.p_s}{1.00}
\defval{gb10_mistral_C.science_qa.linear.power_W}{71.1}
\defval{gb10_mistral_C.science_qa.min_successes}{10}
\defval{gb10_mistral_C.science_qa.power_ratio}{1.02}
\defval{gb10_mistral_C.science_qa.power_ratio_hi}{1.02}
\defval{gb10_mistral_C.science_qa.power_ratio_lo}{1.01}
\defval{gb10_mistral_C.science_qa.share_energy}{1.00}
\defval{gb10_mistral_C.science_qa.success_ratio}{1.00}
\defval{gb10_mistral_C.science_qa.time_ratio}{15.56}
\defval{gb10_mistral_C.science_qa.time_ratio_hi}{18.72}
\defval{gb10_mistral_C.science_qa.time_ratio_lo}{12.48}
\defval{gb10_mistral_C.tg_sequential_2.OOI_EpCG}{2.17}
\defval{gb10_mistral_C.tg_sequential_2.OOI_EpCG_hi}{2.73}
\defval{gb10_mistral_C.tg_sequential_2.OOI_EpCG_lo}{1.63}
\defval{gb10_mistral_C.tg_sequential_2.OOI_EpG}{2.17}
\defval{gb10_mistral_C.tg_sequential_2.OOI_EpG_hi}{2.73}
\defval{gb10_mistral_C.tg_sequential_2.OOI_EpG_lo}{1.63}
\defval{gb10_mistral_C.tg_sequential_2.OOI_EpSG}{undefined}
\defval{gb10_mistral_C.tg_sequential_2.OOI_EpSG_hi}{undefined}
\defval{gb10_mistral_C.tg_sequential_2.OOI_EpSG_lo}{undefined}
\defval{gb10_mistral_C.tg_sequential_2.agentic.E_wrong_goals_J}{0}
\defval{gb10_mistral_C.tg_sequential_2.agentic.EpCG_J}{3,620.8}
\defval{gb10_mistral_C.tg_sequential_2.agentic.EpCall_J}{1,293.1}
\defval{gb10_mistral_C.tg_sequential_2.agentic.EpG_J}{3,620.8}
\defval{gb10_mistral_C.tg_sequential_2.agentic.EpSG_J}{3,620.8}
\defval{gb10_mistral_C.tg_sequential_2.agentic.N}{10}
\defval{gb10_mistral_C.tg_sequential_2.agentic.TpG_s}{48.9}
\defval{gb10_mistral_C.tg_sequential_2.agentic.attempts_per_goal}{2.80}
\defval{gb10_mistral_C.tg_sequential_2.agentic.calls_per_goal}{2.80}
\defval{gb10_mistral_C.tg_sequential_2.agentic.failure_J}{1,439}
\defval{gb10_mistral_C.tg_sequential_2.agentic.n_completed}{10}
\defval{gb10_mistral_C.tg_sequential_2.agentic.n_correct}{10}
\defval{gb10_mistral_C.tg_sequential_2.agentic.p_c}{1.00}
\defval{gb10_mistral_C.tg_sequential_2.agentic.p_s}{1.00}
\defval{gb10_mistral_C.tg_sequential_2.agentic.power_W}{74.1}
\defval{gb10_mistral_C.tg_sequential_2.energy_ratio}{2.17}
\defval{gb10_mistral_C.tg_sequential_2.linear.E_wrong_goals_J}{16,696}
\defval{gb10_mistral_C.tg_sequential_2.linear.EpCG_J}{1,669.6}
\defval{gb10_mistral_C.tg_sequential_2.linear.EpCall_J}{1,669.6}
\defval{gb10_mistral_C.tg_sequential_2.linear.EpG_J}{1,669.6}
\defval{gb10_mistral_C.tg_sequential_2.linear.EpSG_J}{undefined}
\defval{gb10_mistral_C.tg_sequential_2.linear.N}{10}
\defval{gb10_mistral_C.tg_sequential_2.linear.TpG_s}{23.6}
\defval{gb10_mistral_C.tg_sequential_2.linear.attempts_per_goal}{1.00}
\defval{gb10_mistral_C.tg_sequential_2.linear.calls_per_goal}{1.00}
\defval{gb10_mistral_C.tg_sequential_2.linear.n_completed}{10}
\defval{gb10_mistral_C.tg_sequential_2.linear.n_correct}{0}
\defval{gb10_mistral_C.tg_sequential_2.linear.p_c}{1.00}
\defval{gb10_mistral_C.tg_sequential_2.linear.p_s}{0.00}
\defval{gb10_mistral_C.tg_sequential_2.linear.power_W}{70.9}
\defval{gb10_mistral_C.tg_sequential_2.min_successes}{0}
\defval{gb10_mistral_C.tg_sequential_2.power_ratio}{1.05}
\defval{gb10_mistral_C.tg_sequential_2.power_ratio_hi}{1.06}
\defval{gb10_mistral_C.tg_sequential_2.power_ratio_lo}{1.03}
\defval{gb10_mistral_C.tg_sequential_2.share_energy}{undefined}
\defval{gb10_mistral_C.tg_sequential_2.success_ratio}{0.00}
\defval{gb10_mistral_C.tg_sequential_2.time_ratio}{2.07}
\defval{gb10_mistral_C.tg_sequential_2.time_ratio_hi}{2.60}
\defval{gb10_mistral_C.tg_sequential_2.time_ratio_lo}{1.56}
\defval{gb10_mistral_C.tg_single_calc.OOI_EpCG}{2.14}
\defval{gb10_mistral_C.tg_single_calc.OOI_EpCG_hi}{2.58}
\defval{gb10_mistral_C.tg_single_calc.OOI_EpCG_lo}{1.59}
\defval{gb10_mistral_C.tg_single_calc.OOI_EpG}{2.14}
\defval{gb10_mistral_C.tg_single_calc.OOI_EpG_hi}{2.58}
\defval{gb10_mistral_C.tg_single_calc.OOI_EpG_lo}{1.59}
\defval{gb10_mistral_C.tg_single_calc.OOI_EpSG}{0.21}
\defval{gb10_mistral_C.tg_single_calc.OOI_EpSG_hi}{undefined}
\defval{gb10_mistral_C.tg_single_calc.OOI_EpSG_lo}{undefined}
\defval{gb10_mistral_C.tg_single_calc.agentic.E_wrong_goals_J}{0}
\defval{gb10_mistral_C.tg_single_calc.agentic.EpCG_J}{2,265.0}
\defval{gb10_mistral_C.tg_single_calc.agentic.EpCall_J}{808.9}
\defval{gb10_mistral_C.tg_single_calc.agentic.EpG_J}{2,265.0}
\defval{gb10_mistral_C.tg_single_calc.agentic.EpSG_J}{2,265.0}
\defval{gb10_mistral_C.tg_single_calc.agentic.N}{10}
\defval{gb10_mistral_C.tg_single_calc.agentic.TpG_s}{31.1}
\defval{gb10_mistral_C.tg_single_calc.agentic.attempts_per_goal}{2.80}
\defval{gb10_mistral_C.tg_single_calc.agentic.calls_per_goal}{2.80}
\defval{gb10_mistral_C.tg_single_calc.agentic.failure_J}{1,113}
\defval{gb10_mistral_C.tg_single_calc.agentic.n_completed}{10}
\defval{gb10_mistral_C.tg_single_calc.agentic.n_correct}{10}
\defval{gb10_mistral_C.tg_single_calc.agentic.p_c}{1.00}
\defval{gb10_mistral_C.tg_single_calc.agentic.p_s}{1.00}
\defval{gb10_mistral_C.tg_single_calc.agentic.power_W}{72.9}
\defval{gb10_mistral_C.tg_single_calc.energy_ratio}{2.14}
\defval{gb10_mistral_C.tg_single_calc.linear.E_wrong_goals_J}{9,504}
\defval{gb10_mistral_C.tg_single_calc.linear.EpCG_J}{1,059.5}
\defval{gb10_mistral_C.tg_single_calc.linear.EpCall_J}{1,059.5}
\defval{gb10_mistral_C.tg_single_calc.linear.EpG_J}{1,059.5}
\defval{gb10_mistral_C.tg_single_calc.linear.EpSG_J}{10,595.5}
\defval{gb10_mistral_C.tg_single_calc.linear.N}{10}
\defval{gb10_mistral_C.tg_single_calc.linear.TpG_s}{15.1}
\defval{gb10_mistral_C.tg_single_calc.linear.attempts_per_goal}{1.00}
\defval{gb10_mistral_C.tg_single_calc.linear.calls_per_goal}{1.00}
\defval{gb10_mistral_C.tg_single_calc.linear.n_completed}{10}
\defval{gb10_mistral_C.tg_single_calc.linear.n_correct}{1}
\defval{gb10_mistral_C.tg_single_calc.linear.p_c}{1.00}
\defval{gb10_mistral_C.tg_single_calc.linear.p_s}{0.10}
\defval{gb10_mistral_C.tg_single_calc.linear.power_W}{70.0}
\defval{gb10_mistral_C.tg_single_calc.min_successes}{1}
\defval{gb10_mistral_C.tg_single_calc.power_ratio}{1.04}
\defval{gb10_mistral_C.tg_single_calc.power_ratio_hi}{1.05}
\defval{gb10_mistral_C.tg_single_calc.power_ratio_lo}{1.03}
\defval{gb10_mistral_C.tg_single_calc.share_energy}{-0.49}
\defval{gb10_mistral_C.tg_single_calc.success_ratio}{0.10}
\defval{gb10_mistral_C.tg_single_calc.time_ratio}{2.05}
\defval{gb10_mistral_C.tg_single_calc.time_ratio_hi}{2.47}
\defval{gb10_mistral_C.tg_single_calc.time_ratio_lo}{1.53}
\defval{gb10_mistral_C.tg_single_db.OOI_EpCG}{2.42}
\defval{gb10_mistral_C.tg_single_db.OOI_EpCG_hi}{3.01}
\defval{gb10_mistral_C.tg_single_db.OOI_EpCG_lo}{1.80}
\defval{gb10_mistral_C.tg_single_db.OOI_EpG}{2.42}
\defval{gb10_mistral_C.tg_single_db.OOI_EpG_hi}{3.01}
\defval{gb10_mistral_C.tg_single_db.OOI_EpG_lo}{1.80}
\defval{gb10_mistral_C.tg_single_db.OOI_EpSG}{undefined}
\defval{gb10_mistral_C.tg_single_db.OOI_EpSG_hi}{undefined}
\defval{gb10_mistral_C.tg_single_db.OOI_EpSG_lo}{undefined}
\defval{gb10_mistral_C.tg_single_db.agentic.E_wrong_goals_J}{0}
\defval{gb10_mistral_C.tg_single_db.agentic.EpCG_J}{2,704.8}
\defval{gb10_mistral_C.tg_single_db.agentic.EpCall_J}{966.0}
\defval{gb10_mistral_C.tg_single_db.agentic.EpG_J}{2,704.8}
\defval{gb10_mistral_C.tg_single_db.agentic.EpSG_J}{2,704.8}
\defval{gb10_mistral_C.tg_single_db.agentic.N}{10}
\defval{gb10_mistral_C.tg_single_db.agentic.TpG_s}{37.1}
\defval{gb10_mistral_C.tg_single_db.agentic.attempts_per_goal}{2.80}
\defval{gb10_mistral_C.tg_single_db.agentic.calls_per_goal}{2.80}
\defval{gb10_mistral_C.tg_single_db.agentic.failure_J}{1,257}
\defval{gb10_mistral_C.tg_single_db.agentic.n_completed}{10}
\defval{gb10_mistral_C.tg_single_db.agentic.n_correct}{10}
\defval{gb10_mistral_C.tg_single_db.agentic.p_c}{1.00}
\defval{gb10_mistral_C.tg_single_db.agentic.p_s}{1.00}
\defval{gb10_mistral_C.tg_single_db.agentic.power_W}{72.9}
\defval{gb10_mistral_C.tg_single_db.energy_ratio}{2.42}
\defval{gb10_mistral_C.tg_single_db.linear.E_wrong_goals_J}{11,176}
\defval{gb10_mistral_C.tg_single_db.linear.EpCG_J}{1,117.6}
\defval{gb10_mistral_C.tg_single_db.linear.EpCall_J}{1,117.6}
\defval{gb10_mistral_C.tg_single_db.linear.EpG_J}{1,117.6}
\defval{gb10_mistral_C.tg_single_db.linear.EpSG_J}{undefined}
\defval{gb10_mistral_C.tg_single_db.linear.N}{10}
\defval{gb10_mistral_C.tg_single_db.linear.TpG_s}{15.9}
\defval{gb10_mistral_C.tg_single_db.linear.attempts_per_goal}{1.00}
\defval{gb10_mistral_C.tg_single_db.linear.calls_per_goal}{1.00}
\defval{gb10_mistral_C.tg_single_db.linear.n_completed}{10}
\defval{gb10_mistral_C.tg_single_db.linear.n_correct}{0}
\defval{gb10_mistral_C.tg_single_db.linear.p_c}{1.00}
\defval{gb10_mistral_C.tg_single_db.linear.p_s}{0.00}
\defval{gb10_mistral_C.tg_single_db.linear.power_W}{70.1}
\defval{gb10_mistral_C.tg_single_db.min_successes}{0}
\defval{gb10_mistral_C.tg_single_db.power_ratio}{1.04}
\defval{gb10_mistral_C.tg_single_db.power_ratio_hi}{1.06}
\defval{gb10_mistral_C.tg_single_db.power_ratio_lo}{1.03}
\defval{gb10_mistral_C.tg_single_db.share_energy}{undefined}
\defval{gb10_mistral_C.tg_single_db.success_ratio}{0.00}
\defval{gb10_mistral_C.tg_single_db.time_ratio}{2.33}
\defval{gb10_mistral_C.tg_single_db.time_ratio_hi}{2.88}
\defval{gb10_mistral_C.tg_single_db.time_ratio_lo}{1.73}
\defval{gb10_mistral_dist.excluded}{0}
\defval{gb10_mistral_dist.goals}{100}
\defval{gb10_mistral_dist.gsm8k.OOI_EpCG}{29.61}
\defval{gb10_mistral_dist.gsm8k.OOI_EpCG_hi}{44.71}
\defval{gb10_mistral_dist.gsm8k.OOI_EpCG_lo}{19.44}
\defval{gb10_mistral_dist.gsm8k.OOI_EpG}{19.55}
\defval{gb10_mistral_dist.gsm8k.OOI_EpG_hi}{24.87}
\defval{gb10_mistral_dist.gsm8k.OOI_EpG_lo}{14.91}
\defval{gb10_mistral_dist.gsm8k.OOI_EpSG}{41.05}
\defval{gb10_mistral_dist.gsm8k.OOI_EpSG_hi}{90.80}
\defval{gb10_mistral_dist.gsm8k.OOI_EpSG_lo}{20.89}
\defval{gb10_mistral_dist.gsm8k.agentic.E_wrong_goals_J}{155,699}
\defval{gb10_mistral_dist.gsm8k.agentic.EpCG_J}{22,808.7}
\defval{gb10_mistral_dist.gsm8k.agentic.EpG_J}{15,053.7}
\defval{gb10_mistral_dist.gsm8k.agentic.EpSG_J}{75,268.7}
\defval{gb10_mistral_dist.gsm8k.agentic.N}{50}
\defval{gb10_mistral_dist.gsm8k.agentic.TpG_s}{220.8}
\defval{gb10_mistral_dist.gsm8k.agentic.attempts_per_goal}{3.06}
\defval{gb10_mistral_dist.gsm8k.agentic.n_completed}{33}
\defval{gb10_mistral_dist.gsm8k.agentic.n_correct}{10}
\defval{gb10_mistral_dist.gsm8k.agentic.p_c}{0.66}
\defval{gb10_mistral_dist.gsm8k.agentic.p_s}{0.20}
\defval{gb10_mistral_dist.gsm8k.agentic.power_W}{68.2}
\defval{gb10_mistral_dist.gsm8k.energy_ratio}{19.55}
\defval{gb10_mistral_dist.gsm8k.linear.E_wrong_goals_J}{23,911}
\defval{gb10_mistral_dist.gsm8k.linear.EpCG_J}{770.2}
\defval{gb10_mistral_dist.gsm8k.linear.EpG_J}{770.2}
\defval{gb10_mistral_dist.gsm8k.linear.EpSG_J}{1,833.8}
\defval{gb10_mistral_dist.gsm8k.linear.N}{50}
\defval{gb10_mistral_dist.gsm8k.linear.TpG_s}{11.9}
\defval{gb10_mistral_dist.gsm8k.linear.attempts_per_goal}{1.00}
\defval{gb10_mistral_dist.gsm8k.linear.n_completed}{50}
\defval{gb10_mistral_dist.gsm8k.linear.n_correct}{21}
\defval{gb10_mistral_dist.gsm8k.linear.p_c}{1.00}
\defval{gb10_mistral_dist.gsm8k.linear.p_s}{0.42}
\defval{gb10_mistral_dist.gsm8k.linear.power_W}{64.7}
\defval{gb10_mistral_dist.gsm8k.min_successes}{10}
\defval{gb10_mistral_dist.gsm8k.power_ratio}{1.05}
\defval{gb10_mistral_dist.gsm8k.power_ratio_hi}{1.09}
\defval{gb10_mistral_dist.gsm8k.power_ratio_lo}{1.02}
\defval{gb10_mistral_dist.gsm8k.share_energy}{0.80}
\defval{gb10_mistral_dist.gsm8k.success_ratio}{2.10}
\defval{gb10_mistral_dist.gsm8k.time_ratio}{18.55}
\defval{gb10_mistral_dist.gsm8k.time_ratio_hi}{23.35}
\defval{gb10_mistral_dist.gsm8k.time_ratio_lo}{14.30}
\defval{gb10_mistral_tools.excluded}{60}
\defval{gb10_mistral_tools.goals}{120}
\defval{gb10_mistral_tools.tg_sequential_2.OOI_EpCG}{0.65}
\defval{gb10_mistral_tools.tg_sequential_2.OOI_EpCG_hi}{0.76}
\defval{gb10_mistral_tools.tg_sequential_2.OOI_EpCG_lo}{0.57}
\defval{gb10_mistral_tools.tg_sequential_2.OOI_EpG}{0.65}
\defval{gb10_mistral_tools.tg_sequential_2.OOI_EpG_hi}{0.76}
\defval{gb10_mistral_tools.tg_sequential_2.OOI_EpG_lo}{0.57}
\defval{gb10_mistral_tools.tg_sequential_2.OOI_EpSG}{undefined}
\defval{gb10_mistral_tools.tg_sequential_2.OOI_EpSG_hi}{undefined}
\defval{gb10_mistral_tools.tg_sequential_2.OOI_EpSG_lo}{undefined}
\defval{gb10_mistral_tools.tg_sequential_2.agentic.E_wrong_goals_J}{0}
\defval{gb10_mistral_tools.tg_sequential_2.agentic.EpCG_J}{1,224.2}
\defval{gb10_mistral_tools.tg_sequential_2.agentic.EpG_J}{1,224.2}
\defval{gb10_mistral_tools.tg_sequential_2.agentic.EpSG_J}{1,224.2}
\defval{gb10_mistral_tools.tg_sequential_2.agentic.N}{30}
\defval{gb10_mistral_tools.tg_sequential_2.agentic.TpG_s}{16.3}
\defval{gb10_mistral_tools.tg_sequential_2.agentic.attempts_per_goal}{1.00}
\defval{gb10_mistral_tools.tg_sequential_2.agentic.n_completed}{30}
\defval{gb10_mistral_tools.tg_sequential_2.agentic.n_correct}{30}
\defval{gb10_mistral_tools.tg_sequential_2.agentic.p_c}{1.00}
\defval{gb10_mistral_tools.tg_sequential_2.agentic.p_s}{1.00}
\defval{gb10_mistral_tools.tg_sequential_2.agentic.power_W}{75.3}
\defval{gb10_mistral_tools.tg_sequential_2.energy_ratio}{0.65}
\defval{gb10_mistral_tools.tg_sequential_2.linear.E_wrong_goals_J}{56,277}
\defval{gb10_mistral_tools.tg_sequential_2.linear.EpCG_J}{1,875.9}
\defval{gb10_mistral_tools.tg_sequential_2.linear.EpG_J}{1,875.9}
\defval{gb10_mistral_tools.tg_sequential_2.linear.EpSG_J}{undefined}
\defval{gb10_mistral_tools.tg_sequential_2.linear.N}{30}
\defval{gb10_mistral_tools.tg_sequential_2.linear.TpG_s}{25.8}
\defval{gb10_mistral_tools.tg_sequential_2.linear.attempts_per_goal}{1.00}
\defval{gb10_mistral_tools.tg_sequential_2.linear.n_completed}{30}
\defval{gb10_mistral_tools.tg_sequential_2.linear.n_correct}{0}
\defval{gb10_mistral_tools.tg_sequential_2.linear.p_c}{1.00}
\defval{gb10_mistral_tools.tg_sequential_2.linear.p_s}{0.00}
\defval{gb10_mistral_tools.tg_sequential_2.linear.power_W}{72.7}
\defval{gb10_mistral_tools.tg_sequential_2.min_successes}{0}
\defval{gb10_mistral_tools.tg_sequential_2.power_ratio}{1.04}
\defval{gb10_mistral_tools.tg_sequential_2.power_ratio_hi}{1.04}
\defval{gb10_mistral_tools.tg_sequential_2.power_ratio_lo}{1.03}
\defval{gb10_mistral_tools.tg_sequential_2.share_energy}{undefined}
\defval{gb10_mistral_tools.tg_sequential_2.success_ratio}{0.00}
\defval{gb10_mistral_tools.tg_sequential_2.time_ratio}{0.63}
\defval{gb10_mistral_tools.tg_sequential_2.time_ratio_hi}{0.74}
\defval{gb10_mistral_tools.tg_sequential_2.time_ratio_lo}{0.55}
\defval{gb10_mistral_tools.tg_single_db.OOI_EpCG}{0.55}
\defval{gb10_mistral_tools.tg_single_db.OOI_EpCG_hi}{0.77}
\defval{gb10_mistral_tools.tg_single_db.OOI_EpCG_lo}{0.37}
\defval{gb10_mistral_tools.tg_single_db.OOI_EpG}{0.55}
\defval{gb10_mistral_tools.tg_single_db.OOI_EpG_hi}{0.77}
\defval{gb10_mistral_tools.tg_single_db.OOI_EpG_lo}{0.37}
\defval{gb10_mistral_tools.tg_single_db.OOI_EpSG}{undefined}
\defval{gb10_mistral_tools.tg_single_db.OOI_EpSG_hi}{undefined}
\defval{gb10_mistral_tools.tg_single_db.OOI_EpSG_lo}{undefined}
\defval{gb10_mistral_tools.tg_single_db.agentic.E_wrong_goals_J}{0}
\defval{gb10_mistral_tools.tg_single_db.agentic.EpCG_J}{574.4}
\defval{gb10_mistral_tools.tg_single_db.agentic.EpG_J}{574.4}
\defval{gb10_mistral_tools.tg_single_db.agentic.EpSG_J}{574.4}
\defval{gb10_mistral_tools.tg_single_db.agentic.N}{30}
\defval{gb10_mistral_tools.tg_single_db.agentic.TpG_s}{7.7}
\defval{gb10_mistral_tools.tg_single_db.agentic.attempts_per_goal}{1.00}
\defval{gb10_mistral_tools.tg_single_db.agentic.n_completed}{30}
\defval{gb10_mistral_tools.tg_single_db.agentic.n_correct}{30}
\defval{gb10_mistral_tools.tg_single_db.agentic.p_c}{1.00}
\defval{gb10_mistral_tools.tg_single_db.agentic.p_s}{1.00}
\defval{gb10_mistral_tools.tg_single_db.agentic.power_W}{74.2}
\defval{gb10_mistral_tools.tg_single_db.energy_ratio}{0.55}
\defval{gb10_mistral_tools.tg_single_db.linear.E_wrong_goals_J}{31,572}
\defval{gb10_mistral_tools.tg_single_db.linear.EpCG_J}{1,052.4}
\defval{gb10_mistral_tools.tg_single_db.linear.EpG_J}{1,052.4}
\defval{gb10_mistral_tools.tg_single_db.linear.EpSG_J}{undefined}
\defval{gb10_mistral_tools.tg_single_db.linear.N}{30}
\defval{gb10_mistral_tools.tg_single_db.linear.TpG_s}{14.7}
\defval{gb10_mistral_tools.tg_single_db.linear.attempts_per_goal}{1.00}
\defval{gb10_mistral_tools.tg_single_db.linear.n_completed}{30}
\defval{gb10_mistral_tools.tg_single_db.linear.n_correct}{0}
\defval{gb10_mistral_tools.tg_single_db.linear.p_c}{1.00}
\defval{gb10_mistral_tools.tg_single_db.linear.p_s}{0.00}
\defval{gb10_mistral_tools.tg_single_db.linear.power_W}{71.7}
\defval{gb10_mistral_tools.tg_single_db.min_successes}{0}
\defval{gb10_mistral_tools.tg_single_db.power_ratio}{1.03}
\defval{gb10_mistral_tools.tg_single_db.power_ratio_hi}{1.05}
\defval{gb10_mistral_tools.tg_single_db.power_ratio_lo}{1.02}
\defval{gb10_mistral_tools.tg_single_db.share_energy}{undefined}
\defval{gb10_mistral_tools.tg_single_db.success_ratio}{0.00}
\defval{gb10_mistral_tools.tg_single_db.time_ratio}{0.53}
\defval{gb10_mistral_tools.tg_single_db.time_ratio_hi}{0.75}
\defval{gb10_mistral_tools.tg_single_db.time_ratio_lo}{0.36}
\defval{gb10_qwen3.commonsense_reasoning.OOI_EpCG}{15.53}
\defval{gb10_qwen3.commonsense_reasoning.OOI_EpCG_hi}{16.08}
\defval{gb10_qwen3.commonsense_reasoning.OOI_EpCG_lo}{15.03}
\defval{gb10_qwen3.commonsense_reasoning.OOI_EpG}{15.53}
\defval{gb10_qwen3.commonsense_reasoning.OOI_EpG_hi}{16.08}
\defval{gb10_qwen3.commonsense_reasoning.OOI_EpG_lo}{15.03}
\defval{gb10_qwen3.commonsense_reasoning.OOI_EpSG}{15.53}
\defval{gb10_qwen3.commonsense_reasoning.OOI_EpSG_hi}{16.08}
\defval{gb10_qwen3.commonsense_reasoning.OOI_EpSG_lo}{15.03}
\defval{gb10_qwen3.commonsense_reasoning.agentic.E_wrong_goals_J}{0}
\defval{gb10_qwen3.commonsense_reasoning.agentic.EpCG_J}{4,965.8}
\defval{gb10_qwen3.commonsense_reasoning.agentic.EpG_J}{4,965.8}
\defval{gb10_qwen3.commonsense_reasoning.agentic.EpSG_J}{4,965.8}
\defval{gb10_qwen3.commonsense_reasoning.agentic.N}{30}
\defval{gb10_qwen3.commonsense_reasoning.agentic.TpG_s}{79.1}
\defval{gb10_qwen3.commonsense_reasoning.agentic.attempts_per_goal}{1.00}
\defval{gb10_qwen3.commonsense_reasoning.agentic.n_completed}{30}
\defval{gb10_qwen3.commonsense_reasoning.agentic.n_correct}{30}
\defval{gb10_qwen3.commonsense_reasoning.agentic.p_c}{1.00}
\defval{gb10_qwen3.commonsense_reasoning.agentic.p_s}{1.00}
\defval{gb10_qwen3.commonsense_reasoning.agentic.power_W}{62.8}
\defval{gb10_qwen3.commonsense_reasoning.energy_ratio}{15.53}
\defval{gb10_qwen3.commonsense_reasoning.linear.E_wrong_goals_J}{0}
\defval{gb10_qwen3.commonsense_reasoning.linear.EpCG_J}{319.8}
\defval{gb10_qwen3.commonsense_reasoning.linear.EpG_J}{319.8}
\defval{gb10_qwen3.commonsense_reasoning.linear.EpSG_J}{319.8}
\defval{gb10_qwen3.commonsense_reasoning.linear.N}{30}
\defval{gb10_qwen3.commonsense_reasoning.linear.TpG_s}{5.4}
\defval{gb10_qwen3.commonsense_reasoning.linear.attempts_per_goal}{1.00}
\defval{gb10_qwen3.commonsense_reasoning.linear.n_completed}{30}
\defval{gb10_qwen3.commonsense_reasoning.linear.n_correct}{30}
\defval{gb10_qwen3.commonsense_reasoning.linear.p_c}{1.00}
\defval{gb10_qwen3.commonsense_reasoning.linear.p_s}{1.00}
\defval{gb10_qwen3.commonsense_reasoning.linear.power_W}{59.6}
\defval{gb10_qwen3.commonsense_reasoning.min_successes}{30}
\defval{gb10_qwen3.commonsense_reasoning.power_ratio}{1.05}
\defval{gb10_qwen3.commonsense_reasoning.power_ratio_hi}{1.06}
\defval{gb10_qwen3.commonsense_reasoning.power_ratio_lo}{1.05}
\defval{gb10_qwen3.commonsense_reasoning.share_energy}{1.00}
\defval{gb10_qwen3.commonsense_reasoning.success_ratio}{1.00}
\defval{gb10_qwen3.commonsense_reasoning.time_ratio}{14.74}
\defval{gb10_qwen3.commonsense_reasoning.time_ratio_hi}{15.26}
\defval{gb10_qwen3.commonsense_reasoning.time_ratio_lo}{14.27}
\defval{gb10_qwen3.excluded}{0}
\defval{gb10_qwen3.factual_qa.OOI_EpCG}{24.31}
\defval{gb10_qwen3.factual_qa.OOI_EpCG_hi}{25.32}
\defval{gb10_qwen3.factual_qa.OOI_EpCG_lo}{23.34}
\defval{gb10_qwen3.factual_qa.OOI_EpG}{24.31}
\defval{gb10_qwen3.factual_qa.OOI_EpG_hi}{25.32}
\defval{gb10_qwen3.factual_qa.OOI_EpG_lo}{23.34}
\defval{gb10_qwen3.factual_qa.OOI_EpSG}{24.31}
\defval{gb10_qwen3.factual_qa.OOI_EpSG_hi}{25.32}
\defval{gb10_qwen3.factual_qa.OOI_EpSG_lo}{23.34}
\defval{gb10_qwen3.factual_qa.agentic.E_wrong_goals_J}{0}
\defval{gb10_qwen3.factual_qa.agentic.EpCG_J}{1,362.8}
\defval{gb10_qwen3.factual_qa.agentic.EpG_J}{1,362.8}
\defval{gb10_qwen3.factual_qa.agentic.EpSG_J}{1,362.8}
\defval{gb10_qwen3.factual_qa.agentic.N}{30}
\defval{gb10_qwen3.factual_qa.agentic.TpG_s}{23.7}
\defval{gb10_qwen3.factual_qa.agentic.attempts_per_goal}{1.00}
\defval{gb10_qwen3.factual_qa.agentic.n_completed}{30}
\defval{gb10_qwen3.factual_qa.agentic.n_correct}{30}
\defval{gb10_qwen3.factual_qa.agentic.p_c}{1.00}
\defval{gb10_qwen3.factual_qa.agentic.p_s}{1.00}
\defval{gb10_qwen3.factual_qa.agentic.power_W}{57.5}
\defval{gb10_qwen3.factual_qa.energy_ratio}{24.31}
\defval{gb10_qwen3.factual_qa.linear.E_wrong_goals_J}{0}
\defval{gb10_qwen3.factual_qa.linear.EpCG_J}{56.1}
\defval{gb10_qwen3.factual_qa.linear.EpG_J}{56.1}
\defval{gb10_qwen3.factual_qa.linear.EpSG_J}{56.1}
\defval{gb10_qwen3.factual_qa.linear.N}{30}
\defval{gb10_qwen3.factual_qa.linear.TpG_s}{0.9}
\defval{gb10_qwen3.factual_qa.linear.attempts_per_goal}{1.00}
\defval{gb10_qwen3.factual_qa.linear.n_completed}{30}
\defval{gb10_qwen3.factual_qa.linear.n_correct}{30}
\defval{gb10_qwen3.factual_qa.linear.p_c}{1.00}
\defval{gb10_qwen3.factual_qa.linear.p_s}{1.00}
\defval{gb10_qwen3.factual_qa.linear.power_W}{60.2}
\defval{gb10_qwen3.factual_qa.min_successes}{30}
\defval{gb10_qwen3.factual_qa.power_ratio}{0.96}
\defval{gb10_qwen3.factual_qa.power_ratio_hi}{0.96}
\defval{gb10_qwen3.factual_qa.power_ratio_lo}{0.95}
\defval{gb10_qwen3.factual_qa.share_energy}{1.00}
\defval{gb10_qwen3.factual_qa.success_ratio}{1.00}
\defval{gb10_qwen3.factual_qa.time_ratio}{25.44}
\defval{gb10_qwen3.factual_qa.time_ratio_hi}{26.44}
\defval{gb10_qwen3.factual_qa.time_ratio_lo}{24.47}
\defval{gb10_qwen3.goals}{540}
\defval{gb10_qwen3.gsm8k_basic.OOI_EpCG}{14.86}
\defval{gb10_qwen3.gsm8k_basic.OOI_EpCG_hi}{15.83}
\defval{gb10_qwen3.gsm8k_basic.OOI_EpCG_lo}{13.89}
\defval{gb10_qwen3.gsm8k_basic.OOI_EpG}{14.86}
\defval{gb10_qwen3.gsm8k_basic.OOI_EpG_hi}{15.83}
\defval{gb10_qwen3.gsm8k_basic.OOI_EpG_lo}{13.89}
\defval{gb10_qwen3.gsm8k_basic.OOI_EpSG}{14.86}
\defval{gb10_qwen3.gsm8k_basic.OOI_EpSG_hi}{15.83}
\defval{gb10_qwen3.gsm8k_basic.OOI_EpSG_lo}{13.89}
\defval{gb10_qwen3.gsm8k_basic.agentic.E_wrong_goals_J}{0}
\defval{gb10_qwen3.gsm8k_basic.agentic.EpCG_J}{724.4}
\defval{gb10_qwen3.gsm8k_basic.agentic.EpG_J}{724.4}
\defval{gb10_qwen3.gsm8k_basic.agentic.EpSG_J}{724.4}
\defval{gb10_qwen3.gsm8k_basic.agentic.N}{30}
\defval{gb10_qwen3.gsm8k_basic.agentic.TpG_s}{13.0}
\defval{gb10_qwen3.gsm8k_basic.agentic.attempts_per_goal}{1.00}
\defval{gb10_qwen3.gsm8k_basic.agentic.n_completed}{30}
\defval{gb10_qwen3.gsm8k_basic.agentic.n_correct}{30}
\defval{gb10_qwen3.gsm8k_basic.agentic.p_c}{1.00}
\defval{gb10_qwen3.gsm8k_basic.agentic.p_s}{1.00}
\defval{gb10_qwen3.gsm8k_basic.agentic.power_W}{55.8}
\defval{gb10_qwen3.gsm8k_basic.energy_ratio}{14.86}
\defval{gb10_qwen3.gsm8k_basic.linear.E_wrong_goals_J}{0}
\defval{gb10_qwen3.gsm8k_basic.linear.EpCG_J}{48.8}
\defval{gb10_qwen3.gsm8k_basic.linear.EpG_J}{48.8}
\defval{gb10_qwen3.gsm8k_basic.linear.EpSG_J}{48.8}
\defval{gb10_qwen3.gsm8k_basic.linear.N}{30}
\defval{gb10_qwen3.gsm8k_basic.linear.TpG_s}{0.8}
\defval{gb10_qwen3.gsm8k_basic.linear.attempts_per_goal}{1.00}
\defval{gb10_qwen3.gsm8k_basic.linear.n_completed}{30}
\defval{gb10_qwen3.gsm8k_basic.linear.n_correct}{30}
\defval{gb10_qwen3.gsm8k_basic.linear.p_c}{1.00}
\defval{gb10_qwen3.gsm8k_basic.linear.p_s}{1.00}
\defval{gb10_qwen3.gsm8k_basic.linear.power_W}{58.6}
\defval{gb10_qwen3.gsm8k_basic.min_successes}{30}
\defval{gb10_qwen3.gsm8k_basic.power_ratio}{0.95}
\defval{gb10_qwen3.gsm8k_basic.power_ratio_hi}{0.96}
\defval{gb10_qwen3.gsm8k_basic.power_ratio_lo}{0.94}
\defval{gb10_qwen3.gsm8k_basic.share_energy}{1.00}
\defval{gb10_qwen3.gsm8k_basic.success_ratio}{1.00}
\defval{gb10_qwen3.gsm8k_basic.time_ratio}{15.61}
\defval{gb10_qwen3.gsm8k_basic.time_ratio_hi}{16.62}
\defval{gb10_qwen3.gsm8k_basic.time_ratio_lo}{14.62}
\defval{gb10_qwen3.gsm8k_multi_step.OOI_EpCG}{8.24}
\defval{gb10_qwen3.gsm8k_multi_step.OOI_EpCG_hi}{8.77}
\defval{gb10_qwen3.gsm8k_multi_step.OOI_EpCG_lo}{7.79}
\defval{gb10_qwen3.gsm8k_multi_step.OOI_EpG}{8.24}
\defval{gb10_qwen3.gsm8k_multi_step.OOI_EpG_hi}{8.77}
\defval{gb10_qwen3.gsm8k_multi_step.OOI_EpG_lo}{7.79}
\defval{gb10_qwen3.gsm8k_multi_step.OOI_EpSG}{8.24}
\defval{gb10_qwen3.gsm8k_multi_step.OOI_EpSG_hi}{8.77}
\defval{gb10_qwen3.gsm8k_multi_step.OOI_EpSG_lo}{7.79}
\defval{gb10_qwen3.gsm8k_multi_step.agentic.E_wrong_goals_J}{0}
\defval{gb10_qwen3.gsm8k_multi_step.agentic.EpCG_J}{1,916.1}
\defval{gb10_qwen3.gsm8k_multi_step.agentic.EpG_J}{1,916.1}
\defval{gb10_qwen3.gsm8k_multi_step.agentic.EpSG_J}{1,916.1}
\defval{gb10_qwen3.gsm8k_multi_step.agentic.N}{30}
\defval{gb10_qwen3.gsm8k_multi_step.agentic.TpG_s}{31.5}
\defval{gb10_qwen3.gsm8k_multi_step.agentic.attempts_per_goal}{1.00}
\defval{gb10_qwen3.gsm8k_multi_step.agentic.n_completed}{30}
\defval{gb10_qwen3.gsm8k_multi_step.agentic.n_correct}{30}
\defval{gb10_qwen3.gsm8k_multi_step.agentic.p_c}{1.00}
\defval{gb10_qwen3.gsm8k_multi_step.agentic.p_s}{1.00}
\defval{gb10_qwen3.gsm8k_multi_step.agentic.power_W}{60.8}
\defval{gb10_qwen3.gsm8k_multi_step.energy_ratio}{8.24}
\defval{gb10_qwen3.gsm8k_multi_step.linear.E_wrong_goals_J}{0}
\defval{gb10_qwen3.gsm8k_multi_step.linear.EpCG_J}{232.5}
\defval{gb10_qwen3.gsm8k_multi_step.linear.EpG_J}{232.5}
\defval{gb10_qwen3.gsm8k_multi_step.linear.EpSG_J}{232.5}
\defval{gb10_qwen3.gsm8k_multi_step.linear.N}{30}
\defval{gb10_qwen3.gsm8k_multi_step.linear.TpG_s}{3.9}
\defval{gb10_qwen3.gsm8k_multi_step.linear.attempts_per_goal}{1.00}
\defval{gb10_qwen3.gsm8k_multi_step.linear.n_completed}{30}
\defval{gb10_qwen3.gsm8k_multi_step.linear.n_correct}{30}
\defval{gb10_qwen3.gsm8k_multi_step.linear.p_c}{1.00}
\defval{gb10_qwen3.gsm8k_multi_step.linear.p_s}{1.00}
\defval{gb10_qwen3.gsm8k_multi_step.linear.power_W}{60.0}
\defval{gb10_qwen3.gsm8k_multi_step.min_successes}{30}
\defval{gb10_qwen3.gsm8k_multi_step.power_ratio}{1.01}
\defval{gb10_qwen3.gsm8k_multi_step.power_ratio_hi}{1.03}
\defval{gb10_qwen3.gsm8k_multi_step.power_ratio_lo}{0.99}
\defval{gb10_qwen3.gsm8k_multi_step.share_energy}{1.00}
\defval{gb10_qwen3.gsm8k_multi_step.success_ratio}{1.00}
\defval{gb10_qwen3.gsm8k_multi_step.time_ratio}{8.13}
\defval{gb10_qwen3.gsm8k_multi_step.time_ratio_hi}{8.61}
\defval{gb10_qwen3.gsm8k_multi_step.time_ratio_lo}{7.72}
\defval{gb10_qwen3.logical_reasoning.OOI_EpCG}{21.84}
\defval{gb10_qwen3.logical_reasoning.OOI_EpCG_hi}{23.01}
\defval{gb10_qwen3.logical_reasoning.OOI_EpCG_lo}{20.74}
\defval{gb10_qwen3.logical_reasoning.OOI_EpG}{21.84}
\defval{gb10_qwen3.logical_reasoning.OOI_EpG_hi}{23.01}
\defval{gb10_qwen3.logical_reasoning.OOI_EpG_lo}{20.74}
\defval{gb10_qwen3.logical_reasoning.OOI_EpSG}{21.84}
\defval{gb10_qwen3.logical_reasoning.OOI_EpSG_hi}{23.01}
\defval{gb10_qwen3.logical_reasoning.OOI_EpSG_lo}{20.74}
\defval{gb10_qwen3.logical_reasoning.agentic.E_wrong_goals_J}{0}
\defval{gb10_qwen3.logical_reasoning.agentic.EpCG_J}{2,142.0}
\defval{gb10_qwen3.logical_reasoning.agentic.EpG_J}{2,142.0}
\defval{gb10_qwen3.logical_reasoning.agentic.EpSG_J}{2,142.0}
\defval{gb10_qwen3.logical_reasoning.agentic.N}{30}
\defval{gb10_qwen3.logical_reasoning.agentic.TpG_s}{35.1}
\defval{gb10_qwen3.logical_reasoning.agentic.attempts_per_goal}{1.00}
\defval{gb10_qwen3.logical_reasoning.agentic.n_completed}{30}
\defval{gb10_qwen3.logical_reasoning.agentic.n_correct}{30}
\defval{gb10_qwen3.logical_reasoning.agentic.p_c}{1.00}
\defval{gb10_qwen3.logical_reasoning.agentic.p_s}{1.00}
\defval{gb10_qwen3.logical_reasoning.agentic.power_W}{61.1}
\defval{gb10_qwen3.logical_reasoning.energy_ratio}{21.84}
\defval{gb10_qwen3.logical_reasoning.linear.E_wrong_goals_J}{0}
\defval{gb10_qwen3.logical_reasoning.linear.EpCG_J}{98.1}
\defval{gb10_qwen3.logical_reasoning.linear.EpG_J}{98.1}
\defval{gb10_qwen3.logical_reasoning.linear.EpSG_J}{98.1}
\defval{gb10_qwen3.logical_reasoning.linear.N}{30}
\defval{gb10_qwen3.logical_reasoning.linear.TpG_s}{1.6}
\defval{gb10_qwen3.logical_reasoning.linear.attempts_per_goal}{1.00}
\defval{gb10_qwen3.logical_reasoning.linear.n_completed}{30}
\defval{gb10_qwen3.logical_reasoning.linear.n_correct}{30}
\defval{gb10_qwen3.logical_reasoning.linear.p_c}{1.00}
\defval{gb10_qwen3.logical_reasoning.linear.p_s}{1.00}
\defval{gb10_qwen3.logical_reasoning.linear.power_W}{61.6}
\defval{gb10_qwen3.logical_reasoning.min_successes}{30}
\defval{gb10_qwen3.logical_reasoning.power_ratio}{0.99}
\defval{gb10_qwen3.logical_reasoning.power_ratio_hi}{1.00}
\defval{gb10_qwen3.logical_reasoning.power_ratio_lo}{0.98}
\defval{gb10_qwen3.logical_reasoning.share_energy}{1.00}
\defval{gb10_qwen3.logical_reasoning.success_ratio}{1.00}
\defval{gb10_qwen3.logical_reasoning.time_ratio}{22.04}
\defval{gb10_qwen3.logical_reasoning.time_ratio_hi}{23.21}
\defval{gb10_qwen3.logical_reasoning.time_ratio_lo}{20.93}
\defval{gb10_qwen3.science_qa.OOI_EpCG}{39.50}
\defval{gb10_qwen3.science_qa.OOI_EpCG_hi}{43.95}
\defval{gb10_qwen3.science_qa.OOI_EpCG_lo}{35.98}
\defval{gb10_qwen3.science_qa.OOI_EpG}{39.50}
\defval{gb10_qwen3.science_qa.OOI_EpG_hi}{43.95}
\defval{gb10_qwen3.science_qa.OOI_EpG_lo}{35.98}
\defval{gb10_qwen3.science_qa.OOI_EpSG}{39.50}
\defval{gb10_qwen3.science_qa.OOI_EpSG_hi}{43.95}
\defval{gb10_qwen3.science_qa.OOI_EpSG_lo}{35.98}
\defval{gb10_qwen3.science_qa.agentic.E_wrong_goals_J}{0}
\defval{gb10_qwen3.science_qa.agentic.EpCG_J}{779.8}
\defval{gb10_qwen3.science_qa.agentic.EpG_J}{779.8}
\defval{gb10_qwen3.science_qa.agentic.EpSG_J}{779.8}
\defval{gb10_qwen3.science_qa.agentic.N}{30}
\defval{gb10_qwen3.science_qa.agentic.TpG_s}{14.1}
\defval{gb10_qwen3.science_qa.agentic.attempts_per_goal}{1.00}
\defval{gb10_qwen3.science_qa.agentic.n_completed}{30}
\defval{gb10_qwen3.science_qa.agentic.n_correct}{30}
\defval{gb10_qwen3.science_qa.agentic.p_c}{1.00}
\defval{gb10_qwen3.science_qa.agentic.p_s}{1.00}
\defval{gb10_qwen3.science_qa.agentic.power_W}{55.2}
\defval{gb10_qwen3.science_qa.energy_ratio}{39.50}
\defval{gb10_qwen3.science_qa.linear.E_wrong_goals_J}{0}
\defval{gb10_qwen3.science_qa.linear.EpCG_J}{19.7}
\defval{gb10_qwen3.science_qa.linear.EpG_J}{19.7}
\defval{gb10_qwen3.science_qa.linear.EpSG_J}{19.7}
\defval{gb10_qwen3.science_qa.linear.N}{30}
\defval{gb10_qwen3.science_qa.linear.TpG_s}{0.3}
\defval{gb10_qwen3.science_qa.linear.attempts_per_goal}{1.00}
\defval{gb10_qwen3.science_qa.linear.n_completed}{30}
\defval{gb10_qwen3.science_qa.linear.n_correct}{30}
\defval{gb10_qwen3.science_qa.linear.p_c}{1.00}
\defval{gb10_qwen3.science_qa.linear.p_s}{1.00}
\defval{gb10_qwen3.science_qa.linear.power_W}{64.6}
\defval{gb10_qwen3.science_qa.min_successes}{30}
\defval{gb10_qwen3.science_qa.power_ratio}{0.85}
\defval{gb10_qwen3.science_qa.power_ratio_hi}{0.87}
\defval{gb10_qwen3.science_qa.power_ratio_lo}{0.84}
\defval{gb10_qwen3.science_qa.share_energy}{1.00}
\defval{gb10_qwen3.science_qa.success_ratio}{1.00}
\defval{gb10_qwen3.science_qa.time_ratio}{46.27}
\defval{gb10_qwen3.science_qa.time_ratio_hi}{51.73}
\defval{gb10_qwen3.science_qa.time_ratio_lo}{41.71}
\defval{gb10_qwen3.tg_sequential_2.OOI_EpCG}{1.62}
\defval{gb10_qwen3.tg_sequential_2.OOI_EpCG_hi}{1.72}
\defval{gb10_qwen3.tg_sequential_2.OOI_EpCG_lo}{1.53}
\defval{gb10_qwen3.tg_sequential_2.OOI_EpG}{1.62}
\defval{gb10_qwen3.tg_sequential_2.OOI_EpG_hi}{1.72}
\defval{gb10_qwen3.tg_sequential_2.OOI_EpG_lo}{1.53}
\defval{gb10_qwen3.tg_sequential_2.OOI_EpSG}{undefined}
\defval{gb10_qwen3.tg_sequential_2.OOI_EpSG_hi}{undefined}
\defval{gb10_qwen3.tg_sequential_2.OOI_EpSG_lo}{undefined}
\defval{gb10_qwen3.tg_sequential_2.agentic.E_wrong_goals_J}{0}
\defval{gb10_qwen3.tg_sequential_2.agentic.EpCG_J}{923.4}
\defval{gb10_qwen3.tg_sequential_2.agentic.EpG_J}{923.4}
\defval{gb10_qwen3.tg_sequential_2.agentic.EpSG_J}{923.4}
\defval{gb10_qwen3.tg_sequential_2.agentic.N}{30}
\defval{gb10_qwen3.tg_sequential_2.agentic.TpG_s}{16.1}
\defval{gb10_qwen3.tg_sequential_2.agentic.attempts_per_goal}{1.00}
\defval{gb10_qwen3.tg_sequential_2.agentic.n_completed}{30}
\defval{gb10_qwen3.tg_sequential_2.agentic.n_correct}{30}
\defval{gb10_qwen3.tg_sequential_2.agentic.p_c}{1.00}
\defval{gb10_qwen3.tg_sequential_2.agentic.p_s}{1.00}
\defval{gb10_qwen3.tg_sequential_2.agentic.power_W}{57.5}
\defval{gb10_qwen3.tg_sequential_2.energy_ratio}{1.62}
\defval{gb10_qwen3.tg_sequential_2.linear.E_wrong_goals_J}{17,142}
\defval{gb10_qwen3.tg_sequential_2.linear.EpCG_J}{571.4}
\defval{gb10_qwen3.tg_sequential_2.linear.EpG_J}{571.4}
\defval{gb10_qwen3.tg_sequential_2.linear.EpSG_J}{undefined}
\defval{gb10_qwen3.tg_sequential_2.linear.N}{30}
\defval{gb10_qwen3.tg_sequential_2.linear.TpG_s}{10.2}
\defval{gb10_qwen3.tg_sequential_2.linear.attempts_per_goal}{1.00}
\defval{gb10_qwen3.tg_sequential_2.linear.n_completed}{30}
\defval{gb10_qwen3.tg_sequential_2.linear.n_correct}{0}
\defval{gb10_qwen3.tg_sequential_2.linear.p_c}{1.00}
\defval{gb10_qwen3.tg_sequential_2.linear.p_s}{0.00}
\defval{gb10_qwen3.tg_sequential_2.linear.power_W}{55.9}
\defval{gb10_qwen3.tg_sequential_2.min_successes}{0}
\defval{gb10_qwen3.tg_sequential_2.power_ratio}{1.03}
\defval{gb10_qwen3.tg_sequential_2.power_ratio_hi}{1.04}
\defval{gb10_qwen3.tg_sequential_2.power_ratio_lo}{1.02}
\defval{gb10_qwen3.tg_sequential_2.share_energy}{undefined}
\defval{gb10_qwen3.tg_sequential_2.success_ratio}{0.00}
\defval{gb10_qwen3.tg_sequential_2.time_ratio}{1.57}
\defval{gb10_qwen3.tg_sequential_2.time_ratio_hi}{1.68}
\defval{gb10_qwen3.tg_sequential_2.time_ratio_lo}{1.49}
\defval{gb10_qwen3.tg_single_calc.OOI_EpCG}{1.43}
\defval{gb10_qwen3.tg_single_calc.OOI_EpCG_hi}{1.54}
\defval{gb10_qwen3.tg_single_calc.OOI_EpCG_lo}{1.34}
\defval{gb10_qwen3.tg_single_calc.OOI_EpG}{1.43}
\defval{gb10_qwen3.tg_single_calc.OOI_EpG_hi}{1.54}
\defval{gb10_qwen3.tg_single_calc.OOI_EpG_lo}{1.34}
\defval{gb10_qwen3.tg_single_calc.OOI_EpSG}{undefined}
\defval{gb10_qwen3.tg_single_calc.OOI_EpSG_hi}{undefined}
\defval{gb10_qwen3.tg_single_calc.OOI_EpSG_lo}{undefined}
\defval{gb10_qwen3.tg_single_calc.agentic.E_wrong_goals_J}{0}
\defval{gb10_qwen3.tg_single_calc.agentic.EpCG_J}{390.9}
\defval{gb10_qwen3.tg_single_calc.agentic.EpG_J}{390.9}
\defval{gb10_qwen3.tg_single_calc.agentic.EpSG_J}{390.9}
\defval{gb10_qwen3.tg_single_calc.agentic.N}{30}
\defval{gb10_qwen3.tg_single_calc.agentic.TpG_s}{7.1}
\defval{gb10_qwen3.tg_single_calc.agentic.attempts_per_goal}{1.00}
\defval{gb10_qwen3.tg_single_calc.agentic.n_completed}{30}
\defval{gb10_qwen3.tg_single_calc.agentic.n_correct}{30}
\defval{gb10_qwen3.tg_single_calc.agentic.p_c}{1.00}
\defval{gb10_qwen3.tg_single_calc.agentic.p_s}{1.00}
\defval{gb10_qwen3.tg_single_calc.agentic.power_W}{55.3}
\defval{gb10_qwen3.tg_single_calc.energy_ratio}{1.43}
\defval{gb10_qwen3.tg_single_calc.linear.E_wrong_goals_J}{8,186}
\defval{gb10_qwen3.tg_single_calc.linear.EpCG_J}{272.9}
\defval{gb10_qwen3.tg_single_calc.linear.EpG_J}{272.9}
\defval{gb10_qwen3.tg_single_calc.linear.EpSG_J}{undefined}
\defval{gb10_qwen3.tg_single_calc.linear.N}{30}
\defval{gb10_qwen3.tg_single_calc.linear.TpG_s}{5.0}
\defval{gb10_qwen3.tg_single_calc.linear.attempts_per_goal}{1.00}
\defval{gb10_qwen3.tg_single_calc.linear.n_completed}{30}
\defval{gb10_qwen3.tg_single_calc.linear.n_correct}{0}
\defval{gb10_qwen3.tg_single_calc.linear.p_c}{1.00}
\defval{gb10_qwen3.tg_single_calc.linear.p_s}{0.00}
\defval{gb10_qwen3.tg_single_calc.linear.power_W}{54.4}
\defval{gb10_qwen3.tg_single_calc.min_successes}{0}
\defval{gb10_qwen3.tg_single_calc.power_ratio}{1.02}
\defval{gb10_qwen3.tg_single_calc.power_ratio_hi}{1.03}
\defval{gb10_qwen3.tg_single_calc.power_ratio_lo}{1.01}
\defval{gb10_qwen3.tg_single_calc.share_energy}{undefined}
\defval{gb10_qwen3.tg_single_calc.success_ratio}{0.00}
\defval{gb10_qwen3.tg_single_calc.time_ratio}{1.41}
\defval{gb10_qwen3.tg_single_calc.time_ratio_hi}{1.51}
\defval{gb10_qwen3.tg_single_calc.time_ratio_lo}{1.32}
\defval{gb10_qwen3.tg_single_db.OOI_EpCG}{2.16}
\defval{gb10_qwen3.tg_single_db.OOI_EpCG_hi}{2.37}
\defval{gb10_qwen3.tg_single_db.OOI_EpCG_lo}{1.98}
\defval{gb10_qwen3.tg_single_db.OOI_EpG}{2.16}
\defval{gb10_qwen3.tg_single_db.OOI_EpG_hi}{2.37}
\defval{gb10_qwen3.tg_single_db.OOI_EpG_lo}{1.98}
\defval{gb10_qwen3.tg_single_db.OOI_EpSG}{undefined}
\defval{gb10_qwen3.tg_single_db.OOI_EpSG_hi}{undefined}
\defval{gb10_qwen3.tg_single_db.OOI_EpSG_lo}{undefined}
\defval{gb10_qwen3.tg_single_db.agentic.E_wrong_goals_J}{0}
\defval{gb10_qwen3.tg_single_db.agentic.EpCG_J}{530.1}
\defval{gb10_qwen3.tg_single_db.agentic.EpG_J}{530.1}
\defval{gb10_qwen3.tg_single_db.agentic.EpSG_J}{530.1}
\defval{gb10_qwen3.tg_single_db.agentic.N}{30}
\defval{gb10_qwen3.tg_single_db.agentic.TpG_s}{9.6}
\defval{gb10_qwen3.tg_single_db.agentic.attempts_per_goal}{1.00}
\defval{gb10_qwen3.tg_single_db.agentic.n_completed}{30}
\defval{gb10_qwen3.tg_single_db.agentic.n_correct}{30}
\defval{gb10_qwen3.tg_single_db.agentic.p_c}{1.00}
\defval{gb10_qwen3.tg_single_db.agentic.p_s}{1.00}
\defval{gb10_qwen3.tg_single_db.agentic.power_W}{55.0}
\defval{gb10_qwen3.tg_single_db.energy_ratio}{2.16}
\defval{gb10_qwen3.tg_single_db.linear.E_wrong_goals_J}{7,358}
\defval{gb10_qwen3.tg_single_db.linear.EpCG_J}{245.3}
\defval{gb10_qwen3.tg_single_db.linear.EpG_J}{245.3}
\defval{gb10_qwen3.tg_single_db.linear.EpSG_J}{undefined}
\defval{gb10_qwen3.tg_single_db.linear.N}{30}
\defval{gb10_qwen3.tg_single_db.linear.TpG_s}{4.5}
\defval{gb10_qwen3.tg_single_db.linear.attempts_per_goal}{1.00}
\defval{gb10_qwen3.tg_single_db.linear.n_completed}{30}
\defval{gb10_qwen3.tg_single_db.linear.n_correct}{0}
\defval{gb10_qwen3.tg_single_db.linear.p_c}{1.00}
\defval{gb10_qwen3.tg_single_db.linear.p_s}{0.00}
\defval{gb10_qwen3.tg_single_db.linear.power_W}{54.3}
\defval{gb10_qwen3.tg_single_db.min_successes}{0}
\defval{gb10_qwen3.tg_single_db.power_ratio}{1.01}
\defval{gb10_qwen3.tg_single_db.power_ratio_hi}{1.02}
\defval{gb10_qwen3.tg_single_db.power_ratio_lo}{1.01}
\defval{gb10_qwen3.tg_single_db.share_energy}{undefined}
\defval{gb10_qwen3.tg_single_db.success_ratio}{0.00}
\defval{gb10_qwen3.tg_single_db.time_ratio}{2.13}
\defval{gb10_qwen3.tg_single_db.time_ratio_hi}{2.34}
\defval{gb10_qwen3.tg_single_db.time_ratio_lo}{1.95}
\defval{gb10_qwen3_breadth.english_to_hindi.OOI_EpCG}{3.45}
\defval{gb10_qwen3_breadth.english_to_hindi.OOI_EpCG_hi}{4.42}
\defval{gb10_qwen3_breadth.english_to_hindi.OOI_EpCG_lo}{2.77}
\defval{gb10_qwen3_breadth.english_to_hindi.OOI_EpG}{3.45}
\defval{gb10_qwen3_breadth.english_to_hindi.OOI_EpG_hi}{4.42}
\defval{gb10_qwen3_breadth.english_to_hindi.OOI_EpG_lo}{2.77}
\defval{gb10_qwen3_breadth.english_to_hindi.OOI_EpSG}{3.63}
\defval{gb10_qwen3_breadth.english_to_hindi.OOI_EpSG_hi}{5.05}
\defval{gb10_qwen3_breadth.english_to_hindi.OOI_EpSG_lo}{2.77}
\defval{gb10_qwen3_breadth.english_to_hindi.agentic.E_wrong_goals_J}{643}
\defval{gb10_qwen3_breadth.english_to_hindi.agentic.EpCG_J}{274.4}
\defval{gb10_qwen3_breadth.english_to_hindi.agentic.EpG_J}{274.4}
\defval{gb10_qwen3_breadth.english_to_hindi.agentic.EpSG_J}{288.8}
\defval{gb10_qwen3_breadth.english_to_hindi.agentic.N}{20}
\defval{gb10_qwen3_breadth.english_to_hindi.agentic.TpG_s}{5.8}
\defval{gb10_qwen3_breadth.english_to_hindi.agentic.attempts_per_goal}{1.00}
\defval{gb10_qwen3_breadth.english_to_hindi.agentic.n_completed}{20}
\defval{gb10_qwen3_breadth.english_to_hindi.agentic.n_correct}{19}
\defval{gb10_qwen3_breadth.english_to_hindi.agentic.p_c}{1.00}
\defval{gb10_qwen3_breadth.english_to_hindi.agentic.p_s}{0.95}
\defval{gb10_qwen3_breadth.english_to_hindi.agentic.power_W}{47.5}
\defval{gb10_qwen3_breadth.english_to_hindi.energy_ratio}{3.45}
\defval{gb10_qwen3_breadth.english_to_hindi.linear.E_wrong_goals_J}{0}
\defval{gb10_qwen3_breadth.english_to_hindi.linear.EpCG_J}{79.5}
\defval{gb10_qwen3_breadth.english_to_hindi.linear.EpG_J}{79.5}
\defval{gb10_qwen3_breadth.english_to_hindi.linear.EpSG_J}{79.5}
\defval{gb10_qwen3_breadth.english_to_hindi.linear.N}{20}
\defval{gb10_qwen3_breadth.english_to_hindi.linear.TpG_s}{1.6}
\defval{gb10_qwen3_breadth.english_to_hindi.linear.attempts_per_goal}{1.00}
\defval{gb10_qwen3_breadth.english_to_hindi.linear.n_completed}{20}
\defval{gb10_qwen3_breadth.english_to_hindi.linear.n_correct}{20}
\defval{gb10_qwen3_breadth.english_to_hindi.linear.p_c}{1.00}
\defval{gb10_qwen3_breadth.english_to_hindi.linear.p_s}{1.00}
\defval{gb10_qwen3_breadth.english_to_hindi.linear.power_W}{48.7}
\defval{gb10_qwen3_breadth.english_to_hindi.min_successes}{19}
\defval{gb10_qwen3_breadth.english_to_hindi.power_ratio}{0.97}
\defval{gb10_qwen3_breadth.english_to_hindi.power_ratio_hi}{1.00}
\defval{gb10_qwen3_breadth.english_to_hindi.power_ratio_lo}{0.95}
\defval{gb10_qwen3_breadth.english_to_hindi.share_energy}{0.96}
\defval{gb10_qwen3_breadth.english_to_hindi.success_ratio}{1.05}
\defval{gb10_qwen3_breadth.english_to_hindi.time_ratio}{3.54}
\defval{gb10_qwen3_breadth.english_to_hindi.time_ratio_hi}{4.57}
\defval{gb10_qwen3_breadth.english_to_hindi.time_ratio_lo}{2.85}
\defval{gb10_qwen3_breadth.excluded}{0}
\defval{gb10_qwen3_breadth.geography_qa.OOI_EpCG}{51.21}
\defval{gb10_qwen3_breadth.geography_qa.OOI_EpCG_hi}{56.38}
\defval{gb10_qwen3_breadth.geography_qa.OOI_EpCG_lo}{45.73}
\defval{gb10_qwen3_breadth.geography_qa.OOI_EpG}{51.21}
\defval{gb10_qwen3_breadth.geography_qa.OOI_EpG_hi}{56.38}
\defval{gb10_qwen3_breadth.geography_qa.OOI_EpG_lo}{45.73}
\defval{gb10_qwen3_breadth.geography_qa.OOI_EpSG}{51.21}
\defval{gb10_qwen3_breadth.geography_qa.OOI_EpSG_hi}{56.38}
\defval{gb10_qwen3_breadth.geography_qa.OOI_EpSG_lo}{45.73}
\defval{gb10_qwen3_breadth.geography_qa.agentic.E_wrong_goals_J}{0}
\defval{gb10_qwen3_breadth.geography_qa.agentic.EpCG_J}{840.4}
\defval{gb10_qwen3_breadth.geography_qa.agentic.EpG_J}{840.4}
\defval{gb10_qwen3_breadth.geography_qa.agentic.EpSG_J}{840.4}
\defval{gb10_qwen3_breadth.geography_qa.agentic.N}{20}
\defval{gb10_qwen3_breadth.geography_qa.agentic.TpG_s}{16.4}
\defval{gb10_qwen3_breadth.geography_qa.agentic.attempts_per_goal}{1.00}
\defval{gb10_qwen3_breadth.geography_qa.agentic.n_completed}{20}
\defval{gb10_qwen3_breadth.geography_qa.agentic.n_correct}{20}
\defval{gb10_qwen3_breadth.geography_qa.agentic.p_c}{1.00}
\defval{gb10_qwen3_breadth.geography_qa.agentic.p_s}{1.00}
\defval{gb10_qwen3_breadth.geography_qa.agentic.power_W}{51.2}
\defval{gb10_qwen3_breadth.geography_qa.energy_ratio}{51.21}
\defval{gb10_qwen3_breadth.geography_qa.linear.E_wrong_goals_J}{0}
\defval{gb10_qwen3_breadth.geography_qa.linear.EpCG_J}{16.4}
\defval{gb10_qwen3_breadth.geography_qa.linear.EpG_J}{16.4}
\defval{gb10_qwen3_breadth.geography_qa.linear.EpSG_J}{16.4}
\defval{gb10_qwen3_breadth.geography_qa.linear.N}{20}
\defval{gb10_qwen3_breadth.geography_qa.linear.TpG_s}{0.2}
\defval{gb10_qwen3_breadth.geography_qa.linear.attempts_per_goal}{1.00}
\defval{gb10_qwen3_breadth.geography_qa.linear.n_completed}{20}
\defval{gb10_qwen3_breadth.geography_qa.linear.n_correct}{20}
\defval{gb10_qwen3_breadth.geography_qa.linear.p_c}{1.00}
\defval{gb10_qwen3_breadth.geography_qa.linear.p_s}{1.00}
\defval{gb10_qwen3_breadth.geography_qa.linear.power_W}{70.7}
\defval{gb10_qwen3_breadth.geography_qa.min_successes}{20}
\defval{gb10_qwen3_breadth.geography_qa.power_ratio}{0.72}
\defval{gb10_qwen3_breadth.geography_qa.power_ratio_hi}{0.79}
\defval{gb10_qwen3_breadth.geography_qa.power_ratio_lo}{0.67}
\defval{gb10_qwen3_breadth.geography_qa.share_energy}{1.00}
\defval{gb10_qwen3_breadth.geography_qa.success_ratio}{1.00}
\defval{gb10_qwen3_breadth.geography_qa.time_ratio}{70.66}
\defval{gb10_qwen3_breadth.geography_qa.time_ratio_hi}{77.87}
\defval{gb10_qwen3_breadth.geography_qa.time_ratio_lo}{62.42}
\defval{gb10_qwen3_breadth.goals}{680}
\defval{gb10_qwen3_breadth.keyword_extraction.OOI_EpCG}{35.33}
\defval{gb10_qwen3_breadth.keyword_extraction.OOI_EpCG_hi}{42.39}
\defval{gb10_qwen3_breadth.keyword_extraction.OOI_EpCG_lo}{31.13}
\defval{gb10_qwen3_breadth.keyword_extraction.OOI_EpG}{35.33}
\defval{gb10_qwen3_breadth.keyword_extraction.OOI_EpG_hi}{42.39}
\defval{gb10_qwen3_breadth.keyword_extraction.OOI_EpG_lo}{31.13}
\defval{gb10_qwen3_breadth.keyword_extraction.OOI_EpSG}{35.33}
\defval{gb10_qwen3_breadth.keyword_extraction.OOI_EpSG_hi}{42.39}
\defval{gb10_qwen3_breadth.keyword_extraction.OOI_EpSG_lo}{31.13}
\defval{gb10_qwen3_breadth.keyword_extraction.agentic.E_wrong_goals_J}{0}
\defval{gb10_qwen3_breadth.keyword_extraction.agentic.EpCG_J}{871.3}
\defval{gb10_qwen3_breadth.keyword_extraction.agentic.EpG_J}{871.3}
\defval{gb10_qwen3_breadth.keyword_extraction.agentic.EpSG_J}{871.3}
\defval{gb10_qwen3_breadth.keyword_extraction.agentic.N}{20}
\defval{gb10_qwen3_breadth.keyword_extraction.agentic.TpG_s}{19.0}
\defval{gb10_qwen3_breadth.keyword_extraction.agentic.attempts_per_goal}{1.00}
\defval{gb10_qwen3_breadth.keyword_extraction.agentic.n_completed}{20}
\defval{gb10_qwen3_breadth.keyword_extraction.agentic.n_correct}{20}
\defval{gb10_qwen3_breadth.keyword_extraction.agentic.p_c}{1.00}
\defval{gb10_qwen3_breadth.keyword_extraction.agentic.p_s}{1.00}
\defval{gb10_qwen3_breadth.keyword_extraction.agentic.power_W}{45.8}
\defval{gb10_qwen3_breadth.keyword_extraction.energy_ratio}{35.33}
\defval{gb10_qwen3_breadth.keyword_extraction.linear.E_wrong_goals_J}{0}
\defval{gb10_qwen3_breadth.keyword_extraction.linear.EpCG_J}{24.7}
\defval{gb10_qwen3_breadth.keyword_extraction.linear.EpG_J}{24.7}
\defval{gb10_qwen3_breadth.keyword_extraction.linear.EpSG_J}{24.7}
\defval{gb10_qwen3_breadth.keyword_extraction.linear.N}{20}
\defval{gb10_qwen3_breadth.keyword_extraction.linear.TpG_s}{0.5}
\defval{gb10_qwen3_breadth.keyword_extraction.linear.attempts_per_goal}{1.00}
\defval{gb10_qwen3_breadth.keyword_extraction.linear.n_completed}{20}
\defval{gb10_qwen3_breadth.keyword_extraction.linear.n_correct}{20}
\defval{gb10_qwen3_breadth.keyword_extraction.linear.p_c}{1.00}
\defval{gb10_qwen3_breadth.keyword_extraction.linear.p_s}{1.00}
\defval{gb10_qwen3_breadth.keyword_extraction.linear.power_W}{53.1}
\defval{gb10_qwen3_breadth.keyword_extraction.min_successes}{20}
\defval{gb10_qwen3_breadth.keyword_extraction.power_ratio}{0.86}
\defval{gb10_qwen3_breadth.keyword_extraction.power_ratio_hi}{0.89}
\defval{gb10_qwen3_breadth.keyword_extraction.power_ratio_lo}{0.84}
\defval{gb10_qwen3_breadth.keyword_extraction.share_energy}{1.00}
\defval{gb10_qwen3_breadth.keyword_extraction.success_ratio}{1.00}
\defval{gb10_qwen3_breadth.keyword_extraction.time_ratio}{41.00}
\defval{gb10_qwen3_breadth.keyword_extraction.time_ratio_hi}{49.25}
\defval{gb10_qwen3_breadth.keyword_extraction.time_ratio_lo}{36.29}
\defval{gb10_qwen3_breadth.literature_synthesis.OOI_EpCG}{undefined}
\defval{gb10_qwen3_breadth.literature_synthesis.OOI_EpCG_hi}{undefined}
\defval{gb10_qwen3_breadth.literature_synthesis.OOI_EpCG_lo}{undefined}
\defval{gb10_qwen3_breadth.literature_synthesis.OOI_EpG}{18.50}
\defval{gb10_qwen3_breadth.literature_synthesis.OOI_EpG_hi}{20.38}
\defval{gb10_qwen3_breadth.literature_synthesis.OOI_EpG_lo}{16.84}
\defval{gb10_qwen3_breadth.literature_synthesis.OOI_EpSG}{undefined}
\defval{gb10_qwen3_breadth.literature_synthesis.OOI_EpSG_hi}{undefined}
\defval{gb10_qwen3_breadth.literature_synthesis.OOI_EpSG_lo}{undefined}
\defval{gb10_qwen3_breadth.literature_synthesis.agentic.E_wrong_goals_J}{0}
\defval{gb10_qwen3_breadth.literature_synthesis.agentic.EpCG_J}{undefined}
\defval{gb10_qwen3_breadth.literature_synthesis.agentic.EpG_J}{14,961.4}
\defval{gb10_qwen3_breadth.literature_synthesis.agentic.EpSG_J}{undefined}
\defval{gb10_qwen3_breadth.literature_synthesis.agentic.N}{20}
\defval{gb10_qwen3_breadth.literature_synthesis.agentic.TpG_s}{288.5}
\defval{gb10_qwen3_breadth.literature_synthesis.agentic.attempts_per_goal}{6.00}
\defval{gb10_qwen3_breadth.literature_synthesis.agentic.n_completed}{0}
\defval{gb10_qwen3_breadth.literature_synthesis.agentic.n_correct}{0}
\defval{gb10_qwen3_breadth.literature_synthesis.agentic.p_c}{0.00}
\defval{gb10_qwen3_breadth.literature_synthesis.agentic.p_s}{0.00}
\defval{gb10_qwen3_breadth.literature_synthesis.agentic.power_W}{51.9}
\defval{gb10_qwen3_breadth.literature_synthesis.energy_ratio}{18.50}
\defval{gb10_qwen3_breadth.literature_synthesis.linear.E_wrong_goals_J}{0}
\defval{gb10_qwen3_breadth.literature_synthesis.linear.EpCG_J}{808.6}
\defval{gb10_qwen3_breadth.literature_synthesis.linear.EpG_J}{808.6}
\defval{gb10_qwen3_breadth.literature_synthesis.linear.EpSG_J}{undefined}
\defval{gb10_qwen3_breadth.literature_synthesis.linear.N}{20}
\defval{gb10_qwen3_breadth.literature_synthesis.linear.TpG_s}{16.6}
\defval{gb10_qwen3_breadth.literature_synthesis.linear.attempts_per_goal}{1.00}
\defval{gb10_qwen3_breadth.literature_synthesis.linear.n_completed}{20}
\defval{gb10_qwen3_breadth.literature_synthesis.linear.n_correct}{0}
\defval{gb10_qwen3_breadth.literature_synthesis.linear.p_c}{1.00}
\defval{gb10_qwen3_breadth.literature_synthesis.linear.p_s}{undefined}
\defval{gb10_qwen3_breadth.literature_synthesis.linear.power_W}{48.8}
\defval{gb10_qwen3_breadth.literature_synthesis.min_successes}{0}
\defval{gb10_qwen3_breadth.literature_synthesis.power_ratio}{1.06}
\defval{gb10_qwen3_breadth.literature_synthesis.power_ratio_hi}{1.07}
\defval{gb10_qwen3_breadth.literature_synthesis.power_ratio_lo}{1.06}
\defval{gb10_qwen3_breadth.literature_synthesis.share_energy}{undefined}
\defval{gb10_qwen3_breadth.literature_synthesis.success_ratio}{undefined}
\defval{gb10_qwen3_breadth.literature_synthesis.time_ratio}{17.41}
\defval{gb10_qwen3_breadth.literature_synthesis.time_ratio_hi}{19.14}
\defval{gb10_qwen3_breadth.literature_synthesis.time_ratio_lo}{15.86}
\defval{gb10_qwen3_breadth.methodology_comparison.OOI_EpCG}{undefined}
\defval{gb10_qwen3_breadth.methodology_comparison.OOI_EpCG_hi}{undefined}
\defval{gb10_qwen3_breadth.methodology_comparison.OOI_EpCG_lo}{undefined}
\defval{gb10_qwen3_breadth.methodology_comparison.OOI_EpG}{34.53}
\defval{gb10_qwen3_breadth.methodology_comparison.OOI_EpG_hi}{36.84}
\defval{gb10_qwen3_breadth.methodology_comparison.OOI_EpG_lo}{32.30}
\defval{gb10_qwen3_breadth.methodology_comparison.OOI_EpSG}{undefined}
\defval{gb10_qwen3_breadth.methodology_comparison.OOI_EpSG_hi}{undefined}
\defval{gb10_qwen3_breadth.methodology_comparison.OOI_EpSG_lo}{undefined}
\defval{gb10_qwen3_breadth.methodology_comparison.agentic.E_wrong_goals_J}{0}
\defval{gb10_qwen3_breadth.methodology_comparison.agentic.EpCG_J}{undefined}
\defval{gb10_qwen3_breadth.methodology_comparison.agentic.EpG_J}{31,936.9}
\defval{gb10_qwen3_breadth.methodology_comparison.agentic.EpSG_J}{undefined}
\defval{gb10_qwen3_breadth.methodology_comparison.agentic.N}{20}
\defval{gb10_qwen3_breadth.methodology_comparison.agentic.TpG_s}{615.8}
\defval{gb10_qwen3_breadth.methodology_comparison.agentic.attempts_per_goal}{6.00}
\defval{gb10_qwen3_breadth.methodology_comparison.agentic.n_completed}{0}
\defval{gb10_qwen3_breadth.methodology_comparison.agentic.n_correct}{0}
\defval{gb10_qwen3_breadth.methodology_comparison.agentic.p_c}{0.00}
\defval{gb10_qwen3_breadth.methodology_comparison.agentic.p_s}{0.00}
\defval{gb10_qwen3_breadth.methodology_comparison.agentic.power_W}{51.9}
\defval{gb10_qwen3_breadth.methodology_comparison.energy_ratio}{34.53}
\defval{gb10_qwen3_breadth.methodology_comparison.linear.E_wrong_goals_J}{0}
\defval{gb10_qwen3_breadth.methodology_comparison.linear.EpCG_J}{924.8}
\defval{gb10_qwen3_breadth.methodology_comparison.linear.EpG_J}{924.8}
\defval{gb10_qwen3_breadth.methodology_comparison.linear.EpSG_J}{undefined}
\defval{gb10_qwen3_breadth.methodology_comparison.linear.N}{20}
\defval{gb10_qwen3_breadth.methodology_comparison.linear.TpG_s}{18.8}
\defval{gb10_qwen3_breadth.methodology_comparison.linear.attempts_per_goal}{1.00}
\defval{gb10_qwen3_breadth.methodology_comparison.linear.n_completed}{20}
\defval{gb10_qwen3_breadth.methodology_comparison.linear.n_correct}{0}
\defval{gb10_qwen3_breadth.methodology_comparison.linear.p_c}{1.00}
\defval{gb10_qwen3_breadth.methodology_comparison.linear.p_s}{undefined}
\defval{gb10_qwen3_breadth.methodology_comparison.linear.power_W}{49.3}
\defval{gb10_qwen3_breadth.methodology_comparison.min_successes}{0}
\defval{gb10_qwen3_breadth.methodology_comparison.power_ratio}{1.05}
\defval{gb10_qwen3_breadth.methodology_comparison.power_ratio_hi}{1.08}
\defval{gb10_qwen3_breadth.methodology_comparison.power_ratio_lo}{1.02}
\defval{gb10_qwen3_breadth.methodology_comparison.share_energy}{undefined}
\defval{gb10_qwen3_breadth.methodology_comparison.success_ratio}{undefined}
\defval{gb10_qwen3_breadth.methodology_comparison.time_ratio}{32.82}
\defval{gb10_qwen3_breadth.methodology_comparison.time_ratio_hi}{34.95}
\defval{gb10_qwen3_breadth.methodology_comparison.time_ratio_lo}{30.78}
\defval{gb10_qwen3_breadth.multilingual_summary.OOI_EpCG}{3.21}
\defval{gb10_qwen3_breadth.multilingual_summary.OOI_EpCG_hi}{3.36}
\defval{gb10_qwen3_breadth.multilingual_summary.OOI_EpCG_lo}{3.07}
\defval{gb10_qwen3_breadth.multilingual_summary.OOI_EpG}{3.21}
\defval{gb10_qwen3_breadth.multilingual_summary.OOI_EpG_hi}{3.36}
\defval{gb10_qwen3_breadth.multilingual_summary.OOI_EpG_lo}{3.07}
\defval{gb10_qwen3_breadth.multilingual_summary.OOI_EpSG}{3.21}
\defval{gb10_qwen3_breadth.multilingual_summary.OOI_EpSG_hi}{3.36}
\defval{gb10_qwen3_breadth.multilingual_summary.OOI_EpSG_lo}{3.07}
\defval{gb10_qwen3_breadth.multilingual_summary.agentic.E_wrong_goals_J}{0}
\defval{gb10_qwen3_breadth.multilingual_summary.agentic.EpCG_J}{964.4}
\defval{gb10_qwen3_breadth.multilingual_summary.agentic.EpG_J}{964.4}
\defval{gb10_qwen3_breadth.multilingual_summary.agentic.EpSG_J}{964.4}
\defval{gb10_qwen3_breadth.multilingual_summary.agentic.N}{20}
\defval{gb10_qwen3_breadth.multilingual_summary.agentic.TpG_s}{19.9}
\defval{gb10_qwen3_breadth.multilingual_summary.agentic.attempts_per_goal}{1.00}
\defval{gb10_qwen3_breadth.multilingual_summary.agentic.n_completed}{20}
\defval{gb10_qwen3_breadth.multilingual_summary.agentic.n_correct}{20}
\defval{gb10_qwen3_breadth.multilingual_summary.agentic.p_c}{1.00}
\defval{gb10_qwen3_breadth.multilingual_summary.agentic.p_s}{1.00}
\defval{gb10_qwen3_breadth.multilingual_summary.agentic.power_W}{48.4}
\defval{gb10_qwen3_breadth.multilingual_summary.energy_ratio}{3.21}
\defval{gb10_qwen3_breadth.multilingual_summary.linear.E_wrong_goals_J}{0}
\defval{gb10_qwen3_breadth.multilingual_summary.linear.EpCG_J}{300.0}
\defval{gb10_qwen3_breadth.multilingual_summary.linear.EpG_J}{300.0}
\defval{gb10_qwen3_breadth.multilingual_summary.linear.EpSG_J}{300.0}
\defval{gb10_qwen3_breadth.multilingual_summary.linear.N}{20}
\defval{gb10_qwen3_breadth.multilingual_summary.linear.TpG_s}{6.4}
\defval{gb10_qwen3_breadth.multilingual_summary.linear.attempts_per_goal}{1.00}
\defval{gb10_qwen3_breadth.multilingual_summary.linear.n_completed}{20}
\defval{gb10_qwen3_breadth.multilingual_summary.linear.n_correct}{20}
\defval{gb10_qwen3_breadth.multilingual_summary.linear.p_c}{1.00}
\defval{gb10_qwen3_breadth.multilingual_summary.linear.p_s}{1.00}
\defval{gb10_qwen3_breadth.multilingual_summary.linear.power_W}{46.6}
\defval{gb10_qwen3_breadth.multilingual_summary.min_successes}{20}
\defval{gb10_qwen3_breadth.multilingual_summary.power_ratio}{1.04}
\defval{gb10_qwen3_breadth.multilingual_summary.power_ratio_hi}{1.05}
\defval{gb10_qwen3_breadth.multilingual_summary.power_ratio_lo}{1.03}
\defval{gb10_qwen3_breadth.multilingual_summary.share_energy}{1.00}
\defval{gb10_qwen3_breadth.multilingual_summary.success_ratio}{1.00}
\defval{gb10_qwen3_breadth.multilingual_summary.time_ratio}{3.10}
\defval{gb10_qwen3_breadth.multilingual_summary.time_ratio_hi}{3.23}
\defval{gb10_qwen3_breadth.multilingual_summary.time_ratio_lo}{2.96}
\defval{gb10_qwen3_breadth.news_summary.OOI_EpCG}{9.31}
\defval{gb10_qwen3_breadth.news_summary.OOI_EpCG_hi}{9.62}
\defval{gb10_qwen3_breadth.news_summary.OOI_EpCG_lo}{9.03}
\defval{gb10_qwen3_breadth.news_summary.OOI_EpG}{9.31}
\defval{gb10_qwen3_breadth.news_summary.OOI_EpG_hi}{9.62}
\defval{gb10_qwen3_breadth.news_summary.OOI_EpG_lo}{9.03}
\defval{gb10_qwen3_breadth.news_summary.OOI_EpSG}{10.95}
\defval{gb10_qwen3_breadth.news_summary.OOI_EpSG_hi}{13.74}
\defval{gb10_qwen3_breadth.news_summary.OOI_EpSG_lo}{9.37}
\defval{gb10_qwen3_breadth.news_summary.agentic.E_wrong_goals_J}{1,164}
\defval{gb10_qwen3_breadth.news_summary.agentic.EpCG_J}{383.6}
\defval{gb10_qwen3_breadth.news_summary.agentic.EpG_J}{383.6}
\defval{gb10_qwen3_breadth.news_summary.agentic.EpSG_J}{451.3}
\defval{gb10_qwen3_breadth.news_summary.agentic.N}{20}
\defval{gb10_qwen3_breadth.news_summary.agentic.TpG_s}{8.4}
\defval{gb10_qwen3_breadth.news_summary.agentic.attempts_per_goal}{1.00}
\defval{gb10_qwen3_breadth.news_summary.agentic.n_completed}{20}
\defval{gb10_qwen3_breadth.news_summary.agentic.n_correct}{17}
\defval{gb10_qwen3_breadth.news_summary.agentic.p_c}{1.00}
\defval{gb10_qwen3_breadth.news_summary.agentic.p_s}{0.85}
\defval{gb10_qwen3_breadth.news_summary.agentic.power_W}{45.6}
\defval{gb10_qwen3_breadth.news_summary.energy_ratio}{9.31}
\defval{gb10_qwen3_breadth.news_summary.linear.E_wrong_goals_J}{0}
\defval{gb10_qwen3_breadth.news_summary.linear.EpCG_J}{41.2}
\defval{gb10_qwen3_breadth.news_summary.linear.EpG_J}{41.2}
\defval{gb10_qwen3_breadth.news_summary.linear.EpSG_J}{41.2}
\defval{gb10_qwen3_breadth.news_summary.linear.N}{20}
\defval{gb10_qwen3_breadth.news_summary.linear.TpG_s}{0.8}
\defval{gb10_qwen3_breadth.news_summary.linear.attempts_per_goal}{1.00}
\defval{gb10_qwen3_breadth.news_summary.linear.n_completed}{20}
\defval{gb10_qwen3_breadth.news_summary.linear.n_correct}{20}
\defval{gb10_qwen3_breadth.news_summary.linear.p_c}{1.00}
\defval{gb10_qwen3_breadth.news_summary.linear.p_s}{1.00}
\defval{gb10_qwen3_breadth.news_summary.linear.power_W}{49.4}
\defval{gb10_qwen3_breadth.news_summary.min_successes}{17}
\defval{gb10_qwen3_breadth.news_summary.power_ratio}{0.92}
\defval{gb10_qwen3_breadth.news_summary.power_ratio_hi}{0.94}
\defval{gb10_qwen3_breadth.news_summary.power_ratio_lo}{0.91}
\defval{gb10_qwen3_breadth.news_summary.share_energy}{0.93}
\defval{gb10_qwen3_breadth.news_summary.success_ratio}{1.18}
\defval{gb10_qwen3_breadth.news_summary.time_ratio}{10.09}
\defval{gb10_qwen3_breadth.news_summary.time_ratio_hi}{10.34}
\defval{gb10_qwen3_breadth.news_summary.time_ratio_lo}{9.83}
\defval{gb10_qwen3_breadth.research_abstract.OOI_EpCG}{6.48}
\defval{gb10_qwen3_breadth.research_abstract.OOI_EpCG_hi}{7.05}
\defval{gb10_qwen3_breadth.research_abstract.OOI_EpCG_lo}{5.95}
\defval{gb10_qwen3_breadth.research_abstract.OOI_EpG}{6.48}
\defval{gb10_qwen3_breadth.research_abstract.OOI_EpG_hi}{7.05}
\defval{gb10_qwen3_breadth.research_abstract.OOI_EpG_lo}{5.95}
\defval{gb10_qwen3_breadth.research_abstract.OOI_EpSG}{2.59}
\defval{gb10_qwen3_breadth.research_abstract.OOI_EpSG_hi}{undefined}
\defval{gb10_qwen3_breadth.research_abstract.OOI_EpSG_lo}{undefined}
\defval{gb10_qwen3_breadth.research_abstract.agentic.E_wrong_goals_J}{17,226}
\defval{gb10_qwen3_breadth.research_abstract.agentic.EpCG_J}{1,131.3}
\defval{gb10_qwen3_breadth.research_abstract.agentic.EpG_J}{1,131.3}
\defval{gb10_qwen3_breadth.research_abstract.agentic.EpSG_J}{4,525.1}
\defval{gb10_qwen3_breadth.research_abstract.agentic.N}{20}
\defval{gb10_qwen3_breadth.research_abstract.agentic.TpG_s}{20.1}
\defval{gb10_qwen3_breadth.research_abstract.agentic.attempts_per_goal}{1.00}
\defval{gb10_qwen3_breadth.research_abstract.agentic.n_completed}{20}
\defval{gb10_qwen3_breadth.research_abstract.agentic.n_correct}{5}
\defval{gb10_qwen3_breadth.research_abstract.agentic.p_c}{1.00}
\defval{gb10_qwen3_breadth.research_abstract.agentic.p_s}{0.25}
\defval{gb10_qwen3_breadth.research_abstract.agentic.power_W}{56.3}
\defval{gb10_qwen3_breadth.research_abstract.energy_ratio}{6.48}
\defval{gb10_qwen3_breadth.research_abstract.linear.E_wrong_goals_J}{3,139}
\defval{gb10_qwen3_breadth.research_abstract.linear.EpCG_J}{174.5}
\defval{gb10_qwen3_breadth.research_abstract.linear.EpG_J}{174.5}
\defval{gb10_qwen3_breadth.research_abstract.linear.EpSG_J}{1,744.9}
\defval{gb10_qwen3_breadth.research_abstract.linear.N}{20}
\defval{gb10_qwen3_breadth.research_abstract.linear.TpG_s}{3.1}
\defval{gb10_qwen3_breadth.research_abstract.linear.attempts_per_goal}{1.00}
\defval{gb10_qwen3_breadth.research_abstract.linear.n_completed}{20}
\defval{gb10_qwen3_breadth.research_abstract.linear.n_correct}{2}
\defval{gb10_qwen3_breadth.research_abstract.linear.p_c}{1.00}
\defval{gb10_qwen3_breadth.research_abstract.linear.p_s}{0.10}
\defval{gb10_qwen3_breadth.research_abstract.linear.power_W}{56.4}
\defval{gb10_qwen3_breadth.research_abstract.min_successes}{2}
\defval{gb10_qwen3_breadth.research_abstract.power_ratio}{1.00}
\defval{gb10_qwen3_breadth.research_abstract.power_ratio_hi}{1.07}
\defval{gb10_qwen3_breadth.research_abstract.power_ratio_lo}{0.92}
\defval{gb10_qwen3_breadth.research_abstract.share_energy}{1.96}
\defval{gb10_qwen3_breadth.research_abstract.success_ratio}{0.40}
\defval{gb10_qwen3_breadth.research_abstract.time_ratio}{6.50}
\defval{gb10_qwen3_breadth.research_abstract.time_ratio_hi}{6.66}
\defval{gb10_qwen3_breadth.research_abstract.time_ratio_lo}{6.36}
\defval{gb10_qwen3_breadth.research_summary.OOI_EpCG}{26.65}
\defval{gb10_qwen3_breadth.research_summary.OOI_EpCG_hi}{27.32}
\defval{gb10_qwen3_breadth.research_summary.OOI_EpCG_lo}{26.10}
\defval{gb10_qwen3_breadth.research_summary.OOI_EpG}{26.65}
\defval{gb10_qwen3_breadth.research_summary.OOI_EpG_hi}{27.32}
\defval{gb10_qwen3_breadth.research_summary.OOI_EpG_lo}{26.10}
\defval{gb10_qwen3_breadth.research_summary.OOI_EpSG}{26.65}
\defval{gb10_qwen3_breadth.research_summary.OOI_EpSG_hi}{27.32}
\defval{gb10_qwen3_breadth.research_summary.OOI_EpSG_lo}{26.10}
\defval{gb10_qwen3_breadth.research_summary.agentic.E_wrong_goals_J}{0}
\defval{gb10_qwen3_breadth.research_summary.agentic.EpCG_J}{777.7}
\defval{gb10_qwen3_breadth.research_summary.agentic.EpG_J}{777.7}
\defval{gb10_qwen3_breadth.research_summary.agentic.EpSG_J}{777.7}
\defval{gb10_qwen3_breadth.research_summary.agentic.N}{20}
\defval{gb10_qwen3_breadth.research_summary.agentic.TpG_s}{16.9}
\defval{gb10_qwen3_breadth.research_summary.agentic.attempts_per_goal}{1.00}
\defval{gb10_qwen3_breadth.research_summary.agentic.n_completed}{20}
\defval{gb10_qwen3_breadth.research_summary.agentic.n_correct}{20}
\defval{gb10_qwen3_breadth.research_summary.agentic.p_c}{1.00}
\defval{gb10_qwen3_breadth.research_summary.agentic.p_s}{1.00}
\defval{gb10_qwen3_breadth.research_summary.agentic.power_W}{46.0}
\defval{gb10_qwen3_breadth.research_summary.energy_ratio}{26.65}
\defval{gb10_qwen3_breadth.research_summary.linear.E_wrong_goals_J}{0}
\defval{gb10_qwen3_breadth.research_summary.linear.EpCG_J}{29.2}
\defval{gb10_qwen3_breadth.research_summary.linear.EpG_J}{29.2}
\defval{gb10_qwen3_breadth.research_summary.linear.EpSG_J}{29.2}
\defval{gb10_qwen3_breadth.research_summary.linear.N}{20}
\defval{gb10_qwen3_breadth.research_summary.linear.TpG_s}{0.6}
\defval{gb10_qwen3_breadth.research_summary.linear.attempts_per_goal}{1.00}
\defval{gb10_qwen3_breadth.research_summary.linear.n_completed}{20}
\defval{gb10_qwen3_breadth.research_summary.linear.n_correct}{20}
\defval{gb10_qwen3_breadth.research_summary.linear.p_c}{1.00}
\defval{gb10_qwen3_breadth.research_summary.linear.p_s}{1.00}
\defval{gb10_qwen3_breadth.research_summary.linear.power_W}{51.8}
\defval{gb10_qwen3_breadth.research_summary.min_successes}{20}
\defval{gb10_qwen3_breadth.research_summary.power_ratio}{0.89}
\defval{gb10_qwen3_breadth.research_summary.power_ratio_hi}{0.91}
\defval{gb10_qwen3_breadth.research_summary.power_ratio_lo}{0.87}
\defval{gb10_qwen3_breadth.research_summary.share_energy}{1.00}
\defval{gb10_qwen3_breadth.research_summary.success_ratio}{1.00}
\defval{gb10_qwen3_breadth.research_summary.time_ratio}{29.99}
\defval{gb10_qwen3_breadth.research_summary.time_ratio_hi}{30.47}
\defval{gb10_qwen3_breadth.research_summary.time_ratio_lo}{29.54}
\defval{gb10_qwen3_breadth.sentiment_analysis.OOI_EpCG}{11.33}
\defval{gb10_qwen3_breadth.sentiment_analysis.OOI_EpCG_hi}{13.41}
\defval{gb10_qwen3_breadth.sentiment_analysis.OOI_EpCG_lo}{9.82}
\defval{gb10_qwen3_breadth.sentiment_analysis.OOI_EpG}{11.33}
\defval{gb10_qwen3_breadth.sentiment_analysis.OOI_EpG_hi}{13.41}
\defval{gb10_qwen3_breadth.sentiment_analysis.OOI_EpG_lo}{9.82}
\defval{gb10_qwen3_breadth.sentiment_analysis.OOI_EpSG}{11.33}
\defval{gb10_qwen3_breadth.sentiment_analysis.OOI_EpSG_hi}{13.41}
\defval{gb10_qwen3_breadth.sentiment_analysis.OOI_EpSG_lo}{9.82}
\defval{gb10_qwen3_breadth.sentiment_analysis.agentic.E_wrong_goals_J}{0}
\defval{gb10_qwen3_breadth.sentiment_analysis.agentic.EpCG_J}{222.6}
\defval{gb10_qwen3_breadth.sentiment_analysis.agentic.EpG_J}{222.6}
\defval{gb10_qwen3_breadth.sentiment_analysis.agentic.EpSG_J}{222.6}
\defval{gb10_qwen3_breadth.sentiment_analysis.agentic.N}{20}
\defval{gb10_qwen3_breadth.sentiment_analysis.agentic.TpG_s}{5.0}
\defval{gb10_qwen3_breadth.sentiment_analysis.agentic.attempts_per_goal}{1.00}
\defval{gb10_qwen3_breadth.sentiment_analysis.agentic.n_completed}{20}
\defval{gb10_qwen3_breadth.sentiment_analysis.agentic.n_correct}{20}
\defval{gb10_qwen3_breadth.sentiment_analysis.agentic.p_c}{1.00}
\defval{gb10_qwen3_breadth.sentiment_analysis.agentic.p_s}{1.00}
\defval{gb10_qwen3_breadth.sentiment_analysis.agentic.power_W}{44.6}
\defval{gb10_qwen3_breadth.sentiment_analysis.energy_ratio}{11.33}
\defval{gb10_qwen3_breadth.sentiment_analysis.linear.E_wrong_goals_J}{0}
\defval{gb10_qwen3_breadth.sentiment_analysis.linear.EpCG_J}{19.7}
\defval{gb10_qwen3_breadth.sentiment_analysis.linear.EpG_J}{19.7}
\defval{gb10_qwen3_breadth.sentiment_analysis.linear.EpSG_J}{19.7}
\defval{gb10_qwen3_breadth.sentiment_analysis.linear.N}{20}
\defval{gb10_qwen3_breadth.sentiment_analysis.linear.TpG_s}{0.4}
\defval{gb10_qwen3_breadth.sentiment_analysis.linear.attempts_per_goal}{1.00}
\defval{gb10_qwen3_breadth.sentiment_analysis.linear.n_completed}{20}
\defval{gb10_qwen3_breadth.sentiment_analysis.linear.n_correct}{20}
\defval{gb10_qwen3_breadth.sentiment_analysis.linear.p_c}{1.00}
\defval{gb10_qwen3_breadth.sentiment_analysis.linear.p_s}{1.00}
\defval{gb10_qwen3_breadth.sentiment_analysis.linear.power_W}{54.4}
\defval{gb10_qwen3_breadth.sentiment_analysis.min_successes}{20}
\defval{gb10_qwen3_breadth.sentiment_analysis.power_ratio}{0.82}
\defval{gb10_qwen3_breadth.sentiment_analysis.power_ratio_hi}{0.85}
\defval{gb10_qwen3_breadth.sentiment_analysis.power_ratio_lo}{0.78}
\defval{gb10_qwen3_breadth.sentiment_analysis.share_energy}{1.00}
\defval{gb10_qwen3_breadth.sentiment_analysis.success_ratio}{1.00}
\defval{gb10_qwen3_breadth.sentiment_analysis.time_ratio}{13.82}
\defval{gb10_qwen3_breadth.sentiment_analysis.time_ratio_hi}{17.27}
\defval{gb10_qwen3_breadth.sentiment_analysis.time_ratio_lo}{11.60}
\defval{gb10_qwen3_breadth.t3_creative.OOI_EpCG}{28.64}
\defval{gb10_qwen3_breadth.t3_creative.OOI_EpCG_hi}{29.55}
\defval{gb10_qwen3_breadth.t3_creative.OOI_EpCG_lo}{27.83}
\defval{gb10_qwen3_breadth.t3_creative.OOI_EpG}{28.64}
\defval{gb10_qwen3_breadth.t3_creative.OOI_EpG_hi}{29.55}
\defval{gb10_qwen3_breadth.t3_creative.OOI_EpG_lo}{27.83}
\defval{gb10_qwen3_breadth.t3_creative.OOI_EpSG}{57.28}
\defval{gb10_qwen3_breadth.t3_creative.OOI_EpSG_hi}{258.84}
\defval{gb10_qwen3_breadth.t3_creative.OOI_EpSG_lo}{22.33}
\defval{gb10_qwen3_breadth.t3_creative.agentic.E_wrong_goals_J}{56,483}
\defval{gb10_qwen3_breadth.t3_creative.agentic.EpCG_J}{3,505.6}
\defval{gb10_qwen3_breadth.t3_creative.agentic.EpG_J}{3,505.6}
\defval{gb10_qwen3_breadth.t3_creative.agentic.EpSG_J}{17,528.2}
\defval{gb10_qwen3_breadth.t3_creative.agentic.N}{20}
\defval{gb10_qwen3_breadth.t3_creative.agentic.TpG_s}{69.8}
\defval{gb10_qwen3_breadth.t3_creative.agentic.attempts_per_goal}{1.00}
\defval{gb10_qwen3_breadth.t3_creative.agentic.n_completed}{20}
\defval{gb10_qwen3_breadth.t3_creative.agentic.n_correct}{4}
\defval{gb10_qwen3_breadth.t3_creative.agentic.p_c}{1.00}
\defval{gb10_qwen3_breadth.t3_creative.agentic.p_s}{0.20}
\defval{gb10_qwen3_breadth.t3_creative.agentic.power_W}{50.2}
\defval{gb10_qwen3_breadth.t3_creative.energy_ratio}{28.64}
\defval{gb10_qwen3_breadth.t3_creative.linear.E_wrong_goals_J}{1,513}
\defval{gb10_qwen3_breadth.t3_creative.linear.EpCG_J}{122.4}
\defval{gb10_qwen3_breadth.t3_creative.linear.EpG_J}{122.4}
\defval{gb10_qwen3_breadth.t3_creative.linear.EpSG_J}{306.0}
\defval{gb10_qwen3_breadth.t3_creative.linear.N}{20}
\defval{gb10_qwen3_breadth.t3_creative.linear.TpG_s}{2.5}
\defval{gb10_qwen3_breadth.t3_creative.linear.attempts_per_goal}{1.00}
\defval{gb10_qwen3_breadth.t3_creative.linear.n_completed}{20}
\defval{gb10_qwen3_breadth.t3_creative.linear.n_correct}{8}
\defval{gb10_qwen3_breadth.t3_creative.linear.p_c}{1.00}
\defval{gb10_qwen3_breadth.t3_creative.linear.p_s}{0.40}
\defval{gb10_qwen3_breadth.t3_creative.linear.power_W}{49.5}
\defval{gb10_qwen3_breadth.t3_creative.min_successes}{4}
\defval{gb10_qwen3_breadth.t3_creative.power_ratio}{1.01}
\defval{gb10_qwen3_breadth.t3_creative.power_ratio_hi}{1.03}
\defval{gb10_qwen3_breadth.t3_creative.power_ratio_lo}{1.00}
\defval{gb10_qwen3_breadth.t3_creative.share_energy}{0.83}
\defval{gb10_qwen3_breadth.t3_creative.success_ratio}{2.00}
\defval{gb10_qwen3_breadth.t3_creative.time_ratio}{28.24}
\defval{gb10_qwen3_breadth.t3_creative.time_ratio_hi}{29.03}
\defval{gb10_qwen3_breadth.t3_creative.time_ratio_lo}{27.56}
\defval{gb10_qwen3_breadth.t3_planning_notools.OOI_EpCG}{7.38}
\defval{gb10_qwen3_breadth.t3_planning_notools.OOI_EpCG_hi}{7.74}
\defval{gb10_qwen3_breadth.t3_planning_notools.OOI_EpCG_lo}{7.04}
\defval{gb10_qwen3_breadth.t3_planning_notools.OOI_EpG}{7.38}
\defval{gb10_qwen3_breadth.t3_planning_notools.OOI_EpG_hi}{7.74}
\defval{gb10_qwen3_breadth.t3_planning_notools.OOI_EpG_lo}{7.04}
\defval{gb10_qwen3_breadth.t3_planning_notools.OOI_EpSG}{7.38}
\defval{gb10_qwen3_breadth.t3_planning_notools.OOI_EpSG_hi}{7.74}
\defval{gb10_qwen3_breadth.t3_planning_notools.OOI_EpSG_lo}{7.04}
\defval{gb10_qwen3_breadth.t3_planning_notools.agentic.E_wrong_goals_J}{0}
\defval{gb10_qwen3_breadth.t3_planning_notools.agentic.EpCG_J}{9,007.5}
\defval{gb10_qwen3_breadth.t3_planning_notools.agentic.EpG_J}{9,007.5}
\defval{gb10_qwen3_breadth.t3_planning_notools.agentic.EpSG_J}{9,007.5}
\defval{gb10_qwen3_breadth.t3_planning_notools.agentic.N}{20}
\defval{gb10_qwen3_breadth.t3_planning_notools.agentic.TpG_s}{169.1}
\defval{gb10_qwen3_breadth.t3_planning_notools.agentic.attempts_per_goal}{1.00}
\defval{gb10_qwen3_breadth.t3_planning_notools.agentic.n_completed}{20}
\defval{gb10_qwen3_breadth.t3_planning_notools.agentic.n_correct}{20}
\defval{gb10_qwen3_breadth.t3_planning_notools.agentic.p_c}{1.00}
\defval{gb10_qwen3_breadth.t3_planning_notools.agentic.p_s}{1.00}
\defval{gb10_qwen3_breadth.t3_planning_notools.agentic.power_W}{53.3}
\defval{gb10_qwen3_breadth.t3_planning_notools.energy_ratio}{7.38}
\defval{gb10_qwen3_breadth.t3_planning_notools.linear.E_wrong_goals_J}{0}
\defval{gb10_qwen3_breadth.t3_planning_notools.linear.EpCG_J}{1,220.3}
\defval{gb10_qwen3_breadth.t3_planning_notools.linear.EpG_J}{1,220.3}
\defval{gb10_qwen3_breadth.t3_planning_notools.linear.EpSG_J}{1,220.3}
\defval{gb10_qwen3_breadth.t3_planning_notools.linear.N}{20}
\defval{gb10_qwen3_breadth.t3_planning_notools.linear.TpG_s}{24.7}
\defval{gb10_qwen3_breadth.t3_planning_notools.linear.attempts_per_goal}{1.00}
\defval{gb10_qwen3_breadth.t3_planning_notools.linear.n_completed}{20}
\defval{gb10_qwen3_breadth.t3_planning_notools.linear.n_correct}{20}
\defval{gb10_qwen3_breadth.t3_planning_notools.linear.p_c}{1.00}
\defval{gb10_qwen3_breadth.t3_planning_notools.linear.p_s}{1.00}
\defval{gb10_qwen3_breadth.t3_planning_notools.linear.power_W}{49.4}
\defval{gb10_qwen3_breadth.t3_planning_notools.min_successes}{20}
\defval{gb10_qwen3_breadth.t3_planning_notools.power_ratio}{1.08}
\defval{gb10_qwen3_breadth.t3_planning_notools.power_ratio_hi}{1.09}
\defval{gb10_qwen3_breadth.t3_planning_notools.power_ratio_lo}{1.07}
\defval{gb10_qwen3_breadth.t3_planning_notools.share_energy}{1.00}
\defval{gb10_qwen3_breadth.t3_planning_notools.success_ratio}{1.00}
\defval{gb10_qwen3_breadth.t3_planning_notools.time_ratio}{6.85}
\defval{gb10_qwen3_breadth.t3_planning_notools.time_ratio_hi}{7.18}
\defval{gb10_qwen3_breadth.t3_planning_notools.time_ratio_lo}{6.54}
\defval{gb10_qwen3_breadth.t3_reasoning_hard.OOI_EpCG}{1.84}
\defval{gb10_qwen3_breadth.t3_reasoning_hard.OOI_EpCG_hi}{1.92}
\defval{gb10_qwen3_breadth.t3_reasoning_hard.OOI_EpCG_lo}{1.78}
\defval{gb10_qwen3_breadth.t3_reasoning_hard.OOI_EpG}{1.84}
\defval{gb10_qwen3_breadth.t3_reasoning_hard.OOI_EpG_hi}{1.92}
\defval{gb10_qwen3_breadth.t3_reasoning_hard.OOI_EpG_lo}{1.78}
\defval{gb10_qwen3_breadth.t3_reasoning_hard.OOI_EpSG}{1.84}
\defval{gb10_qwen3_breadth.t3_reasoning_hard.OOI_EpSG_hi}{1.92}
\defval{gb10_qwen3_breadth.t3_reasoning_hard.OOI_EpSG_lo}{1.78}
\defval{gb10_qwen3_breadth.t3_reasoning_hard.agentic.E_wrong_goals_J}{0}
\defval{gb10_qwen3_breadth.t3_reasoning_hard.agentic.EpCG_J}{497.7}
\defval{gb10_qwen3_breadth.t3_reasoning_hard.agentic.EpG_J}{497.7}
\defval{gb10_qwen3_breadth.t3_reasoning_hard.agentic.EpSG_J}{497.7}
\defval{gb10_qwen3_breadth.t3_reasoning_hard.agentic.N}{20}
\defval{gb10_qwen3_breadth.t3_reasoning_hard.agentic.TpG_s}{10.8}
\defval{gb10_qwen3_breadth.t3_reasoning_hard.agentic.attempts_per_goal}{1.00}
\defval{gb10_qwen3_breadth.t3_reasoning_hard.agentic.n_completed}{20}
\defval{gb10_qwen3_breadth.t3_reasoning_hard.agentic.n_correct}{20}
\defval{gb10_qwen3_breadth.t3_reasoning_hard.agentic.p_c}{1.00}
\defval{gb10_qwen3_breadth.t3_reasoning_hard.agentic.p_s}{1.00}
\defval{gb10_qwen3_breadth.t3_reasoning_hard.agentic.power_W}{46.2}
\defval{gb10_qwen3_breadth.t3_reasoning_hard.energy_ratio}{1.84}
\defval{gb10_qwen3_breadth.t3_reasoning_hard.linear.E_wrong_goals_J}{0}
\defval{gb10_qwen3_breadth.t3_reasoning_hard.linear.EpCG_J}{270.4}
\defval{gb10_qwen3_breadth.t3_reasoning_hard.linear.EpG_J}{270.4}
\defval{gb10_qwen3_breadth.t3_reasoning_hard.linear.EpSG_J}{270.4}
\defval{gb10_qwen3_breadth.t3_reasoning_hard.linear.N}{20}
\defval{gb10_qwen3_breadth.t3_reasoning_hard.linear.TpG_s}{5.9}
\defval{gb10_qwen3_breadth.t3_reasoning_hard.linear.attempts_per_goal}{1.00}
\defval{gb10_qwen3_breadth.t3_reasoning_hard.linear.n_completed}{20}
\defval{gb10_qwen3_breadth.t3_reasoning_hard.linear.n_correct}{20}
\defval{gb10_qwen3_breadth.t3_reasoning_hard.linear.p_c}{1.00}
\defval{gb10_qwen3_breadth.t3_reasoning_hard.linear.p_s}{1.00}
\defval{gb10_qwen3_breadth.t3_reasoning_hard.linear.power_W}{45.5}
\defval{gb10_qwen3_breadth.t3_reasoning_hard.min_successes}{20}
\defval{gb10_qwen3_breadth.t3_reasoning_hard.power_ratio}{1.02}
\defval{gb10_qwen3_breadth.t3_reasoning_hard.power_ratio_hi}{1.02}
\defval{gb10_qwen3_breadth.t3_reasoning_hard.power_ratio_lo}{1.01}
\defval{gb10_qwen3_breadth.t3_reasoning_hard.share_energy}{1.00}
\defval{gb10_qwen3_breadth.t3_reasoning_hard.success_ratio}{1.00}
\defval{gb10_qwen3_breadth.t3_reasoning_hard.time_ratio}{1.81}
\defval{gb10_qwen3_breadth.t3_reasoning_hard.time_ratio_hi}{1.89}
\defval{gb10_qwen3_breadth.t3_reasoning_hard.time_ratio_lo}{1.75}
\defval{gb10_qwen3_breadth.t3_summarization.OOI_EpCG}{12.50}
\defval{gb10_qwen3_breadth.t3_summarization.OOI_EpCG_hi}{12.97}
\defval{gb10_qwen3_breadth.t3_summarization.OOI_EpCG_lo}{12.07}
\defval{gb10_qwen3_breadth.t3_summarization.OOI_EpG}{12.50}
\defval{gb10_qwen3_breadth.t3_summarization.OOI_EpG_hi}{12.97}
\defval{gb10_qwen3_breadth.t3_summarization.OOI_EpG_lo}{12.07}
\defval{gb10_qwen3_breadth.t3_summarization.OOI_EpSG}{12.50}
\defval{gb10_qwen3_breadth.t3_summarization.OOI_EpSG_hi}{12.97}
\defval{gb10_qwen3_breadth.t3_summarization.OOI_EpSG_lo}{12.07}
\defval{gb10_qwen3_breadth.t3_summarization.agentic.E_wrong_goals_J}{0}
\defval{gb10_qwen3_breadth.t3_summarization.agentic.EpCG_J}{956.1}
\defval{gb10_qwen3_breadth.t3_summarization.agentic.EpG_J}{956.1}
\defval{gb10_qwen3_breadth.t3_summarization.agentic.EpSG_J}{956.1}
\defval{gb10_qwen3_breadth.t3_summarization.agentic.N}{20}
\defval{gb10_qwen3_breadth.t3_summarization.agentic.TpG_s}{20.9}
\defval{gb10_qwen3_breadth.t3_summarization.agentic.attempts_per_goal}{1.00}
\defval{gb10_qwen3_breadth.t3_summarization.agentic.n_completed}{20}
\defval{gb10_qwen3_breadth.t3_summarization.agentic.n_correct}{20}
\defval{gb10_qwen3_breadth.t3_summarization.agentic.p_c}{1.00}
\defval{gb10_qwen3_breadth.t3_summarization.agentic.p_s}{1.00}
\defval{gb10_qwen3_breadth.t3_summarization.agentic.power_W}{45.8}
\defval{gb10_qwen3_breadth.t3_summarization.energy_ratio}{12.50}
\defval{gb10_qwen3_breadth.t3_summarization.linear.E_wrong_goals_J}{0}
\defval{gb10_qwen3_breadth.t3_summarization.linear.EpCG_J}{76.5}
\defval{gb10_qwen3_breadth.t3_summarization.linear.EpG_J}{76.5}
\defval{gb10_qwen3_breadth.t3_summarization.linear.EpSG_J}{76.5}
\defval{gb10_qwen3_breadth.t3_summarization.linear.N}{20}
\defval{gb10_qwen3_breadth.t3_summarization.linear.TpG_s}{1.6}
\defval{gb10_qwen3_breadth.t3_summarization.linear.attempts_per_goal}{1.00}
\defval{gb10_qwen3_breadth.t3_summarization.linear.n_completed}{20}
\defval{gb10_qwen3_breadth.t3_summarization.linear.n_correct}{20}
\defval{gb10_qwen3_breadth.t3_summarization.linear.p_c}{1.00}
\defval{gb10_qwen3_breadth.t3_summarization.linear.p_s}{1.00}
\defval{gb10_qwen3_breadth.t3_summarization.linear.power_W}{46.9}
\defval{gb10_qwen3_breadth.t3_summarization.min_successes}{20}
\defval{gb10_qwen3_breadth.t3_summarization.power_ratio}{0.98}
\defval{gb10_qwen3_breadth.t3_summarization.power_ratio_hi}{0.99}
\defval{gb10_qwen3_breadth.t3_summarization.power_ratio_lo}{0.97}
\defval{gb10_qwen3_breadth.t3_summarization.share_energy}{1.00}
\defval{gb10_qwen3_breadth.t3_summarization.success_ratio}{1.00}
\defval{gb10_qwen3_breadth.t3_summarization.time_ratio}{12.80}
\defval{gb10_qwen3_breadth.t3_summarization.time_ratio_hi}{13.28}
\defval{gb10_qwen3_breadth.t3_summarization.time_ratio_lo}{12.35}
\defval{gb10_qwen3_breadth.t3_translation_hi.OOI_EpCG}{3.29}
\defval{gb10_qwen3_breadth.t3_translation_hi.OOI_EpCG_hi}{4.28}
\defval{gb10_qwen3_breadth.t3_translation_hi.OOI_EpCG_lo}{2.76}
\defval{gb10_qwen3_breadth.t3_translation_hi.OOI_EpG}{3.29}
\defval{gb10_qwen3_breadth.t3_translation_hi.OOI_EpG_hi}{4.28}
\defval{gb10_qwen3_breadth.t3_translation_hi.OOI_EpG_lo}{2.76}
\defval{gb10_qwen3_breadth.t3_translation_hi.OOI_EpSG}{3.29}
\defval{gb10_qwen3_breadth.t3_translation_hi.OOI_EpSG_hi}{4.28}
\defval{gb10_qwen3_breadth.t3_translation_hi.OOI_EpSG_lo}{2.76}
\defval{gb10_qwen3_breadth.t3_translation_hi.agentic.E_wrong_goals_J}{0}
\defval{gb10_qwen3_breadth.t3_translation_hi.agentic.EpCG_J}{269.4}
\defval{gb10_qwen3_breadth.t3_translation_hi.agentic.EpG_J}{269.4}
\defval{gb10_qwen3_breadth.t3_translation_hi.agentic.EpSG_J}{269.4}
\defval{gb10_qwen3_breadth.t3_translation_hi.agentic.N}{20}
\defval{gb10_qwen3_breadth.t3_translation_hi.agentic.TpG_s}{6.0}
\defval{gb10_qwen3_breadth.t3_translation_hi.agentic.attempts_per_goal}{1.00}
\defval{gb10_qwen3_breadth.t3_translation_hi.agentic.n_completed}{20}
\defval{gb10_qwen3_breadth.t3_translation_hi.agentic.n_correct}{20}
\defval{gb10_qwen3_breadth.t3_translation_hi.agentic.p_c}{1.00}
\defval{gb10_qwen3_breadth.t3_translation_hi.agentic.p_s}{1.00}
\defval{gb10_qwen3_breadth.t3_translation_hi.agentic.power_W}{44.9}
\defval{gb10_qwen3_breadth.t3_translation_hi.energy_ratio}{3.29}
\defval{gb10_qwen3_breadth.t3_translation_hi.linear.E_wrong_goals_J}{0}
\defval{gb10_qwen3_breadth.t3_translation_hi.linear.EpCG_J}{81.8}
\defval{gb10_qwen3_breadth.t3_translation_hi.linear.EpG_J}{81.8}
\defval{gb10_qwen3_breadth.t3_translation_hi.linear.EpSG_J}{81.8}
\defval{gb10_qwen3_breadth.t3_translation_hi.linear.N}{20}
\defval{gb10_qwen3_breadth.t3_translation_hi.linear.TpG_s}{1.8}
\defval{gb10_qwen3_breadth.t3_translation_hi.linear.attempts_per_goal}{1.00}
\defval{gb10_qwen3_breadth.t3_translation_hi.linear.n_completed}{20}
\defval{gb10_qwen3_breadth.t3_translation_hi.linear.n_correct}{20}
\defval{gb10_qwen3_breadth.t3_translation_hi.linear.p_c}{1.00}
\defval{gb10_qwen3_breadth.t3_translation_hi.linear.p_s}{1.00}
\defval{gb10_qwen3_breadth.t3_translation_hi.linear.power_W}{46.2}
\defval{gb10_qwen3_breadth.t3_translation_hi.min_successes}{20}
\defval{gb10_qwen3_breadth.t3_translation_hi.power_ratio}{0.97}
\defval{gb10_qwen3_breadth.t3_translation_hi.power_ratio_hi}{0.98}
\defval{gb10_qwen3_breadth.t3_translation_hi.power_ratio_lo}{0.96}
\defval{gb10_qwen3_breadth.t3_translation_hi.share_energy}{1.00}
\defval{gb10_qwen3_breadth.t3_translation_hi.success_ratio}{1.00}
\defval{gb10_qwen3_breadth.t3_translation_hi.time_ratio}{3.38}
\defval{gb10_qwen3_breadth.t3_translation_hi.time_ratio_hi}{4.39}
\defval{gb10_qwen3_breadth.t3_translation_hi.time_ratio_lo}{2.84}
\defval{gb10_qwen3_breadth.technical_blog_post.OOI_EpCG}{13.90}
\defval{gb10_qwen3_breadth.technical_blog_post.OOI_EpCG_hi}{14.49}
\defval{gb10_qwen3_breadth.technical_blog_post.OOI_EpCG_lo}{13.30}
\defval{gb10_qwen3_breadth.technical_blog_post.OOI_EpG}{13.90}
\defval{gb10_qwen3_breadth.technical_blog_post.OOI_EpG_hi}{14.49}
\defval{gb10_qwen3_breadth.technical_blog_post.OOI_EpG_lo}{13.30}
\defval{gb10_qwen3_breadth.technical_blog_post.OOI_EpSG}{9.27}
\defval{gb10_qwen3_breadth.technical_blog_post.OOI_EpSG_hi}{15.11}
\defval{gb10_qwen3_breadth.technical_blog_post.OOI_EpSG_lo}{4.93}
\defval{gb10_qwen3_breadth.technical_blog_post.agentic.E_wrong_goals_J}{13,511}
\defval{gb10_qwen3_breadth.technical_blog_post.agentic.EpCG_J}{2,857.5}
\defval{gb10_qwen3_breadth.technical_blog_post.agentic.EpG_J}{2,857.5}
\defval{gb10_qwen3_breadth.technical_blog_post.agentic.EpSG_J}{3,810.1}
\defval{gb10_qwen3_breadth.technical_blog_post.agentic.N}{20}
\defval{gb10_qwen3_breadth.technical_blog_post.agentic.TpG_s}{54.9}
\defval{gb10_qwen3_breadth.technical_blog_post.agentic.attempts_per_goal}{1.00}
\defval{gb10_qwen3_breadth.technical_blog_post.agentic.n_completed}{20}
\defval{gb10_qwen3_breadth.technical_blog_post.agentic.n_correct}{15}
\defval{gb10_qwen3_breadth.technical_blog_post.agentic.p_c}{1.00}
\defval{gb10_qwen3_breadth.technical_blog_post.agentic.p_s}{0.75}
\defval{gb10_qwen3_breadth.technical_blog_post.agentic.power_W}{52.1}
\defval{gb10_qwen3_breadth.technical_blog_post.energy_ratio}{13.90}
\defval{gb10_qwen3_breadth.technical_blog_post.linear.E_wrong_goals_J}{2,159}
\defval{gb10_qwen3_breadth.technical_blog_post.linear.EpCG_J}{205.6}
\defval{gb10_qwen3_breadth.technical_blog_post.linear.EpG_J}{205.6}
\defval{gb10_qwen3_breadth.technical_blog_post.linear.EpSG_J}{411.2}
\defval{gb10_qwen3_breadth.technical_blog_post.linear.N}{20}
\defval{gb10_qwen3_breadth.technical_blog_post.linear.TpG_s}{4.1}
\defval{gb10_qwen3_breadth.technical_blog_post.linear.attempts_per_goal}{1.00}
\defval{gb10_qwen3_breadth.technical_blog_post.linear.n_completed}{20}
\defval{gb10_qwen3_breadth.technical_blog_post.linear.n_correct}{10}
\defval{gb10_qwen3_breadth.technical_blog_post.linear.p_c}{1.00}
\defval{gb10_qwen3_breadth.technical_blog_post.linear.p_s}{0.50}
\defval{gb10_qwen3_breadth.technical_blog_post.linear.power_W}{50.1}
\defval{gb10_qwen3_breadth.technical_blog_post.min_successes}{10}
\defval{gb10_qwen3_breadth.technical_blog_post.power_ratio}{1.04}
\defval{gb10_qwen3_breadth.technical_blog_post.power_ratio_hi}{1.06}
\defval{gb10_qwen3_breadth.technical_blog_post.power_ratio_lo}{1.02}
\defval{gb10_qwen3_breadth.technical_blog_post.share_energy}{1.18}
\defval{gb10_qwen3_breadth.technical_blog_post.success_ratio}{0.67}
\defval{gb10_qwen3_breadth.technical_blog_post.time_ratio}{13.37}
\defval{gb10_qwen3_breadth.technical_blog_post.time_ratio_hi}{13.81}
\defval{gb10_qwen3_breadth.technical_blog_post.time_ratio_lo}{12.91}
\defval{gb10_qwen3_breadth.topic_classification.OOI_EpCG}{37.89}
\defval{gb10_qwen3_breadth.topic_classification.OOI_EpCG_hi}{39.28}
\defval{gb10_qwen3_breadth.topic_classification.OOI_EpCG_lo}{36.63}
\defval{gb10_qwen3_breadth.topic_classification.OOI_EpG}{37.89}
\defval{gb10_qwen3_breadth.topic_classification.OOI_EpG_hi}{39.28}
\defval{gb10_qwen3_breadth.topic_classification.OOI_EpG_lo}{36.63}
\defval{gb10_qwen3_breadth.topic_classification.OOI_EpSG}{37.89}
\defval{gb10_qwen3_breadth.topic_classification.OOI_EpSG_hi}{39.28}
\defval{gb10_qwen3_breadth.topic_classification.OOI_EpSG_lo}{36.63}
\defval{gb10_qwen3_breadth.topic_classification.agentic.E_wrong_goals_J}{0}
\defval{gb10_qwen3_breadth.topic_classification.agentic.EpCG_J}{449.8}
\defval{gb10_qwen3_breadth.topic_classification.agentic.EpG_J}{449.8}
\defval{gb10_qwen3_breadth.topic_classification.agentic.EpSG_J}{449.8}
\defval{gb10_qwen3_breadth.topic_classification.agentic.N}{20}
\defval{gb10_qwen3_breadth.topic_classification.agentic.TpG_s}{10.1}
\defval{gb10_qwen3_breadth.topic_classification.agentic.attempts_per_goal}{1.00}
\defval{gb10_qwen3_breadth.topic_classification.agentic.n_completed}{20}
\defval{gb10_qwen3_breadth.topic_classification.agentic.n_correct}{20}
\defval{gb10_qwen3_breadth.topic_classification.agentic.p_c}{1.00}
\defval{gb10_qwen3_breadth.topic_classification.agentic.p_s}{1.00}
\defval{gb10_qwen3_breadth.topic_classification.agentic.power_W}{44.7}
\defval{gb10_qwen3_breadth.topic_classification.energy_ratio}{37.89}
\defval{gb10_qwen3_breadth.topic_classification.linear.E_wrong_goals_J}{0}
\defval{gb10_qwen3_breadth.topic_classification.linear.EpCG_J}{11.9}
\defval{gb10_qwen3_breadth.topic_classification.linear.EpG_J}{11.9}
\defval{gb10_qwen3_breadth.topic_classification.linear.EpSG_J}{11.9}
\defval{gb10_qwen3_breadth.topic_classification.linear.N}{20}
\defval{gb10_qwen3_breadth.topic_classification.linear.TpG_s}{0.2}
\defval{gb10_qwen3_breadth.topic_classification.linear.attempts_per_goal}{1.00}
\defval{gb10_qwen3_breadth.topic_classification.linear.n_completed}{20}
\defval{gb10_qwen3_breadth.topic_classification.linear.n_correct}{20}
\defval{gb10_qwen3_breadth.topic_classification.linear.p_c}{1.00}
\defval{gb10_qwen3_breadth.topic_classification.linear.p_s}{1.00}
\defval{gb10_qwen3_breadth.topic_classification.linear.power_W}{64.7}
\defval{gb10_qwen3_breadth.topic_classification.min_successes}{20}
\defval{gb10_qwen3_breadth.topic_classification.power_ratio}{0.69}
\defval{gb10_qwen3_breadth.topic_classification.power_ratio_hi}{0.71}
\defval{gb10_qwen3_breadth.topic_classification.power_ratio_lo}{0.67}
\defval{gb10_qwen3_breadth.topic_classification.share_energy}{1.00}
\defval{gb10_qwen3_breadth.topic_classification.success_ratio}{1.00}
\defval{gb10_qwen3_breadth.topic_classification.time_ratio}{54.94}
\defval{gb10_qwen3_breadth.topic_classification.time_ratio_hi}{56.37}
\defval{gb10_qwen3_breadth.topic_classification.time_ratio_lo}{53.61}
\defval{gb10_qwen3_dist.excluded}{0}
\defval{gb10_qwen3_dist.goals}{300}
\defval{gb10_qwen3_dist.gsm8k.OOI_EpCG}{4.68}
\defval{gb10_qwen3_dist.gsm8k.OOI_EpCG_hi}{7.84}
\defval{gb10_qwen3_dist.gsm8k.OOI_EpCG_lo}{2.76}
\defval{gb10_qwen3_dist.gsm8k.OOI_EpG}{4.50}
\defval{gb10_qwen3_dist.gsm8k.OOI_EpG_hi}{7.14}
\defval{gb10_qwen3_dist.gsm8k.OOI_EpG_lo}{2.76}
\defval{gb10_qwen3_dist.gsm8k.OOI_EpSG}{4.80}
\defval{gb10_qwen3_dist.gsm8k.OOI_EpSG_hi}{8.14}
\defval{gb10_qwen3_dist.gsm8k.OOI_EpSG_lo}{2.83}
\defval{gb10_qwen3_dist.gsm8k.agentic.E_wrong_goals_J}{1,418}
\defval{gb10_qwen3_dist.gsm8k.agentic.EpCG_J}{910.7}
\defval{gb10_qwen3_dist.gsm8k.agentic.EpG_J}{874.3}
\defval{gb10_qwen3_dist.gsm8k.agentic.EpSG_J}{971.4}
\defval{gb10_qwen3_dist.gsm8k.agentic.N}{50}
\defval{gb10_qwen3_dist.gsm8k.agentic.TpG_s}{14.9}
\defval{gb10_qwen3_dist.gsm8k.agentic.attempts_per_goal}{1.22}
\defval{gb10_qwen3_dist.gsm8k.agentic.n_completed}{48}
\defval{gb10_qwen3_dist.gsm8k.agentic.n_correct}{45}
\defval{gb10_qwen3_dist.gsm8k.agentic.p_c}{0.96}
\defval{gb10_qwen3_dist.gsm8k.agentic.p_s}{0.90}
\defval{gb10_qwen3_dist.gsm8k.agentic.power_W}{58.5}
\defval{gb10_qwen3_dist.gsm8k.energy_ratio}{4.50}
\defval{gb10_qwen3_dist.gsm8k.linear.E_wrong_goals_J}{657}
\defval{gb10_qwen3_dist.gsm8k.linear.EpCG_J}{194.4}
\defval{gb10_qwen3_dist.gsm8k.linear.EpG_J}{194.4}
\defval{gb10_qwen3_dist.gsm8k.linear.EpSG_J}{202.5}
\defval{gb10_qwen3_dist.gsm8k.linear.N}{50}
\defval{gb10_qwen3_dist.gsm8k.linear.TpG_s}{3.5}
\defval{gb10_qwen3_dist.gsm8k.linear.attempts_per_goal}{1.00}
\defval{gb10_qwen3_dist.gsm8k.linear.n_completed}{50}
\defval{gb10_qwen3_dist.gsm8k.linear.n_correct}{48}
\defval{gb10_qwen3_dist.gsm8k.linear.p_c}{1.00}
\defval{gb10_qwen3_dist.gsm8k.linear.p_s}{0.96}
\defval{gb10_qwen3_dist.gsm8k.linear.power_W}{55.4}
\defval{gb10_qwen3_dist.gsm8k.min_successes}{45}
\defval{gb10_qwen3_dist.gsm8k.power_ratio}{1.06}
\defval{gb10_qwen3_dist.gsm8k.power_ratio_hi}{1.08}
\defval{gb10_qwen3_dist.gsm8k.power_ratio_lo}{1.02}
\defval{gb10_qwen3_dist.gsm8k.share_energy}{0.96}
\defval{gb10_qwen3_dist.gsm8k.success_ratio}{1.07}
\defval{gb10_qwen3_dist.gsm8k.time_ratio}{4.26}
\defval{gb10_qwen3_dist.gsm8k.time_ratio_hi}{6.61}
\defval{gb10_qwen3_dist.gsm8k.time_ratio_lo}{2.70}
\defval{gb10_qwen3_dist.humaneval.OOI_EpCG}{25.42}
\defval{gb10_qwen3_dist.humaneval.OOI_EpCG_hi}{38.12}
\defval{gb10_qwen3_dist.humaneval.OOI_EpCG_lo}{15.58}
\defval{gb10_qwen3_dist.humaneval.OOI_EpG}{20.33}
\defval{gb10_qwen3_dist.humaneval.OOI_EpG_hi}{27.90}
\defval{gb10_qwen3_dist.humaneval.OOI_EpG_lo}{13.86}
\defval{gb10_qwen3_dist.humaneval.OOI_EpSG}{27.11}
\defval{gb10_qwen3_dist.humaneval.OOI_EpSG_hi}{42.07}
\defval{gb10_qwen3_dist.humaneval.OOI_EpSG_lo}{16.35}
\defval{gb10_qwen3_dist.humaneval.agentic.E_wrong_goals_J}{15,713}
\defval{gb10_qwen3_dist.humaneval.agentic.EpCG_J}{6,366.2}
\defval{gb10_qwen3_dist.humaneval.agentic.EpG_J}{5,093.0}
\defval{gb10_qwen3_dist.humaneval.agentic.EpSG_J}{7,073.6}
\defval{gb10_qwen3_dist.humaneval.agentic.N}{50}
\defval{gb10_qwen3_dist.humaneval.agentic.TpG_s}{82.8}
\defval{gb10_qwen3_dist.humaneval.agentic.attempts_per_goal}{2.06}
\defval{gb10_qwen3_dist.humaneval.agentic.n_completed}{40}
\defval{gb10_qwen3_dist.humaneval.agentic.n_correct}{36}
\defval{gb10_qwen3_dist.humaneval.agentic.p_c}{0.80}
\defval{gb10_qwen3_dist.humaneval.agentic.p_s}{0.72}
\defval{gb10_qwen3_dist.humaneval.agentic.power_W}{61.5}
\defval{gb10_qwen3_dist.humaneval.energy_ratio}{20.33}
\defval{gb10_qwen3_dist.humaneval.linear.E_wrong_goals_J}{401}
\defval{gb10_qwen3_dist.humaneval.linear.EpCG_J}{250.5}
\defval{gb10_qwen3_dist.humaneval.linear.EpG_J}{250.5}
\defval{gb10_qwen3_dist.humaneval.linear.EpSG_J}{260.9}
\defval{gb10_qwen3_dist.humaneval.linear.N}{50}
\defval{gb10_qwen3_dist.humaneval.linear.TpG_s}{4.3}
\defval{gb10_qwen3_dist.humaneval.linear.attempts_per_goal}{1.00}
\defval{gb10_qwen3_dist.humaneval.linear.n_completed}{50}
\defval{gb10_qwen3_dist.humaneval.linear.n_correct}{48}
\defval{gb10_qwen3_dist.humaneval.linear.p_c}{1.00}
\defval{gb10_qwen3_dist.humaneval.linear.p_s}{0.96}
\defval{gb10_qwen3_dist.humaneval.linear.power_W}{57.6}
\defval{gb10_qwen3_dist.humaneval.min_successes}{36}
\defval{gb10_qwen3_dist.humaneval.power_ratio}{1.07}
\defval{gb10_qwen3_dist.humaneval.power_ratio_hi}{1.08}
\defval{gb10_qwen3_dist.humaneval.power_ratio_lo}{1.04}
\defval{gb10_qwen3_dist.humaneval.share_energy}{0.91}
\defval{gb10_qwen3_dist.humaneval.success_ratio}{1.33}
\defval{gb10_qwen3_dist.humaneval.time_ratio}{19.04}
\defval{gb10_qwen3_dist.humaneval.time_ratio_hi}{25.91}
\defval{gb10_qwen3_dist.humaneval.time_ratio_lo}{13.18}
\defval{gb10_qwen3_dist.mmlu.OOI_EpCG}{17.48}
\defval{gb10_qwen3_dist.mmlu.OOI_EpCG_hi}{21.96}
\defval{gb10_qwen3_dist.mmlu.OOI_EpCG_lo}{14.50}
\defval{gb10_qwen3_dist.mmlu.OOI_EpG}{17.48}
\defval{gb10_qwen3_dist.mmlu.OOI_EpG_hi}{21.96}
\defval{gb10_qwen3_dist.mmlu.OOI_EpG_lo}{14.50}
\defval{gb10_qwen3_dist.mmlu.OOI_EpSG}{15.62}
\defval{gb10_qwen3_dist.mmlu.OOI_EpSG_hi}{20.36}
\defval{gb10_qwen3_dist.mmlu.OOI_EpSG_lo}{12.37}
\defval{gb10_qwen3_dist.mmlu.agentic.E_wrong_goals_J}{13,633}
\defval{gb10_qwen3_dist.mmlu.agentic.EpCG_J}{3,657.6}
\defval{gb10_qwen3_dist.mmlu.agentic.EpG_J}{3,657.6}
\defval{gb10_qwen3_dist.mmlu.agentic.EpSG_J}{3,891.1}
\defval{gb10_qwen3_dist.mmlu.agentic.N}{50}
\defval{gb10_qwen3_dist.mmlu.agentic.TpG_s}{61.0}
\defval{gb10_qwen3_dist.mmlu.agentic.attempts_per_goal}{1.00}
\defval{gb10_qwen3_dist.mmlu.agentic.n_completed}{50}
\defval{gb10_qwen3_dist.mmlu.agentic.n_correct}{47}
\defval{gb10_qwen3_dist.mmlu.agentic.p_c}{1.00}
\defval{gb10_qwen3_dist.mmlu.agentic.p_s}{0.94}
\defval{gb10_qwen3_dist.mmlu.agentic.power_W}{60.0}
\defval{gb10_qwen3_dist.mmlu.energy_ratio}{17.48}
\defval{gb10_qwen3_dist.mmlu.linear.E_wrong_goals_J}{1,563}
\defval{gb10_qwen3_dist.mmlu.linear.EpCG_J}{209.2}
\defval{gb10_qwen3_dist.mmlu.linear.EpG_J}{209.2}
\defval{gb10_qwen3_dist.mmlu.linear.EpSG_J}{249.1}
\defval{gb10_qwen3_dist.mmlu.linear.N}{50}
\defval{gb10_qwen3_dist.mmlu.linear.TpG_s}{3.7}
\defval{gb10_qwen3_dist.mmlu.linear.attempts_per_goal}{1.00}
\defval{gb10_qwen3_dist.mmlu.linear.n_completed}{50}
\defval{gb10_qwen3_dist.mmlu.linear.n_correct}{42}
\defval{gb10_qwen3_dist.mmlu.linear.p_c}{1.00}
\defval{gb10_qwen3_dist.mmlu.linear.p_s}{0.84}
\defval{gb10_qwen3_dist.mmlu.linear.power_W}{57.0}
\defval{gb10_qwen3_dist.mmlu.min_successes}{42}
\defval{gb10_qwen3_dist.mmlu.power_ratio}{1.05}
\defval{gb10_qwen3_dist.mmlu.power_ratio_hi}{1.07}
\defval{gb10_qwen3_dist.mmlu.power_ratio_lo}{1.03}
\defval{gb10_qwen3_dist.mmlu.share_energy}{1.04}
\defval{gb10_qwen3_dist.mmlu.success_ratio}{0.89}
\defval{gb10_qwen3_dist.mmlu.time_ratio}{16.62}
\defval{gb10_qwen3_dist.mmlu.time_ratio_hi}{21.14}
\defval{gb10_qwen3_dist.mmlu.time_ratio_lo}{13.63}
\defval{gb10_tinyllama.commonsense_reasoning.OOI_EpCG}{5.51}
\defval{gb10_tinyllama.commonsense_reasoning.OOI_EpCG_hi}{7.60}
\defval{gb10_tinyllama.commonsense_reasoning.OOI_EpCG_lo}{4.27}
\defval{gb10_tinyllama.commonsense_reasoning.OOI_EpG}{5.51}
\defval{gb10_tinyllama.commonsense_reasoning.OOI_EpG_hi}{7.60}
\defval{gb10_tinyllama.commonsense_reasoning.OOI_EpG_lo}{4.27}
\defval{gb10_tinyllama.commonsense_reasoning.OOI_EpSG}{2.64}
\defval{gb10_tinyllama.commonsense_reasoning.OOI_EpSG_hi}{4.98}
\defval{gb10_tinyllama.commonsense_reasoning.OOI_EpSG_lo}{1.25}
\defval{gb10_tinyllama.commonsense_reasoning.agentic.E_wrong_goals_J}{25,516}
\defval{gb10_tinyllama.commonsense_reasoning.agentic.EpCG_J}{683.9}
\defval{gb10_tinyllama.commonsense_reasoning.agentic.EpG_J}{683.9}
\defval{gb10_tinyllama.commonsense_reasoning.agentic.EpSG_J}{1,641.2}
\defval{gb10_tinyllama.commonsense_reasoning.agentic.N}{60}
\defval{gb10_tinyllama.commonsense_reasoning.agentic.TpG_s}{10.1}
\defval{gb10_tinyllama.commonsense_reasoning.agentic.attempts_per_goal}{1.00}
\defval{gb10_tinyllama.commonsense_reasoning.agentic.n_completed}{60}
\defval{gb10_tinyllama.commonsense_reasoning.agentic.n_correct}{25}
\defval{gb10_tinyllama.commonsense_reasoning.agentic.p_c}{1.00}
\defval{gb10_tinyllama.commonsense_reasoning.agentic.p_s}{0.42}
\defval{gb10_tinyllama.commonsense_reasoning.agentic.power_W}{67.4}
\defval{gb10_tinyllama.commonsense_reasoning.energy_ratio}{5.51}
\defval{gb10_tinyllama.commonsense_reasoning.linear.E_wrong_goals_J}{5,715}
\defval{gb10_tinyllama.commonsense_reasoning.linear.EpCG_J}{124.1}
\defval{gb10_tinyllama.commonsense_reasoning.linear.EpG_J}{124.1}
\defval{gb10_tinyllama.commonsense_reasoning.linear.EpSG_J}{620.6}
\defval{gb10_tinyllama.commonsense_reasoning.linear.N}{60}
\defval{gb10_tinyllama.commonsense_reasoning.linear.TpG_s}{1.8}
\defval{gb10_tinyllama.commonsense_reasoning.linear.attempts_per_goal}{1.00}
\defval{gb10_tinyllama.commonsense_reasoning.linear.n_completed}{60}
\defval{gb10_tinyllama.commonsense_reasoning.linear.n_correct}{12}
\defval{gb10_tinyllama.commonsense_reasoning.linear.p_c}{1.00}
\defval{gb10_tinyllama.commonsense_reasoning.linear.p_s}{0.20}
\defval{gb10_tinyllama.commonsense_reasoning.linear.power_W}{68.0}
\defval{gb10_tinyllama.commonsense_reasoning.min_successes}{12}
\defval{gb10_tinyllama.commonsense_reasoning.power_ratio}{0.99}
\defval{gb10_tinyllama.commonsense_reasoning.power_ratio_hi}{1.01}
\defval{gb10_tinyllama.commonsense_reasoning.power_ratio_lo}{0.96}
\defval{gb10_tinyllama.commonsense_reasoning.share_energy}{1.75}
\defval{gb10_tinyllama.commonsense_reasoning.success_ratio}{0.48}
\defval{gb10_tinyllama.commonsense_reasoning.time_ratio}{5.56}
\defval{gb10_tinyllama.commonsense_reasoning.time_ratio_hi}{7.89}
\defval{gb10_tinyllama.commonsense_reasoning.time_ratio_lo}{4.24}
\defval{gb10_tinyllama.excluded}{0}
\defval{gb10_tinyllama.factual_qa.OOI_EpCG}{4.29}
\defval{gb10_tinyllama.factual_qa.OOI_EpCG_hi}{6.07}
\defval{gb10_tinyllama.factual_qa.OOI_EpCG_lo}{3.34}
\defval{gb10_tinyllama.factual_qa.OOI_EpG}{4.29}
\defval{gb10_tinyllama.factual_qa.OOI_EpG_hi}{6.07}
\defval{gb10_tinyllama.factual_qa.OOI_EpG_lo}{3.34}
\defval{gb10_tinyllama.factual_qa.OOI_EpSG}{2.15}
\defval{gb10_tinyllama.factual_qa.OOI_EpSG_hi}{3.90}
\defval{gb10_tinyllama.factual_qa.OOI_EpSG_lo}{1.16}
\defval{gb10_tinyllama.factual_qa.agentic.E_wrong_goals_J}{16,528}
\defval{gb10_tinyllama.factual_qa.agentic.EpCG_J}{541.3}
\defval{gb10_tinyllama.factual_qa.agentic.EpG_J}{541.3}
\defval{gb10_tinyllama.factual_qa.agentic.EpSG_J}{1,082.5}
\defval{gb10_tinyllama.factual_qa.agentic.N}{60}
\defval{gb10_tinyllama.factual_qa.agentic.TpG_s}{8.2}
\defval{gb10_tinyllama.factual_qa.agentic.attempts_per_goal}{1.00}
\defval{gb10_tinyllama.factual_qa.agentic.n_completed}{60}
\defval{gb10_tinyllama.factual_qa.agentic.n_correct}{30}
\defval{gb10_tinyllama.factual_qa.agentic.p_c}{1.00}
\defval{gb10_tinyllama.factual_qa.agentic.p_s}{0.50}
\defval{gb10_tinyllama.factual_qa.agentic.power_W}{66.4}
\defval{gb10_tinyllama.factual_qa.energy_ratio}{4.29}
\defval{gb10_tinyllama.factual_qa.linear.E_wrong_goals_J}{5,802}
\defval{gb10_tinyllama.factual_qa.linear.EpCG_J}{126.1}
\defval{gb10_tinyllama.factual_qa.linear.EpG_J}{126.1}
\defval{gb10_tinyllama.factual_qa.linear.EpSG_J}{504.4}
\defval{gb10_tinyllama.factual_qa.linear.N}{60}
\defval{gb10_tinyllama.factual_qa.linear.TpG_s}{1.9}
\defval{gb10_tinyllama.factual_qa.linear.attempts_per_goal}{1.00}
\defval{gb10_tinyllama.factual_qa.linear.n_completed}{60}
\defval{gb10_tinyllama.factual_qa.linear.n_correct}{15}
\defval{gb10_tinyllama.factual_qa.linear.p_c}{1.00}
\defval{gb10_tinyllama.factual_qa.linear.p_s}{0.25}
\defval{gb10_tinyllama.factual_qa.linear.power_W}{65.9}
\defval{gb10_tinyllama.factual_qa.min_successes}{15}
\defval{gb10_tinyllama.factual_qa.power_ratio}{1.01}
\defval{gb10_tinyllama.factual_qa.power_ratio_hi}{1.02}
\defval{gb10_tinyllama.factual_qa.power_ratio_lo}{0.98}
\defval{gb10_tinyllama.factual_qa.share_energy}{1.91}
\defval{gb10_tinyllama.factual_qa.success_ratio}{0.50}
\defval{gb10_tinyllama.factual_qa.time_ratio}{4.26}
\defval{gb10_tinyllama.factual_qa.time_ratio_hi}{6.14}
\defval{gb10_tinyllama.factual_qa.time_ratio_lo}{3.28}
\defval{gb10_tinyllama.goals}{1,080}
\defval{gb10_tinyllama.gsm8k_basic.OOI_EpCG}{3.66}
\defval{gb10_tinyllama.gsm8k_basic.OOI_EpCG_hi}{4.69}
\defval{gb10_tinyllama.gsm8k_basic.OOI_EpCG_lo}{2.96}
\defval{gb10_tinyllama.gsm8k_basic.OOI_EpG}{3.66}
\defval{gb10_tinyllama.gsm8k_basic.OOI_EpG_hi}{4.69}
\defval{gb10_tinyllama.gsm8k_basic.OOI_EpG_lo}{2.96}
\defval{gb10_tinyllama.gsm8k_basic.OOI_EpSG}{2.44}
\defval{gb10_tinyllama.gsm8k_basic.OOI_EpSG_hi}{10.90}
\defval{gb10_tinyllama.gsm8k_basic.OOI_EpSG_lo}{0.51}
\defval{gb10_tinyllama.gsm8k_basic.agentic.E_wrong_goals_J}{33,742}
\defval{gb10_tinyllama.gsm8k_basic.agentic.EpCG_J}{622.4}
\defval{gb10_tinyllama.gsm8k_basic.agentic.EpG_J}{622.4}
\defval{gb10_tinyllama.gsm8k_basic.agentic.EpSG_J}{6,224.3}
\defval{gb10_tinyllama.gsm8k_basic.agentic.N}{60}
\defval{gb10_tinyllama.gsm8k_basic.agentic.TpG_s}{9.2}
\defval{gb10_tinyllama.gsm8k_basic.agentic.attempts_per_goal}{1.00}
\defval{gb10_tinyllama.gsm8k_basic.agentic.n_completed}{60}
\defval{gb10_tinyllama.gsm8k_basic.agentic.n_correct}{6}
\defval{gb10_tinyllama.gsm8k_basic.agentic.p_c}{1.00}
\defval{gb10_tinyllama.gsm8k_basic.agentic.p_s}{0.10}
\defval{gb10_tinyllama.gsm8k_basic.agentic.power_W}{67.8}
\defval{gb10_tinyllama.gsm8k_basic.energy_ratio}{3.66}
\defval{gb10_tinyllama.gsm8k_basic.linear.E_wrong_goals_J}{10,003}
\defval{gb10_tinyllama.gsm8k_basic.linear.EpCG_J}{170.1}
\defval{gb10_tinyllama.gsm8k_basic.linear.EpG_J}{170.1}
\defval{gb10_tinyllama.gsm8k_basic.linear.EpSG_J}{2,551.2}
\defval{gb10_tinyllama.gsm8k_basic.linear.N}{60}
\defval{gb10_tinyllama.gsm8k_basic.linear.TpG_s}{2.6}
\defval{gb10_tinyllama.gsm8k_basic.linear.attempts_per_goal}{1.00}
\defval{gb10_tinyllama.gsm8k_basic.linear.n_completed}{60}
\defval{gb10_tinyllama.gsm8k_basic.linear.n_correct}{4}
\defval{gb10_tinyllama.gsm8k_basic.linear.p_c}{1.00}
\defval{gb10_tinyllama.gsm8k_basic.linear.p_s}{0.07}
\defval{gb10_tinyllama.gsm8k_basic.linear.power_W}{66.7}
\defval{gb10_tinyllama.gsm8k_basic.min_successes}{4}
\defval{gb10_tinyllama.gsm8k_basic.power_ratio}{1.02}
\defval{gb10_tinyllama.gsm8k_basic.power_ratio_hi}{1.04}
\defval{gb10_tinyllama.gsm8k_basic.power_ratio_lo}{0.99}
\defval{gb10_tinyllama.gsm8k_basic.share_energy}{1.45}
\defval{gb10_tinyllama.gsm8k_basic.success_ratio}{0.67}
\defval{gb10_tinyllama.gsm8k_basic.time_ratio}{3.60}
\defval{gb10_tinyllama.gsm8k_basic.time_ratio_hi}{4.68}
\defval{gb10_tinyllama.gsm8k_basic.time_ratio_lo}{2.87}
\defval{gb10_tinyllama.gsm8k_multi_step.OOI_EpCG}{8.88}
\defval{gb10_tinyllama.gsm8k_multi_step.OOI_EpCG_hi}{13.89}
\defval{gb10_tinyllama.gsm8k_multi_step.OOI_EpCG_lo}{6.36}
\defval{gb10_tinyllama.gsm8k_multi_step.OOI_EpG}{8.88}
\defval{gb10_tinyllama.gsm8k_multi_step.OOI_EpG_hi}{13.89}
\defval{gb10_tinyllama.gsm8k_multi_step.OOI_EpG_lo}{6.36}
\defval{gb10_tinyllama.gsm8k_multi_step.OOI_EpSG}{undefined}
\defval{gb10_tinyllama.gsm8k_multi_step.OOI_EpSG_hi}{undefined}
\defval{gb10_tinyllama.gsm8k_multi_step.OOI_EpSG_lo}{undefined}
\defval{gb10_tinyllama.gsm8k_multi_step.agentic.E_wrong_goals_J}{35,706}
\defval{gb10_tinyllama.gsm8k_multi_step.agentic.EpCG_J}{595.1}
\defval{gb10_tinyllama.gsm8k_multi_step.agentic.EpG_J}{595.1}
\defval{gb10_tinyllama.gsm8k_multi_step.agentic.EpSG_J}{undefined}
\defval{gb10_tinyllama.gsm8k_multi_step.agentic.N}{60}
\defval{gb10_tinyllama.gsm8k_multi_step.agentic.TpG_s}{8.9}
\defval{gb10_tinyllama.gsm8k_multi_step.agentic.attempts_per_goal}{1.00}
\defval{gb10_tinyllama.gsm8k_multi_step.agentic.n_completed}{60}
\defval{gb10_tinyllama.gsm8k_multi_step.agentic.n_correct}{0}
\defval{gb10_tinyllama.gsm8k_multi_step.agentic.p_c}{1.00}
\defval{gb10_tinyllama.gsm8k_multi_step.agentic.p_s}{0.00}
\defval{gb10_tinyllama.gsm8k_multi_step.agentic.power_W}{67.0}
\defval{gb10_tinyllama.gsm8k_multi_step.energy_ratio}{8.88}
\defval{gb10_tinyllama.gsm8k_multi_step.linear.E_wrong_goals_J}{3,926}
\defval{gb10_tinyllama.gsm8k_multi_step.linear.EpCG_J}{67.0}
\defval{gb10_tinyllama.gsm8k_multi_step.linear.EpG_J}{67.0}
\defval{gb10_tinyllama.gsm8k_multi_step.linear.EpSG_J}{2,009.6}
\defval{gb10_tinyllama.gsm8k_multi_step.linear.N}{60}
\defval{gb10_tinyllama.gsm8k_multi_step.linear.TpG_s}{0.9}
\defval{gb10_tinyllama.gsm8k_multi_step.linear.attempts_per_goal}{1.00}
\defval{gb10_tinyllama.gsm8k_multi_step.linear.n_completed}{60}
\defval{gb10_tinyllama.gsm8k_multi_step.linear.n_correct}{2}
\defval{gb10_tinyllama.gsm8k_multi_step.linear.p_c}{1.00}
\defval{gb10_tinyllama.gsm8k_multi_step.linear.p_s}{0.03}
\defval{gb10_tinyllama.gsm8k_multi_step.linear.power_W}{72.7}
\defval{gb10_tinyllama.gsm8k_multi_step.min_successes}{0}
\defval{gb10_tinyllama.gsm8k_multi_step.power_ratio}{0.92}
\defval{gb10_tinyllama.gsm8k_multi_step.power_ratio_hi}{0.95}
\defval{gb10_tinyllama.gsm8k_multi_step.power_ratio_lo}{0.87}
\defval{gb10_tinyllama.gsm8k_multi_step.share_energy}{undefined}
\defval{gb10_tinyllama.gsm8k_multi_step.success_ratio}{undefined}
\defval{gb10_tinyllama.gsm8k_multi_step.time_ratio}{9.64}
\defval{gb10_tinyllama.gsm8k_multi_step.time_ratio_hi}{15.83}
\defval{gb10_tinyllama.gsm8k_multi_step.time_ratio_lo}{6.70}
\defval{gb10_tinyllama.logical_reasoning.OOI_EpCG}{6.13}
\defval{gb10_tinyllama.logical_reasoning.OOI_EpCG_hi}{9.03}
\defval{gb10_tinyllama.logical_reasoning.OOI_EpCG_lo}{4.53}
\defval{gb10_tinyllama.logical_reasoning.OOI_EpG}{6.13}
\defval{gb10_tinyllama.logical_reasoning.OOI_EpG_hi}{9.03}
\defval{gb10_tinyllama.logical_reasoning.OOI_EpG_lo}{4.53}
\defval{gb10_tinyllama.logical_reasoning.OOI_EpSG}{4.38}
\defval{gb10_tinyllama.logical_reasoning.OOI_EpSG_hi}{10.20}
\defval{gb10_tinyllama.logical_reasoning.OOI_EpSG_lo}{1.76}
\defval{gb10_tinyllama.logical_reasoning.agentic.E_wrong_goals_J}{29,249}
\defval{gb10_tinyllama.logical_reasoning.agentic.EpCG_J}{625.7}
\defval{gb10_tinyllama.logical_reasoning.agentic.EpG_J}{625.7}
\defval{gb10_tinyllama.logical_reasoning.agentic.EpSG_J}{2,681.5}
\defval{gb10_tinyllama.logical_reasoning.agentic.N}{60}
\defval{gb10_tinyllama.logical_reasoning.agentic.TpG_s}{9.3}
\defval{gb10_tinyllama.logical_reasoning.agentic.attempts_per_goal}{1.00}
\defval{gb10_tinyllama.logical_reasoning.agentic.n_completed}{60}
\defval{gb10_tinyllama.logical_reasoning.agentic.n_correct}{14}
\defval{gb10_tinyllama.logical_reasoning.agentic.p_c}{1.00}
\defval{gb10_tinyllama.logical_reasoning.agentic.p_s}{0.23}
\defval{gb10_tinyllama.logical_reasoning.agentic.power_W}{67.6}
\defval{gb10_tinyllama.logical_reasoning.energy_ratio}{6.13}
\defval{gb10_tinyllama.logical_reasoning.linear.E_wrong_goals_J}{5,534}
\defval{gb10_tinyllama.logical_reasoning.linear.EpCG_J}{102.1}
\defval{gb10_tinyllama.logical_reasoning.linear.EpG_J}{102.1}
\defval{gb10_tinyllama.logical_reasoning.linear.EpSG_J}{612.5}
\defval{gb10_tinyllama.logical_reasoning.linear.N}{60}
\defval{gb10_tinyllama.logical_reasoning.linear.TpG_s}{1.5}
\defval{gb10_tinyllama.logical_reasoning.linear.attempts_per_goal}{1.00}
\defval{gb10_tinyllama.logical_reasoning.linear.n_completed}{60}
\defval{gb10_tinyllama.logical_reasoning.linear.n_correct}{10}
\defval{gb10_tinyllama.logical_reasoning.linear.p_c}{1.00}
\defval{gb10_tinyllama.logical_reasoning.linear.p_s}{0.17}
\defval{gb10_tinyllama.logical_reasoning.linear.power_W}{68.6}
\defval{gb10_tinyllama.logical_reasoning.min_successes}{10}
\defval{gb10_tinyllama.logical_reasoning.power_ratio}{0.98}
\defval{gb10_tinyllama.logical_reasoning.power_ratio_hi}{1.01}
\defval{gb10_tinyllama.logical_reasoning.power_ratio_lo}{0.95}
\defval{gb10_tinyllama.logical_reasoning.share_energy}{1.23}
\defval{gb10_tinyllama.logical_reasoning.success_ratio}{0.71}
\defval{gb10_tinyllama.logical_reasoning.time_ratio}{6.23}
\defval{gb10_tinyllama.logical_reasoning.time_ratio_hi}{9.41}
\defval{gb10_tinyllama.logical_reasoning.time_ratio_lo}{4.55}
\defval{gb10_tinyllama.science_qa.OOI_EpCG}{8.16}
\defval{gb10_tinyllama.science_qa.OOI_EpCG_hi}{14.10}
\defval{gb10_tinyllama.science_qa.OOI_EpCG_lo}{5.50}
\defval{gb10_tinyllama.science_qa.OOI_EpG}{8.16}
\defval{gb10_tinyllama.science_qa.OOI_EpG_hi}{14.10}
\defval{gb10_tinyllama.science_qa.OOI_EpG_lo}{5.50}
\defval{gb10_tinyllama.science_qa.OOI_EpSG}{undefined}
\defval{gb10_tinyllama.science_qa.OOI_EpSG_hi}{undefined}
\defval{gb10_tinyllama.science_qa.OOI_EpSG_lo}{undefined}
\defval{gb10_tinyllama.science_qa.agentic.E_wrong_goals_J}{26,347}
\defval{gb10_tinyllama.science_qa.agentic.EpCG_J}{555.2}
\defval{gb10_tinyllama.science_qa.agentic.EpG_J}{555.2}
\defval{gb10_tinyllama.science_qa.agentic.EpSG_J}{2,379.6}
\defval{gb10_tinyllama.science_qa.agentic.N}{60}
\defval{gb10_tinyllama.science_qa.agentic.TpG_s}{8.6}
\defval{gb10_tinyllama.science_qa.agentic.attempts_per_goal}{1.00}
\defval{gb10_tinyllama.science_qa.agentic.n_completed}{60}
\defval{gb10_tinyllama.science_qa.agentic.n_correct}{14}
\defval{gb10_tinyllama.science_qa.agentic.p_c}{1.00}
\defval{gb10_tinyllama.science_qa.agentic.p_s}{0.23}
\defval{gb10_tinyllama.science_qa.agentic.power_W}{64.5}
\defval{gb10_tinyllama.science_qa.energy_ratio}{8.16}
\defval{gb10_tinyllama.science_qa.linear.E_wrong_goals_J}{4,082}
\defval{gb10_tinyllama.science_qa.linear.EpCG_J}{68.0}
\defval{gb10_tinyllama.science_qa.linear.EpG_J}{68.0}
\defval{gb10_tinyllama.science_qa.linear.EpSG_J}{undefined}
\defval{gb10_tinyllama.science_qa.linear.N}{60}
\defval{gb10_tinyllama.science_qa.linear.TpG_s}{1.2}
\defval{gb10_tinyllama.science_qa.linear.attempts_per_goal}{1.00}
\defval{gb10_tinyllama.science_qa.linear.n_completed}{60}
\defval{gb10_tinyllama.science_qa.linear.n_correct}{0}
\defval{gb10_tinyllama.science_qa.linear.p_c}{1.00}
\defval{gb10_tinyllama.science_qa.linear.p_s}{0.00}
\defval{gb10_tinyllama.science_qa.linear.power_W}{59.0}
\defval{gb10_tinyllama.science_qa.min_successes}{0}
\defval{gb10_tinyllama.science_qa.power_ratio}{1.09}
\defval{gb10_tinyllama.science_qa.power_ratio_hi}{1.12}
\defval{gb10_tinyllama.science_qa.power_ratio_lo}{1.04}
\defval{gb10_tinyllama.science_qa.share_energy}{undefined}
\defval{gb10_tinyllama.science_qa.success_ratio}{0.00}
\defval{gb10_tinyllama.science_qa.time_ratio}{7.46}
\defval{gb10_tinyllama.science_qa.time_ratio_hi}{13.40}
\defval{gb10_tinyllama.science_qa.time_ratio_lo}{4.93}
\defval{gb10_tinyllama.tg_sequential_2.OOI_EpCG}{1.75}
\defval{gb10_tinyllama.tg_sequential_2.OOI_EpCG_hi}{2.69}
\defval{gb10_tinyllama.tg_sequential_2.OOI_EpCG_lo}{1.22}
\defval{gb10_tinyllama.tg_sequential_2.OOI_EpG}{1.75}
\defval{gb10_tinyllama.tg_sequential_2.OOI_EpG_hi}{2.69}
\defval{gb10_tinyllama.tg_sequential_2.OOI_EpG_lo}{1.22}
\defval{gb10_tinyllama.tg_sequential_2.OOI_EpSG}{undefined}
\defval{gb10_tinyllama.tg_sequential_2.OOI_EpSG_hi}{undefined}
\defval{gb10_tinyllama.tg_sequential_2.OOI_EpSG_lo}{undefined}
\defval{gb10_tinyllama.tg_sequential_2.agentic.E_wrong_goals_J}{0}
\defval{gb10_tinyllama.tg_sequential_2.agentic.EpCG_J}{128.1}
\defval{gb10_tinyllama.tg_sequential_2.agentic.EpG_J}{128.1}
\defval{gb10_tinyllama.tg_sequential_2.agentic.EpSG_J}{128.1}
\defval{gb10_tinyllama.tg_sequential_2.agentic.N}{60}
\defval{gb10_tinyllama.tg_sequential_2.agentic.TpG_s}{2.2}
\defval{gb10_tinyllama.tg_sequential_2.agentic.attempts_per_goal}{1.00}
\defval{gb10_tinyllama.tg_sequential_2.agentic.n_completed}{60}
\defval{gb10_tinyllama.tg_sequential_2.agentic.n_correct}{60}
\defval{gb10_tinyllama.tg_sequential_2.agentic.p_c}{1.00}
\defval{gb10_tinyllama.tg_sequential_2.agentic.p_s}{1.00}
\defval{gb10_tinyllama.tg_sequential_2.agentic.power_W}{58.1}
\defval{gb10_tinyllama.tg_sequential_2.energy_ratio}{1.75}
\defval{gb10_tinyllama.tg_sequential_2.linear.E_wrong_goals_J}{4,382}
\defval{gb10_tinyllama.tg_sequential_2.linear.EpCG_J}{73.0}
\defval{gb10_tinyllama.tg_sequential_2.linear.EpG_J}{73.0}
\defval{gb10_tinyllama.tg_sequential_2.linear.EpSG_J}{undefined}
\defval{gb10_tinyllama.tg_sequential_2.linear.N}{60}
\defval{gb10_tinyllama.tg_sequential_2.linear.TpG_s}{1.2}
\defval{gb10_tinyllama.tg_sequential_2.linear.attempts_per_goal}{1.00}
\defval{gb10_tinyllama.tg_sequential_2.linear.n_completed}{60}
\defval{gb10_tinyllama.tg_sequential_2.linear.n_correct}{0}
\defval{gb10_tinyllama.tg_sequential_2.linear.p_c}{1.00}
\defval{gb10_tinyllama.tg_sequential_2.linear.p_s}{0.00}
\defval{gb10_tinyllama.tg_sequential_2.linear.power_W}{58.5}
\defval{gb10_tinyllama.tg_sequential_2.min_successes}{0}
\defval{gb10_tinyllama.tg_sequential_2.power_ratio}{0.99}
\defval{gb10_tinyllama.tg_sequential_2.power_ratio_hi}{1.02}
\defval{gb10_tinyllama.tg_sequential_2.power_ratio_lo}{0.96}
\defval{gb10_tinyllama.tg_sequential_2.share_energy}{undefined}
\defval{gb10_tinyllama.tg_sequential_2.success_ratio}{0.00}
\defval{gb10_tinyllama.tg_sequential_2.time_ratio}{1.77}
\defval{gb10_tinyllama.tg_sequential_2.time_ratio_hi}{2.76}
\defval{gb10_tinyllama.tg_sequential_2.time_ratio_lo}{1.21}
\defval{gb10_tinyllama.tg_single_calc.OOI_EpCG}{0.40}
\defval{gb10_tinyllama.tg_single_calc.OOI_EpCG_hi}{0.69}
\defval{gb10_tinyllama.tg_single_calc.OOI_EpCG_lo}{0.23}
\defval{gb10_tinyllama.tg_single_calc.OOI_EpG}{0.40}
\defval{gb10_tinyllama.tg_single_calc.OOI_EpG_hi}{0.69}
\defval{gb10_tinyllama.tg_single_calc.OOI_EpG_lo}{0.23}
\defval{gb10_tinyllama.tg_single_calc.OOI_EpSG}{undefined}
\defval{gb10_tinyllama.tg_single_calc.OOI_EpSG_hi}{undefined}
\defval{gb10_tinyllama.tg_single_calc.OOI_EpSG_lo}{undefined}
\defval{gb10_tinyllama.tg_single_calc.agentic.E_wrong_goals_J}{868}
\defval{gb10_tinyllama.tg_single_calc.agentic.EpCG_J}{32.9}
\defval{gb10_tinyllama.tg_single_calc.agentic.EpG_J}{32.9}
\defval{gb10_tinyllama.tg_single_calc.agentic.EpSG_J}{38.7}
\defval{gb10_tinyllama.tg_single_calc.agentic.N}{60}
\defval{gb10_tinyllama.tg_single_calc.agentic.TpG_s}{0.4}
\defval{gb10_tinyllama.tg_single_calc.agentic.attempts_per_goal}{1.00}
\defval{gb10_tinyllama.tg_single_calc.agentic.n_completed}{60}
\defval{gb10_tinyllama.tg_single_calc.agentic.n_correct}{51}
\defval{gb10_tinyllama.tg_single_calc.agentic.p_c}{1.00}
\defval{gb10_tinyllama.tg_single_calc.agentic.p_s}{0.85}
\defval{gb10_tinyllama.tg_single_calc.agentic.power_W}{75.1}
\defval{gb10_tinyllama.tg_single_calc.energy_ratio}{0.40}
\defval{gb10_tinyllama.tg_single_calc.linear.E_wrong_goals_J}{4,877}
\defval{gb10_tinyllama.tg_single_calc.linear.EpCG_J}{81.3}
\defval{gb10_tinyllama.tg_single_calc.linear.EpG_J}{81.3}
\defval{gb10_tinyllama.tg_single_calc.linear.EpSG_J}{undefined}
\defval{gb10_tinyllama.tg_single_calc.linear.N}{60}
\defval{gb10_tinyllama.tg_single_calc.linear.TpG_s}{1.2}
\defval{gb10_tinyllama.tg_single_calc.linear.attempts_per_goal}{1.00}
\defval{gb10_tinyllama.tg_single_calc.linear.n_completed}{60}
\defval{gb10_tinyllama.tg_single_calc.linear.n_correct}{0}
\defval{gb10_tinyllama.tg_single_calc.linear.p_c}{1.00}
\defval{gb10_tinyllama.tg_single_calc.linear.p_s}{0.00}
\defval{gb10_tinyllama.tg_single_calc.linear.power_W}{67.7}
\defval{gb10_tinyllama.tg_single_calc.min_successes}{0}
\defval{gb10_tinyllama.tg_single_calc.power_ratio}{1.11}
\defval{gb10_tinyllama.tg_single_calc.power_ratio_hi}{1.26}
\defval{gb10_tinyllama.tg_single_calc.power_ratio_lo}{1.03}
\defval{gb10_tinyllama.tg_single_calc.share_energy}{undefined}
\defval{gb10_tinyllama.tg_single_calc.success_ratio}{0.00}
\defval{gb10_tinyllama.tg_single_calc.time_ratio}{0.37}
\defval{gb10_tinyllama.tg_single_calc.time_ratio_hi}{0.65}
\defval{gb10_tinyllama.tg_single_calc.time_ratio_lo}{0.19}
\defval{gb10_tinyllama.tg_single_db.OOI_EpCG}{0.53}
\defval{gb10_tinyllama.tg_single_db.OOI_EpCG_hi}{0.90}
\defval{gb10_tinyllama.tg_single_db.OOI_EpCG_lo}{0.32}
\defval{gb10_tinyllama.tg_single_db.OOI_EpG}{0.53}
\defval{gb10_tinyllama.tg_single_db.OOI_EpG_hi}{0.90}
\defval{gb10_tinyllama.tg_single_db.OOI_EpG_lo}{0.32}
\defval{gb10_tinyllama.tg_single_db.OOI_EpSG}{undefined}
\defval{gb10_tinyllama.tg_single_db.OOI_EpSG_hi}{undefined}
\defval{gb10_tinyllama.tg_single_db.OOI_EpSG_lo}{undefined}
\defval{gb10_tinyllama.tg_single_db.agentic.E_wrong_goals_J}{0}
\defval{gb10_tinyllama.tg_single_db.agentic.EpCG_J}{34.2}
\defval{gb10_tinyllama.tg_single_db.agentic.EpG_J}{34.2}
\defval{gb10_tinyllama.tg_single_db.agentic.EpSG_J}{34.2}
\defval{gb10_tinyllama.tg_single_db.agentic.N}{60}
\defval{gb10_tinyllama.tg_single_db.agentic.TpG_s}{0.5}
\defval{gb10_tinyllama.tg_single_db.agentic.attempts_per_goal}{1.00}
\defval{gb10_tinyllama.tg_single_db.agentic.n_completed}{60}
\defval{gb10_tinyllama.tg_single_db.agentic.n_correct}{60}
\defval{gb10_tinyllama.tg_single_db.agentic.p_c}{1.00}
\defval{gb10_tinyllama.tg_single_db.agentic.p_s}{1.00}
\defval{gb10_tinyllama.tg_single_db.agentic.power_W}{72.0}
\defval{gb10_tinyllama.tg_single_db.energy_ratio}{0.53}
\defval{gb10_tinyllama.tg_single_db.linear.E_wrong_goals_J}{3,851}
\defval{gb10_tinyllama.tg_single_db.linear.EpCG_J}{64.2}
\defval{gb10_tinyllama.tg_single_db.linear.EpG_J}{64.2}
\defval{gb10_tinyllama.tg_single_db.linear.EpSG_J}{undefined}
\defval{gb10_tinyllama.tg_single_db.linear.N}{60}
\defval{gb10_tinyllama.tg_single_db.linear.TpG_s}{0.9}
\defval{gb10_tinyllama.tg_single_db.linear.attempts_per_goal}{1.00}
\defval{gb10_tinyllama.tg_single_db.linear.n_completed}{60}
\defval{gb10_tinyllama.tg_single_db.linear.n_correct}{0}
\defval{gb10_tinyllama.tg_single_db.linear.p_c}{1.00}
\defval{gb10_tinyllama.tg_single_db.linear.p_s}{0.00}
\defval{gb10_tinyllama.tg_single_db.linear.power_W}{68.3}
\defval{gb10_tinyllama.tg_single_db.min_successes}{0}
\defval{gb10_tinyllama.tg_single_db.power_ratio}{1.05}
\defval{gb10_tinyllama.tg_single_db.power_ratio_hi}{1.16}
\defval{gb10_tinyllama.tg_single_db.power_ratio_lo}{0.98}
\defval{gb10_tinyllama.tg_single_db.share_energy}{undefined}
\defval{gb10_tinyllama.tg_single_db.success_ratio}{0.00}
\defval{gb10_tinyllama.tg_single_db.time_ratio}{0.51}
\defval{gb10_tinyllama.tg_single_db.time_ratio_hi}{0.91}
\defval{gb10_tinyllama.tg_single_db.time_ratio_lo}{0.27}
\defval{mac_tinyllama.commonsense_reasoning.OOI_EpCG}{7.53}
\defval{mac_tinyllama.commonsense_reasoning.OOI_EpCG_hi}{13.89}
\defval{mac_tinyllama.commonsense_reasoning.OOI_EpCG_lo}{4.76}
\defval{mac_tinyllama.commonsense_reasoning.OOI_EpG}{7.53}
\defval{mac_tinyllama.commonsense_reasoning.OOI_EpG_hi}{13.89}
\defval{mac_tinyllama.commonsense_reasoning.OOI_EpG_lo}{4.76}
\defval{mac_tinyllama.commonsense_reasoning.OOI_EpSG}{1.56}
\defval{mac_tinyllama.commonsense_reasoning.OOI_EpSG_hi}{3.69}
\defval{mac_tinyllama.commonsense_reasoning.OOI_EpSG_lo}{0.45}
\defval{mac_tinyllama.commonsense_reasoning.agentic.E_wrong_goals_J}{11,903}
\defval{mac_tinyllama.commonsense_reasoning.agentic.EpCG_J}{653.3}
\defval{mac_tinyllama.commonsense_reasoning.agentic.EpG_J}{653.3}
\defval{mac_tinyllama.commonsense_reasoning.agentic.EpSG_J}{1,446.5}
\defval{mac_tinyllama.commonsense_reasoning.agentic.N}{31}
\defval{mac_tinyllama.commonsense_reasoning.agentic.TpG_s}{149.4}
\defval{mac_tinyllama.commonsense_reasoning.agentic.attempts_per_goal}{1.00}
\defval{mac_tinyllama.commonsense_reasoning.agentic.n_completed}{31}
\defval{mac_tinyllama.commonsense_reasoning.agentic.n_correct}{14}
\defval{mac_tinyllama.commonsense_reasoning.agentic.p_c}{1.00}
\defval{mac_tinyllama.commonsense_reasoning.agentic.p_s}{0.45}
\defval{mac_tinyllama.commonsense_reasoning.agentic.power_W}{4.4}
\defval{mac_tinyllama.commonsense_reasoning.energy_ratio}{7.53}
\defval{mac_tinyllama.commonsense_reasoning.linear.E_wrong_goals_J}{1,951}
\defval{mac_tinyllama.commonsense_reasoning.linear.EpCG_J}{86.8}
\defval{mac_tinyllama.commonsense_reasoning.linear.EpG_J}{86.8}
\defval{mac_tinyllama.commonsense_reasoning.linear.EpSG_J}{925.7}
\defval{mac_tinyllama.commonsense_reasoning.linear.N}{32}
\defval{mac_tinyllama.commonsense_reasoning.linear.TpG_s}{29.9}
\defval{mac_tinyllama.commonsense_reasoning.linear.attempts_per_goal}{1.00}
\defval{mac_tinyllama.commonsense_reasoning.linear.n_completed}{32}
\defval{mac_tinyllama.commonsense_reasoning.linear.n_correct}{3}
\defval{mac_tinyllama.commonsense_reasoning.linear.p_c}{1.00}
\defval{mac_tinyllama.commonsense_reasoning.linear.p_s}{0.09}
\defval{mac_tinyllama.commonsense_reasoning.linear.power_W}{2.9}
\defval{mac_tinyllama.commonsense_reasoning.min_successes}{3}
\defval{mac_tinyllama.commonsense_reasoning.power_ratio}{1.50}
\defval{mac_tinyllama.commonsense_reasoning.power_ratio_hi}{1.85}
\defval{mac_tinyllama.commonsense_reasoning.power_ratio_lo}{1.25}
\defval{mac_tinyllama.commonsense_reasoning.share_energy}{4.52}
\defval{mac_tinyllama.commonsense_reasoning.success_ratio}{0.21}
\defval{mac_tinyllama.commonsense_reasoning.time_ratio}{5.00}
\defval{mac_tinyllama.commonsense_reasoning.time_ratio_hi}{9.19}
\defval{mac_tinyllama.commonsense_reasoning.time_ratio_lo}{3.25}
\defval{mac_tinyllama.excluded}{561}
\defval{mac_tinyllama.factual_qa.OOI_EpCG}{4.24}
\defval{mac_tinyllama.factual_qa.OOI_EpCG_hi}{7.40}
\defval{mac_tinyllama.factual_qa.OOI_EpCG_lo}{2.94}
\defval{mac_tinyllama.factual_qa.OOI_EpG}{4.24}
\defval{mac_tinyllama.factual_qa.OOI_EpG_hi}{7.40}
\defval{mac_tinyllama.factual_qa.OOI_EpG_lo}{2.94}
\defval{mac_tinyllama.factual_qa.OOI_EpSG}{3.52}
\defval{mac_tinyllama.factual_qa.OOI_EpSG_hi}{7.64}
\defval{mac_tinyllama.factual_qa.OOI_EpSG_lo}{1.77}
\defval{mac_tinyllama.factual_qa.agentic.E_wrong_goals_J}{11,014}
\defval{mac_tinyllama.factual_qa.agentic.EpCG_J}{720.1}
\defval{mac_tinyllama.factual_qa.agentic.EpG_J}{720.1}
\defval{mac_tinyllama.factual_qa.agentic.EpSG_J}{1,495.6}
\defval{mac_tinyllama.factual_qa.agentic.N}{27}
\defval{mac_tinyllama.factual_qa.agentic.TpG_s}{181.3}
\defval{mac_tinyllama.factual_qa.agentic.attempts_per_goal}{1.00}
\defval{mac_tinyllama.factual_qa.agentic.n_completed}{27}
\defval{mac_tinyllama.factual_qa.agentic.n_correct}{13}
\defval{mac_tinyllama.factual_qa.agentic.p_c}{1.00}
\defval{mac_tinyllama.factual_qa.agentic.p_s}{0.48}
\defval{mac_tinyllama.factual_qa.agentic.power_W}{4.0}
\defval{mac_tinyllama.factual_qa.energy_ratio}{4.24}
\defval{mac_tinyllama.factual_qa.linear.E_wrong_goals_J}{2,329}
\defval{mac_tinyllama.factual_qa.linear.EpCG_J}{169.9}
\defval{mac_tinyllama.factual_qa.linear.EpG_J}{169.9}
\defval{mac_tinyllama.factual_qa.linear.EpSG_J}{424.8}
\defval{mac_tinyllama.factual_qa.linear.N}{30}
\defval{mac_tinyllama.factual_qa.linear.TpG_s}{48.5}
\defval{mac_tinyllama.factual_qa.linear.attempts_per_goal}{1.00}
\defval{mac_tinyllama.factual_qa.linear.n_completed}{30}
\defval{mac_tinyllama.factual_qa.linear.n_correct}{12}
\defval{mac_tinyllama.factual_qa.linear.p_c}{1.00}
\defval{mac_tinyllama.factual_qa.linear.p_s}{0.40}
\defval{mac_tinyllama.factual_qa.linear.power_W}{3.5}
\defval{mac_tinyllama.factual_qa.min_successes}{12}
\defval{mac_tinyllama.factual_qa.power_ratio}{1.13}
\defval{mac_tinyllama.factual_qa.power_ratio_hi}{1.48}
\defval{mac_tinyllama.factual_qa.power_ratio_lo}{0.86}
\defval{mac_tinyllama.factual_qa.share_energy}{1.15}
\defval{mac_tinyllama.factual_qa.success_ratio}{0.83}
\defval{mac_tinyllama.factual_qa.time_ratio}{3.74}
\defval{mac_tinyllama.factual_qa.time_ratio_hi}{6.43}
\defval{mac_tinyllama.factual_qa.time_ratio_lo}{2.50}
\defval{mac_tinyllama.goals}{486}
\defval{mac_tinyllama.gsm8k_basic.OOI_EpCG}{3.95}
\defval{mac_tinyllama.gsm8k_basic.OOI_EpCG_hi}{5.43}
\defval{mac_tinyllama.gsm8k_basic.OOI_EpCG_lo}{2.93}
\defval{mac_tinyllama.gsm8k_basic.OOI_EpG}{3.95}
\defval{mac_tinyllama.gsm8k_basic.OOI_EpG_hi}{5.43}
\defval{mac_tinyllama.gsm8k_basic.OOI_EpG_lo}{2.93}
\defval{mac_tinyllama.gsm8k_basic.OOI_EpSG}{1.84}
\defval{mac_tinyllama.gsm8k_basic.OOI_EpSG_hi}{undefined}
\defval{mac_tinyllama.gsm8k_basic.OOI_EpSG_lo}{undefined}
\defval{mac_tinyllama.gsm8k_basic.agentic.E_wrong_goals_J}{10,117}
\defval{mac_tinyllama.gsm8k_basic.agentic.EpCG_J}{247.4}
\defval{mac_tinyllama.gsm8k_basic.agentic.EpG_J}{247.4}
\defval{mac_tinyllama.gsm8k_basic.agentic.EpSG_J}{5,195.9}
\defval{mac_tinyllama.gsm8k_basic.agentic.N}{42}
\defval{mac_tinyllama.gsm8k_basic.agentic.TpG_s}{32.1}
\defval{mac_tinyllama.gsm8k_basic.agentic.attempts_per_goal}{1.00}
\defval{mac_tinyllama.gsm8k_basic.agentic.n_completed}{42}
\defval{mac_tinyllama.gsm8k_basic.agentic.n_correct}{2}
\defval{mac_tinyllama.gsm8k_basic.agentic.p_c}{1.00}
\defval{mac_tinyllama.gsm8k_basic.agentic.p_s}{0.05}
\defval{mac_tinyllama.gsm8k_basic.agentic.power_W}{7.7}
\defval{mac_tinyllama.gsm8k_basic.energy_ratio}{3.95}
\defval{mac_tinyllama.gsm8k_basic.linear.E_wrong_goals_J}{2,817}
\defval{mac_tinyllama.gsm8k_basic.linear.EpCG_J}{62.6}
\defval{mac_tinyllama.gsm8k_basic.linear.EpG_J}{62.6}
\defval{mac_tinyllama.gsm8k_basic.linear.EpSG_J}{2,816.7}
\defval{mac_tinyllama.gsm8k_basic.linear.N}{45}
\defval{mac_tinyllama.gsm8k_basic.linear.TpG_s}{8.6}
\defval{mac_tinyllama.gsm8k_basic.linear.attempts_per_goal}{1.00}
\defval{mac_tinyllama.gsm8k_basic.linear.n_completed}{45}
\defval{mac_tinyllama.gsm8k_basic.linear.n_correct}{1}
\defval{mac_tinyllama.gsm8k_basic.linear.p_c}{1.00}
\defval{mac_tinyllama.gsm8k_basic.linear.p_s}{0.02}
\defval{mac_tinyllama.gsm8k_basic.linear.power_W}{7.3}
\defval{mac_tinyllama.gsm8k_basic.min_successes}{1}
\defval{mac_tinyllama.gsm8k_basic.power_ratio}{1.06}
\defval{mac_tinyllama.gsm8k_basic.power_ratio_hi}{1.25}
\defval{mac_tinyllama.gsm8k_basic.power_ratio_lo}{0.86}
\defval{mac_tinyllama.gsm8k_basic.share_energy}{2.24}
\defval{mac_tinyllama.gsm8k_basic.success_ratio}{0.47}
\defval{mac_tinyllama.gsm8k_basic.time_ratio}{3.74}
\defval{mac_tinyllama.gsm8k_basic.time_ratio_hi}{4.87}
\defval{mac_tinyllama.gsm8k_basic.time_ratio_lo}{2.98}
\defval{mac_tinyllama.gsm8k_multi_step.OOI_EpCG}{3.17}
\defval{mac_tinyllama.gsm8k_multi_step.OOI_EpCG_hi}{6.37}
\defval{mac_tinyllama.gsm8k_multi_step.OOI_EpCG_lo}{1.85}
\defval{mac_tinyllama.gsm8k_multi_step.OOI_EpG}{3.17}
\defval{mac_tinyllama.gsm8k_multi_step.OOI_EpG_hi}{6.37}
\defval{mac_tinyllama.gsm8k_multi_step.OOI_EpG_lo}{1.85}
\defval{mac_tinyllama.gsm8k_multi_step.OOI_EpSG}{undefined}
\defval{mac_tinyllama.gsm8k_multi_step.OOI_EpSG_hi}{undefined}
\defval{mac_tinyllama.gsm8k_multi_step.OOI_EpSG_lo}{undefined}
\defval{mac_tinyllama.gsm8k_multi_step.agentic.E_wrong_goals_J}{5,198}
\defval{mac_tinyllama.gsm8k_multi_step.agentic.EpCG_J}{123.8}
\defval{mac_tinyllama.gsm8k_multi_step.agentic.EpG_J}{123.8}
\defval{mac_tinyllama.gsm8k_multi_step.agentic.EpSG_J}{undefined}
\defval{mac_tinyllama.gsm8k_multi_step.agentic.N}{42}
\defval{mac_tinyllama.gsm8k_multi_step.agentic.TpG_s}{27.6}
\defval{mac_tinyllama.gsm8k_multi_step.agentic.attempts_per_goal}{1.00}
\defval{mac_tinyllama.gsm8k_multi_step.agentic.n_completed}{42}
\defval{mac_tinyllama.gsm8k_multi_step.agentic.n_correct}{0}
\defval{mac_tinyllama.gsm8k_multi_step.agentic.p_c}{1.00}
\defval{mac_tinyllama.gsm8k_multi_step.agentic.p_s}{0.00}
\defval{mac_tinyllama.gsm8k_multi_step.agentic.power_W}{4.5}
\defval{mac_tinyllama.gsm8k_multi_step.energy_ratio}{3.17}
\defval{mac_tinyllama.gsm8k_multi_step.linear.E_wrong_goals_J}{1,716}
\defval{mac_tinyllama.gsm8k_multi_step.linear.EpCG_J}{39.0}
\defval{mac_tinyllama.gsm8k_multi_step.linear.EpG_J}{39.0}
\defval{mac_tinyllama.gsm8k_multi_step.linear.EpSG_J}{undefined}
\defval{mac_tinyllama.gsm8k_multi_step.linear.N}{44}
\defval{mac_tinyllama.gsm8k_multi_step.linear.TpG_s}{7.9}
\defval{mac_tinyllama.gsm8k_multi_step.linear.attempts_per_goal}{1.00}
\defval{mac_tinyllama.gsm8k_multi_step.linear.n_completed}{44}
\defval{mac_tinyllama.gsm8k_multi_step.linear.n_correct}{0}
\defval{mac_tinyllama.gsm8k_multi_step.linear.p_c}{1.00}
\defval{mac_tinyllama.gsm8k_multi_step.linear.p_s}{0.00}
\defval{mac_tinyllama.gsm8k_multi_step.linear.power_W}{4.9}
\defval{mac_tinyllama.gsm8k_multi_step.min_successes}{0}
\defval{mac_tinyllama.gsm8k_multi_step.power_ratio}{0.91}
\defval{mac_tinyllama.gsm8k_multi_step.power_ratio_hi}{1.35}
\defval{mac_tinyllama.gsm8k_multi_step.power_ratio_lo}{0.59}
\defval{mac_tinyllama.gsm8k_multi_step.share_energy}{undefined}
\defval{mac_tinyllama.gsm8k_multi_step.success_ratio}{undefined}
\defval{mac_tinyllama.gsm8k_multi_step.time_ratio}{3.48}
\defval{mac_tinyllama.gsm8k_multi_step.time_ratio_hi}{6.39}
\defval{mac_tinyllama.gsm8k_multi_step.time_ratio_lo}{2.21}
\defval{mac_tinyllama.logical_reasoning.OOI_EpCG}{6.00}
\defval{mac_tinyllama.logical_reasoning.OOI_EpCG_hi}{10.94}
\defval{mac_tinyllama.logical_reasoning.OOI_EpCG_lo}{3.95}
\defval{mac_tinyllama.logical_reasoning.OOI_EpG}{6.00}
\defval{mac_tinyllama.logical_reasoning.OOI_EpG_hi}{10.94}
\defval{mac_tinyllama.logical_reasoning.OOI_EpG_lo}{3.95}
\defval{mac_tinyllama.logical_reasoning.OOI_EpSG}{5.64}
\defval{mac_tinyllama.logical_reasoning.OOI_EpSG_hi}{16.31}
\defval{mac_tinyllama.logical_reasoning.OOI_EpSG_lo}{2.18}
\defval{mac_tinyllama.logical_reasoning.agentic.E_wrong_goals_J}{12,059}
\defval{mac_tinyllama.logical_reasoning.agentic.EpCG_J}{520.2}
\defval{mac_tinyllama.logical_reasoning.agentic.EpG_J}{520.2}
\defval{mac_tinyllama.logical_reasoning.agentic.EpSG_J}{1,849.7}
\defval{mac_tinyllama.logical_reasoning.agentic.N}{32}
\defval{mac_tinyllama.logical_reasoning.agentic.TpG_s}{128.0}
\defval{mac_tinyllama.logical_reasoning.agentic.attempts_per_goal}{1.00}
\defval{mac_tinyllama.logical_reasoning.agentic.n_completed}{32}
\defval{mac_tinyllama.logical_reasoning.agentic.n_correct}{9}
\defval{mac_tinyllama.logical_reasoning.agentic.p_c}{1.00}
\defval{mac_tinyllama.logical_reasoning.agentic.p_s}{0.28}
\defval{mac_tinyllama.logical_reasoning.agentic.power_W}{4.1}
\defval{mac_tinyllama.logical_reasoning.energy_ratio}{6.00}
\defval{mac_tinyllama.logical_reasoning.linear.E_wrong_goals_J}{2,023}
\defval{mac_tinyllama.logical_reasoning.linear.EpCG_J}{86.7}
\defval{mac_tinyllama.logical_reasoning.linear.EpG_J}{86.7}
\defval{mac_tinyllama.logical_reasoning.linear.EpSG_J}{327.7}
\defval{mac_tinyllama.logical_reasoning.linear.N}{34}
\defval{mac_tinyllama.logical_reasoning.linear.TpG_s}{22.0}
\defval{mac_tinyllama.logical_reasoning.linear.attempts_per_goal}{1.00}
\defval{mac_tinyllama.logical_reasoning.linear.n_completed}{34}
\defval{mac_tinyllama.logical_reasoning.linear.n_correct}{9}
\defval{mac_tinyllama.logical_reasoning.linear.p_c}{1.00}
\defval{mac_tinyllama.logical_reasoning.linear.p_s}{0.26}
\defval{mac_tinyllama.logical_reasoning.linear.power_W}{3.9}
\defval{mac_tinyllama.logical_reasoning.min_successes}{9}
\defval{mac_tinyllama.logical_reasoning.power_ratio}{1.03}
\defval{mac_tinyllama.logical_reasoning.power_ratio_hi}{1.36}
\defval{mac_tinyllama.logical_reasoning.power_ratio_lo}{0.76}
\defval{mac_tinyllama.logical_reasoning.share_energy}{1.04}
\defval{mac_tinyllama.logical_reasoning.success_ratio}{0.94}
\defval{mac_tinyllama.logical_reasoning.time_ratio}{5.82}
\defval{mac_tinyllama.logical_reasoning.time_ratio_hi}{10.42}
\defval{mac_tinyllama.logical_reasoning.time_ratio_lo}{3.98}
\defval{mac_tinyllama.science_qa.OOI_EpCG}{17.71}
\defval{mac_tinyllama.science_qa.OOI_EpCG_hi}{247.03}
\defval{mac_tinyllama.science_qa.OOI_EpCG_lo}{8.02}
\defval{mac_tinyllama.science_qa.OOI_EpG}{17.71}
\defval{mac_tinyllama.science_qa.OOI_EpG_hi}{247.03}
\defval{mac_tinyllama.science_qa.OOI_EpG_lo}{8.02}
\defval{mac_tinyllama.science_qa.OOI_EpSG}{3.54}
\defval{mac_tinyllama.science_qa.OOI_EpSG_hi}{undefined}
\defval{mac_tinyllama.science_qa.OOI_EpSG_lo}{undefined}
\defval{mac_tinyllama.science_qa.agentic.E_wrong_goals_J}{25,546}
\defval{mac_tinyllama.science_qa.agentic.EpCG_J}{1,022.5}
\defval{mac_tinyllama.science_qa.agentic.EpG_J}{1,022.5}
\defval{mac_tinyllama.science_qa.agentic.EpSG_J}{5,930.5}
\defval{mac_tinyllama.science_qa.agentic.N}{29}
\defval{mac_tinyllama.science_qa.agentic.TpG_s}{238.8}
\defval{mac_tinyllama.science_qa.agentic.attempts_per_goal}{1.00}
\defval{mac_tinyllama.science_qa.agentic.n_completed}{29}
\defval{mac_tinyllama.science_qa.agentic.n_correct}{5}
\defval{mac_tinyllama.science_qa.agentic.p_c}{1.00}
\defval{mac_tinyllama.science_qa.agentic.p_s}{0.17}
\defval{mac_tinyllama.science_qa.agentic.power_W}{4.3}
\defval{mac_tinyllama.science_qa.energy_ratio}{17.71}
\defval{mac_tinyllama.science_qa.linear.E_wrong_goals_J}{1,672}
\defval{mac_tinyllama.science_qa.linear.EpCG_J}{57.7}
\defval{mac_tinyllama.science_qa.linear.EpG_J}{57.7}
\defval{mac_tinyllama.science_qa.linear.EpSG_J}{1,674.4}
\defval{mac_tinyllama.science_qa.linear.N}{29}
\defval{mac_tinyllama.science_qa.linear.TpG_s}{13.7}
\defval{mac_tinyllama.science_qa.linear.attempts_per_goal}{1.00}
\defval{mac_tinyllama.science_qa.linear.n_completed}{29}
\defval{mac_tinyllama.science_qa.linear.n_correct}{1}
\defval{mac_tinyllama.science_qa.linear.p_c}{1.00}
\defval{mac_tinyllama.science_qa.linear.p_s}{0.03}
\defval{mac_tinyllama.science_qa.linear.power_W}{4.2}
\defval{mac_tinyllama.science_qa.min_successes}{1}
\defval{mac_tinyllama.science_qa.power_ratio}{1.02}
\defval{mac_tinyllama.science_qa.power_ratio_hi}{2.77}
\defval{mac_tinyllama.science_qa.power_ratio_lo}{0.51}
\defval{mac_tinyllama.science_qa.share_energy}{2.27}
\defval{mac_tinyllama.science_qa.success_ratio}{0.20}
\defval{mac_tinyllama.science_qa.time_ratio}{17.37}
\defval{mac_tinyllama.science_qa.time_ratio_hi}{92.51}
\defval{mac_tinyllama.science_qa.time_ratio_lo}{7.67}
\defval{mac_tinyllama.tg_single_db.OOI_EpCG}{0.43}
\defval{mac_tinyllama.tg_single_db.OOI_EpCG_hi}{0.97}
\defval{mac_tinyllama.tg_single_db.OOI_EpCG_lo}{0.20}
\defval{mac_tinyllama.tg_single_db.OOI_EpG}{0.43}
\defval{mac_tinyllama.tg_single_db.OOI_EpG_hi}{0.97}
\defval{mac_tinyllama.tg_single_db.OOI_EpG_lo}{0.20}
\defval{mac_tinyllama.tg_single_db.OOI_EpSG}{undefined}
\defval{mac_tinyllama.tg_single_db.OOI_EpSG_hi}{undefined}
\defval{mac_tinyllama.tg_single_db.OOI_EpSG_lo}{undefined}
\defval{mac_tinyllama.tg_single_db.agentic.E_wrong_goals_J}{0}
\defval{mac_tinyllama.tg_single_db.agentic.EpCG_J}{58.0}
\defval{mac_tinyllama.tg_single_db.agentic.EpG_J}{58.0}
\defval{mac_tinyllama.tg_single_db.agentic.EpSG_J}{58.0}
\defval{mac_tinyllama.tg_single_db.agentic.N}{33}
\defval{mac_tinyllama.tg_single_db.agentic.TpG_s}{21.1}
\defval{mac_tinyllama.tg_single_db.agentic.attempts_per_goal}{1.00}
\defval{mac_tinyllama.tg_single_db.agentic.n_completed}{33}
\defval{mac_tinyllama.tg_single_db.agentic.n_correct}{33}
\defval{mac_tinyllama.tg_single_db.agentic.p_c}{1.00}
\defval{mac_tinyllama.tg_single_db.agentic.p_s}{1.00}
\defval{mac_tinyllama.tg_single_db.agentic.power_W}{2.8}
\defval{mac_tinyllama.tg_single_db.energy_ratio}{0.43}
\defval{mac_tinyllama.tg_single_db.linear.E_wrong_goals_J}{4,827}
\defval{mac_tinyllama.tg_single_db.linear.EpCG_J}{134.1}
\defval{mac_tinyllama.tg_single_db.linear.EpG_J}{134.1}
\defval{mac_tinyllama.tg_single_db.linear.EpSG_J}{undefined}
\defval{mac_tinyllama.tg_single_db.linear.N}{36}
\defval{mac_tinyllama.tg_single_db.linear.TpG_s}{46.7}
\defval{mac_tinyllama.tg_single_db.linear.attempts_per_goal}{1.00}
\defval{mac_tinyllama.tg_single_db.linear.n_completed}{36}
\defval{mac_tinyllama.tg_single_db.linear.n_correct}{0}
\defval{mac_tinyllama.tg_single_db.linear.p_c}{1.00}
\defval{mac_tinyllama.tg_single_db.linear.p_s}{0.00}
\defval{mac_tinyllama.tg_single_db.linear.power_W}{2.9}
\defval{mac_tinyllama.tg_single_db.min_successes}{0}
\defval{mac_tinyllama.tg_single_db.power_ratio}{0.96}
\defval{mac_tinyllama.tg_single_db.power_ratio_hi}{1.09}
\defval{mac_tinyllama.tg_single_db.power_ratio_lo}{0.83}
\defval{mac_tinyllama.tg_single_db.share_energy}{undefined}
\defval{mac_tinyllama.tg_single_db.success_ratio}{0.00}
\defval{mac_tinyllama.tg_single_db.time_ratio}{0.45}
\defval{mac_tinyllama.tg_single_db.time_ratio_hi}{1.00}
\defval{mac_tinyllama.tg_single_db.time_ratio_lo}{0.21}
\defval{mac_tinyllama_unfiltered.commonsense_reasoning.OOI_EpCG}{5.59}
\defval{mac_tinyllama_unfiltered.commonsense_reasoning.OOI_EpCG_hi}{7.60}
\defval{mac_tinyllama_unfiltered.commonsense_reasoning.OOI_EpCG_lo}{4.35}
\defval{mac_tinyllama_unfiltered.commonsense_reasoning.OOI_EpG}{5.59}
\defval{mac_tinyllama_unfiltered.commonsense_reasoning.OOI_EpG_hi}{7.60}
\defval{mac_tinyllama_unfiltered.commonsense_reasoning.OOI_EpG_lo}{4.35}
\defval{mac_tinyllama_unfiltered.commonsense_reasoning.OOI_EpSG}{1.20}
\defval{mac_tinyllama_unfiltered.commonsense_reasoning.OOI_EpSG_hi}{3.22}
\defval{mac_tinyllama_unfiltered.commonsense_reasoning.OOI_EpSG_lo}{0.32}
\defval{mac_tinyllama_unfiltered.commonsense_reasoning.agentic.E_wrong_goals_J}{11,903}
\defval{mac_tinyllama_unfiltered.commonsense_reasoning.agentic.EpCG_J}{760.7}
\defval{mac_tinyllama_unfiltered.commonsense_reasoning.agentic.EpG_J}{760.7}
\defval{mac_tinyllama_unfiltered.commonsense_reasoning.agentic.EpSG_J}{3,260.2}
\defval{mac_tinyllama_unfiltered.commonsense_reasoning.agentic.N}{60}
\defval{mac_tinyllama_unfiltered.commonsense_reasoning.agentic.TpG_s}{104.3}
\defval{mac_tinyllama_unfiltered.commonsense_reasoning.agentic.attempts_per_goal}{1.00}
\defval{mac_tinyllama_unfiltered.commonsense_reasoning.agentic.n_completed}{60}
\defval{mac_tinyllama_unfiltered.commonsense_reasoning.agentic.n_correct}{14}
\defval{mac_tinyllama_unfiltered.commonsense_reasoning.agentic.p_c}{1.00}
\defval{mac_tinyllama_unfiltered.commonsense_reasoning.agentic.p_s}{0.23}
\defval{mac_tinyllama_unfiltered.commonsense_reasoning.agentic.power_W}{7.3}
\defval{mac_tinyllama_unfiltered.commonsense_reasoning.energy_ratio}{5.59}
\defval{mac_tinyllama_unfiltered.commonsense_reasoning.linear.E_wrong_goals_J}{1,951}
\defval{mac_tinyllama_unfiltered.commonsense_reasoning.linear.EpCG_J}{136.0}
\defval{mac_tinyllama_unfiltered.commonsense_reasoning.linear.EpG_J}{136.0}
\defval{mac_tinyllama_unfiltered.commonsense_reasoning.linear.EpSG_J}{2,720.6}
\defval{mac_tinyllama_unfiltered.commonsense_reasoning.linear.N}{60}
\defval{mac_tinyllama_unfiltered.commonsense_reasoning.linear.TpG_s}{20.0}
\defval{mac_tinyllama_unfiltered.commonsense_reasoning.linear.attempts_per_goal}{1.00}
\defval{mac_tinyllama_unfiltered.commonsense_reasoning.linear.n_completed}{60}
\defval{mac_tinyllama_unfiltered.commonsense_reasoning.linear.n_correct}{3}
\defval{mac_tinyllama_unfiltered.commonsense_reasoning.linear.p_c}{1.00}
\defval{mac_tinyllama_unfiltered.commonsense_reasoning.linear.p_s}{0.05}
\defval{mac_tinyllama_unfiltered.commonsense_reasoning.linear.power_W}{6.8}
\defval{mac_tinyllama_unfiltered.commonsense_reasoning.min_successes}{3}
\defval{mac_tinyllama_unfiltered.commonsense_reasoning.power_ratio}{1.07}
\defval{mac_tinyllama_unfiltered.commonsense_reasoning.power_ratio_hi}{1.55}
\defval{mac_tinyllama_unfiltered.commonsense_reasoning.power_ratio_lo}{0.70}
\defval{mac_tinyllama_unfiltered.commonsense_reasoning.share_energy}{9.51}
\defval{mac_tinyllama_unfiltered.commonsense_reasoning.success_ratio}{0.21}
\defval{mac_tinyllama_unfiltered.commonsense_reasoning.time_ratio}{5.22}
\defval{mac_tinyllama_unfiltered.commonsense_reasoning.time_ratio_hi}{8.63}
\defval{mac_tinyllama_unfiltered.commonsense_reasoning.time_ratio_lo}{3.54}
\defval{mac_tinyllama_unfiltered.excluded}{207}
\defval{mac_tinyllama_unfiltered.factual_qa.OOI_EpCG}{4.48}
\defval{mac_tinyllama_unfiltered.factual_qa.OOI_EpCG_hi}{6.07}
\defval{mac_tinyllama_unfiltered.factual_qa.OOI_EpCG_lo}{3.48}
\defval{mac_tinyllama_unfiltered.factual_qa.OOI_EpG}{4.48}
\defval{mac_tinyllama_unfiltered.factual_qa.OOI_EpG_hi}{6.07}
\defval{mac_tinyllama_unfiltered.factual_qa.OOI_EpG_lo}{3.48}
\defval{mac_tinyllama_unfiltered.factual_qa.OOI_EpSG}{4.14}
\defval{mac_tinyllama_unfiltered.factual_qa.OOI_EpSG_hi}{9.06}
\defval{mac_tinyllama_unfiltered.factual_qa.OOI_EpSG_lo}{1.83}
\defval{mac_tinyllama_unfiltered.factual_qa.agentic.E_wrong_goals_J}{11,014}
\defval{mac_tinyllama_unfiltered.factual_qa.agentic.EpCG_J}{800.6}
\defval{mac_tinyllama_unfiltered.factual_qa.agentic.EpG_J}{800.6}
\defval{mac_tinyllama_unfiltered.factual_qa.agentic.EpSG_J}{3,695.0}
\defval{mac_tinyllama_unfiltered.factual_qa.agentic.N}{60}
\defval{mac_tinyllama_unfiltered.factual_qa.agentic.TpG_s}{111.7}
\defval{mac_tinyllama_unfiltered.factual_qa.agentic.attempts_per_goal}{1.00}
\defval{mac_tinyllama_unfiltered.factual_qa.agentic.n_completed}{60}
\defval{mac_tinyllama_unfiltered.factual_qa.agentic.n_correct}{13}
\defval{mac_tinyllama_unfiltered.factual_qa.agentic.p_c}{1.00}
\defval{mac_tinyllama_unfiltered.factual_qa.agentic.p_s}{0.22}
\defval{mac_tinyllama_unfiltered.factual_qa.agentic.power_W}{7.2}
\defval{mac_tinyllama_unfiltered.factual_qa.energy_ratio}{4.48}
\defval{mac_tinyllama_unfiltered.factual_qa.linear.E_wrong_goals_J}{2,329}
\defval{mac_tinyllama_unfiltered.factual_qa.linear.EpCG_J}{178.6}
\defval{mac_tinyllama_unfiltered.factual_qa.linear.EpG_J}{178.6}
\defval{mac_tinyllama_unfiltered.factual_qa.linear.EpSG_J}{892.8}
\defval{mac_tinyllama_unfiltered.factual_qa.linear.N}{60}
\defval{mac_tinyllama_unfiltered.factual_qa.linear.TpG_s}{28.1}
\defval{mac_tinyllama_unfiltered.factual_qa.linear.attempts_per_goal}{1.00}
\defval{mac_tinyllama_unfiltered.factual_qa.linear.n_completed}{60}
\defval{mac_tinyllama_unfiltered.factual_qa.linear.n_correct}{12}
\defval{mac_tinyllama_unfiltered.factual_qa.linear.p_c}{1.00}
\defval{mac_tinyllama_unfiltered.factual_qa.linear.p_s}{0.20}
\defval{mac_tinyllama_unfiltered.factual_qa.linear.power_W}{6.4}
\defval{mac_tinyllama_unfiltered.factual_qa.min_successes}{12}
\defval{mac_tinyllama_unfiltered.factual_qa.power_ratio}{1.13}
\defval{mac_tinyllama_unfiltered.factual_qa.power_ratio_hi}{1.62}
\defval{mac_tinyllama_unfiltered.factual_qa.power_ratio_lo}{0.74}
\defval{mac_tinyllama_unfiltered.factual_qa.share_energy}{1.06}
\defval{mac_tinyllama_unfiltered.factual_qa.success_ratio}{0.92}
\defval{mac_tinyllama_unfiltered.factual_qa.time_ratio}{3.98}
\defval{mac_tinyllama_unfiltered.factual_qa.time_ratio_hi}{6.95}
\defval{mac_tinyllama_unfiltered.factual_qa.time_ratio_lo}{2.58}
\defval{mac_tinyllama_unfiltered.goals}{840}
\defval{mac_tinyllama_unfiltered.gsm8k_basic.OOI_EpCG}{3.71}
\defval{mac_tinyllama_unfiltered.gsm8k_basic.OOI_EpCG_hi}{5.03}
\defval{mac_tinyllama_unfiltered.gsm8k_basic.OOI_EpCG_lo}{2.75}
\defval{mac_tinyllama_unfiltered.gsm8k_basic.OOI_EpG}{3.71}
\defval{mac_tinyllama_unfiltered.gsm8k_basic.OOI_EpG_hi}{5.03}
\defval{mac_tinyllama_unfiltered.gsm8k_basic.OOI_EpG_lo}{2.75}
\defval{mac_tinyllama_unfiltered.gsm8k_basic.OOI_EpSG}{1.85}
\defval{mac_tinyllama_unfiltered.gsm8k_basic.OOI_EpSG_hi}{undefined}
\defval{mac_tinyllama_unfiltered.gsm8k_basic.OOI_EpSG_lo}{undefined}
\defval{mac_tinyllama_unfiltered.gsm8k_basic.agentic.E_wrong_goals_J}{10,117}
\defval{mac_tinyllama_unfiltered.gsm8k_basic.agentic.EpCG_J}{358.8}
\defval{mac_tinyllama_unfiltered.gsm8k_basic.agentic.EpG_J}{358.8}
\defval{mac_tinyllama_unfiltered.gsm8k_basic.agentic.EpSG_J}{10,763.5}
\defval{mac_tinyllama_unfiltered.gsm8k_basic.agentic.N}{60}
\defval{mac_tinyllama_unfiltered.gsm8k_basic.agentic.TpG_s}{31.2}
\defval{mac_tinyllama_unfiltered.gsm8k_basic.agentic.attempts_per_goal}{1.00}
\defval{mac_tinyllama_unfiltered.gsm8k_basic.agentic.n_completed}{60}
\defval{mac_tinyllama_unfiltered.gsm8k_basic.agentic.n_correct}{2}
\defval{mac_tinyllama_unfiltered.gsm8k_basic.agentic.p_c}{1.00}
\defval{mac_tinyllama_unfiltered.gsm8k_basic.agentic.p_s}{0.03}
\defval{mac_tinyllama_unfiltered.gsm8k_basic.agentic.power_W}{11.5}
\defval{mac_tinyllama_unfiltered.gsm8k_basic.energy_ratio}{3.71}
\defval{mac_tinyllama_unfiltered.gsm8k_basic.linear.E_wrong_goals_J}{2,817}
\defval{mac_tinyllama_unfiltered.gsm8k_basic.linear.EpCG_J}{96.7}
\defval{mac_tinyllama_unfiltered.gsm8k_basic.linear.EpG_J}{96.7}
\defval{mac_tinyllama_unfiltered.gsm8k_basic.linear.EpSG_J}{5,802.5}
\defval{mac_tinyllama_unfiltered.gsm8k_basic.linear.N}{60}
\defval{mac_tinyllama_unfiltered.gsm8k_basic.linear.TpG_s}{8.7}
\defval{mac_tinyllama_unfiltered.gsm8k_basic.linear.attempts_per_goal}{1.00}
\defval{mac_tinyllama_unfiltered.gsm8k_basic.linear.n_completed}{60}
\defval{mac_tinyllama_unfiltered.gsm8k_basic.linear.n_correct}{1}
\defval{mac_tinyllama_unfiltered.gsm8k_basic.linear.p_c}{1.00}
\defval{mac_tinyllama_unfiltered.gsm8k_basic.linear.p_s}{0.02}
\defval{mac_tinyllama_unfiltered.gsm8k_basic.linear.power_W}{11.2}
\defval{mac_tinyllama_unfiltered.gsm8k_basic.min_successes}{1}
\defval{mac_tinyllama_unfiltered.gsm8k_basic.power_ratio}{1.03}
\defval{mac_tinyllama_unfiltered.gsm8k_basic.power_ratio_hi}{1.32}
\defval{mac_tinyllama_unfiltered.gsm8k_basic.power_ratio_lo}{0.80}
\defval{mac_tinyllama_unfiltered.gsm8k_basic.share_energy}{2.12}
\defval{mac_tinyllama_unfiltered.gsm8k_basic.success_ratio}{0.50}
\defval{mac_tinyllama_unfiltered.gsm8k_basic.time_ratio}{3.60}
\defval{mac_tinyllama_unfiltered.gsm8k_basic.time_ratio_hi}{4.52}
\defval{mac_tinyllama_unfiltered.gsm8k_basic.time_ratio_lo}{2.96}
\defval{mac_tinyllama_unfiltered.gsm8k_multi_step.OOI_EpCG}{3.67}
\defval{mac_tinyllama_unfiltered.gsm8k_multi_step.OOI_EpCG_hi}{5.48}
\defval{mac_tinyllama_unfiltered.gsm8k_multi_step.OOI_EpCG_lo}{2.46}
\defval{mac_tinyllama_unfiltered.gsm8k_multi_step.OOI_EpG}{3.67}
\defval{mac_tinyllama_unfiltered.gsm8k_multi_step.OOI_EpG_hi}{5.48}
\defval{mac_tinyllama_unfiltered.gsm8k_multi_step.OOI_EpG_lo}{2.46}
\defval{mac_tinyllama_unfiltered.gsm8k_multi_step.OOI_EpSG}{undefined}
\defval{mac_tinyllama_unfiltered.gsm8k_multi_step.OOI_EpSG_hi}{undefined}
\defval{mac_tinyllama_unfiltered.gsm8k_multi_step.OOI_EpSG_lo}{undefined}
\defval{mac_tinyllama_unfiltered.gsm8k_multi_step.agentic.E_wrong_goals_J}{5,198}
\defval{mac_tinyllama_unfiltered.gsm8k_multi_step.agentic.EpCG_J}{201.4}
\defval{mac_tinyllama_unfiltered.gsm8k_multi_step.agentic.EpG_J}{201.4}
\defval{mac_tinyllama_unfiltered.gsm8k_multi_step.agentic.EpSG_J}{undefined}
\defval{mac_tinyllama_unfiltered.gsm8k_multi_step.agentic.N}{60}
\defval{mac_tinyllama_unfiltered.gsm8k_multi_step.agentic.TpG_s}{25.1}
\defval{mac_tinyllama_unfiltered.gsm8k_multi_step.agentic.attempts_per_goal}{1.00}
\defval{mac_tinyllama_unfiltered.gsm8k_multi_step.agentic.n_completed}{60}
\defval{mac_tinyllama_unfiltered.gsm8k_multi_step.agentic.n_correct}{0}
\defval{mac_tinyllama_unfiltered.gsm8k_multi_step.agentic.p_c}{1.00}
\defval{mac_tinyllama_unfiltered.gsm8k_multi_step.agentic.p_s}{0.00}
\defval{mac_tinyllama_unfiltered.gsm8k_multi_step.agentic.power_W}{8.0}
\defval{mac_tinyllama_unfiltered.gsm8k_multi_step.energy_ratio}{3.67}
\defval{mac_tinyllama_unfiltered.gsm8k_multi_step.linear.E_wrong_goals_J}{1,716}
\defval{mac_tinyllama_unfiltered.gsm8k_multi_step.linear.EpCG_J}{54.9}
\defval{mac_tinyllama_unfiltered.gsm8k_multi_step.linear.EpG_J}{54.9}
\defval{mac_tinyllama_unfiltered.gsm8k_multi_step.linear.EpSG_J}{undefined}
\defval{mac_tinyllama_unfiltered.gsm8k_multi_step.linear.N}{60}
\defval{mac_tinyllama_unfiltered.gsm8k_multi_step.linear.TpG_s}{6.6}
\defval{mac_tinyllama_unfiltered.gsm8k_multi_step.linear.attempts_per_goal}{1.00}
\defval{mac_tinyllama_unfiltered.gsm8k_multi_step.linear.n_completed}{60}
\defval{mac_tinyllama_unfiltered.gsm8k_multi_step.linear.n_correct}{0}
\defval{mac_tinyllama_unfiltered.gsm8k_multi_step.linear.p_c}{1.00}
\defval{mac_tinyllama_unfiltered.gsm8k_multi_step.linear.p_s}{0.00}
\defval{mac_tinyllama_unfiltered.gsm8k_multi_step.linear.power_W}{8.3}
\defval{mac_tinyllama_unfiltered.gsm8k_multi_step.min_successes}{0}
\defval{mac_tinyllama_unfiltered.gsm8k_multi_step.power_ratio}{0.97}
\defval{mac_tinyllama_unfiltered.gsm8k_multi_step.power_ratio_hi}{1.45}
\defval{mac_tinyllama_unfiltered.gsm8k_multi_step.power_ratio_lo}{0.61}
\defval{mac_tinyllama_unfiltered.gsm8k_multi_step.share_energy}{undefined}
\defval{mac_tinyllama_unfiltered.gsm8k_multi_step.success_ratio}{undefined}
\defval{mac_tinyllama_unfiltered.gsm8k_multi_step.time_ratio}{3.78}
\defval{mac_tinyllama_unfiltered.gsm8k_multi_step.time_ratio_hi}{6.51}
\defval{mac_tinyllama_unfiltered.gsm8k_multi_step.time_ratio_lo}{2.53}
\defval{mac_tinyllama_unfiltered.logical_reasoning.OOI_EpCG}{4.00}
\defval{mac_tinyllama_unfiltered.logical_reasoning.OOI_EpCG_hi}{5.41}
\defval{mac_tinyllama_unfiltered.logical_reasoning.OOI_EpCG_lo}{3.07}
\defval{mac_tinyllama_unfiltered.logical_reasoning.OOI_EpG}{4.00}
\defval{mac_tinyllama_unfiltered.logical_reasoning.OOI_EpG_hi}{5.41}
\defval{mac_tinyllama_unfiltered.logical_reasoning.OOI_EpG_lo}{3.07}
\defval{mac_tinyllama_unfiltered.logical_reasoning.OOI_EpSG}{4.00}
\defval{mac_tinyllama_unfiltered.logical_reasoning.OOI_EpSG_hi}{11.55}
\defval{mac_tinyllama_unfiltered.logical_reasoning.OOI_EpSG_lo}{1.35}
\defval{mac_tinyllama_unfiltered.logical_reasoning.agentic.E_wrong_goals_J}{12,059}
\defval{mac_tinyllama_unfiltered.logical_reasoning.agentic.EpCG_J}{626.7}
\defval{mac_tinyllama_unfiltered.logical_reasoning.agentic.EpG_J}{626.7}
\defval{mac_tinyllama_unfiltered.logical_reasoning.agentic.EpSG_J}{4,177.8}
\defval{mac_tinyllama_unfiltered.logical_reasoning.agentic.N}{60}
\defval{mac_tinyllama_unfiltered.logical_reasoning.agentic.TpG_s}{86.6}
\defval{mac_tinyllama_unfiltered.logical_reasoning.agentic.attempts_per_goal}{1.00}
\defval{mac_tinyllama_unfiltered.logical_reasoning.agentic.n_completed}{60}
\defval{mac_tinyllama_unfiltered.logical_reasoning.agentic.n_correct}{9}
\defval{mac_tinyllama_unfiltered.logical_reasoning.agentic.p_c}{1.00}
\defval{mac_tinyllama_unfiltered.logical_reasoning.agentic.p_s}{0.15}
\defval{mac_tinyllama_unfiltered.logical_reasoning.agentic.power_W}{7.2}
\defval{mac_tinyllama_unfiltered.logical_reasoning.energy_ratio}{4.00}
\defval{mac_tinyllama_unfiltered.logical_reasoning.linear.E_wrong_goals_J}{2,023}
\defval{mac_tinyllama_unfiltered.logical_reasoning.linear.EpCG_J}{156.7}
\defval{mac_tinyllama_unfiltered.logical_reasoning.linear.EpG_J}{156.7}
\defval{mac_tinyllama_unfiltered.logical_reasoning.linear.EpSG_J}{1,044.5}
\defval{mac_tinyllama_unfiltered.logical_reasoning.linear.N}{60}
\defval{mac_tinyllama_unfiltered.logical_reasoning.linear.TpG_s}{17.1}
\defval{mac_tinyllama_unfiltered.logical_reasoning.linear.attempts_per_goal}{1.00}
\defval{mac_tinyllama_unfiltered.logical_reasoning.linear.n_completed}{60}
\defval{mac_tinyllama_unfiltered.logical_reasoning.linear.n_correct}{9}
\defval{mac_tinyllama_unfiltered.logical_reasoning.linear.p_c}{1.00}
\defval{mac_tinyllama_unfiltered.logical_reasoning.linear.p_s}{0.15}
\defval{mac_tinyllama_unfiltered.logical_reasoning.linear.power_W}{9.2}
\defval{mac_tinyllama_unfiltered.logical_reasoning.min_successes}{9}
\defval{mac_tinyllama_unfiltered.logical_reasoning.power_ratio}{0.79}
\defval{mac_tinyllama_unfiltered.logical_reasoning.power_ratio_hi}{1.15}
\defval{mac_tinyllama_unfiltered.logical_reasoning.power_ratio_lo}{0.55}
\defval{mac_tinyllama_unfiltered.logical_reasoning.share_energy}{1.00}
\defval{mac_tinyllama_unfiltered.logical_reasoning.success_ratio}{1.00}
\defval{mac_tinyllama_unfiltered.logical_reasoning.time_ratio}{5.06}
\defval{mac_tinyllama_unfiltered.logical_reasoning.time_ratio_hi}{7.64}
\defval{mac_tinyllama_unfiltered.logical_reasoning.time_ratio_lo}{3.55}
\defval{mac_tinyllama_unfiltered.science_qa.OOI_EpCG}{7.59}
\defval{mac_tinyllama_unfiltered.science_qa.OOI_EpCG_hi}{12.00}
\defval{mac_tinyllama_unfiltered.science_qa.OOI_EpCG_lo}{5.29}
\defval{mac_tinyllama_unfiltered.science_qa.OOI_EpG}{7.59}
\defval{mac_tinyllama_unfiltered.science_qa.OOI_EpG_hi}{12.00}
\defval{mac_tinyllama_unfiltered.science_qa.OOI_EpG_lo}{5.29}
\defval{mac_tinyllama_unfiltered.science_qa.OOI_EpSG}{1.52}
\defval{mac_tinyllama_unfiltered.science_qa.OOI_EpSG_hi}{undefined}
\defval{mac_tinyllama_unfiltered.science_qa.OOI_EpSG_lo}{undefined}
\defval{mac_tinyllama_unfiltered.science_qa.agentic.E_wrong_goals_J}{25,546}
\defval{mac_tinyllama_unfiltered.science_qa.agentic.EpCG_J}{1,037.9}
\defval{mac_tinyllama_unfiltered.science_qa.agentic.EpG_J}{1,037.9}
\defval{mac_tinyllama_unfiltered.science_qa.agentic.EpSG_J}{12,454.4}
\defval{mac_tinyllama_unfiltered.science_qa.agentic.N}{60}
\defval{mac_tinyllama_unfiltered.science_qa.agentic.TpG_s}{152.5}
\defval{mac_tinyllama_unfiltered.science_qa.agentic.attempts_per_goal}{1.00}
\defval{mac_tinyllama_unfiltered.science_qa.agentic.n_completed}{60}
\defval{mac_tinyllama_unfiltered.science_qa.agentic.n_correct}{5}
\defval{mac_tinyllama_unfiltered.science_qa.agentic.p_c}{1.00}
\defval{mac_tinyllama_unfiltered.science_qa.agentic.p_s}{0.08}
\defval{mac_tinyllama_unfiltered.science_qa.agentic.power_W}{6.8}
\defval{mac_tinyllama_unfiltered.science_qa.energy_ratio}{7.59}
\defval{mac_tinyllama_unfiltered.science_qa.linear.E_wrong_goals_J}{1,672}
\defval{mac_tinyllama_unfiltered.science_qa.linear.EpCG_J}{136.7}
\defval{mac_tinyllama_unfiltered.science_qa.linear.EpG_J}{136.7}
\defval{mac_tinyllama_unfiltered.science_qa.linear.EpSG_J}{8,203.1}
\defval{mac_tinyllama_unfiltered.science_qa.linear.N}{60}
\defval{mac_tinyllama_unfiltered.science_qa.linear.TpG_s}{14.4}
\defval{mac_tinyllama_unfiltered.science_qa.linear.attempts_per_goal}{1.00}
\defval{mac_tinyllama_unfiltered.science_qa.linear.n_completed}{60}
\defval{mac_tinyllama_unfiltered.science_qa.linear.n_correct}{1}
\defval{mac_tinyllama_unfiltered.science_qa.linear.p_c}{1.00}
\defval{mac_tinyllama_unfiltered.science_qa.linear.p_s}{0.02}
\defval{mac_tinyllama_unfiltered.science_qa.linear.power_W}{9.5}
\defval{mac_tinyllama_unfiltered.science_qa.min_successes}{1}
\defval{mac_tinyllama_unfiltered.science_qa.power_ratio}{0.71}
\defval{mac_tinyllama_unfiltered.science_qa.power_ratio_hi}{1.15}
\defval{mac_tinyllama_unfiltered.science_qa.power_ratio_lo}{0.42}
\defval{mac_tinyllama_unfiltered.science_qa.share_energy}{4.85}
\defval{mac_tinyllama_unfiltered.science_qa.success_ratio}{0.20}
\defval{mac_tinyllama_unfiltered.science_qa.time_ratio}{10.62}
\defval{mac_tinyllama_unfiltered.science_qa.time_ratio_hi}{21.90}
\defval{mac_tinyllama_unfiltered.science_qa.time_ratio_lo}{5.97}
\defval{mac_tinyllama_unfiltered.tg_single_db.OOI_EpCG}{0.52}
\defval{mac_tinyllama_unfiltered.tg_single_db.OOI_EpCG_hi}{0.84}
\defval{mac_tinyllama_unfiltered.tg_single_db.OOI_EpCG_lo}{0.33}
\defval{mac_tinyllama_unfiltered.tg_single_db.OOI_EpG}{0.52}
\defval{mac_tinyllama_unfiltered.tg_single_db.OOI_EpG_hi}{0.84}
\defval{mac_tinyllama_unfiltered.tg_single_db.OOI_EpG_lo}{0.33}
\defval{mac_tinyllama_unfiltered.tg_single_db.OOI_EpSG}{undefined}
\defval{mac_tinyllama_unfiltered.tg_single_db.OOI_EpSG_hi}{undefined}
\defval{mac_tinyllama_unfiltered.tg_single_db.OOI_EpSG_lo}{undefined}
\defval{mac_tinyllama_unfiltered.tg_single_db.agentic.E_wrong_goals_J}{0}
\defval{mac_tinyllama_unfiltered.tg_single_db.agentic.EpCG_J}{92.9}
\defval{mac_tinyllama_unfiltered.tg_single_db.agentic.EpG_J}{92.9}
\defval{mac_tinyllama_unfiltered.tg_single_db.agentic.EpSG_J}{168.9}
\defval{mac_tinyllama_unfiltered.tg_single_db.agentic.N}{60}
\defval{mac_tinyllama_unfiltered.tg_single_db.agentic.TpG_s}{14.5}
\defval{mac_tinyllama_unfiltered.tg_single_db.agentic.attempts_per_goal}{1.00}
\defval{mac_tinyllama_unfiltered.tg_single_db.agentic.n_completed}{60}
\defval{mac_tinyllama_unfiltered.tg_single_db.agentic.n_correct}{33}
\defval{mac_tinyllama_unfiltered.tg_single_db.agentic.p_c}{1.00}
\defval{mac_tinyllama_unfiltered.tg_single_db.agentic.p_s}{0.55}
\defval{mac_tinyllama_unfiltered.tg_single_db.agentic.power_W}{6.4}
\defval{mac_tinyllama_unfiltered.tg_single_db.energy_ratio}{0.52}
\defval{mac_tinyllama_unfiltered.tg_single_db.linear.E_wrong_goals_J}{4,827}
\defval{mac_tinyllama_unfiltered.tg_single_db.linear.EpCG_J}{179.9}
\defval{mac_tinyllama_unfiltered.tg_single_db.linear.EpG_J}{179.9}
\defval{mac_tinyllama_unfiltered.tg_single_db.linear.EpSG_J}{undefined}
\defval{mac_tinyllama_unfiltered.tg_single_db.linear.N}{60}
\defval{mac_tinyllama_unfiltered.tg_single_db.linear.TpG_s}{33.8}
\defval{mac_tinyllama_unfiltered.tg_single_db.linear.attempts_per_goal}{1.00}
\defval{mac_tinyllama_unfiltered.tg_single_db.linear.n_completed}{60}
\defval{mac_tinyllama_unfiltered.tg_single_db.linear.n_correct}{0}
\defval{mac_tinyllama_unfiltered.tg_single_db.linear.p_c}{1.00}
\defval{mac_tinyllama_unfiltered.tg_single_db.linear.p_s}{0.00}
\defval{mac_tinyllama_unfiltered.tg_single_db.linear.power_W}{5.3}
\defval{mac_tinyllama_unfiltered.tg_single_db.min_successes}{0}
\defval{mac_tinyllama_unfiltered.tg_single_db.power_ratio}{1.20}
\defval{mac_tinyllama_unfiltered.tg_single_db.power_ratio_hi}{2.02}
\defval{mac_tinyllama_unfiltered.tg_single_db.power_ratio_lo}{0.75}
\defval{mac_tinyllama_unfiltered.tg_single_db.share_energy}{undefined}
\defval{mac_tinyllama_unfiltered.tg_single_db.success_ratio}{0.00}
\defval{mac_tinyllama_unfiltered.tg_single_db.time_ratio}{0.43}
\defval{mac_tinyllama_unfiltered.tg_single_db.time_ratio_hi}{0.82}
\defval{mac_tinyllama_unfiltered.tg_single_db.time_ratio_lo}{0.23}
\defval{modelsel.commonsense_reasoning.n_correct_other}{30}
\defval{modelsel.commonsense_reasoning.n_correct_this}{30}
\defval{modelsel.commonsense_reasoning.ratio}{0.74}
\defval{modelsel.commonsense_reasoning.ratio_hi}{0.85}
\defval{modelsel.commonsense_reasoning.ratio_lo}{0.64}
\defval{modelsel.factual_qa.n_correct_other}{30}
\defval{modelsel.factual_qa.n_correct_this}{30}
\defval{modelsel.factual_qa.ratio}{0.19}
\defval{modelsel.factual_qa.ratio_hi}{0.21}
\defval{modelsel.factual_qa.ratio_lo}{0.18}
\defval{modelsel.gsm8k_basic.n_correct_other}{29}
\defval{modelsel.gsm8k_basic.n_correct_this}{30}
\defval{modelsel.gsm8k_basic.ratio}{0.21}
\defval{modelsel.gsm8k_basic.ratio_hi}{0.24}
\defval{modelsel.gsm8k_basic.ratio_lo}{0.18}
\defval{modelsel.gsm8k_multi_step.n_correct_other}{30}
\defval{modelsel.gsm8k_multi_step.n_correct_this}{30}
\defval{modelsel.gsm8k_multi_step.ratio}{0.36}
\defval{modelsel.gsm8k_multi_step.ratio_hi}{0.39}
\defval{modelsel.gsm8k_multi_step.ratio_lo}{0.34}
\defval{modelsel.logical_reasoning.n_correct_other}{30}
\defval{modelsel.logical_reasoning.n_correct_this}{30}
\defval{modelsel.logical_reasoning.ratio}{0.36}
\defval{modelsel.logical_reasoning.ratio_hi}{0.39}
\defval{modelsel.logical_reasoning.ratio_lo}{0.33}
\defval{modelsel.science_qa.n_correct_other}{30}
\defval{modelsel.science_qa.n_correct_this}{30}
\defval{modelsel.science_qa.ratio}{0.12}
\defval{modelsel.science_qa.ratio_hi}{0.13}
\defval{modelsel.science_qa.ratio_lo}{0.11}
\defval{modelsel.tg_single_calc.n_correct_other}{3}
\defval{modelsel.tg_single_calc.n_correct_this}{0}
\defval{modelsel.tg_single_calc.ratio}{undefined}
\defval{modelsel.tg_single_calc.ratio_hi}{undefined}
\defval{modelsel.tg_single_calc.ratio_lo}{undefined}
\defval{x86.excluded}{98}
\defval{x86.factual_qa.OOI_EpCG}{12.67}
\defval{x86.factual_qa.OOI_EpCG_hi}{16.83}
\defval{x86.factual_qa.OOI_EpCG_lo}{9.91}
\defval{x86.factual_qa.OOI_EpG}{5.05}
\defval{x86.factual_qa.OOI_EpG_hi}{6.53}
\defval{x86.factual_qa.OOI_EpG_lo}{3.94}
\defval{x86.factual_qa.OOI_EpSG}{3.76}
\defval{x86.factual_qa.OOI_EpSG_hi}{6.35}
\defval{x86.factual_qa.OOI_EpSG_lo}{2.20}
\defval{x86.factual_qa.agentic.E_wrong_goals_J}{52,954}
\defval{x86.factual_qa.agentic.EpCG_J}{2,708.2}
\defval{x86.factual_qa.agentic.EpG_J}{1,072.4}
\defval{x86.factual_qa.agentic.EpSG_J}{5,706.7}
\defval{x86.factual_qa.agentic.N}{149}
\defval{x86.factual_qa.agentic.TpG_s}{42.5}
\defval{x86.factual_qa.agentic.attempts_per_goal}{1.00}
\defval{x86.factual_qa.agentic.n_completed}{59}
\defval{x86.factual_qa.agentic.n_correct}{28}
\defval{x86.factual_qa.agentic.p_c}{0.40}
\defval{x86.factual_qa.agentic.p_s}{0.19}
\defval{x86.factual_qa.agentic.power_W}{25.2}
\defval{x86.factual_qa.energy_ratio}{5.05}
\defval{x86.factual_qa.linear.E_wrong_goals_J}{24,989}
\defval{x86.factual_qa.linear.EpCG_J}{213.7}
\defval{x86.factual_qa.linear.EpG_J}{212.3}
\defval{x86.factual_qa.linear.EpSG_J}{1,516.4}
\defval{x86.factual_qa.linear.N}{150}
\defval{x86.factual_qa.linear.TpG_s}{8.3}
\defval{x86.factual_qa.linear.attempts_per_goal}{1.00}
\defval{x86.factual_qa.linear.n_completed}{149}
\defval{x86.factual_qa.linear.n_correct}{21}
\defval{x86.factual_qa.linear.p_c}{0.99}
\defval{x86.factual_qa.linear.p_s}{0.14}
\defval{x86.factual_qa.linear.power_W}{25.6}
\defval{x86.factual_qa.min_successes}{21}
\defval{x86.factual_qa.power_ratio}{0.99}
\defval{x86.factual_qa.power_ratio_hi}{1.03}
\defval{x86.factual_qa.power_ratio_lo}{0.95}
\defval{x86.factual_qa.share_energy}{1.22}
\defval{x86.factual_qa.success_ratio}{0.74}
\defval{x86.factual_qa.time_ratio}{5.13}
\defval{x86.factual_qa.time_ratio_hi}{6.60}
\defval{x86.factual_qa.time_ratio_lo}{4.02}
\defval{x86.goals}{1,347}
\defval{x86.gsm8k_basic.OOI_EpCG}{6.36}
\defval{x86.gsm8k_basic.OOI_EpCG_hi}{7.64}
\defval{x86.gsm8k_basic.OOI_EpCG_lo}{5.41}
\defval{x86.gsm8k_basic.OOI_EpG}{3.05}
\defval{x86.gsm8k_basic.OOI_EpG_hi}{3.60}
\defval{x86.gsm8k_basic.OOI_EpG_lo}{2.61}
\defval{x86.gsm8k_basic.OOI_EpSG}{10.61}
\defval{x86.gsm8k_basic.OOI_EpSG_hi}{37.25}
\defval{x86.gsm8k_basic.OOI_EpSG_lo}{4.92}
\defval{x86.gsm8k_basic.agentic.E_wrong_goals_J}{134,826}
\defval{x86.gsm8k_basic.agentic.EpCG_J}{2,457.4}
\defval{x86.gsm8k_basic.agentic.EpG_J}{1,172.2}
\defval{x86.gsm8k_basic.agentic.EpSG_J}{33,994.2}
\defval{x86.gsm8k_basic.agentic.N}{174}
\defval{x86.gsm8k_basic.agentic.TpG_s}{43.5}
\defval{x86.gsm8k_basic.agentic.attempts_per_goal}{1.02}
\defval{x86.gsm8k_basic.agentic.n_completed}{83}
\defval{x86.gsm8k_basic.agentic.n_correct}{6}
\defval{x86.gsm8k_basic.agentic.p_c}{0.48}
\defval{x86.gsm8k_basic.agentic.p_s}{0.03}
\defval{x86.gsm8k_basic.agentic.power_W}{27.0}
\defval{x86.gsm8k_basic.energy_ratio}{3.05}
\defval{x86.gsm8k_basic.linear.E_wrong_goals_J}{62,777}
\defval{x86.gsm8k_basic.linear.EpCG_J}{386.5}
\defval{x86.gsm8k_basic.linear.EpG_J}{384.3}
\defval{x86.gsm8k_basic.linear.EpSG_J}{3,202.8}
\defval{x86.gsm8k_basic.linear.N}{175}
\defval{x86.gsm8k_basic.linear.TpG_s}{14.0}
\defval{x86.gsm8k_basic.linear.attempts_per_goal}{1.00}
\defval{x86.gsm8k_basic.linear.n_completed}{174}
\defval{x86.gsm8k_basic.linear.n_correct}{21}
\defval{x86.gsm8k_basic.linear.p_c}{0.99}
\defval{x86.gsm8k_basic.linear.p_s}{0.12}
\defval{x86.gsm8k_basic.linear.power_W}{27.4}
\defval{x86.gsm8k_basic.min_successes}{6}
\defval{x86.gsm8k_basic.power_ratio}{0.99}
\defval{x86.gsm8k_basic.power_ratio_hi}{1.04}
\defval{x86.gsm8k_basic.power_ratio_lo}{0.93}
\defval{x86.gsm8k_basic.share_energy}{0.47}
\defval{x86.gsm8k_basic.success_ratio}{3.48}
\defval{x86.gsm8k_basic.time_ratio}{3.10}
\defval{x86.gsm8k_basic.time_ratio_hi}{3.62}
\defval{x86.gsm8k_basic.time_ratio_lo}{2.68}
\defval{x86.gsm8k_multi_step.OOI_EpCG}{7.41}
\defval{x86.gsm8k_multi_step.OOI_EpCG_hi}{13.26}
\defval{x86.gsm8k_multi_step.OOI_EpCG_lo}{4.83}
\defval{x86.gsm8k_multi_step.OOI_EpG}{7.52}
\defval{x86.gsm8k_multi_step.OOI_EpG_hi}{13.42}
\defval{x86.gsm8k_multi_step.OOI_EpG_lo}{4.92}
\defval{x86.gsm8k_multi_step.OOI_EpSG}{undefined}
\defval{x86.gsm8k_multi_step.OOI_EpSG_hi}{undefined}
\defval{x86.gsm8k_multi_step.OOI_EpSG_lo}{undefined}
\defval{x86.gsm8k_multi_step.agentic.E_wrong_goals_J}{76,419}
\defval{x86.gsm8k_multi_step.agentic.EpCG_J}{1,151.4}
\defval{x86.gsm8k_multi_step.agentic.EpG_J}{1,151.4}
\defval{x86.gsm8k_multi_step.agentic.EpSG_J}{77,142.0}
\defval{x86.gsm8k_multi_step.agentic.N}{67}
\defval{x86.gsm8k_multi_step.agentic.TpG_s}{30.3}
\defval{x86.gsm8k_multi_step.agentic.attempts_per_goal}{1.13}
\defval{x86.gsm8k_multi_step.agentic.n_completed}{67}
\defval{x86.gsm8k_multi_step.agentic.n_correct}{1}
\defval{x86.gsm8k_multi_step.agentic.p_c}{1.00}
\defval{x86.gsm8k_multi_step.agentic.p_s}{0.01}
\defval{x86.gsm8k_multi_step.agentic.power_W}{38.0}
\defval{x86.gsm8k_multi_step.energy_ratio}{7.52}
\defval{x86.gsm8k_multi_step.linear.E_wrong_goals_J}{10,254}
\defval{x86.gsm8k_multi_step.linear.EpCG_J}{155.3}
\defval{x86.gsm8k_multi_step.linear.EpG_J}{153.0}
\defval{x86.gsm8k_multi_step.linear.EpSG_J}{undefined}
\defval{x86.gsm8k_multi_step.linear.N}{68}
\defval{x86.gsm8k_multi_step.linear.TpG_s}{5.3}
\defval{x86.gsm8k_multi_step.linear.attempts_per_goal}{1.06}
\defval{x86.gsm8k_multi_step.linear.n_completed}{67}
\defval{x86.gsm8k_multi_step.linear.n_correct}{0}
\defval{x86.gsm8k_multi_step.linear.p_c}{0.99}
\defval{x86.gsm8k_multi_step.linear.p_s}{0.00}
\defval{x86.gsm8k_multi_step.linear.power_W}{29.1}
\defval{x86.gsm8k_multi_step.min_successes}{0}
\defval{x86.gsm8k_multi_step.power_ratio}{1.31}
\defval{x86.gsm8k_multi_step.power_ratio_hi}{1.59}
\defval{x86.gsm8k_multi_step.power_ratio_lo}{1.10}
\defval{x86.gsm8k_multi_step.share_energy}{undefined}
\defval{x86.gsm8k_multi_step.success_ratio}{0.00}
\defval{x86.gsm8k_multi_step.time_ratio}{5.75}
\defval{x86.gsm8k_multi_step.time_ratio_hi}{9.92}
\defval{x86.gsm8k_multi_step.time_ratio_lo}{3.80}
\defval{x86.logical_reasoning.OOI_EpCG}{4.68}
\defval{x86.logical_reasoning.OOI_EpCG_hi}{6.55}
\defval{x86.logical_reasoning.OOI_EpCG_lo}{3.61}
\defval{x86.logical_reasoning.OOI_EpG}{4.68}
\defval{x86.logical_reasoning.OOI_EpG_hi}{6.55}
\defval{x86.logical_reasoning.OOI_EpG_lo}{3.61}
\defval{x86.logical_reasoning.OOI_EpSG}{3.68}
\defval{x86.logical_reasoning.OOI_EpSG_hi}{8.19}
\defval{x86.logical_reasoning.OOI_EpSG_lo}{1.69}
\defval{x86.logical_reasoning.agentic.E_wrong_goals_J}{72,030}
\defval{x86.logical_reasoning.agentic.EpCG_J}{1,917.3}
\defval{x86.logical_reasoning.agentic.EpG_J}{1,917.3}
\defval{x86.logical_reasoning.agentic.EpSG_J}{6,984.6}
\defval{x86.logical_reasoning.agentic.N}{51}
\defval{x86.logical_reasoning.agentic.TpG_s}{70.1}
\defval{x86.logical_reasoning.agentic.attempts_per_goal}{1.00}
\defval{x86.logical_reasoning.agentic.n_completed}{51}
\defval{x86.logical_reasoning.agentic.n_correct}{14}
\defval{x86.logical_reasoning.agentic.p_c}{1.00}
\defval{x86.logical_reasoning.agentic.p_s}{0.27}
\defval{x86.logical_reasoning.agentic.power_W}{27.3}
\defval{x86.logical_reasoning.energy_ratio}{4.68}
\defval{x86.logical_reasoning.linear.E_wrong_goals_J}{16,673}
\defval{x86.logical_reasoning.linear.EpCG_J}{409.6}
\defval{x86.logical_reasoning.linear.EpG_J}{409.6}
\defval{x86.logical_reasoning.linear.EpSG_J}{1,899.1}
\defval{x86.logical_reasoning.linear.N}{51}
\defval{x86.logical_reasoning.linear.TpG_s}{14.0}
\defval{x86.logical_reasoning.linear.attempts_per_goal}{1.00}
\defval{x86.logical_reasoning.linear.n_completed}{51}
\defval{x86.logical_reasoning.linear.n_correct}{11}
\defval{x86.logical_reasoning.linear.p_c}{1.00}
\defval{x86.logical_reasoning.linear.p_s}{0.22}
\defval{x86.logical_reasoning.linear.power_W}{29.3}
\defval{x86.logical_reasoning.min_successes}{11}
\defval{x86.logical_reasoning.power_ratio}{0.93}
\defval{x86.logical_reasoning.power_ratio_hi}{0.96}
\defval{x86.logical_reasoning.power_ratio_lo}{0.91}
\defval{x86.logical_reasoning.share_energy}{1.19}
\defval{x86.logical_reasoning.success_ratio}{0.79}
\defval{x86.logical_reasoning.time_ratio}{5.01}
\defval{x86.logical_reasoning.time_ratio_hi}{7.02}
\defval{x86.logical_reasoning.time_ratio_lo}{3.87}
\defval{x86.science_qa.OOI_EpCG}{6.09}
\defval{x86.science_qa.OOI_EpCG_hi}{9.38}
\defval{x86.science_qa.OOI_EpCG_lo}{4.25}
\defval{x86.science_qa.OOI_EpG}{6.09}
\defval{x86.science_qa.OOI_EpG_hi}{9.38}
\defval{x86.science_qa.OOI_EpG_lo}{4.25}
\defval{x86.science_qa.OOI_EpSG}{1.02}
\defval{x86.science_qa.OOI_EpSG_hi}{undefined}
\defval{x86.science_qa.OOI_EpSG_lo}{undefined}
\defval{x86.science_qa.agentic.E_wrong_goals_J}{66,567}
\defval{x86.science_qa.agentic.EpCG_J}{1,420.3}
\defval{x86.science_qa.agentic.EpG_J}{1,420.3}
\defval{x86.science_qa.agentic.EpSG_J}{7,101.5}
\defval{x86.science_qa.agentic.N}{60}
\defval{x86.science_qa.agentic.TpG_s}{65.0}
\defval{x86.science_qa.agentic.attempts_per_goal}{1.00}
\defval{x86.science_qa.agentic.n_completed}{60}
\defval{x86.science_qa.agentic.n_correct}{12}
\defval{x86.science_qa.agentic.p_c}{1.00}
\defval{x86.science_qa.agentic.p_s}{0.20}
\defval{x86.science_qa.agentic.power_W}{21.8}
\defval{x86.science_qa.energy_ratio}{6.09}
\defval{x86.science_qa.linear.E_wrong_goals_J}{13,955}
\defval{x86.science_qa.linear.EpCG_J}{233.1}
\defval{x86.science_qa.linear.EpG_J}{233.1}
\defval{x86.science_qa.linear.EpSG_J}{6,992.5}
\defval{x86.science_qa.linear.N}{60}
\defval{x86.science_qa.linear.TpG_s}{9.7}
\defval{x86.science_qa.linear.attempts_per_goal}{1.00}
\defval{x86.science_qa.linear.n_completed}{60}
\defval{x86.science_qa.linear.n_correct}{2}
\defval{x86.science_qa.linear.p_c}{1.00}
\defval{x86.science_qa.linear.p_s}{0.03}
\defval{x86.science_qa.linear.power_W}{24.0}
\defval{x86.science_qa.min_successes}{2}
\defval{x86.science_qa.power_ratio}{0.91}
\defval{x86.science_qa.power_ratio_hi}{0.98}
\defval{x86.science_qa.power_ratio_lo}{0.83}
\defval{x86.science_qa.share_energy}{undefined}
\defval{x86.science_qa.success_ratio}{0.17}
\defval{x86.science_qa.time_ratio}{6.69}
\defval{x86.science_qa.time_ratio_hi}{10.49}
\defval{x86.science_qa.time_ratio_lo}{4.62}
\defval{x86.tg_sequential_2.OOI_EpCG}{1.55}
\defval{x86.tg_sequential_2.OOI_EpCG_hi}{2.60}
\defval{x86.tg_sequential_2.OOI_EpCG_lo}{1.04}
\defval{x86.tg_sequential_2.OOI_EpG}{1.55}
\defval{x86.tg_sequential_2.OOI_EpG_hi}{2.60}
\defval{x86.tg_sequential_2.OOI_EpG_lo}{1.04}
\defval{x86.tg_sequential_2.OOI_EpSG}{undefined}
\defval{x86.tg_sequential_2.OOI_EpSG_hi}{undefined}
\defval{x86.tg_sequential_2.OOI_EpSG_lo}{undefined}
\defval{x86.tg_sequential_2.agentic.E_wrong_goals_J}{0}
\defval{x86.tg_sequential_2.agentic.EpCG_J}{329.2}
\defval{x86.tg_sequential_2.agentic.EpG_J}{329.2}
\defval{x86.tg_sequential_2.agentic.EpSG_J}{329.2}
\defval{x86.tg_sequential_2.agentic.N}{51}
\defval{x86.tg_sequential_2.agentic.TpG_s}{13.2}
\defval{x86.tg_sequential_2.agentic.attempts_per_goal}{1.00}
\defval{x86.tg_sequential_2.agentic.n_completed}{51}
\defval{x86.tg_sequential_2.agentic.n_correct}{51}
\defval{x86.tg_sequential_2.agentic.p_c}{1.00}
\defval{x86.tg_sequential_2.agentic.p_s}{1.00}
\defval{x86.tg_sequential_2.agentic.power_W}{24.9}
\defval{x86.tg_sequential_2.energy_ratio}{1.55}
\defval{x86.tg_sequential_2.linear.E_wrong_goals_J}{10,835}
\defval{x86.tg_sequential_2.linear.EpCG_J}{212.4}
\defval{x86.tg_sequential_2.linear.EpG_J}{212.4}
\defval{x86.tg_sequential_2.linear.EpSG_J}{undefined}
\defval{x86.tg_sequential_2.linear.N}{51}
\defval{x86.tg_sequential_2.linear.TpG_s}{9.2}
\defval{x86.tg_sequential_2.linear.attempts_per_goal}{1.00}
\defval{x86.tg_sequential_2.linear.n_completed}{51}
\defval{x86.tg_sequential_2.linear.n_correct}{0}
\defval{x86.tg_sequential_2.linear.p_c}{1.00}
\defval{x86.tg_sequential_2.linear.p_s}{0.00}
\defval{x86.tg_sequential_2.linear.power_W}{23.2}
\defval{x86.tg_sequential_2.min_successes}{0}
\defval{x86.tg_sequential_2.power_ratio}{1.07}
\defval{x86.tg_sequential_2.power_ratio_hi}{1.14}
\defval{x86.tg_sequential_2.power_ratio_lo}{1.01}
\defval{x86.tg_sequential_2.share_energy}{undefined}
\defval{x86.tg_sequential_2.success_ratio}{0.00}
\defval{x86.tg_sequential_2.time_ratio}{1.44}
\defval{x86.tg_sequential_2.time_ratio_hi}{2.47}
\defval{x86.tg_sequential_2.time_ratio_lo}{0.93}
\defval{x86.tg_single_calc.OOI_EpCG}{0.62}
\defval{x86.tg_single_calc.OOI_EpCG_hi}{0.97}
\defval{x86.tg_single_calc.OOI_EpCG_lo}{0.35}
\defval{x86.tg_single_calc.OOI_EpG}{0.62}
\defval{x86.tg_single_calc.OOI_EpG_hi}{0.97}
\defval{x86.tg_single_calc.OOI_EpG_lo}{0.35}
\defval{x86.tg_single_calc.OOI_EpSG}{undefined}
\defval{x86.tg_single_calc.OOI_EpSG_hi}{undefined}
\defval{x86.tg_single_calc.OOI_EpSG_lo}{undefined}
\defval{x86.tg_single_calc.agentic.E_wrong_goals_J}{5,726}
\defval{x86.tg_single_calc.agentic.EpCG_J}{132.3}
\defval{x86.tg_single_calc.agentic.EpG_J}{132.3}
\defval{x86.tg_single_calc.agentic.EpSG_J}{176.4}
\defval{x86.tg_single_calc.agentic.N}{60}
\defval{x86.tg_single_calc.agentic.TpG_s}{5.1}
\defval{x86.tg_single_calc.agentic.attempts_per_goal}{1.00}
\defval{x86.tg_single_calc.agentic.n_completed}{60}
\defval{x86.tg_single_calc.agentic.n_correct}{45}
\defval{x86.tg_single_calc.agentic.p_c}{1.00}
\defval{x86.tg_single_calc.agentic.p_s}{0.75}
\defval{x86.tg_single_calc.agentic.power_W}{26.1}
\defval{x86.tg_single_calc.energy_ratio}{0.62}
\defval{x86.tg_single_calc.linear.E_wrong_goals_J}{12,885}
\defval{x86.tg_single_calc.linear.EpCG_J}{214.8}
\defval{x86.tg_single_calc.linear.EpG_J}{214.8}
\defval{x86.tg_single_calc.linear.EpSG_J}{undefined}
\defval{x86.tg_single_calc.linear.N}{60}
\defval{x86.tg_single_calc.linear.TpG_s}{8.9}
\defval{x86.tg_single_calc.linear.attempts_per_goal}{1.00}
\defval{x86.tg_single_calc.linear.n_completed}{60}
\defval{x86.tg_single_calc.linear.n_correct}{0}
\defval{x86.tg_single_calc.linear.p_c}{1.00}
\defval{x86.tg_single_calc.linear.p_s}{0.00}
\defval{x86.tg_single_calc.linear.power_W}{24.2}
\defval{x86.tg_single_calc.min_successes}{0}
\defval{x86.tg_single_calc.power_ratio}{1.08}
\defval{x86.tg_single_calc.power_ratio_hi}{1.16}
\defval{x86.tg_single_calc.power_ratio_lo}{1.00}
\defval{x86.tg_single_calc.share_energy}{undefined}
\defval{x86.tg_single_calc.success_ratio}{0.00}
\defval{x86.tg_single_calc.time_ratio}{0.57}
\defval{x86.tg_single_calc.time_ratio_hi}{0.90}
\defval{x86.tg_single_calc.time_ratio_lo}{0.33}
\defval{x86.tg_single_db.OOI_EpCG}{0.96}
\defval{x86.tg_single_db.OOI_EpCG_hi}{1.63}
\defval{x86.tg_single_db.OOI_EpCG_lo}{0.58}
\defval{x86.tg_single_db.OOI_EpG}{0.96}
\defval{x86.tg_single_db.OOI_EpG_hi}{1.63}
\defval{x86.tg_single_db.OOI_EpG_lo}{0.58}
\defval{x86.tg_single_db.OOI_EpSG}{undefined}
\defval{x86.tg_single_db.OOI_EpSG_hi}{undefined}
\defval{x86.tg_single_db.OOI_EpSG_lo}{undefined}
\defval{x86.tg_single_db.agentic.E_wrong_goals_J}{3,811}
\defval{x86.tg_single_db.agentic.EpCG_J}{140.7}
\defval{x86.tg_single_db.agentic.EpG_J}{140.7}
\defval{x86.tg_single_db.agentic.EpSG_J}{156.4}
\defval{x86.tg_single_db.agentic.N}{60}
\defval{x86.tg_single_db.agentic.TpG_s}{5.0}
\defval{x86.tg_single_db.agentic.attempts_per_goal}{1.00}
\defval{x86.tg_single_db.agentic.n_completed}{60}
\defval{x86.tg_single_db.agentic.n_correct}{54}
\defval{x86.tg_single_db.agentic.p_c}{1.00}
\defval{x86.tg_single_db.agentic.p_s}{0.90}
\defval{x86.tg_single_db.agentic.power_W}{28.4}
\defval{x86.tg_single_db.energy_ratio}{0.96}
\defval{x86.tg_single_db.linear.E_wrong_goals_J}{8,804}
\defval{x86.tg_single_db.linear.EpCG_J}{146.7}
\defval{x86.tg_single_db.linear.EpG_J}{146.7}
\defval{x86.tg_single_db.linear.EpSG_J}{undefined}
\defval{x86.tg_single_db.linear.N}{60}
\defval{x86.tg_single_db.linear.TpG_s}{5.7}
\defval{x86.tg_single_db.linear.attempts_per_goal}{1.00}
\defval{x86.tg_single_db.linear.n_completed}{60}
\defval{x86.tg_single_db.linear.n_correct}{0}
\defval{x86.tg_single_db.linear.p_c}{1.00}
\defval{x86.tg_single_db.linear.p_s}{0.00}
\defval{x86.tg_single_db.linear.power_W}{25.6}
\defval{x86.tg_single_db.min_successes}{0}
\defval{x86.tg_single_db.power_ratio}{1.11}
\defval{x86.tg_single_db.power_ratio_hi}{1.17}
\defval{x86.tg_single_db.power_ratio_lo}{1.04}
\defval{x86.tg_single_db.share_energy}{undefined}
\defval{x86.tg_single_db.success_ratio}{0.00}
\defval{x86.tg_single_db.time_ratio}{0.86}
\defval{x86.tg_single_db.time_ratio_hi}{1.45}
\defval{x86.tg_single_db.time_ratio_lo}{0.53}
\defval{x86_remote_nim.commonsense_reasoning.OOI_EpCG}{8.31}
\defval{x86_remote_nim.commonsense_reasoning.OOI_EpCG_hi}{10.98}
\defval{x86_remote_nim.commonsense_reasoning.OOI_EpCG_lo}{6.47}
\defval{x86_remote_nim.commonsense_reasoning.OOI_EpG}{8.31}
\defval{x86_remote_nim.commonsense_reasoning.OOI_EpG_hi}{10.98}
\defval{x86_remote_nim.commonsense_reasoning.OOI_EpG_lo}{6.47}
\defval{x86_remote_nim.commonsense_reasoning.OOI_EpSG}{8.31}
\defval{x86_remote_nim.commonsense_reasoning.OOI_EpSG_hi}{10.98}
\defval{x86_remote_nim.commonsense_reasoning.OOI_EpSG_lo}{6.47}
\defval{x86_remote_nim.commonsense_reasoning.agentic.E_wrong_goals_J}{0}
\defval{x86_remote_nim.commonsense_reasoning.agentic.EpCG_J}{37.3}
\defval{x86_remote_nim.commonsense_reasoning.agentic.EpG_J}{37.3}
\defval{x86_remote_nim.commonsense_reasoning.agentic.EpSG_J}{37.3}
\defval{x86_remote_nim.commonsense_reasoning.agentic.N}{30}
\defval{x86_remote_nim.commonsense_reasoning.agentic.TpG_s}{210.6}
\defval{x86_remote_nim.commonsense_reasoning.agentic.attempts_per_goal}{1.00}
\defval{x86_remote_nim.commonsense_reasoning.agentic.n_completed}{30}
\defval{x86_remote_nim.commonsense_reasoning.agentic.n_correct}{30}
\defval{x86_remote_nim.commonsense_reasoning.agentic.p_c}{1.00}
\defval{x86_remote_nim.commonsense_reasoning.agentic.p_s}{1.00}
\defval{x86_remote_nim.commonsense_reasoning.agentic.power_W}{0.2}
\defval{x86_remote_nim.commonsense_reasoning.energy_ratio}{8.31}
\defval{x86_remote_nim.commonsense_reasoning.linear.E_wrong_goals_J}{0}
\defval{x86_remote_nim.commonsense_reasoning.linear.EpCG_J}{4.5}
\defval{x86_remote_nim.commonsense_reasoning.linear.EpG_J}{4.5}
\defval{x86_remote_nim.commonsense_reasoning.linear.EpSG_J}{4.5}
\defval{x86_remote_nim.commonsense_reasoning.linear.N}{30}
\defval{x86_remote_nim.commonsense_reasoning.linear.TpG_s}{28.1}
\defval{x86_remote_nim.commonsense_reasoning.linear.attempts_per_goal}{1.00}
\defval{x86_remote_nim.commonsense_reasoning.linear.n_completed}{30}
\defval{x86_remote_nim.commonsense_reasoning.linear.n_correct}{30}
\defval{x86_remote_nim.commonsense_reasoning.linear.p_c}{1.00}
\defval{x86_remote_nim.commonsense_reasoning.linear.p_s}{1.00}
\defval{x86_remote_nim.commonsense_reasoning.linear.power_W}{0.2}
\defval{x86_remote_nim.commonsense_reasoning.min_successes}{30}
\defval{x86_remote_nim.commonsense_reasoning.power_ratio}{1.11}
\defval{x86_remote_nim.commonsense_reasoning.power_ratio_hi}{1.21}
\defval{x86_remote_nim.commonsense_reasoning.power_ratio_lo}{1.01}
\defval{x86_remote_nim.commonsense_reasoning.share_energy}{1.00}
\defval{x86_remote_nim.commonsense_reasoning.success_ratio}{1.00}
\defval{x86_remote_nim.commonsense_reasoning.time_ratio}{7.50}
\defval{x86_remote_nim.commonsense_reasoning.time_ratio_hi}{9.87}
\defval{x86_remote_nim.commonsense_reasoning.time_ratio_lo}{5.84}
\defval{x86_remote_nim.excluded}{0}
\defval{x86_remote_nim.factual_qa.OOI_EpCG}{2.47}
\defval{x86_remote_nim.factual_qa.OOI_EpCG_hi}{5.65}
\defval{x86_remote_nim.factual_qa.OOI_EpCG_lo}{1.36}
\defval{x86_remote_nim.factual_qa.OOI_EpG}{2.47}
\defval{x86_remote_nim.factual_qa.OOI_EpG_hi}{5.65}
\defval{x86_remote_nim.factual_qa.OOI_EpG_lo}{1.36}
\defval{x86_remote_nim.factual_qa.OOI_EpSG}{2.47}
\defval{x86_remote_nim.factual_qa.OOI_EpSG_hi}{5.65}
\defval{x86_remote_nim.factual_qa.OOI_EpSG_lo}{1.36}
\defval{x86_remote_nim.factual_qa.agentic.E_wrong_goals_J}{0}
\defval{x86_remote_nim.factual_qa.agentic.EpCG_J}{10.5}
\defval{x86_remote_nim.factual_qa.agentic.EpG_J}{10.5}
\defval{x86_remote_nim.factual_qa.agentic.EpSG_J}{10.5}
\defval{x86_remote_nim.factual_qa.agentic.N}{30}
\defval{x86_remote_nim.factual_qa.agentic.TpG_s}{74.8}
\defval{x86_remote_nim.factual_qa.agentic.attempts_per_goal}{1.00}
\defval{x86_remote_nim.factual_qa.agentic.n_completed}{30}
\defval{x86_remote_nim.factual_qa.agentic.n_correct}{30}
\defval{x86_remote_nim.factual_qa.agentic.p_c}{1.00}
\defval{x86_remote_nim.factual_qa.agentic.p_s}{1.00}
\defval{x86_remote_nim.factual_qa.agentic.power_W}{0.1}
\defval{x86_remote_nim.factual_qa.energy_ratio}{2.47}
\defval{x86_remote_nim.factual_qa.linear.E_wrong_goals_J}{0}
\defval{x86_remote_nim.factual_qa.linear.EpCG_J}{4.2}
\defval{x86_remote_nim.factual_qa.linear.EpG_J}{4.2}
\defval{x86_remote_nim.factual_qa.linear.EpSG_J}{4.2}
\defval{x86_remote_nim.factual_qa.linear.N}{30}
\defval{x86_remote_nim.factual_qa.linear.TpG_s}{29.9}
\defval{x86_remote_nim.factual_qa.linear.attempts_per_goal}{1.00}
\defval{x86_remote_nim.factual_qa.linear.n_completed}{30}
\defval{x86_remote_nim.factual_qa.linear.n_correct}{30}
\defval{x86_remote_nim.factual_qa.linear.p_c}{1.00}
\defval{x86_remote_nim.factual_qa.linear.p_s}{1.00}
\defval{x86_remote_nim.factual_qa.linear.power_W}{0.1}
\defval{x86_remote_nim.factual_qa.min_successes}{30}
\defval{x86_remote_nim.factual_qa.power_ratio}{0.98}
\defval{x86_remote_nim.factual_qa.power_ratio_hi}{1.11}
\defval{x86_remote_nim.factual_qa.power_ratio_lo}{0.90}
\defval{x86_remote_nim.factual_qa.share_energy}{1.00}
\defval{x86_remote_nim.factual_qa.success_ratio}{1.00}
\defval{x86_remote_nim.factual_qa.time_ratio}{2.50}
\defval{x86_remote_nim.factual_qa.time_ratio_hi}{5.39}
\defval{x86_remote_nim.factual_qa.time_ratio_lo}{1.42}
\defval{x86_remote_nim.goals}{608}
\defval{x86_remote_nim.gsm8k_basic.OOI_EpCG}{6.42}
\defval{x86_remote_nim.gsm8k_basic.OOI_EpCG_hi}{17.92}
\defval{x86_remote_nim.gsm8k_basic.OOI_EpCG_lo}{2.94}
\defval{x86_remote_nim.gsm8k_basic.OOI_EpG}{6.42}
\defval{x86_remote_nim.gsm8k_basic.OOI_EpG_hi}{17.92}
\defval{x86_remote_nim.gsm8k_basic.OOI_EpG_lo}{2.94}
\defval{x86_remote_nim.gsm8k_basic.OOI_EpSG}{6.52}
\defval{x86_remote_nim.gsm8k_basic.OOI_EpSG_hi}{18.21}
\defval{x86_remote_nim.gsm8k_basic.OOI_EpSG_lo}{2.97}
\defval{x86_remote_nim.gsm8k_basic.agentic.E_wrong_goals_J}{96}
\defval{x86_remote_nim.gsm8k_basic.agentic.EpCG_J}{13.7}
\defval{x86_remote_nim.gsm8k_basic.agentic.EpG_J}{13.7}
\defval{x86_remote_nim.gsm8k_basic.agentic.EpSG_J}{13.9}
\defval{x86_remote_nim.gsm8k_basic.agentic.N}{64}
\defval{x86_remote_nim.gsm8k_basic.agentic.TpG_s}{81.9}
\defval{x86_remote_nim.gsm8k_basic.agentic.attempts_per_goal}{1.02}
\defval{x86_remote_nim.gsm8k_basic.agentic.n_completed}{64}
\defval{x86_remote_nim.gsm8k_basic.agentic.n_correct}{63}
\defval{x86_remote_nim.gsm8k_basic.agentic.p_c}{1.00}
\defval{x86_remote_nim.gsm8k_basic.agentic.p_s}{0.98}
\defval{x86_remote_nim.gsm8k_basic.agentic.power_W}{0.2}
\defval{x86_remote_nim.gsm8k_basic.energy_ratio}{6.42}
\defval{x86_remote_nim.gsm8k_basic.linear.E_wrong_goals_J}{0}
\defval{x86_remote_nim.gsm8k_basic.linear.EpCG_J}{2.1}
\defval{x86_remote_nim.gsm8k_basic.linear.EpG_J}{2.1}
\defval{x86_remote_nim.gsm8k_basic.linear.EpSG_J}{2.1}
\defval{x86_remote_nim.gsm8k_basic.linear.N}{64}
\defval{x86_remote_nim.gsm8k_basic.linear.TpG_s}{12.5}
\defval{x86_remote_nim.gsm8k_basic.linear.attempts_per_goal}{1.02}
\defval{x86_remote_nim.gsm8k_basic.linear.n_completed}{64}
\defval{x86_remote_nim.gsm8k_basic.linear.n_correct}{64}
\defval{x86_remote_nim.gsm8k_basic.linear.p_c}{1.00}
\defval{x86_remote_nim.gsm8k_basic.linear.p_s}{1.00}
\defval{x86_remote_nim.gsm8k_basic.linear.power_W}{0.2}
\defval{x86_remote_nim.gsm8k_basic.min_successes}{63}
\defval{x86_remote_nim.gsm8k_basic.power_ratio}{0.98}
\defval{x86_remote_nim.gsm8k_basic.power_ratio_hi}{2.13}
\defval{x86_remote_nim.gsm8k_basic.power_ratio_lo}{0.58}
\defval{x86_remote_nim.gsm8k_basic.share_energy}{0.99}
\defval{x86_remote_nim.gsm8k_basic.success_ratio}{1.02}
\defval{x86_remote_nim.gsm8k_basic.time_ratio}{6.53}
\defval{x86_remote_nim.gsm8k_basic.time_ratio_hi}{10.60}
\defval{x86_remote_nim.gsm8k_basic.time_ratio_lo}{4.23}
\defval{x86_remote_nim.gsm8k_multi_step.OOI_EpCG}{7.79}
\defval{x86_remote_nim.gsm8k_multi_step.OOI_EpCG_hi}{16.94}
\defval{x86_remote_nim.gsm8k_multi_step.OOI_EpCG_lo}{3.92}
\defval{x86_remote_nim.gsm8k_multi_step.OOI_EpG}{7.79}
\defval{x86_remote_nim.gsm8k_multi_step.OOI_EpG_hi}{16.94}
\defval{x86_remote_nim.gsm8k_multi_step.OOI_EpG_lo}{3.92}
\defval{x86_remote_nim.gsm8k_multi_step.OOI_EpSG}{7.79}
\defval{x86_remote_nim.gsm8k_multi_step.OOI_EpSG_hi}{16.94}
\defval{x86_remote_nim.gsm8k_multi_step.OOI_EpSG_lo}{3.92}
\defval{x86_remote_nim.gsm8k_multi_step.agentic.E_wrong_goals_J}{0}
\defval{x86_remote_nim.gsm8k_multi_step.agentic.EpCG_J}{42.9}
\defval{x86_remote_nim.gsm8k_multi_step.agentic.EpG_J}{42.9}
\defval{x86_remote_nim.gsm8k_multi_step.agentic.EpSG_J}{42.9}
\defval{x86_remote_nim.gsm8k_multi_step.agentic.N}{30}
\defval{x86_remote_nim.gsm8k_multi_step.agentic.TpG_s}{134.4}
\defval{x86_remote_nim.gsm8k_multi_step.agentic.attempts_per_goal}{1.00}
\defval{x86_remote_nim.gsm8k_multi_step.agentic.n_completed}{30}
\defval{x86_remote_nim.gsm8k_multi_step.agentic.n_correct}{30}
\defval{x86_remote_nim.gsm8k_multi_step.agentic.p_c}{1.00}
\defval{x86_remote_nim.gsm8k_multi_step.agentic.p_s}{1.00}
\defval{x86_remote_nim.gsm8k_multi_step.agentic.power_W}{0.3}
\defval{x86_remote_nim.gsm8k_multi_step.energy_ratio}{7.79}
\defval{x86_remote_nim.gsm8k_multi_step.linear.E_wrong_goals_J}{0}
\defval{x86_remote_nim.gsm8k_multi_step.linear.EpCG_J}{5.5}
\defval{x86_remote_nim.gsm8k_multi_step.linear.EpG_J}{5.5}
\defval{x86_remote_nim.gsm8k_multi_step.linear.EpSG_J}{5.5}
\defval{x86_remote_nim.gsm8k_multi_step.linear.N}{30}
\defval{x86_remote_nim.gsm8k_multi_step.linear.TpG_s}{16.4}
\defval{x86_remote_nim.gsm8k_multi_step.linear.attempts_per_goal}{1.00}
\defval{x86_remote_nim.gsm8k_multi_step.linear.n_completed}{30}
\defval{x86_remote_nim.gsm8k_multi_step.linear.n_correct}{30}
\defval{x86_remote_nim.gsm8k_multi_step.linear.p_c}{1.00}
\defval{x86_remote_nim.gsm8k_multi_step.linear.p_s}{1.00}
\defval{x86_remote_nim.gsm8k_multi_step.linear.power_W}{0.3}
\defval{x86_remote_nim.gsm8k_multi_step.min_successes}{30}
\defval{x86_remote_nim.gsm8k_multi_step.power_ratio}{0.95}
\defval{x86_remote_nim.gsm8k_multi_step.power_ratio_hi}{1.37}
\defval{x86_remote_nim.gsm8k_multi_step.power_ratio_lo}{0.75}
\defval{x86_remote_nim.gsm8k_multi_step.share_energy}{1.00}
\defval{x86_remote_nim.gsm8k_multi_step.success_ratio}{1.00}
\defval{x86_remote_nim.gsm8k_multi_step.time_ratio}{8.19}
\defval{x86_remote_nim.gsm8k_multi_step.time_ratio_hi}{14.48}
\defval{x86_remote_nim.gsm8k_multi_step.time_ratio_lo}{4.86}
\defval{x86_remote_nim.logical_reasoning.OOI_EpCG}{8.64}
\defval{x86_remote_nim.logical_reasoning.OOI_EpCG_hi}{16.47}
\defval{x86_remote_nim.logical_reasoning.OOI_EpCG_lo}{4.92}
\defval{x86_remote_nim.logical_reasoning.OOI_EpG}{8.64}
\defval{x86_remote_nim.logical_reasoning.OOI_EpG_hi}{16.47}
\defval{x86_remote_nim.logical_reasoning.OOI_EpG_lo}{4.92}
\defval{x86_remote_nim.logical_reasoning.OOI_EpSG}{8.36}
\defval{x86_remote_nim.logical_reasoning.OOI_EpSG_hi}{15.95}
\defval{x86_remote_nim.logical_reasoning.OOI_EpSG_lo}{4.80}
\defval{x86_remote_nim.logical_reasoning.agentic.E_wrong_goals_J}{0}
\defval{x86_remote_nim.logical_reasoning.agentic.EpCG_J}{16.8}
\defval{x86_remote_nim.logical_reasoning.agentic.EpG_J}{16.8}
\defval{x86_remote_nim.logical_reasoning.agentic.EpSG_J}{16.8}
\defval{x86_remote_nim.logical_reasoning.agentic.N}{30}
\defval{x86_remote_nim.logical_reasoning.agentic.TpG_s}{66.6}
\defval{x86_remote_nim.logical_reasoning.agentic.attempts_per_goal}{1.00}
\defval{x86_remote_nim.logical_reasoning.agentic.n_completed}{30}
\defval{x86_remote_nim.logical_reasoning.agentic.n_correct}{30}
\defval{x86_remote_nim.logical_reasoning.agentic.p_c}{1.00}
\defval{x86_remote_nim.logical_reasoning.agentic.p_s}{1.00}
\defval{x86_remote_nim.logical_reasoning.agentic.power_W}{0.3}
\defval{x86_remote_nim.logical_reasoning.energy_ratio}{8.64}
\defval{x86_remote_nim.logical_reasoning.linear.E_wrong_goals_J}{1}
\defval{x86_remote_nim.logical_reasoning.linear.EpCG_J}{1.9}
\defval{x86_remote_nim.logical_reasoning.linear.EpG_J}{1.9}
\defval{x86_remote_nim.logical_reasoning.linear.EpSG_J}{2.0}
\defval{x86_remote_nim.logical_reasoning.linear.N}{30}
\defval{x86_remote_nim.logical_reasoning.linear.TpG_s}{8.7}
\defval{x86_remote_nim.logical_reasoning.linear.attempts_per_goal}{1.00}
\defval{x86_remote_nim.logical_reasoning.linear.n_completed}{30}
\defval{x86_remote_nim.logical_reasoning.linear.n_correct}{29}
\defval{x86_remote_nim.logical_reasoning.linear.p_c}{1.00}
\defval{x86_remote_nim.logical_reasoning.linear.p_s}{0.97}
\defval{x86_remote_nim.logical_reasoning.linear.power_W}{0.2}
\defval{x86_remote_nim.logical_reasoning.min_successes}{29}
\defval{x86_remote_nim.logical_reasoning.power_ratio}{1.13}
\defval{x86_remote_nim.logical_reasoning.power_ratio_hi}{1.50}
\defval{x86_remote_nim.logical_reasoning.power_ratio_lo}{0.88}
\defval{x86_remote_nim.logical_reasoning.share_energy}{1.02}
\defval{x86_remote_nim.logical_reasoning.success_ratio}{0.97}
\defval{x86_remote_nim.logical_reasoning.time_ratio}{7.68}
\defval{x86_remote_nim.logical_reasoning.time_ratio_hi}{12.20}
\defval{x86_remote_nim.logical_reasoning.time_ratio_lo}{4.99}
\defval{x86_remote_nim.science_qa.OOI_EpCG}{4.24}
\defval{x86_remote_nim.science_qa.OOI_EpCG_hi}{5.70}
\defval{x86_remote_nim.science_qa.OOI_EpCG_lo}{3.25}
\defval{x86_remote_nim.science_qa.OOI_EpG}{4.24}
\defval{x86_remote_nim.science_qa.OOI_EpG_hi}{5.70}
\defval{x86_remote_nim.science_qa.OOI_EpG_lo}{3.25}
\defval{x86_remote_nim.science_qa.OOI_EpSG}{4.40}
\defval{x86_remote_nim.science_qa.OOI_EpSG_hi}{6.52}
\defval{x86_remote_nim.science_qa.OOI_EpSG_lo}{3.03}
\defval{x86_remote_nim.science_qa.agentic.E_wrong_goals_J}{216}
\defval{x86_remote_nim.science_qa.agentic.EpCG_J}{52.9}
\defval{x86_remote_nim.science_qa.agentic.EpG_J}{52.9}
\defval{x86_remote_nim.science_qa.agentic.EpSG_J}{58.7}
\defval{x86_remote_nim.science_qa.agentic.N}{30}
\defval{x86_remote_nim.science_qa.agentic.TpG_s}{334.5}
\defval{x86_remote_nim.science_qa.agentic.attempts_per_goal}{1.00}
\defval{x86_remote_nim.science_qa.agentic.n_completed}{30}
\defval{x86_remote_nim.science_qa.agentic.n_correct}{27}
\defval{x86_remote_nim.science_qa.agentic.p_c}{1.00}
\defval{x86_remote_nim.science_qa.agentic.p_s}{0.90}
\defval{x86_remote_nim.science_qa.agentic.power_W}{0.2}
\defval{x86_remote_nim.science_qa.energy_ratio}{4.24}
\defval{x86_remote_nim.science_qa.linear.E_wrong_goals_J}{69}
\defval{x86_remote_nim.science_qa.linear.EpCG_J}{12.5}
\defval{x86_remote_nim.science_qa.linear.EpG_J}{12.5}
\defval{x86_remote_nim.science_qa.linear.EpSG_J}{13.3}
\defval{x86_remote_nim.science_qa.linear.N}{30}
\defval{x86_remote_nim.science_qa.linear.TpG_s}{82.9}
\defval{x86_remote_nim.science_qa.linear.attempts_per_goal}{1.00}
\defval{x86_remote_nim.science_qa.linear.n_completed}{30}
\defval{x86_remote_nim.science_qa.linear.n_correct}{28}
\defval{x86_remote_nim.science_qa.linear.p_c}{1.00}
\defval{x86_remote_nim.science_qa.linear.p_s}{0.93}
\defval{x86_remote_nim.science_qa.linear.power_W}{0.2}
\defval{x86_remote_nim.science_qa.min_successes}{27}
\defval{x86_remote_nim.science_qa.power_ratio}{1.05}
\defval{x86_remote_nim.science_qa.power_ratio_hi}{1.18}
\defval{x86_remote_nim.science_qa.power_ratio_lo}{0.95}
\defval{x86_remote_nim.science_qa.share_energy}{0.98}
\defval{x86_remote_nim.science_qa.success_ratio}{1.04}
\defval{x86_remote_nim.science_qa.time_ratio}{4.03}
\defval{x86_remote_nim.science_qa.time_ratio_hi}{5.46}
\defval{x86_remote_nim.science_qa.time_ratio_lo}{3.07}
\defval{x86_remote_nim.tg_sequential_2.OOI_EpCG}{1.03}
\defval{x86_remote_nim.tg_sequential_2.OOI_EpCG_hi}{1.35}
\defval{x86_remote_nim.tg_sequential_2.OOI_EpCG_lo}{0.79}
\defval{x86_remote_nim.tg_sequential_2.OOI_EpG}{1.03}
\defval{x86_remote_nim.tg_sequential_2.OOI_EpG_hi}{1.35}
\defval{x86_remote_nim.tg_sequential_2.OOI_EpG_lo}{0.79}
\defval{x86_remote_nim.tg_sequential_2.OOI_EpSG}{undefined}
\defval{x86_remote_nim.tg_sequential_2.OOI_EpSG_hi}{undefined}
\defval{x86_remote_nim.tg_sequential_2.OOI_EpSG_lo}{undefined}
\defval{x86_remote_nim.tg_sequential_2.agentic.E_wrong_goals_J}{0}
\defval{x86_remote_nim.tg_sequential_2.agentic.EpCG_J}{27.0}
\defval{x86_remote_nim.tg_sequential_2.agentic.EpG_J}{27.0}
\defval{x86_remote_nim.tg_sequential_2.agentic.EpSG_J}{27.0}
\defval{x86_remote_nim.tg_sequential_2.agentic.N}{30}
\defval{x86_remote_nim.tg_sequential_2.agentic.TpG_s}{122.8}
\defval{x86_remote_nim.tg_sequential_2.agentic.attempts_per_goal}{1.00}
\defval{x86_remote_nim.tg_sequential_2.agentic.n_completed}{30}
\defval{x86_remote_nim.tg_sequential_2.agentic.n_correct}{30}
\defval{x86_remote_nim.tg_sequential_2.agentic.p_c}{1.00}
\defval{x86_remote_nim.tg_sequential_2.agentic.p_s}{1.00}
\defval{x86_remote_nim.tg_sequential_2.agentic.power_W}{0.2}
\defval{x86_remote_nim.tg_sequential_2.energy_ratio}{1.03}
\defval{x86_remote_nim.tg_sequential_2.linear.E_wrong_goals_J}{788}
\defval{x86_remote_nim.tg_sequential_2.linear.EpCG_J}{26.3}
\defval{x86_remote_nim.tg_sequential_2.linear.EpG_J}{26.3}
\defval{x86_remote_nim.tg_sequential_2.linear.EpSG_J}{undefined}
\defval{x86_remote_nim.tg_sequential_2.linear.N}{30}
\defval{x86_remote_nim.tg_sequential_2.linear.TpG_s}{118.9}
\defval{x86_remote_nim.tg_sequential_2.linear.attempts_per_goal}{1.00}
\defval{x86_remote_nim.tg_sequential_2.linear.n_completed}{30}
\defval{x86_remote_nim.tg_sequential_2.linear.n_correct}{0}
\defval{x86_remote_nim.tg_sequential_2.linear.p_c}{1.00}
\defval{x86_remote_nim.tg_sequential_2.linear.p_s}{0.00}
\defval{x86_remote_nim.tg_sequential_2.linear.power_W}{0.2}
\defval{x86_remote_nim.tg_sequential_2.min_successes}{0}
\defval{x86_remote_nim.tg_sequential_2.power_ratio}{1.00}
\defval{x86_remote_nim.tg_sequential_2.power_ratio_hi}{1.15}
\defval{x86_remote_nim.tg_sequential_2.power_ratio_lo}{0.85}
\defval{x86_remote_nim.tg_sequential_2.share_energy}{undefined}
\defval{x86_remote_nim.tg_sequential_2.success_ratio}{0.00}
\defval{x86_remote_nim.tg_sequential_2.time_ratio}{1.03}
\defval{x86_remote_nim.tg_sequential_2.time_ratio_hi}{1.30}
\defval{x86_remote_nim.tg_sequential_2.time_ratio_lo}{0.83}
\defval{x86_remote_nim.tg_sequential_3.OOI_EpCG}{1.68}
\defval{x86_remote_nim.tg_sequential_3.OOI_EpCG_hi}{2.12}
\defval{x86_remote_nim.tg_sequential_3.OOI_EpCG_lo}{1.34}
\defval{x86_remote_nim.tg_sequential_3.OOI_EpG}{1.68}
\defval{x86_remote_nim.tg_sequential_3.OOI_EpG_hi}{2.12}
\defval{x86_remote_nim.tg_sequential_3.OOI_EpG_lo}{1.34}
\defval{x86_remote_nim.tg_sequential_3.OOI_EpSG}{undefined}
\defval{x86_remote_nim.tg_sequential_3.OOI_EpSG_hi}{undefined}
\defval{x86_remote_nim.tg_sequential_3.OOI_EpSG_lo}{undefined}
\defval{x86_remote_nim.tg_sequential_3.agentic.E_wrong_goals_J}{3,061}
\defval{x86_remote_nim.tg_sequential_3.agentic.EpCG_J}{102.0}
\defval{x86_remote_nim.tg_sequential_3.agentic.EpG_J}{102.0}
\defval{x86_remote_nim.tg_sequential_3.agentic.EpSG_J}{undefined}
\defval{x86_remote_nim.tg_sequential_3.agentic.N}{30}
\defval{x86_remote_nim.tg_sequential_3.agentic.TpG_s}{286.1}
\defval{x86_remote_nim.tg_sequential_3.agentic.attempts_per_goal}{1.00}
\defval{x86_remote_nim.tg_sequential_3.agentic.n_completed}{30}
\defval{x86_remote_nim.tg_sequential_3.agentic.n_correct}{0}
\defval{x86_remote_nim.tg_sequential_3.agentic.p_c}{1.00}
\defval{x86_remote_nim.tg_sequential_3.agentic.p_s}{0.00}
\defval{x86_remote_nim.tg_sequential_3.agentic.power_W}{0.4}
\defval{x86_remote_nim.tg_sequential_3.energy_ratio}{1.68}
\defval{x86_remote_nim.tg_sequential_3.linear.E_wrong_goals_J}{1,821}
\defval{x86_remote_nim.tg_sequential_3.linear.EpCG_J}{60.7}
\defval{x86_remote_nim.tg_sequential_3.linear.EpG_J}{60.7}
\defval{x86_remote_nim.tg_sequential_3.linear.EpSG_J}{undefined}
\defval{x86_remote_nim.tg_sequential_3.linear.N}{30}
\defval{x86_remote_nim.tg_sequential_3.linear.TpG_s}{170.1}
\defval{x86_remote_nim.tg_sequential_3.linear.attempts_per_goal}{1.00}
\defval{x86_remote_nim.tg_sequential_3.linear.n_completed}{30}
\defval{x86_remote_nim.tg_sequential_3.linear.n_correct}{0}
\defval{x86_remote_nim.tg_sequential_3.linear.p_c}{1.00}
\defval{x86_remote_nim.tg_sequential_3.linear.p_s}{0.00}
\defval{x86_remote_nim.tg_sequential_3.linear.power_W}{0.4}
\defval{x86_remote_nim.tg_sequential_3.min_successes}{0}
\defval{x86_remote_nim.tg_sequential_3.power_ratio}{1.00}
\defval{x86_remote_nim.tg_sequential_3.power_ratio_hi}{1.12}
\defval{x86_remote_nim.tg_sequential_3.power_ratio_lo}{0.89}
\defval{x86_remote_nim.tg_sequential_3.share_energy}{undefined}
\defval{x86_remote_nim.tg_sequential_3.success_ratio}{undefined}
\defval{x86_remote_nim.tg_sequential_3.time_ratio}{1.68}
\defval{x86_remote_nim.tg_sequential_3.time_ratio_hi}{2.00}
\defval{x86_remote_nim.tg_sequential_3.time_ratio_lo}{1.41}
\defval{x86_remote_nim.tg_single_db.OOI_EpCG}{1.33}
\defval{x86_remote_nim.tg_single_db.OOI_EpCG_hi}{1.74}
\defval{x86_remote_nim.tg_single_db.OOI_EpCG_lo}{1.01}
\defval{x86_remote_nim.tg_single_db.OOI_EpG}{1.33}
\defval{x86_remote_nim.tg_single_db.OOI_EpG_hi}{1.74}
\defval{x86_remote_nim.tg_single_db.OOI_EpG_lo}{1.01}
\defval{x86_remote_nim.tg_single_db.OOI_EpSG}{undefined}
\defval{x86_remote_nim.tg_single_db.OOI_EpSG_hi}{undefined}
\defval{x86_remote_nim.tg_single_db.OOI_EpSG_lo}{undefined}
\defval{x86_remote_nim.tg_single_db.agentic.E_wrong_goals_J}{13}
\defval{x86_remote_nim.tg_single_db.agentic.EpCG_J}{20.8}
\defval{x86_remote_nim.tg_single_db.agentic.EpG_J}{20.8}
\defval{x86_remote_nim.tg_single_db.agentic.EpSG_J}{21.6}
\defval{x86_remote_nim.tg_single_db.agentic.N}{30}
\defval{x86_remote_nim.tg_single_db.agentic.TpG_s}{125.5}
\defval{x86_remote_nim.tg_single_db.agentic.attempts_per_goal}{1.00}
\defval{x86_remote_nim.tg_single_db.agentic.n_completed}{30}
\defval{x86_remote_nim.tg_single_db.agentic.n_correct}{29}
\defval{x86_remote_nim.tg_single_db.agentic.p_c}{1.00}
\defval{x86_remote_nim.tg_single_db.agentic.p_s}{0.97}
\defval{x86_remote_nim.tg_single_db.agentic.power_W}{0.2}
\defval{x86_remote_nim.tg_single_db.energy_ratio}{1.33}
\defval{x86_remote_nim.tg_single_db.linear.E_wrong_goals_J}{472}
\defval{x86_remote_nim.tg_single_db.linear.EpCG_J}{15.7}
\defval{x86_remote_nim.tg_single_db.linear.EpG_J}{15.7}
\defval{x86_remote_nim.tg_single_db.linear.EpSG_J}{undefined}
\defval{x86_remote_nim.tg_single_db.linear.N}{30}
\defval{x86_remote_nim.tg_single_db.linear.TpG_s}{99.6}
\defval{x86_remote_nim.tg_single_db.linear.attempts_per_goal}{1.00}
\defval{x86_remote_nim.tg_single_db.linear.n_completed}{30}
\defval{x86_remote_nim.tg_single_db.linear.n_correct}{0}
\defval{x86_remote_nim.tg_single_db.linear.p_c}{1.00}
\defval{x86_remote_nim.tg_single_db.linear.p_s}{0.00}
\defval{x86_remote_nim.tg_single_db.linear.power_W}{0.2}
\defval{x86_remote_nim.tg_single_db.min_successes}{0}
\defval{x86_remote_nim.tg_single_db.power_ratio}{1.05}
\defval{x86_remote_nim.tg_single_db.power_ratio_hi}{1.24}
\defval{x86_remote_nim.tg_single_db.power_ratio_lo}{0.89}
\defval{x86_remote_nim.tg_single_db.share_energy}{undefined}
\defval{x86_remote_nim.tg_single_db.success_ratio}{0.00}
\defval{x86_remote_nim.tg_single_db.time_ratio}{1.26}
\defval{x86_remote_nim.tg_single_db.time_ratio_hi}{1.55}
\defval{x86_remote_nim.tg_single_db.time_ratio_lo}{1.02}

%% file: sections/front.tex
\hyphenation{
  im-ple-men-ta-tion
  im-ple-men-ta-tion-level
  op-er-a-tion-al-ly
  em-pir-i-cal-ly
  re-pro-du-ci-bil-i-ty
  or-ches-tra-tion
  in-fer-ence
  work-flow
  a-gen-tic
  mea-sure-ment
  ac-count-a-bil-i-ty
  bench-mark-ing
  at-tri-bu-tion
  con-fig-u-ra-tion
  mi-cro-code
}
\title[Energy per Successful Goal]{Energy per Successful Goal: Goal-Level Energy Accounting for Agentic AI Systems}

\author{Deepak Panigrahy$^{1,*}$, Aakash Tyagi$^{1}$, Stephen Abkin$^{1}$, David Kebo Houngninou$^{1}$, Rabi Mahapatra$^{1}$}
\affiliation{%
  \institution{$^{1}$Texas A\&M University, Department of Computer Science and Engineering}
  \city{College Station}
  \state{Texas}
  \country{USA}}
\email{deepakp@tamu.edu}

\renewcommand{\shortauthors}{Panigrahy et al.}

\begin{abstract}
AI agents plan, call tools and retry failed attempts, so they decide at run time how many model calls a goal takes.
Energy is commonly reported per call or per token, which cannot show the energy needed to deliver a successful goal.
We introduce energy per submitted, completed and successful goal (\EpG{}, \EpCG{}, \EpSG{}), which charge every attempt to its goal and distinguish completion from correctness.
We also introduce the Orchestration Overhead Index (\OOI{}), the ratio of agentic to linear \EpSG{}, where the linear workflow is a single model call, with a rule that states when the agent's higher success rate compensates for its extra energy.
We evaluate the metrics on \val{facts.design.n_goals} goals with \val{facts.design.n_models} models on three hardware platforms (Intel x86, NVIDIA GB10 and Apple M4), using a measurement method that declares the window, the idle correction and the hardware boundary.
Across \val{facts.ooi_all} comparisons over \val{facts.n_goal_types_all} goal types, \OOI{} ranged from \val{facts.calc.ooi_point}$\times$ to \val{facts.newsets.ooi_max}$\times$: the agent cost more per successful goal wherever a single call already succeeded often, and less on a calculation goal, where it handed the computation to a calculator that the single call usually got wrong.
Ranking models by \EpSG{} instead of \EpG{} changed the chosen model in \val{facts.flip.changed} of \val{facts.flip.n} rankings, because \EpG{} can favor a model that spends little energy on each goal but often answers incorrectly.
Each failed attempt added about the energy of one complete attempt, while energy per call stayed almost unchanged, so reporting per call hides what retries cost.
Engineers can therefore use \EpSG{} and \OOI{} to choose the workflow and model that meet their success, energy and latency requirements, and the measured energy of a failed attempt to set retry budgets.
\end{abstract}

\ccsdesc[500]{General and reference~Measurement}
\ccsdesc[500]{General and reference~Metrics}
\ccsdesc[500]{Hardware~Power and energy}
\ccsdesc[300]{Hardware~Power estimation and optimization}
\ccsdesc[300]{Computing methodologies~Intelligent agents}
\ccsdesc[300]{Software and its engineering~Software performance}

\keywords{energy measurement, agentic AI, energy per successful goal, LLM inference, workflow selection, model selection, reproducibility}
\maketitle

%% file: sections/sec1_intro.tex
\section{Introduction}\label{sec:intro}

\begin{figure}[t]
\centering
\begin{tikzpicture}[x=1cm,y=1cm,font=\scriptsize]
\foreach \a/\c in {0/orange!70,1/orange!70,2/green!55!black!60}{
  \fill[\c] (2.65*\a,0) rectangle ++(2.5,0.4);
  \foreach \k in {1,...,6}{\draw[white,line width=0.6pt] (2.65*\a+0.357*\k,0) -- ++(0,0.4);}
}
\node[above,align=center] at (1.25,0.4) {Attempt 1\\failed, \val{facts.g612.a1}\,J};
\node[above,align=center] at (3.9,0.4) {Attempt 2\\failed, \val{facts.g612.a2}\,J};
\node[above,align=center] at (6.55,0.4) {Attempt 3\\passed, \val{facts.g612.a3}\,J};
\draw[decorate,decoration={brace,mirror,amplitude=3pt}] (5.3,-0.05) -- (7.8,-0.05)
  node[midway,below=3pt] {\val{facts.g612.share}\% of the goal};
\draw[decorate,decoration={brace,mirror,amplitude=3pt}] (0,-0.45) -- (7.8,-0.45)
  node[midway,below=3pt] {goal: \val{facts.g612.goal}\,J over \val{facts.g612.calls} calls; one call: about \val{facts.g612.call}\,J};
\end{tikzpicture}
\caption{\textbf{One goal, attempt by attempt.} A logical reasoning goal run by Mistral 7B in the agentic workflow on GB10; each block is one model call.}
\label{fig:goal}
\end{figure}

\textbf{Energy cost of agentic workflows.}
Electricity use by data centers is projected to more than double by 2030, with AI among the main drivers of that growth ~\cite{iea2025energyai}.
Efficiency gains lower the energy of each AI request, yet total consumption keeps rising because demand grows faster than efficiency improves~\cite{iea2026keyquestions}.
Spending on AI follows the same pattern: Uber’s price per token halved over a year while the number of tokens it used grew by 450\%, and it exhausted its annual AI budget within months~\cite{fortune2026uber}.
In each case the cost of a single token or request fell while the total rose, so a cost per unit cannot show whether the growing total is buying proportionately more useful results. Answering that question requires measuring what the energy or money of an AI system delivers for tasks, and not only what each unit costs.

AI systems are moving from answering prompts to completing goals.
An agent that resolves a support ticket plans, calls tools, checks its output and retries when something fails.
It therefore decides at run time how many model calls a goal requires, and some of those calls belong to failed attempts (Figure~\ref{fig:goal}).
Energy, however, is commonly reported per inference, per token or per prompt~\cite{strubell2019energy,patterson2021carbon,chung2025mlenergy,niu2026tokenpowerbench,elsworth2025google,cavagna2026sweetspot}.
These measures give the energy of an individual call but not the total energy needed to deliver a successful result.
Let's take for instance a smaller or quantized model that lowers the energy of each call, but if it fails or answers wrongly more often, it can raise the energy per successful goal while a per-call report shows a saving. To know what an AI system’s energy delivers, we must measure it per successful goal.

Industry has begun to apply the same principle to cost: OpenAI's finance leadership argues that the relevant measure is the full cost of a successful task, including retries and rework, rather than the price of a token~\cite{openai2026scorecard}.
We apply this principle to energy.
Doing so also requires a measurement method that captures the energy of the hardware over the full execution of a goal, so that engineers can reliably evaluate and optimize AI workflows.

\begin{figure*}[t]
\centering
\IfFileExists{figures/fig1_overview.pdf}{\includegraphics[width=\textwidth]{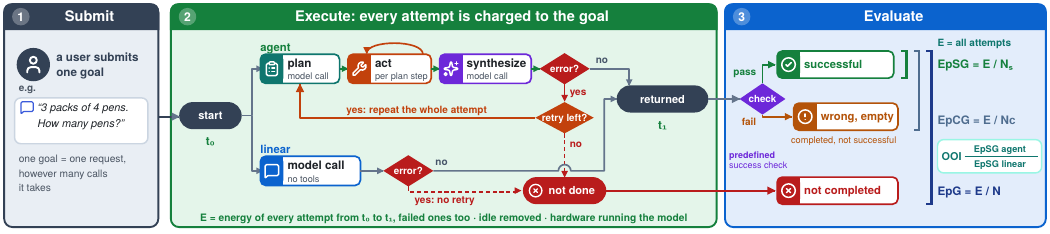}}{\phd{3.5}{fig1\_overview}{figure in progress}}
\caption{\textbf{Goal life cycle.} The linear workflow answers a goal with one model call; the agentic workflow plans, acts and synthesizes, and starts a new attempt after each failure until it returns an answer or its retry budget runs out. Every attempt is charged to the goal, and a check sorts goals into successful, completed but unsuccessful, and not completed.}
\label{fig:overview}
\end{figure*}

\textbf{Limits of energy per call.}
Energy per call cannot tell us what a successful result costs, for two reasons.
First, a goal consists of attempts and each attempt of several model calls.
Energy per call is an average over calls, so it cannot show how many calls a goal needed or how many of them belonged to attempts that failed.
In the goal of Figure~\ref{fig:goal}, the agent failed twice and completed successfully on its third attempt, so the goal used \val{facts.g612.a1} + \val{facts.g612.a2} + \val{facts.g612.a3} = \val{facts.g612.goal}\,J over \val{facts.g612.calls} model calls.
A report per call gives about \val{facts.g612.call}\,J, and counting only the successful attempt gives \val{facts.g612.a3}\,J, \val{facts.g612.share}\% of the energy the goal used.
Second, finishing is not the same as succeeding, because agent frameworks record whether a run finished but not whether its answer was correct.
When we ran \val{facts.bfcl.n_conv} goals of the multi-turn tool benchmark of the Berkeley Function Calling Leaderboard (BFCL)~\cite{patil2025bfcl} with four models on GB10, Granite 8B completed every goal but passed none of the benchmark's own state checks.
Its energy per completed goal is therefore finite while its energy per successful goal is undefined.
Mistral Small 24B completed every goal but succeeded on only \val{facts.unacc.completed_pct}\% of them, so its energy per completed goal was about a quarter lower than its energy per successful goal.
Dividing by completed goals counts the goals with wrong answers, so it spreads the energy over goals that delivered no correct result.

\textbf{Measurement Practices.}
The number also depends on how the energy is measured.
Google reports 0.24\,Wh for a median Gemini prompt when the whole serving stack is counted, but 0.10\,Wh when only the active accelerators are~\cite{elsworth2025google}.
In our runs the gap is larger: the model runs in a separate server on the GPU, so metering only the agent process captured \val{facts.gb10.process_share_min_pct} to \val{facts.gb10.process_share_max_pct}\% of the energy of a goal.
Leaving idle power in or counting clean up after the goal adds energy the goal did not use.
The metered hardware, the window and the idle correction must therefore be stated before results can be compared.

\textbf{Agentic versus linear workflows.}
Given the energy cost of agentic AI, does every goal actually need an agent?
Many goals can be attempted by a \emph{linear} workflow: one model call with the goal's prompt, without tools, planning or retries.
An agent spends additional energy on planning, coordination and retries, which is worth spending only if it raises the success rate enough to justify that energy, for example by handing a deterministic step, such as a calculation, to a tool instead of working it out in text.
We therefore ask which workflow and which model reach a successful goal with the least energy and time, for a given goal type and machine.
The comparison of workflows applies only to goals that a single model response can produce.

From the three outcomes of a goal in Figure~\ref{fig:overview}, we define energy per submitted, completed and successful goal (\EpG{}, \EpCG{} and \EpSG{}), each with a time counterpart that measures how long a user waits per goal, so that a team can check whether the choice that saves energy also saves time.
The Orchestration Overhead Index (\OOI{}) is the ratio of \EpSG{} between the agentic and the linear workflow.
A break-even rule shows when an option with higher energy per goal can still cost less per successful goal, and a second condition shows when energy and time point to different choices.
Together they let a team choose a workflow and a model on a sample of goals before deployment.
We evaluate the metrics by asking what they reveal about four energy-accounting questions that energy per call alone cannot answer:
\begin{enumerate}[leftmargin=*,label={\textbf{RQ\arabic*:}}]
\item \textbf{Workflow choice.} Does \OOI{}, through its break-even rule, identify when the success gain of an agentic workflow justifies the energy consumed by its extra calls?
\item \textbf{Model choice.} Does ranking models by \EpSG{} instead of \EpG{} change which model is selected?
\item \textbf{Retry cost.} How much energy do failed attempts add to a goal, and how much of this cost does energy per call hide?
\item \textbf{Measurement choice.} How sensitive are \EpG{}, \EpSG{} and \OOI{} to the metered hardware, the measurement time window and the idle correction?
\end{enumerate}

\textbf{Summary of findings.}
Across \val{facts.ooi_all} comparisons over \val{facts.n_goal_types_all} goal types, \OOI{} was above 1 in the \val{facts.ooi_total} over the \val{facts.n_goal_types_ooi} goal types that the model can answer on its own, because a single call already succeeded too often for the agent to pass the break-even rule, and below 1 in all \val{facts.calc.n} on a calculator goal that a single call rarely solved (Section~\ref{sec:c4}).
Ranking models by \EpSG{} instead of \EpG{} changed the chosen model in \val{facts.flip.changed} of \val{facts.flip.n} rankings (Section~\ref{sec:res:model}).
In the retry experiments, each failed attempt on a reasoning goal added about as much energy as a goal that succeeded on its first attempt, yet energy per call barely changed (Section~\ref{sec:c5}).
Metering only the agent process missed most of the energy of a goal, and counting clean up or idle power shrank the measured overhead of an agent without changing which workflow costs less (Section~\ref{sec:c3}).
Each comparison fixes one goal type, model and machine, so it shows how often extra work buys a successful goal for that goal type, not for agents in general.

\textbf{Contributions.}
\begin{enumerate}[leftmargin=*]
\item \textbf{Goal-level energy metrics} (Section~\ref{sec:epg}).
We introduce \EpG{}, \EpCG{} and \EpSG{} to account for energy at the goal level rather than the call level.
Energy from every attempt, including failed attempts, is attributed to its goal.
The metrics distinguish submitted goals, completed goals and successful goals, with success determined by a defined output check.
\item \textbf{The Orchestration Overhead Index} (Section~\ref{sec:ooi}).
We introduce \OOI{} to compare the energy per successful goal of agentic and linear workflows on the same goals.
Its decomposition into time, power and success ratios explains the sources of their energy differences.
The associated break-even rule identifies when an agentic workflow's success-rate advantage compensates for its higher energy per goal.
The same formulation supports comparisons between models under a fixed workflow.
\item \textbf{A goal-level measurement method} (Section~\ref{sec:alems}).
We develop a measurement method that makes the hardware boundary, measurement window and idle-power correction explicit.
We evaluate how these choices affect \EpG{}, \EpSG{} and \OOI{} across three platforms: Intel x86, NVIDIA GB10 and Apple M4.
\end{enumerate}

%% file: sections/sec2_related.tex
\section{Background and Related Work}\label{sec:related}

Prior work measures and reports AI energy per inference, per token or per program run, and recent studies apply these measures to agentic systems.

\textbf{Per call and per token energy reporting.}
Strubell et al.~\cite{strubell2019energy}, Patterson et al.~\cite{patterson2021carbon} and Henderson et al.~\cite{henderson2020systematic} established energy and carbon reporting for training.
For inference, MLPerf Power~\cite{tschand2024mlperfpower} benchmarks systems from microwatts to megawatts, and the ML.ENERGY benchmark~\cite{chung2025mlenergy} automates measurement across many models.
TokenPowerBench~\cite{niu2026tokenpowerbench} and other studies~\cite{samsi2023words,luccioni2023counting,jegham2025howhungry} report energy per token, per query or per request.
Google reports the energy of a median Gemini prompt under several measurement boundaries~\cite{elsworth2025google}, and analytical models predict energy per request from model and hardware parameters~\cite{cavagna2026sweetspot,simple2026}.
Each of these divides by a call, which is the right unit when one request is one model call but not when the system decides how many calls a goal takes.
Oviedo et al. show that energy per inference can fall while the energy of the whole task rises, once models spend more computation on each answer~\cite{oviedo2025energyinference}.

\textbf{Energy measurement tools.}
Each platform exposes its own energy counters: \RAPL{}~\cite{david2010rapl,intel_rapl} gives package energy on x86, NVIDIA reports GPU energy through NVML and DCGM~\cite{nvidia_dcgm}, and Apple silicon reports CPU and GPU energy as separate domains~\cite{apple_powermetrics}.
Their boundaries differ, so the same goal can be charged different parts of the machine depending on where it runs.
Recent work studies how to use them over workflows~\cite{thamm2025raplworkflows} and how software and hardware readings compare on CPU and GPU systems~\cite{koszczal2025cpugpuenergy,niaz2025cpugpudatatransfer}.
Profilers such as PowerAPI~\cite{powerapi}, Scaphandre~\cite{scaphandre} and CodeCarbon~\cite{codecarbon} turn these counters into energy per process or per program run, and calculators estimate carbon from hardware type and runtime~\cite{mlco2}.
At larger scale, Patel et al. characterize GPU power in cloud inference~\cite{patel2024polca}, and request level attribution splits the energy of batched serving among requests~\cite{requestattr2026}.
These tools show where and when energy is drawn, but they leave three choices to the user: when to start and stop counting, whether to remove idle power, and which part of the machine to count.
Small, undocumented measurement choices can change results substantially~\cite{mytkowicz2009producing}, and Section~\ref{sec:c3} shows that each of these three choices changes the measured energy of an agent's goal.
None of the tools defines a goal or counts its attempts.

\textbf{Energy studies of agentic systems.}
Luo et al. profile agent inference on Blackwell GPUs and find that an agent sending one request at a time uses far more energy per output token than the same server processing many requests in parallel~\cite{luo2026agentblackwell}.
Studies of software engineering agents compare frameworks and find large differences in energy between multi agent and single query designs~\cite{sustainableagents2026,tripathy2025swenergy,wattsdebts2026}.
Others ask where agents should run between edge and cloud~\cite{agentsedgecloud2026} or reduce agent energy through caching and serving~\cite{librarian2026,kairos2026}.
Raj et al. characterize agent execution from the CPU side and reduce its latency~\cite{raj2025agentic}, Sherlock improves the reliability of agentic workflows~\cite{sherlock2025}, and a survey covers energy efficiency in networked agentic inference~\cite{chen2026networking}.
Joules-to-GDP measures end-to-end energy and divides the economic value of completed work by it~\cite{zeinali2026joulesgdp}, and ProgRouter routes multi-agent workflows under a measured energy budget~\cite{li2026progrouter}.
These studies normalize by token, request or run, or report totals per experiment.
A configuration that fails more often is therefore not charged for its failures, and the studies discussed above do not check whether the energy bought a correct result.

Cost of pass~\cite{erol2025costofpass} is the closest idea to ours.
It estimates the cost of one correct solution as the average cost of one attempt divided by the share of attempts that are correct.
It measures cost in dollars from token prices, not energy.
It also calculates the attempts a goal needs from that share, as if attempts were independent and repeated until one is correct, rather than counting the attempts each goal actually made.
A real agent cannot tell when its answer is wrong, so it can stop with a wrong answer, which is why we separate completed from successful goals.

%% file: sections/sec3_metrics.tex
\section{Energy per Goal Metrics}\label{sec:epg}
\subsection{Goals, Attempts and Outcomes}\label{sec:ontology}\label{sec:goal}\label{sec:ontology:success}

A \emph{goal} is one unit of user intent: a submitted prompt together with the outcome it asks for.
It is defined by what the user asks for, not by how the system does the work.
A prompt therefore counts as one goal even when it contains several questions or needs several steps.
Every model call, tool call, planning step and retry spent on a goal is charged to it.
In our experiments, a goal is one submitted question of a goal type or one BFCL conversation.

An \emph{attempt} is one try at a goal, from the start of its execution until it returns an answer or fails.
Only a failure leads to another attempt within the retry budget.
Examples of failures are timeouts, errors returned by the model server and replies that cannot be parsed.
An attempt that returns an answer ends the goal, whether or not the answer is correct, because the check runs after the goal ends and the agent never sees its result.
A linear workflow therefore always has exactly one attempt.
We record the energy, duration and outcome of every attempt (Figure~\ref{fig:overview}), and the energy of a goal is the sum over all its attempts:
\begin{equation}
  E_g = \sum_{i=1}^{n_g} E_{g,i},
  \label{eq:goalenergy}
\end{equation}
where $n_g \geq 1$ is the number of attempts and $E_{g,i}$ is the energy of attempt $i$.
Figure~\ref{fig:goal} shows one such goal.

The last attempt decides the outcome of a goal.
The goal is \emph{successful} when that attempt returns an answer that passes the check, \emph{completed but unsuccessful} when the answer fails the check or is empty, and \emph{not completed} when the attempt fails and no retry remains (Figure~\ref{fig:overview}).
We keep completion apart from success because a system can finish every goal without delivering a single successful result.

The check is defined for each goal type and returns pass or fail, with no partial credit.
It need not compare against a known answer: in production, it can be any declared success signal, such as passing tests, a closed ticket or a required tool state, and where no check exists or it cannot judge an output, goals are classed only as completed or not completed and no \EpSG{} is reported.
When an attempt produces several outputs, only the last is judged, while the energy of all of them is counted.
Section~\ref{sec:checker} describes our checks, and Table~\ref{tab:tasks} lists them for each goal type.

\subsection{Goal-Level Energy Metrics}\label{sec:epg:def}\label{sec:epg:undef}
A single energy per goal cannot separate energy spent on goals that never finish from energy spent on goals that finish with a wrong answer.
We therefore define three metrics, for submitted, completed and successful goals.
\EpSG{} is the quantity to reduce, and \EpG{} and \EpCG{} show where its energy goes.

Consider $N$ goals submitted to one workflow on one machine, of which $N_c$ are completed and $N_s$ successful, so $N_s \leq N_c \leq N$.
Let $E = \sum_g E_g$ be the energy of all attempts of all goals (Equation~\ref{eq:goalenergy}).
The three metrics share this numerator and differ only in the denominator (Figure~\ref{fig:overview}):
\begin{equation}
\EpG = \frac{E}{N}, \qquad \EpCG = \frac{E}{N_c} = \frac{\EpG}{p_c}, \qquad \EpSG = \frac{E}{N_s} = \frac{\EpG}{p_s},
\label{eq:epg}
\end{equation}
where $p_c = N_c / N$ is the completion rate and $p_s = N_s / N$ the success rate.
Each energy metric has a time counterpart with the same denominator: $\TpG = T_{\mathrm{total}} / N$, $\TpCG = T_{\mathrm{total}} / N_c$ and $\TpSG = T_{\mathrm{total}} / N_s$, where $T_{\mathrm{total}}$ is the duration of all attempts.
\TpSG{} is the time a user waits per successful goal, failed attempts included.
Each energy metric equals the mean power $\bar P = E / T_{\mathrm{total}}$ times its time counterpart, for example $\EpSG = \bar P \cdot \TpSG$, so energy and time can be compared on the same goals.

All three metrics are in joules per goal and serve two purposes.
For a single goal, they give its energy as soon as the goal ends, so a system can use them at run time, for example to enforce an energy budget per goal.
Over a set of goals, such as a benchmark or all the goals a service handles in one day, they describe the system, and only there is the energy of failed goals charged to the goals that succeed.
Because $p_s \leq p_c \leq 1$, the metrics are ordered, $\EpG \leq \EpCG \leq \EpSG$.
\EpCG{} exceeds \EpG{} by the factor $1/p_c$, set by goals that did not finish, and \EpSG{} exceeds \EpCG{} by $p_c/p_s$, set by completed goals that were unsuccessful.
\EpG{} is always defined, while \EpCG{} is undefined when no goal is completed and \EpSG{} when no goal succeeds; we then report it as undefined, never as zero or infinite, together with the energy and the counts.
For a set with no successful goal, we also report a lower bound on \EpSG{} from the 95\% upper bound $p_s^{+}$ of its success rate~\cite{clopper1934binomial}: $\EpSG \geq \EpG / p_s^{+}$.

\EpSG{} compares systems that use the same workflow: it shows which model or configuration delivers its successful goals with the least energy.
The formulas are the same for any measurement window and hardware boundary, but the measured energy differs between them, so we compare energy values only within one machine and one boundary.
Dividing energy by goals rather than calls compares systems by the energy they need to accomplish the same work, regardless of how many calls they make~\cite{bertsch2025energyutilization,kulkarni2026ceec,oviedo2025energyinference}.
A system cannot appear more energy efficient simply by spreading its energy across more calls.

%% file: sections/sec4_ooi.tex
\section{Orchestration Overhead Index}\label{sec:ooi}
\subsection{Definition}\label{sec:ooi:def}

A designer choosing between the agentic and the linear workflow needs to know which one spends less energy per successful goal.
Energy per goal alone cannot tell, because an agent that spends more on each goal may still succeed often enough to cost less per success.
We therefore let both workflows attempt the same goals, with the same model, check, machine and energy boundary, and define the Orchestration Overhead Index as the ratio of their energy per successful goal:
\begin{equation}
  \OOI{} = \frac{\EpSG_{\mathrm{agentic}}}{\EpSG_{\mathrm{linear}}}.
  \label{eq:ooi}
\end{equation}
A value below 1 means that the agentic workflow reaches a successful goal with less energy, and a value above 1 means that the linear workflow does.
We name the cheaper workflow only when the whole 95\% confidence interval of \OOI{} lies on one side of 1.
\OOI{} applies to one goal type, model and machine at a time, so it says whether orchestration justifies its energy in that setting, not whether agents are efficient in general.
It is meant for decisions before deployment: measured on a sample of goals of each type, it shows for which goal types a service needs an agentic workflow and for which a linear one is enough.

\subsection{Break-Even Rule}\label{sec:ooi:split}

Substituting $\EpSG = \EpG / p_s$ from Equation~\ref{eq:epg} splits \OOI{} into two factors:
\begin{equation}
\OOI{} =
\underbrace{\frac{\EpG_{\mathrm{agentic}}}{\EpG_{\mathrm{linear}}}}_{\text{energy ratio } R_E}
\times
\underbrace{\frac{p_{s,\mathrm{linear}}}{p_{s,\mathrm{agentic}}}}_{\text{success ratio } R_S}.
\label{eq:ooisplit}
\end{equation}
The energy ratio compares the energy each workflow spends per submitted goal, and the success ratio compares how often each succeeds.
When the agentic workflow succeeds more often, $R_S$ is below 1 and pulls \OOI{} below $R_E$.
When the linear workflow succeeds more often, $R_S$ pushes \OOI{} above $R_E$.

\textbf{Break-even rule.}
The agentic workflow has the lower \EpSG{} only if
\begin{equation}
p_{s,\mathrm{agentic}} > R_E \, p_{s,\mathrm{linear}},
\label{eq:breakeven}
\end{equation}
where $p_{s,\mathrm{agentic}}$ and $p_{s,\mathrm{linear}}$ are the success rates of the two workflows and $R_E$ is their energy ratio from Equation~\ref{eq:ooisplit}.
Because a success rate cannot exceed 1, this is possible only if $p_{s,\mathrm{linear}} < 1/R_E$, which we call the ceiling.
Being below the ceiling is necessary but not sufficient, since the agent must still succeed often enough.
For example, with $R_E = 10$, the linear workflow must fail on more than 90\% of its goals for the agent to have any chance of using less energy per successful goal.
The same rule compares two models under one workflow.
When the linear workflow has no successful goal, \OOI{} is undefined, and the 95\% upper confidence bound $p_{s,\mathrm{linear}}^{+}$ gives $\OOI{} \leq R_E \, p_{s,\mathrm{linear}}^{+} / p_{s,\mathrm{agentic}}$, with the observed agentic success rate taken as fixed.

The energy ratio itself splits into time and power.
With $T = \TpG$ the time per goal (Section~\ref{sec:epg:def}) and $P = \EpG / T$ the mean power while the goal runs,
\begin{equation}
R_E = \frac{T_{\mathrm{agentic}}}{T_{\mathrm{linear}}} \times \frac{P_{\mathrm{agentic}}}{P_{\mathrm{linear}}},
\label{eq:timepower}
\end{equation}
and substituting this into Equation~\ref{eq:ooisplit} splits \OOI{} into three factors:
\begin{equation}
\OOI{} =
\underbrace{\frac{T_{\mathrm{agentic}}}{T_{\mathrm{linear}}}}_{\text{time ratio}}
\times
\underbrace{\frac{P_{\mathrm{agentic}}}{P_{\mathrm{linear}}}}_{\text{power ratio}}
\times
\underbrace{\frac{p_{s,\mathrm{linear}}}{p_{s,\mathrm{agentic}}}}_{\text{success ratio}}.
\label{eq:ooi3}
\end{equation}
This shows whether the agentic workflow costs more per successful goal because it runs longer, draws more power or succeeds less often.
Dividing time by the success rate in the same way gives $\EpSG = P \times \TpSG$, so for any two options $A$ and $B$,
\begin{equation}
\frac{\EpSG_A}{\EpSG_B} = \frac{\TpSG_A}{\TpSG_B} \times \frac{P_A}{P_B}.
\label{eq:tradeoff}
\end{equation}
Energy and time per successful goal therefore point to different choices only when the faster option draws enough additional power to outweigh its saving in time, that is, when $\TpSG_A < \TpSG_B$ and $P_A / P_B > \TpSG_B / \TpSG_A$.

\subsection{Linear Baseline Eligibility}\label{sec:ooi:elig}

The linear workflow is one model call with the goal's prompt, without tools, planning or retries.
\OOI{} applies only to goals that a single model response can succeed on, so we classify each goal type as eligible or not before computing any \OOI{}.
A goal type is eligible when its answer can be given as text, as in facts, reasoning and arithmetic, including arithmetic that the agent solves with a calculator.
It is not eligible when success requires an external change through a tool, such as a database query, or a sequence of tool calls checked by their final state, as in BFCL; for these types we report only the agent's \EpSG{}.
Eligibility depends only on the goal type, not on the resulting \OOI{}.

%% file: sections/sec5_measure.tex
\section{Goal Level Energy Measurement}\label{sec:alems}

The metrics of Sections~\ref{sec:epg} and~\ref{sec:ooi} depend on how accurately energy is attributed to each attempt.
Three measurement choices determine the energy charged: when counting starts and stops, how idle power and other programs are removed, and which part of the machine is counted; Section~\ref{sec:c3} shows how much each choice changes the results.

\textbf{\ALEMS{}.}
Our measurement tool, the Agentic LLM Energy Measurement System (\ALEMS{}), wraps each run of a workflow on x86, NVIDIA GB10 and Apple silicon and records every attempt of a goal, its outcome and the energy counters of the machine in measurement stores, from which scripts compute every metric in this paper.
The subsections below describe how it makes each of the three choices, and the check that decides which goals succeed.

\begin{figure}[t]
\centering
\resizebox{\columnwidth}{!}{%
\begin{tikzpicture}[>=Stealth, font=\small]
  \draw[very thick] (0,0) -- (10,0);
  \foreach \x/\lab in {0.8/{t_{\mathrm{pre}}}, 2.2/{\tzero}, 6.8/{\tone}, 9.2/{\ttwo}}
  { \draw[thick] (\x,-0.2) -- (\x,0.2); \node[above=4pt] at (\x,0.2) {$\lab$}; }
  \fill[blue!18]   (0.8, 0.02) rectangle (2.2,-0.7);
  \fill[green!22]  (2.2, 0.02) rectangle (6.8,-0.7);
  \fill[orange!20] (6.8, 0.02) rectangle (9.2,-0.7);
  \node[font=\scriptsize\bfseries, align=center] at (1.5,-0.35) {set up\\excluded};
  \node[font=\small\bfseries] at (4.5,-0.35) {charged to the goal $[\tzero,\tone]$};
  \node[font=\scriptsize\bfseries, align=center] at (8.0,-0.35) {clean up\\excluded};
  \node[font=\tiny, blue!60] at (4.5,-0.9) {all attempts, failed ones included};
\end{tikzpicture}}
\caption{Measurement window of one run.}
\label{fig:timeline}
\end{figure}

\subsection{Measurement Window}\label{sec:boundary}
Energy accumulates over time, so its measured value depends on when counting starts and stops.
\ALEMS{} reads the energy counters four times in every run, which divides the run into three intervals (Figure~\ref{fig:timeline}).
The pre task interval $[t_{\mathrm{pre}},\tzero]$ covers setting up the workflow.
The task interval $[\tzero,\tone]$ runs from the read just before the first model call to the return of the final answer, so it contains every attempt of the goal, failed attempts included.
The post task interval $[\tone,\ttwo]$ covers the framework's clean up, such as logging, writing results and shutting down.
Only the task interval is charged to the goal; the other two are recorded but excluded, because set up and clean up are steps of the framework rather than work the goal needs, and counting them would charge the goal energy it did not use, most visibly on short goals.

\subsection{Idle and Background Energy Removal}\label{sec:l1}\label{sec:l2}
The energy measured over the task interval also includes the machine's idle energy and the energy of other programs, and \ALEMS{} removes both (Appendix~\ref{appendix:layers}).

A machine draws power even when idle, about \val{facts.x86.idle_W}\,W at the package of our x86 machine, \val{facts.gb10.idle_W}\,W on GB10 and \val{facts.valid.mac_tinyllama.idle_W}\,W on the Mac~\cite{barroso2007energy,aquino2025energyindex,bertsch2025energyutilization}.
This energy grows with duration, so it would add proportionally more to short linear goals than to long agentic ones.
We therefore subtract the idle power over the length of the window:
\begin{equation}
  \Edyn = \max\!\bigl(0,\;\Epkg - P_{\mathrm{baseline}} \cdot \Delta t\bigr),
\end{equation}
where $P_{\mathrm{baseline}}$ is the mean power over several separate idle windows on the same machine, measured with fixed CPU affinity and with samples more than two standard deviations from the mean dropped.

How other programs are removed depends on where the model runs.
When the model runs inside the agent's own process, as with llama.cpp on our x86 machine, \ALEMS{} attributes to the goal the share of dynamic energy that matches the process's share of CPU time:
\begin{equation}
  \label{eq:attr}E_{\mathrm{attr}} = \fcpu \cdot \Edyn, \quad
  \fcpu = \frac{\Delta\mathrm{ticks}_{\mathrm{pid}}}
               {\Delta\mathrm{ticks}_{\mathrm{total}}},
\end{equation}
with both tick counts read from \texttt{/proc/\{pid\}/stat} and \texttt{/proc/stat} over $[\tzero, \tone]$ and $E_{\mathrm{attr}} \leq \Edyn$ enforced on every run.
CPU time is a common but approximate basis for attribution, since phases bound by memory or disk can draw energy that CPU ticks do not show.
On x86 it attributes \val{facts.x86.share_ag_pct}\% and \val{facts.x86.share_lin_pct}\% of the dynamic package energy to the agent process in the agentic and linear workflows, so charging the whole package instead would change \OOI{} only by a factor of \val{facts.x86.share_factor}.
When the model runs outside the agent's process, in a separate vLLM server on GB10 or on the GPU of the Mac, the agent's CPU share misses most of the work.
There we dedicate the machine to the experiment and run one goal at a time, so the dynamic energy of the chip belongs to the goal.
The harness also records the CPU load of other processes in every run.
On the Mac, other processes used \val{facts.valid.mac_tinyllama.bg_cap}\% of the CPU or more in \val{facts.valid.mac_tinyllama.bg_flagged} of \val{facts.valid.mac_tinyllama.runs_all} runs, and we exclude the goals of those runs; GPU use by other programs is not recorded and is kept low only by the dedicated setup.

\subsection{Hardware Energy Boundary}\label{sec:l0}
The third choice is the hardware boundary, which parts of the machine the counters cover, and we declare it for each machine (Appendix~\ref{appendix:boundary_detail}).
On x86, \RAPL~\cite{david2010rapl,intel_rapl} reports the energy of the CPU package, cores and uncore, updated by the processor about every millisecond~\cite{thamm2025raplworkflows,koszczal2025cpugpuenergy}, and we read it at 100\,Hz.
Memory is a separate domain that we do not count, and since the model runs on the CPU inside the agent process, the attributed package energy covers the work of the goal.
On GB10 (NVIDIA Grace Blackwell), the firmware power manager (SPBM) reports the energy of the whole system on chip, CPU and GPU together, at 10\,Hz.
We checked on every run that the GPU energy reported by DCGM never exceeded this total, so we count the package alone, since adding DCGM would count the GPU twice.
On the Mac (Apple silicon), the CPU and GPU are separate domains read at 10\,Hz, and we count both, removing CPU idle power and leaving in GPU idle power of about 0.07\,W.
Because the boundaries differ, we compare energy in joules only within one machine, and we compare machines through ratios such as \OOI{}, each computed within its own machine.

\subsection{Correctness Checking}\label{sec:checker}
A goal is successful only if the last output of its last attempt passes the check of its goal type (Section~\ref{sec:ontology}); an empty output counts as completed but unsuccessful.
Answers with a known value are compared with it: numbers by their final value, including the calculator goal, and facts by the expected answer or a listed equivalent.
Open text is checked against the requirements stated in its prompt, following IFEval~\cite{zhou2023ifeval}: the requested length, within 25\% in words or one sentence, the required terms and the requested script; BFCL uses the benchmark's own state checker; and the other tool goals confirm only that each required tool ran, so they enter no \OOI{} (Section~\ref{sec:ooi:elig}).
All checks run after the experiment and apply identically to both workflows; Appendix~\ref{appendix:checks} gives each rule.

We tested the checks on the \val{facts.lenient.n} comparisons of the six fixed goal types.
A manual audit of \val{facts.audit.n} randomly sampled outputs, judged without the checker's verdict, agreed with the checker on \val{facts.audit.agree}; every disagreement was a correct arithmetic answer followed by a number restated from the question.
We also tried two looser checks: one accepts an output that lists several lettered options, and one accepts any appearance of the expected number.
They make \val{facts.lenient.flipped} and \val{facts.numloose.flipped} more goals successful, mostly in the linear workflow, but every \OOI{} stays above 1 under both.

\subsection{Reproducibility}\label{sec:repro}
\ALEMS{} stores fingerprints of the hardware, the software and the runtime state with every run, so that a difference between two runs can be traced to one of them (Appendix~\ref{appendix:repro_detail}).
Every number in this paper is computed by scripts from the raw measurement stores, which record each excluded goal and the reason (Appendix~\ref{appendix:method}).

%% file: sections/sec6_setup.tex
\section{Experimental Setup}\label{sec:validation}\label{sec:apparatus}\label{sec:c1}\label{sec:c2}\label{sec:setup:models}\label{sec:setup:goals}

\begin{table*}[t]
\centering
\footnotesize
\definecolor{fzebraD}{HTML}{E8EDF3}
\renewcommand{\arraystretch}{1.3}
\caption{\textbf{Research questions and the runs used to answer them.}
A set of runs that serves two questions is counted once in the total; Table~\ref{tab:setup} gives the goals per goal type and workflow.}
\label{tab:tasks}
\IfFileExists{generated/tab_design.tex}{\input{generated/tab_design.tex}}{\TBD[: tab\_design]}
\end{table*}

\subsection{Design and Procedure}\label{sec:setup:workflows}\label{sec:setup:runs}

The experiments test whether goal-level metrics show what energy per call cannot: whether an agent is worth its extra energy, which model costs least per successful goal, and how much failed attempts add to the energy of a goal.

\textbf{Goal Coverage.}
We evaluate the metrics across a range of goal types (Table~\ref{tab:tasks}).
A fixed question is one prompt submitted many times, and distinct questions are different prompts from a benchmark, each submitted once per workflow.
The six fixed questions run with TinyLlama 1B and the remote Llama 3.3 70B on x86, with TinyLlama 1B, Mistral 7B and Qwen3 30B on GB10, and with TinyLlama 1B on the Mac, which gives $6 \times (2+3+1) = \val{facts.design.n_fixed_comp}$ comparisons.
On GB10, Qwen3 30B also runs \val{facts.n_breadth_types} further goal types with one fixed question each and 50 distinct questions from each of GSM8K~\cite{cobbe2021gsm8k}, MMLU~\cite{hendrycks2021mmlu} and HumanEval~\cite{chen2021humaneval}, and Mistral 7B runs the 50 GSM8K questions, which gives $\val{facts.n_breadth_types} + 3 + 1 = \val{facts.design.n_later_comp}$ comparisons.
A calculator goal runs with TinyLlama 1B on x86 and GB10, Mistral 7B and Qwen3 30B, which gives $\val{facts.design.n_calc_comp}$ more and $\val{facts.design.n_fixed_comp} + \val{facts.design.n_later_comp} + \val{facts.design.n_calc_comp} = \val{facts.design.n_comp}$ in all.
Each run of a fixed question counts as a separate goal, so repeated runs show how energy varies for the same question, and the distinct questions show how it varies across questions.
Two benchmarks are examined in greater depth: HumanEval, where unit tests check the generated code, and BFCL~\cite{patil2025bfcl}, where completing a goal requires several tool calls.

\textbf{Platforms and Models.}
We run the experiments on three machines, an x86 system, an ARM based GB10 and an Apple M4, which differ in their energy sampling rate and in the part of the hardware they measure, and on one remote service, Llama 3.3 70B through NVIDIA NIM, for which we measure only the client.
The models range from 1B parameters to 70B, including a 30B mixture of experts model, which lets us test whether a model that uses more energy per goal can still use less energy per successful goal.

\textbf{Workflow Comparison.}
A comparison is one goal type run with one model on one machine in both workflows, on goals that both can attempt.
The linear workflow sends the goal's prompt in one request and is never retried.
The agent's framework wraps the same prompt in its own instructions to plan, use tools and answer thoroughly, as agent frameworks do; the agent then plans its steps, executes them and produces its answer.
If a step times out, the model server returns an error or the framework cannot parse a reply, the agent restarts the entire goal, with up to five additional attempts, but never after an incorrect answer.
For the distinct questions and the further goal types, the agent can also use a calculator or execute code when it decides these tools are needed.
BFCL goals require several tool calls and cannot be completed by a single call, so on BFCL we compare models rather than workflows.
All runs sample at temperature 0.7, with the per-model limits on generated tokens set in the configuration of the artifact.

\textbf{Measurement and Statistical Analysis.}
We run one goal at a time on each machine and subtract its measured idle power from every run (Section~\ref{sec:alems}).
Failed and timed out goals remain in the analysis as not completed.
We exclude records only when the experiment was marked invalid at run time, a measurement fault was documented, the stored outcome contradicted the last attempt, or other programs used heavy CPU on the Mac (Section~\ref{sec:l1}; counts in Appendix~\ref{appendix:config}).
We compute ratios from each workflow's total energy, time and counts.
Their 95\% confidence intervals use a percentile bootstrap with 2{,}000 resamples and a fixed seed~\cite{efron1979bootstrap}.
For fixed questions, we resample goals within each workflow; for distinct questions, we resample whole questions, keeping both workflows' results for each question together.
Because \OOI{} is unstable with few successes, we report an interval only when both workflows have at least 10 successful goals.
Otherwise, we report a point estimate, or an upper bound when the linear workflow has no successes (Section~\ref{sec:ooi:split}).
Since this threshold retains comparisons in which the linear workflow succeeds often, Section~\ref{sec:c4} also tests lower thresholds.

\begin{figure}[t]
\centering
\IfFileExists{figures/fig_ooi_decomposition.pdf}{\includegraphics[width=\columnwidth]{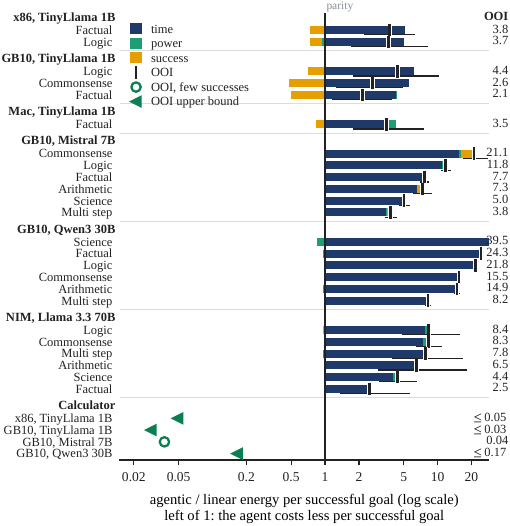}}{\ph{4}{fig\_ooi\_decomposition}{figure in progress}}
\caption{\textbf{The agentic workflow costs more per successful goal because it runs longer.} Each row is the \OOI{} of one goal type, drawn as the product of its time, power and success ratios on a log scale, with the success ratio to the left when the agentic workflow is correct more often. Thin line: 95\% interval. Calculator: \OOI{} as a point estimate (circle) where the linear workflow succeeded a few times, and as an upper bound (arrow) where it never did.}
\label{fig:ooi}
\end{figure}

%% file: generated/tab_design.tex
\setlength{\tabcolsep}{4pt}\begin{tabular}{@{}>{\raggedright\arraybackslash}p{0.03\textwidth}>{\raggedright\arraybackslash}p{0.15\textwidth}>{\raggedright\arraybackslash}p{0.13\textwidth}>{\raggedright\arraybackslash}p{0.17\textwidth}>{\raggedright\arraybackslash}p{0.26\textwidth}>{\raggedright\arraybackslash}p{0.10\textwidth}>{\raggedleft\arraybackslash}p{0.05\textwidth}@{}}
\toprule
\rowcolor{tMhd}\textbf{RQ} & \textbf{Question} & \textbf{Platforms} & \textbf{Models} & \textbf{Goal types} & \textbf{Comparisons} & \textbf{Goals} \\
\midrule
\rowcolor{fzebraD}RQ1 & Does an agent justify its extra energy? & x86, GB10, Apple M4, remote client & TinyLlama 1B, Mistral 7B, Qwen3 30B, Llama 3.3 70B & Six fixed questions (factual, science, logic, commonsense, arithmetic, multi step) on every machine; on GB10, 17 further goal types and GSM8K, MMLU and HumanEval; the calculator goal on x86 and GB10 & $\val{facts.design.n_fixed_comp} + \val{facts.design.n_later_comp} + \val{facts.design.n_calc_comp} = \val{facts.design.n_comp}$ & 5,561 \\
RQ2 & Which model costs least per successful goal? & GB10 & TinyLlama 1B, Mistral 7B, Qwen3 30B, gpt-oss 20B, Mistral Small 3.2 24B, Granite 3.3 8B, Qwen3 4B & Six fixed questions and distinct GSM8K; BFCL multi turn goals, which need several tool calls & \val{facts.flip.n} rankings; \val{facts.bfcl.n_conv} BFCL conversations & 7,840 \\
\rowcolor{fzebraD}RQ3 & What does a failed attempt cost? & GB10 & Mistral 7B & Arithmetic and tool goals, with zero, one or two forced timeouts per agentic goal & 3 conditions & 660 \\
RQ4 & Does the measurement setup change the answer? & x86 (CPU package), GB10 (CPU and GPU together), Apple M4 (CPU and GPU separate) & TinyLlama 1B, Mistral 7B, Qwen3 30B & All goal types run on GB10; fixed questions on x86 and the Mac & 3 choices & 4,853 \\
\midrule
 & \textbf{Total} & \textbf{3 machines and 1 remote service} & \textbf{8 models} & \textbf{31 goal types} & & \textbf{11,621} \\
\bottomrule
\end{tabular}

%% file: sections/sec7_results.tex
\section{Results}\label{sec:results}\label{sec:c6}\label{sec:c7}\label{sec:c8}

\subsection{Workflow Choice}\label{sec:c4}

Each comparison keeps the model, machine and goals fixed and asks whether the agentic workflow's higher success rate justifies its extra energy.
Where both workflows had at least 10 successful goals, \val{facts.ooi_total} comparisons over \val{facts.n_goal_types_ooi} goal types, the agentic workflow had the higher \EpSG{} every time, with \OOI{} from \val{facts.newsets.ooi_min}$\times$ to \val{facts.newsets.ooi_max}$\times$ and every 95\% interval above 1 (Figure~\ref{fig:ooi}).
For reference, Table~\ref{tab:ratios} in Appendix~\ref{appendix:full} lists every comparison.

The break-even rule of Section~\ref{sec:ooi:split} explains this result.
On these goal types the single call already succeeded so often that the agent could not win.
In \val{facts.be.n_impossible} of the \val{facts.be.n} comparisons, the agent would have used more energy per successful goal even if it had solved every goal.
In the remaining comparison, the agent had to solve \val{facts.be.need_min_pct}\% of its goals to match the single call, and it solved fewer.
Correctness changes the size of the gap in both directions.
On GB10 with Mistral 7B, the linear workflow was correct more often on commonsense goals, which raised the energy ratio of \val{gb10.commonsense_reasoning.energy_ratio}$\times$ to an \OOI{} of \val{gb10.commonsense_reasoning.OOI_EpSG}$\times$.
On x86 with TinyLlama, the agentic workflow was correct more often on factual goals, which lowered the energy ratio of \val{x86.factual_qa.energy_ratio}$\times$ to an \OOI{} of \val{x86.factual_qa.OOI_EpSG}$\times$, still above 1.

Giving the agent tools did not help on goals that the single call already solved.
In the later runs the agent could call a calculator or run code whenever it chose.
On the GSM8K word problems it used a tool on \val{facts.tool.gsm8k.n_tool} of its \val{facts.tool.gsm8k.n_ag} goals but solved \val{facts.tool.gsm8k.ag_ok}, fewer than the \val{facts.tool.gsm8k.lin_ok} solved by the single call (Table~\ref{tab:tooluse}), so \OOI{} remained \val{facts.tool.gsm8k.ooi}$\times$.
On the hard reasoning goal both workflows solved every goal, so \OOI{} equals the energy ratio, \val{facts.tool.reasoning.ooi}$\times$, the lowest among these goal types.

\begin{table}[!tb]
\centering
\footnotesize
\caption{\textbf{Tool use in the agentic workflow.} Agent goals in which a tool ran and succeeded; \OOI{} with its 95\% interval.}
\label{tab:tooluse}
\setlength{\tabcolsep}{3pt}
\begin{tabular}{@{}lrrrr@{}}
\toprule
\rowcolor{tMhd}Goal type, model & Agent goals & \multicolumn{2}{c}{Correct goals} & \OOI{} \\
\rowcolor{tMhd} & using a tool & agentic & linear & \\
\midrule
GSM8K, Qwen3 & \val{facts.tool.gsm8k.n_tool}/\val{facts.tool.gsm8k.n_ag} & \val{facts.tool.gsm8k.ag_ok} & \val{facts.tool.gsm8k.lin_ok} & \val{facts.tool.gsm8k.ooi} {\tiny[\val{facts.tool.gsm8k.ooi_lo}, \val{facts.tool.gsm8k.ooi_hi}]} \\
HumanEval, Qwen3 & \val{facts.tool.humaneval.n_tool}/\val{facts.tool.humaneval.n_ag} & \val{facts.tool.humaneval.ag_ok} & \val{facts.tool.humaneval.lin_ok} & \val{facts.tool.humaneval.ooi} {\tiny[\val{facts.tool.humaneval.ooi_lo}, \val{facts.tool.humaneval.ooi_hi}]} \\
Hard reasoning, Qwen3 & \val{facts.tool.reasoning.n_tool}/\val{facts.tool.reasoning.n_ag} & \val{facts.tool.reasoning.ag_ok} & \val{facts.tool.reasoning.lin_ok} & \val{facts.tool.reasoning.ooi} {\tiny[\val{facts.tool.reasoning.ooi_lo}, \val{facts.tool.reasoning.ooi_hi}]} \\
Planning, Qwen3 & \val{facts.tool.planning.n_tool}/\val{facts.tool.planning.n_ag} & \val{facts.tool.planning.ag_ok} & \val{facts.tool.planning.lin_ok} & \val{facts.tool.planning.ooi} {\tiny[\val{facts.tool.planning.ooi_lo}, \val{facts.tool.planning.ooi_hi}]} \\
\midrule
GSM8K, Mistral 7B & \val{facts.tool.mistral.n_tool}/\val{facts.tool.mistral.n_ag} & \val{facts.tool.mistral.ag_ok} & \val{facts.tool.mistral.lin_ok} & \val{facts.tool.mistral.ooi} {\tiny[\val{facts.tool.mistral.ooi_lo}, \val{facts.tool.mistral.ooi_hi}]} \\
\bottomrule
\end{tabular}
\end{table}

Table~\ref{tab:example} follows one GSM8K goal to show where the extra energy goes.
The linear workflow solved the question with one model call.
The agentic workflow made four model calls, a plan, two solutions and a final answer, and checked the result with the calculator, which the table does not list because it is not a model call.
It produced \val{facts.ex.ag.tok} output tokens instead of \val{facts.ex.lin.tok} and used \val{facts.ex.ratio}$\times$ the energy for the same successful goal.
No single call in this goal was unusually expensive; the agent used more energy because it made more calls, and only a measure per goal counts them.

\begin{table}[!tb]
\centering
\footnotesize
\caption{\textbf{One GSM8K goal, call by call.} Both workflows succeeded; the agentic goal is the one nearest the mean energy of its goal type.}
\label{tab:example}
\setlength{\tabcolsep}{4pt}
\begin{tabular}{@{}llrrr@{}}
\toprule
\rowcolor{tMhd}Workflow & Step & Tokens out & Time (s) & Energy (J) \\
\midrule
Linear & answer & \val{facts.ex.lin.tok} & \val{facts.ex.lin.s} & \val{facts.ex.lin.J} \\
\midrule
Agentic & plan & \val{facts.ex.ag1.tok} & \val{facts.ex.ag1.s} & \val{facts.ex.ag1.J} \\
 & solve & \val{facts.ex.ag2.tok} & \val{facts.ex.ag2.s} & \val{facts.ex.ag2.J} \\
 & solve again & \val{facts.ex.ag3.tok} & \val{facts.ex.ag3.s} & \val{facts.ex.ag3.J} \\
 & final answer & \val{facts.ex.ag4.tok} & \val{facts.ex.ag4.s} & \val{facts.ex.ag4.J} \\
 & \textbf{total} & \textbf{\val{facts.ex.ag.tok}} & \textbf{\val{facts.ex.ag.s}} & \textbf{\val{facts.ex.ag.J}} \\
\bottomrule
\end{tabular}
\end{table}

The calculator goal shows the opposite result.
Unlike the GSM8K problems, it asks for one exact computation, which a single call can answer in text but rarely computed correctly, while the agent handed it, a deterministic step, to the calculator.
With Mistral 7B, the linear workflow solved \val{gb10.tg_single_calc.linear.n_correct} of \val{gb10.tg_single_calc.linear.N} goals and \OOI{} was \val{facts.calc.ooi_point}$\times$; with the other three models it solved none, and the upper bound of Section~\ref{sec:ooi:split} places \OOI{} at or below \val{facts.calc.bound_max}$\times$.
\OOI{} was therefore below 1 in all \val{facts.calc.n} comparisons.
The same break-even rule explains it: the single call solved at most \val{facts.calc.plin_max_pct}\% of its goals, so the agent could win, and it did.
This held even where the agent had the higher \EpG{}: with Qwen3 30B it was \val{facts.calc.re_max}$\times$ that of the single call, but the agent solved every goal while the single call solved none.
With Mistral 7B, the agent's \EpG{} was lower, \val{facts.calc.re_min}$\times$ that of the single call.

The rule of at least 10 successful goals leaves out comparisons in which the single call rarely succeeds, which is where the agent could win, so we checked that it does not hide such wins.
Lowering it to 5, 3 and 1 successful goals adds \val{facts.thresh.n5}, \val{facts.thresh.n3} and \val{facts.thresh.n1} comparisons, none of them with an \OOI{} below 1; the lowest is \val{facts.thresh.ooi_min}$\times$, with TinyLlama on x86 science goals.
Where the single call never succeeded, \OOI{} is undefined, and we report its upper bound instead.
With TinyLlama on GB10 science goals, the single call returned no answer on most goals and solved none of \val{gb10_tinyllama.science_qa.linear.N}, while the agent solved \val{gb10_tinyllama.science_qa.agentic.n_correct}.
The upper bound is $\OOI{} \leq \val{gb10_tinyllama.science_qa.energy_ratio} \times \val{facts.zero.p_upper_pct}\% / \val{facts.zero.p_ag_pct}\% = \val{facts.zero.ooi_upper}$, so the data cannot tell which workflow costs less per successful goal, but only the agent produced successful goals.

For Llama 3.3 70B on NVIDIA NIM, we measured only the client, which used \val{x86_remote_nim.factual_qa.OOI_EpSG}$\times$ to \val{x86_remote_nim.logical_reasoning.OOI_EpSG}$\times$ the energy per successful goal in the agentic workflow; without the energy of the provider, the comparison for the whole service cannot be determined (Appendix~\ref{appendix:boundary_detail}).
Goals that require an external action through a tool, such as a database query, a tool chain or a multi-turn BFCL conversation, cannot be performed by a single call; for these, \EpSG{} compares models and agent designs instead (Section~\ref{sec:res:model}).
Our agent retries only when a step fails with an error and never checks whether its answer is correct.
Agents that check and fix their answers, for example by running tests on code they wrote, may succeed more often, and the break-even rule tells how much more often they must succeed to justify their extra energy.

\begin{finding}{RQ1}
An agentic workflow can cost less per successful goal only when the linear workflow succeeds on fewer than $1/R_E$ of its goals.
In our runs the linear workflow stayed above this limit in all but one comparison where a single call already succeeded, and \OOI{} was above 1 there; on the calculator goal, where it rarely succeeded, \OOI{} was below 1 in every comparison.
Teams can apply this rule on a sample of goals to decide whether a goal type needs an agent.
\end{finding}

\subsection{Model Choice}\label{sec:res:model}\label{sec:res:correct}

\begin{figure}[t]
\centering
\IfFileExists{figures/fig_selection.pdf}{\includegraphics[width=\columnwidth]{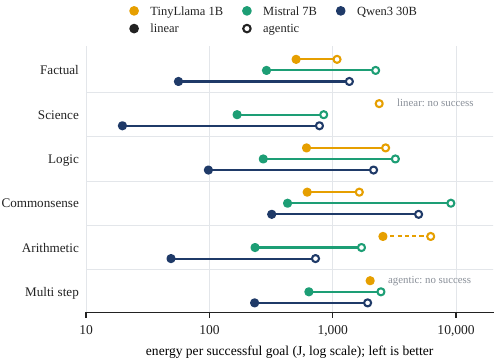}}{\ph{3}{fig\_selection}{figure in progress}}
\caption{\textbf{Energy per successful goal of each model in each workflow.} GB10, same goals and checks; colour: model with its engine (TinyLlama in llama.cpp, Mistral and Qwen3 in vLLM). Filled: linear workflow; hollow: agentic workflow; line length: \OOI{}. Dashed: fewer than 10 successful goals in one workflow.}
\label{fig:selection}
\end{figure}

\textbf{Ranking by \EpSG{}}
Section~\ref{sec:c4} held the model fixed to compare workflows.
Here, we hold the workflow fixed and compare three serving configurations on GB10: TinyLlama 1B, Mistral 7B and Qwen3 30B, each with its own precision and engine (Figure~\ref{fig:selection}).
For each goal type and workflow we rank the configurations once by \EpG{} and once by \EpSG{}, which gives the \val{facts.flip.n} rankings of Table~\ref{tab:modelflip}: seven goal types, each in both workflows, where the distinct GSM8K questions compare only Mistral 7B and Qwen3 30B.
TinyLlama 1B had the lowest \EpG{} in \val{facts.flip.tiny_epg} of the \val{facts.flip.tiny_n} rankings that include it.
However, it had the lowest \EpSG{} in only \val{facts.flip.tiny_epsg} rankings.
In the remaining \val{facts.flip.changed} rankings, Qwen3 30B had the lowest \EpSG{}.
Under the break-even rule of Section~\ref{sec:ooi:split}, a configuration that uses more energy per goal can still use less energy per successful goal if it succeeds on enough more goals, and in these \val{facts.flip.changed} rankings Qwen3 30B succeeded on so many more goals than TinyLlama that its higher energy per goal was more than offset.
In the two rankings where TinyLlama kept the lower \EpSG{}, it succeeded often enough that Qwen3's extra successes did not make up for its extra energy.
The rankings agree when the configuration with the lowest \EpG{} is also the most accurate, as on the distinct GSM8K questions, where Qwen3 30B had both the lowest \EpG{} and the lowest \EpSG{}, and on the six fixed questions, where the linear workflow with Qwen3 30B had a \val{facts.modelsel.red_min} to \val{facts.modelsel.red_max}\% lower \EpSG{} than with Mistral 7B, with each 95\% interval of the ratio below 1; model and precision changed together there, so the saving belongs to the configuration and not to the model alone.

\begin{table}[!tb]
\centering
\footnotesize
\caption{\textbf{Ranking by \EpSG{} instead of \EpG{} changes the cheapest GB10 configuration.}
Lowest \EpG{} $\rightarrow$ lowest \EpSG{}; one name: both agree; $\dagger$: the change does not hold at the 95\% binomial bounds.}
\label{tab:modelflip}
\IfFileExists{generated/tab_modelflip.tex}{\input{generated/tab_modelflip.tex}}{\TBD[: tab\_modelflip]}
\end{table}

\textbf{Completion is not success.}
Energy per completed goal does not tell which model delivers results either.
On the BFCL benchmark, four models ran the same \val{facts.bfcl.n_conv} multi-turn conversations (Table~\ref{tab:bfclcommon}).
Granite 8B had a lower \EpG{} than gpt-oss 20B and Mistral Small 24B and completed every goal, so by \EpG{} and \EpCG{} it looks like the second-best choice of the four.
But it solved none of them, so its \EpSG{} is undefined and it is the worst choice.
We checked six of these goals by hand: the framework marked each as solved, but Granite either wrote its function calls as text instead of executing them or stated that an action had been taken without calling the tool, so the environment never reached the required state.
This agrees with the multi-turn accuracy of 7.4 to 7.5\% that the BFCL leaderboard reports for Granite 3.x 8B~\cite{bfcl_leaderboard}, at which no success in \val{facts.bfcl.n_conv} conversations has a probability of about 10\%.
Among the three models that did succeed, the highest \EpSG{} was \val{facts.bfcl.epsg_x}$\times$ the lowest and the highest \TpSG{} \val{facts.bfcl.tpsg_x}$\times$ the lowest, and Qwen3 30B was best on both.

\begin{table}[!tb]
\centering
\footnotesize
\caption{\textbf{BFCL, four models on the same \val{facts.bfcl.n_conv} conversations, sorted by \EpG{}.} Energy in joules and \TpSG{} in seconds, whole chip, idle removed.}
\label{tab:bfclcommon}
\IfFileExists{generated/tab_bfcl_common.tex}{\input{generated/tab_bfcl_common.tex}}{\TBD[: tab\_bfcl\_common]}
\end{table}

\textbf{Robustness.}
Some of the ranking changes rest on few successful TinyLlama goals, so TinyLlama might simply have been unlucky in our runs.
To test this, we gave TinyLlama the highest success rate its data allows, the upper end of its 95\% exact binomial interval, and gave Qwen3 the lowest, the lower end of its interval~\cite{clopper1934binomial}.
Even under this assumption in TinyLlama's favour, \val{facts.flip.holds} of the \val{facts.flip.changed} changes held (the other is marked in Table~\ref{tab:modelflip}), so the change of ranking is not an effect of small counts.

\textbf{Time per successful goal.}
By Equation~\ref{eq:tradeoff}, energy and time per successful goal point to different choices only when the faster configuration draws more power by a larger factor than it saves in time.
In all \val{facts.flip.best_same_n} rankings on GB10, the configuration with the lowest \EpSG{} also had the lowest \TpSG{}, so one configuration was best for both energy and time.
Choosing by \EpG{} therefore cost time as well as energy: in \val{facts.flip.tpsg_n} of the changed rankings, TinyLlama's \TpSG{} was \val{facts.flip.tpsg_x_min}$\times$ to \val{facts.flip.tpsg_x_max}$\times$ that of Qwen3 30B, and in the remaining one TinyLlama had no successful goals, so its \EpSG{} and \TpSG{} were undefined.

\textbf{Tool access.}
Offering tools did not by itself raise success: on the same GSM8K questions, Mistral 7B used the calculator but often passed it named quantities, such as the cost per night, instead of numbers, while Qwen3 30B used it correctly (Table~\ref{tab:tooluse}).

\begin{finding}{RQ2}
Only \EpSG{} shows whether the energy of a goal bought a successful result.
On identical BFCL conversations, a model that completed every goal but solved none looked like the second-best choice by \EpG{} and \EpCG{}, and on GB10 the configuration with the lowest \EpG{} lost its lead in \val{facts.flip.changed} of \val{facts.flip.tiny_n} rankings.
Teams selecting a model for goals where correctness matters should compare \EpSG{} on a sample of their goals.
\end{finding}

\subsection{Retry Cost}\label{sec:c5}

\begin{figure}[t]
\centering
\IfFileExists{figures/fig_claim1_failures.pdf}{\includegraphics[width=\columnwidth]{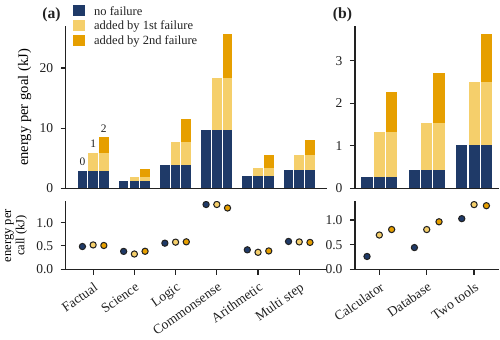}}{\ph{2.2}{fig\_claim1\_failures}{figure in progress}}
\caption{\textbf{Forced failures raise energy per goal; energy per call stays flat only on reasoning goals.} Mistral 7B on GB10, agentic workflow; bars: zero, one and two forced failures. Top: energy per goal. Bottom: energy per model call. (a) Reasoning goals. (b) Tool goals.}
\label{fig:claim1}
\end{figure}

The energy values in Sections~\ref{sec:c4} and~\ref{sec:res:model} include every attempt of a goal, failed ones as well.
To measure how much failed attempts add, we ran the agentic workflow with Mistral 7B on GB10 under three conditions, with \val{facts.fail.n_ag} goals of each goal type per condition: no forced failure, one forced timeout and two forced timeouts, where a timeout injected during an attempt makes it fail after \val{facts.fail.fail_s}\,s on average and the agent starts a new attempt; the timeout reached \val{facts.fail.n_hit} of the \val{facts.fail.n_ag} goals of each type.
On average, a goal then took \val{facts.claim1.A_attempts}, \val{facts.claim1.B_attempts} and \val{facts.claim1.C_attempts} attempts under the three conditions.

On reasoning goals, each failed attempt repeated the planning and model calls of a complete attempt, so the model calls per goal rose, from \val{facts.claim1.calls_A} to \val{facts.claim1.calls_C} on the arithmetic goal.
Energy per goal (\EpG{}) therefore rose \val{facts.claim1.reason_B_epg_x_min}$\times$ to \val{facts.claim1.reason_B_epg_x_max}$\times$ with one forced failure and \val{facts.claim1.reason_C_epg_x_min}$\times$ to \val{facts.claim1.reason_C_epg_x_max}$\times$ with two, each failed attempt adding about as much energy as a goal that succeeded on its first attempt (\val{facts.claim1.reason_fail_x_min}$\times$ to \val{facts.claim1.reason_fail_x_max}$\times$ as much).
Energy per call, by contrast, stayed within \val{facts.claim1.reason_call_dev_pct}\% of its value without failures, so a report per call would have shown almost no change (Figure~\ref{fig:claim1}a).

On tool goals, the agentic workflow made one model call per attempt, so a forced failure added a single long call that consumed energy while it waited for its timeout.
Energy per call therefore rose as well, \val{facts.claim1.tool_C_call_x_min}$\times$ to \val{facts.claim1.tool_C_call_x_max}$\times$ with two forced failures, but energy per goal rose more, \val{facts.claim1.tool_C_epg_x_min}$\times$ to \val{facts.claim1.tool_C_epg_x_max}$\times$ (Figure~\ref{fig:claim1}b).

\begin{finding}{RQ3}
Energy per call stayed almost unchanged while the energy of a goal grew with every failed attempt, so only \EpG{} shows what retries cost.
A retry budget can be evaluated from \EpG{} as the energy of one attempt times the attempts it allows, with timed-out attempts counted separately, since a call that waits for its timeout cost up to \val{facts.claim1.tool_fail_x_max}$\times$ the energy of a goal that succeeded at once.
\end{finding}

\subsection{Measurement Choice}\label{sec:c3}

\begin{figure}[t]
\centering
\includegraphics[width=\columnwidth]{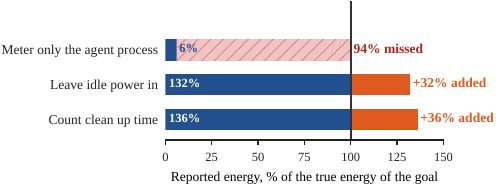}
\caption{\textbf{Each measurement choice changes the energy of a goal.} GB10, the same goals throughout; 100\% is the energy charged to the goal by our method.}
\label{fig:rq4}
\end{figure}
Clean up is a fixed step after each goal, logging, writing results and shutting down, so its energy does not grow with the work of the goal and is a larger share of short goals than of long ones.
On GB10, it adds \val{facts.claim4.gb10.cleanup_pct_lin}\% to a linear goal but only \val{facts.claim4.gb10.cleanup_pct_ag}\% to an agentic one (Figure~\ref{fig:rq4}), so counting it lowers \OOI{} to between \val{facts.claim4.gb10.cleanup_x_min} and \val{facts.claim4.gb10.cleanup_x_max} of its value.
The effect is largest for short goal types: on topic classification, an \OOI{} of \val{gb10_qwen3_breadth.topic_classification.OOI_EpSG}$\times$ would shrink by a factor of \val{facts.claim4.gb10.topic_classification.cleanup_x} if clean up were counted.
On x86, where goals take longer, the effect stays below 1\%.

Leaving idle power in raises the energy of a GB10 goal to \val{facts.unacc.idle_pct}\% of the energy our method charges and lowers \OOI{} to between \val{facts.claim4.gb10.idle_x_min} and \val{facts.claim4.gb10.idle_x_max} of its value on GB10, and to between \val{facts.claim4.x86.idle_x_min} and \val{facts.claim4.x86.idle_x_max} on x86.
Both choices therefore make the agentic workflow look cheaper than it is, yet no \OOI{} crossed 1 under either, so they change the size of the overhead but not which workflow costs less.

\begin{finding}{RQ4}
Where the model runs in a separate server, metering only the agent process misses most of the energy of a goal.
Counting clean up or leaving idle power in shrinks the measured overhead of an agent, by up to a factor of about three on short goals, without changing which workflow costs less.
Energy measurements of agents should therefore cover the hardware that runs the model and state the metered hardware, the window and the idle treatment, so that results can be compared.
\end{finding}

%% file: generated/tab_modelflip.tex
\begin{tabular}{@{}lll@{}}
\toprule
\rowcolor{tMhd}Goal type & Linear & Agentic \\
\midrule
Factual & Qwen3 & TinyLlama \\
Science & Qwen3 & TinyLlama $\rightarrow$ Qwen3 \\
Logic & Qwen3 & TinyLlama $\rightarrow$ Qwen3$^{\dagger}$ \\
Commonsense & TinyLlama $\rightarrow$ Qwen3 & TinyLlama \\
Arithmetic & Qwen3 & TinyLlama $\rightarrow$ Qwen3 \\
Multistep & TinyLlama $\rightarrow$ Qwen3 & TinyLlama $\rightarrow$ Qwen3 \\
GSM8K, distinct & Qwen3 & Qwen3 \\
\bottomrule
\end{tabular}

%% file: generated/tab_bfcl_common.tex
\setlength{\tabcolsep}{2.5pt}\begin{tabular}{@{}lrrrrrrr@{}}
\toprule
\rowcolor{tMhd}Model & Goals & Compl. & Succ. & \EpG{} & \EpCG{} & \EpSG{} & \TpSG{} \\
\midrule
Qwen3 30B & 3,591 & 96\% & 86\% & 983 & 1,029 & 1,145 & 22.3 \\
Granite 8B & 60 & 100\% & 0\% & 1,860 & 1,860 & \vna{undef.} & \vna{undef.} \\
gpt-oss 20B & 376 & 96\% & 78\% & 2,223 & 2,322 & 2,834 & 50.0 \\
Mistral Small 24B & 255 & 100\% & 76\% & 3,458 & 3,458 & 4,546 & 72.5 \\
\bottomrule
\end{tabular}

%% file: sections/sec8_discussion.tex
\section{Discussion and Limitations}\label{sec:discussion}\label{sec:limitations}

Measuring energy per successful goal supports the engineering decisions in Table~\ref{tab:decisions}, each backed by evidence from our runs.
Two of these decisions are controlled comparisons on the same goals: with the workflow fixed, \EpSG{} compares models, and with the model fixed, \OOI{} compares workflows (Figures~\ref{fig:selection} and~\ref{fig:ooi}).

\begin{table}[t]
\centering
\footnotesize
\definecolor{fzebraD}{HTML}{E8EDF3}
\renewcommand{\arraystretch}{1.2}
\caption{\textbf{Engineering decisions supported by goal-level energy accounting.}}
\label{tab:decisions}
\setlength{\tabcolsep}{3pt}
\begin{tabular}{@{}>{\raggedright\arraybackslash}p{0.27\columnwidth}>{\raggedright\arraybackslash}p{0.19\columnwidth}>{\raggedright\arraybackslash}p{0.50\columnwidth}@{}}
\toprule
\rowcolor{tMhd}Objective & Metric & Rule \\
\midrule
\rowcolor{fzebraD}Choose a workflow & \OOI{} & agentic only if \OOI{} is below 1 (RQ1) \\
Select a configuration & \EpSG{} & lowest \EpSG{}, not lowest \EpG{} (RQ2) \\
\rowcolor{fzebraD}Accuracy-critical goals & success rate, \EpSG{} & meet the required success rate, then lowest \EpSG{} \\
Time-constrained goals & \TpSG{}, \EpSG{} & meet the deadline by \TpSG{}, then lowest \EpSG{} (RQ2) \\
\rowcolor{fzebraD}Budget retries & \EpG{} & energy of one attempt times the attempts allowed (RQ3) \\
\bottomrule
\end{tabular}
\end{table}

The decomposition of \OOI{} into time, power and success ratios shows which factor makes an agent more expensive, but it does not predict how much energy a given optimization would save.
The failure runs measure the energy cost of one failure; choosing a retry budget or timeout also requires knowing how often a retry succeeds, which these runs do not measure.
\OOI{} applies only to goals a single call can reach.
For goals that require an external action through a tool, such as a database query, an agent is necessary, and \EpSG{} then compares its models and configurations.
For open-text goals, the check confirms only that an output meets the requirements of its prompt, such as its length and required terms, and does not judge whether its content is accurate.

Our energy values depend on the counters each machine offers, and we rely on prior work for their accuracy~\cite{thamm2025raplworkflows,koszczal2025cpugpuenergy}.
Removing other programs by their share of CPU time is an approximation.
A wall power meter could check total energy, but it cannot split energy by attempt.
A fixed multiplicative bias in a counter cancels in every ratio we report, because each ratio uses the same counter for both quantities.
Appendix~\ref{appendix:gaming} discusses ways to game the metrics and how to detect them.

%% file: sections/sec9_conclusion.tex
\section{Conclusion}\label{sec:conclusion}\label{sec:future}

An agent decides for itself how much work a goal takes, so energy per call cannot describe the cost of achieving that goal.
This paper makes the goal the unit of energy accounting: every attempt is charged to its goal, and energy is reported per submitted, completed and successful goal.

Measured per goal, a failed attempt costs about as much as a goal that succeeds at once, a cost that no report per call shows.
Finishing a goal is not delivering it, so a system can look efficient while delivering nothing.
An agent justifies its extra calls in two cases: when the linear workflow fails often, or when a tool does part of the work, such as a calculation, with less energy than the model needs to write it out in text; both can be checked on a sample of goals before deployment.
The metered hardware decides what is seen: a boundary that misses the hardware running the model misses most of the energy.

We see these metrics as the basis of a goal-level benchmark: a fixed set of goals with declared checks, run with every model on every machine under a declared window, idle correction and hardware boundary, and reported as \EpG{}, \EpSG{} and \OOI{} per goal type.
Such a benchmark would let a team choose the model, machine and workflow for its goals by measured energy per successful goal.
Framework authors could use it to set retries and timeouts, and researchers to make their measurements comparable.
It extends to multi-agent frameworks, routers that decide whether a goal needs an agent, and remote inference once providers report the energy of each call under a declared measurement boundary.

%% file: sections/appendix.tex
\renewcommand{\topfraction}{0.95}
\renewcommand{\dbltopfraction}{0.95}
\renewcommand{\textfraction}{0.05}
\renewcommand{\floatpagefraction}{0.85}
\renewcommand{\dblfloatpagefraction}{0.5}
\setcounter{topnumber}{3}
\setcounter{dbltopnumber}{3}

\section{Measurement Details}\label{appendix:boundary_detail}\label{appendix:layers}

\paragraph{Window.}
Four counter reads split every run into three intervals (Figure~\ref{fig:timeline}).
Only $[\tzero,\tone]$ is charged to the goal; it holds every attempt of the goal, failed attempts included.
Set up before $\tzero$ and clean up after $\tone$ (logging, writing results, shutting down) are recorded and excluded.
The energy of the window is the difference of two counter reads and needs no sampling.
Samples taken in between serve as checks; their coverage of the window is
\begin{equation}
C = \frac{\left|\left(\bigcup_i [s_i,e_i]\right)\cap[\tzero,\tone]\right|}{\tone-\tzero}\times100\,\%,
\label{eq:coverage}
\end{equation}
where $[s_i,e_i]$ is the interval of sample $i$.

\paragraph{Energy charged to a goal.}
Every reported energy is dynamic energy: the counter difference over the window minus idle power times the window length, clamped at zero.
Idle power $P_{\mathrm{baseline}}$ is the mean of several separate idle windows measured before each experiment on the same machine, with the sampler pinned to fixed cores and samples more than two standard deviations from the mean dropped; it is \val{facts.x86.idle_W}\,W at the x86 package and \val{facts.gb10.idle_W}\,W at the GB10 package.
The machines differ in which energy counts as the goal's:
\begin{align*}
\text{x86:}\quad & E = \frac{\Delta\mathrm{ticks}_{\mathrm{pid}}}{\Delta\mathrm{ticks}_{\mathrm{total}}}\,\bigl(E_{\mathrm{pkg}} - P_{\mathrm{baseline}}\,\Delta t\bigr),\\
\text{GB10:}\quad & E = E_{\mathrm{pkg}} - P_{\mathrm{baseline}}\,\Delta t,\\
\text{Mac:}\quad & E = \bigl(E_{\mathrm{cpu}} - P_{\mathrm{cpu}}\,\Delta t\bigr) + E_{\mathrm{gpu}}.
\end{align*}
On x86 the model runs inside the agent process, so the goal gets an estimate: the agent process share of CPU ticks over the window (from \texttt{/proc/\{pid\}/stat} and \texttt{/proc/stat}); the pipeline enforces $E \leq E_{\mathrm{pkg}} - P_{\mathrm{baseline}}\Delta t$.
The same rule applies to the remote runs, measured on the x86 client, where the window includes the wait for the provider; this is client energy only, not the energy of the provider's servers.
On GB10 the model runs in a separate server process and on the GPU, both inside the package, so the whole package counts; the agent process alone holds only a small share of it (Table~\ref{tab:sensitivity}).
On the Mac the CPU and GPU are separate domains and both count; CPU idle power comes from the idle baseline linked to each run (\val{facts.valid.mac_tinyllama.idle_W}\,W on average), while GPU idle power, about 0.07\,W, is not subtracted, so each goal carries at most that power times its window extra.

\paragraph{Remote runs and the provider's energy.}
Let $a$ and $l$ be the client energy per successful goal of the agentic and linear workflows, so client \OOI{} is $a/l$, and let $p_A$ and $p_L$ be the provider's energy per successful goal, which is not reported.
The full \OOI{} is $(a + p_A)/(l + p_L)$, which lies between $a/l$ and $p_A/p_L$; it is therefore at least the client \OOI{} exactly when $p_A/p_L \geq a/l$.
Taking provider energy as proportional to the tokens it processes, which is conservative here since the energy per token grows with context length~\cite{simple2026} and the agent's requests carry longer contexts, the agent sends \val{facts.remote.tok_ratio_min} to \val{facts.remote.tok_ratio_max} times the tokens per successful goal of the single call, which exceeds the client \OOI{} in \val{facts.remote.n_bound_holds} of \val{facts.remote.n_types} goal types (Table~\ref{tab:remote}); the client side is therefore a lower bound on \OOI{} for these runs.

\begin{table}[!tp]
\centering
\footnotesize
\caption{\textbf{Remote model: tokens sent to the provider.} Llama 3.3 70B via NVIDIA NIM. For each goal type, the agentic workflow's tokens per successful goal divided by the linear workflow's. If the provider's energy grows with tokens, the full \OOI{} is at least the client \OOI{} whenever this token ratio is at least as large.}
\label{tab:remote}
\IfFileExists{generated/tab_remote.tex}{\input{generated/tab_remote.tex}}{\TBD[: tab\_remote]}
\end{table}

\paragraph{Other processes on the Mac.}
macOS reports energy per domain, not per process, and a share of CPU time would miss the GPU, where the model runs; Mac energy therefore includes any concurrent process.
The harness records the CPU load of other processes for every run, so runs affected by concurrent work can be identified afterwards.
On the Mac, \val{facts.valid.mac_tinyllama.bg_flagged} of \val{facts.valid.mac_tinyllama.runs_all} runs (\val{facts.valid.mac_tinyllama.bg_flagged_pct}\%) had other processes at \val{facts.valid.mac_tinyllama.bg_cap}\% CPU or more, with no error raised during collection; the goals of these runs are excluded, counted per goal in Table~\ref{tab:excl}.
The load was bimodal: moving the cap anywhere between \val{facts.valid.mac_tinyllama.bg_cap}\% and 90\% reclassifies only \val{facts.valid.mac_tinyllama.bg_mid} runs.
The filter was applied after collection, using the background CPU recorded during it; all Mac results in this paper use the filtered runs, and Table~\ref{tab:bgfilter} shows both versions.
x86 needs no such rule, because the process share already removes other load, and GB10 had none (Table~\ref{tab:validity}).

\paragraph{Validity checks.}
Table~\ref{tab:validity} gives, per set, the coverage of the window by check samples, the spread of the idle baselines used, the CPU load of other processes over the runs kept, the runs removed for high background load and the runs flagged for thermal throttling.
Samples is the mean coverage of the window by check samples (Equation~\ref{eq:coverage}); energy itself comes from the counter reads at the window ends, and coverage is lower only for the shortest windows (TinyLlama), where one sampler interval is a large share of the window.
Idle CV is the standard deviation over the mean of the idle baselines used.
Other processes is the CPU load of processes other than the agent and the model, median and 95th percentile over the runs kept; it matters on GB10 and the Mac, where energy is not split by process.
Removed counts runs, not goals, at or above the background load cap (Mac only).
No run was throttled, and other processes stay in the single digits of CPU at the 95th percentile on GB10.

\begin{table*}[!tp]
\centering
\footnotesize
\caption{\textbf{Measurement validity per set.} Columns as described in Appendix~\ref{appendix:boundary_detail}, Validity checks.}
\label{tab:validity}
\IfFileExists{generated/tab_validity.tex}{\input{generated/tab_validity.tex}}{\TBD[: tab\_validity]}
\end{table*}

\begin{table}[!tp]
\centering
\footnotesize
\caption{\textbf{Runs.} Goals per goal type and workflow, after exclusions.}
\label{tab:setup}
\IfFileExists{generated/tab_setup.tex}{\input{generated/tab_setup.tex}}{\TBD[: tab\_setup]}
\end{table}

\section{From Records to Metrics}\label{appendix:method}
The pipeline reads the measurement stores read only and computes every number in four steps.
\begin{enumerate}[leftmargin=*]
\item \emph{Goals.} Each attempt points to the run that holds it; one run can hold several attempts of the same goal (retries inside a run), and no run belongs to more than one goal, which we checked on every store. A goal's energy and time are the sums over its distinct runs, so each run is counted once. An attempt with no recorded run keeps the goal in the counts; the goal's energy is then a lower bound and is flagged; Table~\ref{tab:validity} counts such goals per set, and only BFCL has any.
\item \emph{Outcomes.} The last attempt decides one of four classes: not completed; completed and wrong; completed and correct; completed with correctness unknown, when the checker could not judge the output. Correctness comes from the declared check (Appendix~\ref{appendix:checks}), never from the success flag written by the agent framework. A set with goals of unknown correctness has no \EpSG{}; its energy per goal is still reported (Section~\ref{sec:epg:undef}).
\item \emph{Metrics.} For each set, goal type and workflow: \EpG{}, \EpCG{}, \EpSG{}, \TpG{} and mean power, from total energy, total time and total counts (Equation~\ref{eq:epg}).
\item \emph{Ratios.} For each goal type run in both workflows: the ratios of the three metrics, the success ratio, and the time and power ratios (Equations~\ref{eq:ooisplit} and~\ref{eq:timepower}). Intervals are 95\% percentile bootstrap intervals with 2{,}000 resamples and a fixed seed. For a fixed question, the resampling unit is one execution, resampled independently within each workflow, because the two workflows hold independent observations of the same question rather than matched executions of one goal.
For distinct questions, the resampling unit is one question, and both workflows of a resampled question are kept together, because each question was run once in each workflow. Each resample recomputes the total energy and the goal counts of each workflow before the ratio is formed. Intervals for fixed questions therefore describe variation from run to run, and intervals for distinct questions describe variation across questions.
\end{enumerate}
The scripts stop on a failed assertion if any of the following fails: $N_s \leq N_c \leq N$; $\TpG \times P = \EpG$; energy ratio $\times$ success ratio $=$ \OOI{}; time ratio $\times$ power ratio $=$ energy ratio; $\EpG \leq \EpCG \leq \EpSG$; the energy of the four outcome classes sums to the total energy; every ratio row has its two metric rows; no excluded goal type enters a table.
Every query behind a number, with its result, is recorded in the analysis notebook of the artifact.

\section{Goal Types and Checks}\label{appendix:checks}
Every check is defined in the rules file of the artifact, and a goal is \emph{successful} when it is completed and its output passes the check; in tables we write \emph{correct} for the same thing.
An empty output is completed and wrong.

\paragraph{Fixed questions.}
Each fixed question is submitted repeatedly with the same prompt.
Numeric goals pass when the final number of the output equals the answer.
Factual and science goals pass when the output contains the answer or a declared equivalent wording, such as ``flightless'' for ``cannot fly'', and logic and commonsense goals pass when the output contains any phrase of a list declared for the goal.
An output that lists several lettered options without choosing one fails, and for each such goal a lenient result without this rule is also recorded.
The logic question states a false premise (all birds can fly) and asks whether penguins can fly, and the check accepts the answer that rejects the premise.
Of the further goal types, four have a declared answer (geography, sentiment, topic and a reasoning question with two numeric answers) and are checked the same way; the other thirteen produce open text, such as summaries, translations and plans, and are checked against the requirements stated in their prompts (Section~\ref{sec:checker}); the two that require a database, literature synthesis and methodology comparison, cannot be reached by a single call and have no \OOI{}.

\paragraph{Benchmark questions.}
GSM8K questions pass when the final number equals the dataset answer.
MMLU questions pass when the chosen option, taken from an explicit answer phrase or else from the last option letter of the output, equals the dataset answer.
HumanEval questions pass when the last code block of the output passes the dataset's unit tests, run in an isolated process with a time limit of 10\,s.

\paragraph{Tool and BFCL goals.}
Tool goals pass when every required tool ran and reported success, which checks that the actions were carried out, not that the resulting state is right.
BFCL goals use the benchmark's state checker, which compares the final state with the ground truth; a goal for which it found no final state has unknown correctness, and we checked six Granite goals by hand: each was marked solved by the framework, but its function calls were written as text or its actions only claimed, so the final state did not match the ground truth.
The calculator goal asks for an overhead fraction, and the check accepts it as a fraction or a percentage, with or without the percent sign, rounded to two decimals, for both workflows alike.
An earlier, stricter version of this check rejected two correct Qwen3 agentic answers written as a percentage without the sign; no linear answer changed.

\section{Full Results}\label{appendix:ci}\label{appendix:full}

\paragraph{Reading the tables.}
Table~\ref{tab:setup} gives the goals per goal type and workflow behind every number.
Tables~\ref{tab:ladder_x86}, \ref{tab:ladder_gb10} and~\ref{tab:ladder_new} give the metrics of each workflow on its own: along a row, \EpG{}, \EpCG{} and \EpSG{} grow apart as goals fail or end wrong, and an undefined \EpSG{} marks a workflow with no correct goal.
Table~\ref{tab:ratios} sets the two workflows side by side; its \EpSG{} ratio is \OOI{} and equals the \EpG{} ratio times the success ratio, and rows without a ratio state why (no correct goal, or a goal a single call cannot reach): in all \val{facts.ooi_n} eligible comparisons the interval of the \EpSG{} ratio lies above 1, and the success ratio column shows where correctness moved the comparison (below 1, the agent was correct more often).
Figure~\ref{fig:coverage} extends the comparison to every cell through energy per submitted goal, which needs no correct goal: the agent costs more in all \val{facts.epg_cells_above1} reasoning comparisons, and less only on tool goals, where one tool call replaces long text generation.
Table~\ref{tab:split} shows where the energy of each workflow went, by the outcome of its goals, with a failed attempt of a goal that later succeeds counted with that goal, and Table~\ref{tab:failures} gives the forced failure runs behind Figure~\ref{fig:claim1}: for reasoning goals the incremental energy per added attempt is close to the energy of a goal that succeeds at once, and for tool goals, where a failure is one long timed-out call, it is larger.

\begin{table*}[!tp]
\centering
\footnotesize
\caption{\textbf{Metrics per goal type and workflow, x86 machine (local model and remote client).} $N$ goals submitted, $N_c$ completed, $N_s$ correct; energy per goal over the boundary of Appendix~\ref{appendix:boundary_detail}; undef.: no correct goal. }
\label{tab:ladder_x86}
\IfFileExists{generated/tab_ladder_x86.tex}{\input{generated/tab_ladder_x86.tex}}{\TBD[: tab\_ladder\_x86]}
\end{table*}

\begin{table*}[!tp]
\centering
\scriptsize
\caption{\textbf{Metrics per goal type and workflow, GB10 and Mac.} Columns as in Table~\ref{tab:ladder_x86}.}
\label{tab:ladder_gb10}
\IfFileExists{generated/tab_ladder_gb10.tex}{\input{generated/tab_ladder_gb10.tex}}{\TBD[: tab\_ladder\_gb10]}
\end{table*}

\begin{table*}[!tp]
\centering
\footnotesize
\caption{\textbf{Metrics per goal type and workflow, GB10, distinct questions and further goal types.} Columns as in Table~\ref{tab:ladder_x86}; distinct questions are grouped by benchmark, and goal types without a declared check have no \EpSG{}.}
\label{tab:ladder_new}
\IfFileExists{generated/tab_ladder_new.tex}{\input{generated/tab_ladder_new.tex}}{\TBD[: tab\_ladder\_new]}
\end{table*}

\begin{table*}[!tp]
\centering
\scriptsize
\caption{\textbf{Agentic over linear for every goal type.} Ratios with 95\% bootstrap intervals: \vabove{bold navy}, interval above 1; \vbelow{bold green}, below 1; plain, contains 1. Gray italics: fewer than 10 correct goals in one workflow, point estimate only (Section~\ref{sec:c4}). min.\ succ.: fewer correct goals of the two workflows.}
\label{tab:ratios}
\IfFileExists{generated/tab_ratio_full.tex}{\input{generated/tab_ratio_full.tex}}{\TBD[: tab\_ratio\_full]}
\end{table*}

\begin{figure}[!tp]
\centering
\IfFileExists{figures/fig_coverage.pdf}{\includegraphics[width=\columnwidth]{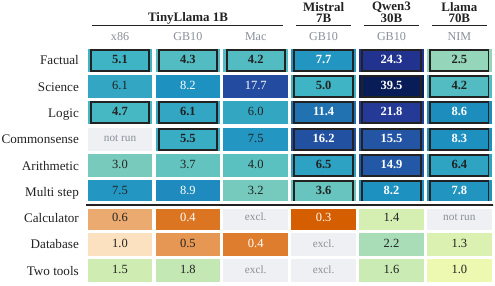}}{\ph{3}{fig\_coverage}{figure in progress}}
\caption{\textbf{On every reasoning goal, the agent costs more energy per goal.} \EpG{} ratio, agentic over linear, log scale: darker cells, larger ratio; orange, below 1. Framed: the cell also has an \OOI{}. On the calculator goal, \OOI{} is below 1 in every column despite ratios up to \val{facts.calc.re_max}, because the linear workflow rarely succeeded (Section~\ref{sec:c4}). Cells marked not run or excl.\ have no data (Appendix~\ref{appendix:config}).}
\label{fig:coverage}
\end{figure}

\begin{table}[!tp]
\centering
\footnotesize
\caption{\textbf{Where the energy went.} ag.: agentic; lin.: linear. Energy of all goals of a set and workflow by the outcome of the goal: \textcolor{fgray}{not completed}, \textcolor{fwrong}{completed but wrong}, \textcolor{fna}{correctness unknown}, \textcolor{fnavy}{\textbf{correct}}; goals a single call cannot reach are left out.}
\label{tab:split}
\IfFileExists{generated/tab_split.tex}{\input{generated/tab_split.tex}}{\TBD[: tab\_split]}
\end{table}

\begin{table}[!tp]
\centering
\footnotesize
\caption{\textbf{Forced failure runs.} Mistral 7B on GB10, agentic workflow, per goal type and number of forced failures per goal (Fail.): attempts (Att.) and model calls per goal, energy per goal and per call, and the incremental energy per added attempt, $(\EpG_2 - \EpG_0) / (\text{Att}_2 - \text{Att}_0)$ (Incr.).}
\label{tab:failures}
\IfFileExists{generated/tab_failures.tex}{\input{generated/tab_failures.tex}}{\TBD[: tab\_failures]}
\end{table}

\section{Exclusions and Sensitivity}\label{appendix:config}

Records are excluded only under documented rules of data integrity or measurement validity, applied independently of the outcome; the Mac background load rule was applied after collection, using telemetry recorded during it, and its effect is shown in Table~\ref{tab:bgfilter}.
Goals that fail for reasons inside the measured system, such as provider API errors or timeouts, are not excluded: they count as not completed.
An outage that stops a whole experiment is a measurement fault: runs of the remote model through a second provider were excluded as a whole, because in six experiments the provider rejected almost every agentic goal at its first attempt, with no recorded cause.
A record whose stored success flag contradicts its last attempt is excluded because the goal and its attempts then disagree on whether the goal completed; the checker decides correctness, not completion, so such a goal cannot be placed in the counts.

\paragraph{Reading the tables.}
Table~\ref{tab:bgfilter} compares the Mac with all runs and with the busy runs removed: the filter changes each ratio (Change column) but reverses none, since every ratio stays above 1 with and without it; it leaves fewer goal types with at least 10 correct goals in both workflows.
Table~\ref{tab:sensitivity} gives, per goal type, the factor by which each measurement choice moves the agentic over linear ratio, below 1 lowering it and above 1 raising it, for every goal type run on GB10, including those without an \OOI{}: counting clean up lowers the GB10 ratio by a factor of up to \val{facts.claim4.gb10.cleanup_x_min}, leaving idle power in changes it by a few percent at most, and across all \val{facts.gb10.process_share_n} goal types run on GB10, the agent process alone holds \val{facts.gb10.process_share_min_pct} to \val{facts.gb10.process_share_max_pct}\% of the package energy.
Table~\ref{tab:excl} lists every exclusion with its rule, set and number of goals; the rules are invalid experiment (marked at run time), measurement fault (declared for a goal type in the configuration), data integrity (stored outcome contradicts the last attempt) and high background load (Mac only, applied after collection): experiments marked invalid at run time, goal types whose tool backend failed on every call, a run stopped by a GPU memory failure, goals whose stored success flag contradicts their last attempt, and Mac goals run under high background load.

\begin{table}[!tp]
\centering
\footnotesize
\caption{\textbf{High background load on the Mac.} Agentic over linear energy per goal ratio with all runs and with the runs at or above the background load cap removed, each with its 95\% bootstrap interval; min.\ succ.: all runs / filtered, muted below 10, where no \OOI{} is estimated.}
\label{tab:bgfilter}
\IfFileExists{generated/tab_bgfilter.tex}{\input{generated/tab_bgfilter.tex}}{\TBD[: tab\_bgfilter]}
\end{table}

\begin{table}[!tp]
\centering
\footnotesize
\caption{\textbf{Effect of each measurement choice.} Factor on \OOI{} when clean up is counted and when idle power is left in (muted: within 2\% of 1); process: agent process share of GB10 package energy.}
\label{tab:sensitivity}
\IfFileExists{generated/tab_sensitivity.tex}{\input{generated/tab_sensitivity.tex}}{\TBD[: tab\_sensitivity]}
\end{table}

\begin{table*}[!tp]
\centering
\footnotesize
\caption{\textbf{Excluded records} by set and rule; counts are goals, not runs.}
\label{tab:excl}
\IfFileExists{generated/tab_excl.tex}{\input{generated/tab_excl.tex}}{\TBD[: tab\_excl]}
\end{table*}

\section{Reproducibility}\label{appendix:repro_detail}
\ALEMS{} stores three fingerprints with every run:
\begin{align}
  \Hhard &= \mathrm{SHA256}(M_{\mathrm{cpu}} \| V_{\mu} \| K \| D_{\mathrm{hw}}), \\
  \Henv  &= \mathrm{SHA256}(P \| O \| G_{\mathrm{commit}} \| G_{\mathrm{dirty}} \| F_{\mathrm{ver}} \| S_{\mathrm{schema}}), \\
  \Hrun  &= \mathrm{SHA256}(G_{\mathrm{gov}} \| T_{\mathrm{turbo}} \| \Hhard \| \Henv \| B_{\mathrm{id}}),
\end{align}
where $M_{\mathrm{cpu}}$ is the CPU model, $V_{\mu}$ the microcode version, $K$ the kernel release, $D_{\mathrm{hw}}$ the energy domains, $P$ the Python version, $O$ the operating system, $G_{\mathrm{commit}}$ the code commit, $G_{\mathrm{dirty}}$ a flag that is set when the code had uncommitted changes, $F_{\mathrm{ver}}$ the framework version, $S_{\mathrm{schema}}$ the store schema version, $G_{\mathrm{gov}}$ the frequency governor, $T_{\mathrm{turbo}}$ the turbo state and $B_{\mathrm{id}}$ the idle baseline used.
When two runs differ, a mismatch in $\Hrun$ alone points to runtime state, in $\Henv$ to software and in $\Hhard$ to hardware.
The two workflows of a comparison run in the same session, so they share $\Henv$; $\Hrun$ also matches unless the idle baseline was measured again between them.
To rerun the analysis, one needs the measurement stores, the configuration and the scripts of the artifact: the configuration records, for each set, its store, selection rule, energy expression, boundary and exclusions, and one command rebuilds every number, table and data figure of the paper from the stores.

\section{Metric Properties}\label{appendix:epg_properties}\label{appendix:gaming}
\EpG{} is defined for every non empty set of goals; \EpCG{} and \EpSG{} are undefined when their denominator is zero and are then reported as undefined with the energy and counts (Section~\ref{sec:epg:undef}).
Because the numerator holds every attempt, a system cannot lower \EpSG{} by hiding failed attempts.
Three ways remain to make the metrics look better without improving the system: choosing easier goals, stopping goals early, and shrinking the measurement boundary.
Reporting the goal set and the completion and success rates beside the metric exposes the first two; declaring the window and hardware boundary, and storing the fingerprints, exposes the third.
Values are comparable only on the same goals, check, machine and boundary, with any change in model or serving configuration, such as precision, declared.

%% file: generated/tab_remote.tex
\setlength{\tabcolsep}{4pt}\begin{tabular}{@{}lrrc@{}}
\toprule
\rowcolor{tNhd}Goal type & Token ratio & Client \OOI{} & Bound holds \\
\midrule
\rowcolor{tNbd}Factual & 4.1 & 2.47 & \vabove{yes} \\
Science & 23.8 & 4.40 & \vabove{yes} \\
\rowcolor{tNbd}Logic & 10.9 & 8.36 & \vabove{yes} \\
Commonsense & 9.2 & 8.31 & \vabove{yes} \\
\rowcolor{tNbd}Arithmetic & 10.6 & 6.52 & \vabove{yes} \\
Multi step & 8.7 & 7.79 & \vabove{yes} \\
\bottomrule
\end{tabular}

%% file: generated/tab_validity.tex
\setlength{\tabcolsep}{4pt}\begin{tabular}{@{}lrrrrrrrr@{}}
\toprule
 & & Samples & Idle & \multicolumn{3}{c}{Other processes, CPU} & & Goals \\
\cmidrule(lr){5-7}
Set & Runs & coverage (\%) & CV (\%) & median (\%) & p95 (\%) & removed & Throttled & lower b. \\
\midrule
\rowcolor{tXbd}x86 TinyLlama & 1,440 & 97.8 & 9.9 & 2.5 & 19.5 & \vna{--} & 0 & \vna{0} \\
\rowcolor{tMbd}GB10 Mistral & 600 & 99.4 & 0.2 & 0.5 & 1.5 & \vna{--} & 0 & \vna{0} \\
\rowcolor{tBbd}GB10 BFCL & 10,155 & 98.1 & 3.5 & 1.0 & 5.0 & \vna{--} & 0 & \textbf{2} \\
\rowcolor{tNbd}x86 client NIM 70B & 608 & 99.4 & 11.5 & 0.0 & 12.5 & \vna{--} & 0 & \vna{0} \\
\rowcolor{tTbd}GB10 Mistral, 1 failure & 261 & 99.5 & 0.0 & 1.5 & 6.0 & \vna{--} & 0 & \vna{0} \\
\rowcolor{tTbd}GB10 Mistral tools & 330 & 99.7 & 0.0 & 1.5 & 5.5 & \vna{--} & 0 & \vna{0} \\
\rowcolor{tTbd}GB10 Mistral, 0 failures & 180 & 99.3 & 0.0 & 1.0 & 5.5 & \vna{--} & 0 & \vna{0} \\
\rowcolor{tTbd}GB10 Mistral, 2 failures & 342 & 99.6 & 0.0 & 0.5 & 1.0 & \vna{--} & 0 & \vna{0} \\
\rowcolor{tGbd}GB10 TinyLlama & 1,080 & 79.4 & 0.0 & 1.5 & 7.0 & \vna{--} & 0 & \vna{0} \\
\rowcolor{tQbd}GB10 Qwen3 & 540 & 97.4 & 0.0 & 0.5 & 4.0 & \vna{--} & 0 & \vna{0} \\
\rowcolor{tCbd}Mac TinyLlama & 647 & 62.5 & \vna{--} & 7.3 & 22.7 & 399 & 0 & \vna{0} \\
\bottomrule
\end{tabular}

%% file: generated/tab_setup.tex
\setlength{\tabcolsep}{4pt}\begin{tabular}{@{}llrlr@{}}
\toprule
Machine & Model & Types & Per type & Goals \\
\midrule
\rowcolor{tXbd}x86 & TinyLlama 1B & 8 & 51 to 175 & 1,347 \\
\rowcolor{tNbd}x86 client & Llama 3.3 70B, NIM & 9 & 30 to 64 & 608 \\
\rowcolor{tGbd}GB10 & TinyLlama 1B & 9 & 60 & 1,080 \\
\rowcolor{tMbd}GB10 & Mistral 7B & 7 & 30 & 420 \\
\rowcolor{tTbd}GB10 & Mistral 7B, failures & 9 & 10 per condition & 540 \\
\rowcolor{tTbd}GB10 & Mistral 7B, tools & 2 & 30 & 120 \\
\rowcolor{tQbd}GB10 & Qwen3 30B & 9 & 30 & 540 \\
\rowcolor{tQbd}GB10 & Qwen3 30B, distinct & 3 & 50 & 300 \\
\rowcolor{tMbd}GB10 & Mistral 7B, GSM8K & 1 & 50 & 100 \\
\rowcolor{tQbd}GB10 & Qwen3 30B, further & 17 & 20 & 680 \\
\rowcolor{tCbd}Mac & TinyLlama 1B & 7 & 27 to 45 & 486 \\
\rowcolor{tBbd}GB10 & BFCL, five models & multi turn & 5,400 goals & 5,400 \\
\bottomrule
\end{tabular}

%% file: generated/tab_ladder_x86.tex
\begin{tabular}{@{}llrrrrrrr@{}}
\toprule
Goal type & Workflow & $N$ & $N_c$ & $N_s$ & \EpG{} (J) & \EpCG{} (J) & \EpSG{} (J) & \TpG{} (s) \\
\midrule
\rowcolor{tXhd}\multicolumn{9}{@{}l}{\textbf{x86, TinyLlama 1B}} \\
\rowcolor{tXbd}Factual & agentic & 149 & 59 & 28 & 1,072.4 & 2,708.2 & 5,706.7 & 42.5 \\
\rowcolor{tXbd} & linear & 150 & 149 & 21 & 212.3 & 213.7 & 1,516.4 & 8.3 \\
Science & agentic & 60 & 60 & 12 & 1,420.3 & 1,420.3 & 7,101.5 & 65.0 \\
 & linear & 60 & 60 & 2 & 233.1 & 233.1 & 6,992.5 & 9.7 \\
\rowcolor{tXbd}Logic & agentic & 51 & 51 & 14 & 1,917.3 & 1,917.3 & 6,984.6 & 70.1 \\
\rowcolor{tXbd} & linear & 51 & 51 & 11 & 409.6 & 409.6 & 1,899.1 & 14.0 \\
Arithmetic & agentic & 174 & 83 & 6 & 1,172.2 & 2,457.4 & 33,994.2 & 43.5 \\
 & linear & 175 & 174 & 21 & 384.3 & 386.5 & 3,202.8 & 14.0 \\
\rowcolor{tXbd}Multi step & agentic & 67 & 67 & 1 & 1,151.4 & 1,151.4 & 77,142.0 & 30.3 \\
\rowcolor{tXbd} & linear & 68 & 67 & 0 & 153.0 & 155.3 & \vna{undef.} & 5.3 \\
Calculator & agentic & 60 & 60 & 45 & 132.3 & 132.3 & 176.4 & 5.1 \\
 & linear & 60 & 60 & 0 & 214.8 & 214.8 & \vna{undef.} & 8.9 \\
\rowcolor{tXbd}Database & agentic & 60 & 60 & 54 & 140.7 & 140.7 & 156.4 & 5.0 \\
\rowcolor{tXbd} & linear & 60 & 60 & 0 & 146.7 & 146.7 & \vna{undef.} & 5.7 \\
Two tools & agentic & 51 & 51 & 51 & 329.2 & 329.2 & 329.2 & 13.2 \\
 & linear & 51 & 51 & 0 & 212.4 & 212.4 & \vna{undef.} & 9.2 \\
\midrule
\rowcolor{tNhd}\multicolumn{9}{@{}l}{\textbf{x86 client, Llama 3.3 70B via NVIDIA NIM}} \\
\rowcolor{tNbd}Factual & agentic & 30 & 30 & 30 & 10.5 & 10.5 & 10.5 & 74.8 \\
\rowcolor{tNbd} & linear & 30 & 30 & 30 & 4.2 & 4.2 & 4.2 & 29.9 \\
Science & agentic & 30 & 30 & 27 & 52.9 & 52.9 & 58.7 & 334.5 \\
 & linear & 30 & 30 & 28 & 12.5 & 12.5 & 13.3 & 82.9 \\
\rowcolor{tNbd}Logic & agentic & 30 & 30 & 30 & 16.8 & 16.8 & 16.8 & 66.6 \\
\rowcolor{tNbd} & linear & 30 & 30 & 29 & 1.9 & 1.9 & 2.0 & 8.7 \\
Commonsense & agentic & 30 & 30 & 30 & 37.3 & 37.3 & 37.3 & 210.6 \\
 & linear & 30 & 30 & 30 & 4.5 & 4.5 & 4.5 & 28.1 \\
\rowcolor{tNbd}Arithmetic & agentic & 64 & 64 & 63 & 13.7 & 13.7 & 13.9 & 81.9 \\
\rowcolor{tNbd} & linear & 64 & 64 & 64 & 2.1 & 2.1 & 2.1 & 12.5 \\
Multi step & agentic & 30 & 30 & 30 & 42.9 & 42.9 & 42.9 & 134.4 \\
 & linear & 30 & 30 & 30 & 5.5 & 5.5 & 5.5 & 16.4 \\
\rowcolor{tNbd}Database & agentic & 30 & 30 & 29 & 20.8 & 20.8 & 21.6 & 125.5 \\
\rowcolor{tNbd} & linear & 30 & 30 & 0 & 15.7 & 15.7 & \vna{undef.} & 99.6 \\
Two tools & agentic & 30 & 30 & 30 & 27.0 & 27.0 & 27.0 & 122.8 \\
 & linear & 30 & 30 & 0 & 26.3 & 26.3 & \vna{undef.} & 118.9 \\
\rowcolor{tNbd}Three tools & agentic & 30 & 30 & 0 & 102.0 & 102.0 & \vna{undef.} & 286.1 \\
\rowcolor{tNbd} & linear & 30 & 30 & 0 & 60.7 & 60.7 & \vna{undef.} & 170.1 \\
\bottomrule
\end{tabular}

%% file: generated/tab_ladder_gb10.tex
\begin{tabular}{@{}llrrrrrrr@{}}
\toprule
Goal type & Workflow & $N$ & $N_c$ & $N_s$ & \EpG{} (J) & \EpCG{} (J) & \EpSG{} (J) & \TpG{} (s) \\
\midrule
\rowcolor{tGhd}\multicolumn{9}{@{}l}{\textbf{GB10, TinyLlama 1B}} \\
\rowcolor{tGbd}Factual & agentic & 60 & 60 & 30 & 541.3 & 541.3 & 1,082.5 & 8.2 \\
\rowcolor{tGbd} & linear & 60 & 60 & 15 & 126.1 & 126.1 & 504.4 & 1.9 \\
Science & agentic & 60 & 60 & 14 & 555.2 & 555.2 & 2,379.6 & 8.6 \\
 & linear & 60 & 60 & 0 & 68.0 & 68.0 & \vna{undef.} & 1.2 \\
\rowcolor{tGbd}Logic & agentic & 60 & 60 & 14 & 625.7 & 625.7 & 2,681.5 & 9.3 \\
\rowcolor{tGbd} & linear & 60 & 60 & 10 & 102.1 & 102.1 & 612.5 & 1.5 \\
Commonsense & agentic & 60 & 60 & 25 & 683.9 & 683.9 & 1,641.2 & 10.1 \\
 & linear & 60 & 60 & 12 & 124.1 & 124.1 & 620.6 & 1.8 \\
\rowcolor{tGbd}Arithmetic & agentic & 60 & 60 & 6 & 622.4 & 622.4 & 6,224.3 & 9.2 \\
\rowcolor{tGbd} & linear & 60 & 60 & 4 & 170.1 & 170.1 & 2,551.2 & 2.6 \\
Multi step & agentic & 60 & 60 & 0 & 595.1 & 595.1 & \vna{undef.} & 8.9 \\
 & linear & 60 & 60 & 2 & 67.0 & 67.0 & 2,009.6 & 0.9 \\
\rowcolor{tGbd}Calculator & agentic & 60 & 60 & 51 & 32.9 & 32.9 & 38.7 & 0.4 \\
\rowcolor{tGbd} & linear & 60 & 60 & 0 & 81.3 & 81.3 & \vna{undef.} & 1.2 \\
Database & agentic & 60 & 60 & 60 & 34.2 & 34.2 & 34.2 & 0.5 \\
 & linear & 60 & 60 & 0 & 64.2 & 64.2 & \vna{undef.} & 0.9 \\
\rowcolor{tGbd}Two tools & agentic & 60 & 60 & 60 & 128.1 & 128.1 & 128.1 & 2.2 \\
\rowcolor{tGbd} & linear & 60 & 60 & 0 & 73.0 & 73.0 & \vna{undef.} & 1.2 \\
\midrule
\rowcolor{tMhd}\multicolumn{9}{@{}l}{\textbf{GB10, Mistral 7B}} \\
\rowcolor{tMbd}Factual & agentic & 30 & 30 & 30 & 2,226.2 & 2,226.2 & 2,226.2 & 40.8 \\
\rowcolor{tMbd} & linear & 30 & 30 & 30 & 290.1 & 290.1 & 290.1 & 5.4 \\
Science & agentic & 30 & 30 & 30 & 843.9 & 843.9 & 843.9 & 16.1 \\
 & linear & 30 & 30 & 30 & 167.6 & 167.6 & 167.6 & 3.2 \\
\rowcolor{tMbd}Logic & agentic & 30 & 30 & 29 & 3,110.7 & 3,110.7 & 3,217.9 & 53.0 \\
\rowcolor{tMbd} & linear & 30 & 30 & 30 & 273.1 & 273.1 & 273.1 & 4.8 \\
Commonsense & agentic & 30 & 30 & 23 & 6,953.3 & 6,953.3 & 9,069.5 & 121.2 \\
 & linear & 30 & 30 & 30 & 430.0 & 430.0 & 430.0 & 7.8 \\
\rowcolor{tMbd}Arithmetic & agentic & 30 & 30 & 26 & 1,481.5 & 1,481.5 & 1,709.4 & 27.0 \\
\rowcolor{tMbd} & linear & 30 & 30 & 29 & 226.5 & 226.5 & 234.3 & 4.1 \\
Multi step & agentic & 30 & 30 & 28 & 2,294.6 & 2,294.6 & 2,458.5 & 39.3 \\
 & linear & 30 & 30 & 30 & 640.3 & 640.3 & 640.3 & 11.3 \\
\rowcolor{tMbd}Calculator & agentic & 30 & 30 & 27 & 399.9 & 399.9 & 444.3 & 5.6 \\
\rowcolor{tMbd} & linear & 30 & 30 & 3 & 1,187.9 & 1,187.9 & 11,879.4 & 17.4 \\
\midrule
\rowcolor{tThd}\multicolumn{9}{@{}l}{\textbf{GB10, Mistral 7B, tool runs}} \\
\rowcolor{tTbd}Database & agentic & 30 & 30 & 30 & 574.4 & 574.4 & 574.4 & 7.7 \\
\rowcolor{tTbd} & linear & 30 & 30 & 0 & 1,052.4 & 1,052.4 & \vna{undef.} & 14.7 \\
Two tools & agentic & 30 & 30 & 30 & 1,224.2 & 1,224.2 & 1,224.2 & 16.3 \\
 & linear & 30 & 30 & 0 & 1,875.9 & 1,875.9 & \vna{undef.} & 25.8 \\
\midrule
\rowcolor{tQhd}\multicolumn{9}{@{}l}{\textbf{GB10, Qwen3 30B}} \\
\rowcolor{tQbd}Factual & agentic & 30 & 30 & 30 & 1,362.8 & 1,362.8 & 1,362.8 & 23.7 \\
\rowcolor{tQbd} & linear & 30 & 30 & 30 & 56.1 & 56.1 & 56.1 & 0.9 \\
Science & agentic & 30 & 30 & 30 & 779.8 & 779.8 & 779.8 & 14.1 \\
 & linear & 30 & 30 & 30 & 19.7 & 19.7 & 19.7 & 0.3 \\
\rowcolor{tQbd}Logic & agentic & 30 & 30 & 30 & 2,142.0 & 2,142.0 & 2,142.0 & 35.1 \\
\rowcolor{tQbd} & linear & 30 & 30 & 30 & 98.1 & 98.1 & 98.1 & 1.6 \\
Commonsense & agentic & 30 & 30 & 30 & 4,965.8 & 4,965.8 & 4,965.8 & 79.1 \\
 & linear & 30 & 30 & 30 & 319.8 & 319.8 & 319.8 & 5.4 \\
\rowcolor{tQbd}Arithmetic & agentic & 30 & 30 & 30 & 724.4 & 724.4 & 724.4 & 13.0 \\
\rowcolor{tQbd} & linear & 30 & 30 & 30 & 48.8 & 48.8 & 48.8 & 0.8 \\
Multi step & agentic & 30 & 30 & 30 & 1,916.1 & 1,916.1 & 1,916.1 & 31.5 \\
 & linear & 30 & 30 & 30 & 232.5 & 232.5 & 232.5 & 3.9 \\
\rowcolor{tQbd}Calculator & agentic & 30 & 30 & 30 & 390.9 & 390.9 & 390.9 & 7.1 \\
\rowcolor{tQbd} & linear & 30 & 30 & 0 & 272.9 & 272.9 & \vna{undef.} & 5.0 \\
Database & agentic & 30 & 30 & 30 & 530.1 & 530.1 & 530.1 & 9.6 \\
 & linear & 30 & 30 & 0 & 245.3 & 245.3 & \vna{undef.} & 4.5 \\
\rowcolor{tQbd}Two tools & agentic & 30 & 30 & 30 & 923.4 & 923.4 & 923.4 & 16.1 \\
\rowcolor{tQbd} & linear & 30 & 30 & 0 & 571.4 & 571.4 & \vna{undef.} & 10.2 \\
\midrule
\rowcolor{tChd}\multicolumn{9}{@{}l}{\textbf{Mac, TinyLlama 1B}} \\
\rowcolor{tCbd}Factual & agentic & 27 & 27 & 13 & 720.1 & 720.1 & 1,495.6 & 181.3 \\
\rowcolor{tCbd} & linear & 30 & 30 & 12 & 169.9 & 169.9 & 424.8 & 48.5 \\
Science & agentic & 29 & 29 & 5 & 1,022.5 & 1,022.5 & 5,930.5 & 238.8 \\
 & linear & 29 & 29 & 1 & 57.7 & 57.7 & 1,674.4 & 13.7 \\
\rowcolor{tCbd}Logic & agentic & 32 & 32 & 9 & 520.2 & 520.2 & 1,849.7 & 128.0 \\
\rowcolor{tCbd} & linear & 34 & 34 & 9 & 86.7 & 86.7 & 327.7 & 22.0 \\
Commonsense & agentic & 31 & 31 & 14 & 653.3 & 653.3 & 1,446.5 & 149.4 \\
 & linear & 32 & 32 & 3 & 86.8 & 86.8 & 925.7 & 29.9 \\
\rowcolor{tCbd}Arithmetic & agentic & 42 & 42 & 2 & 247.4 & 247.4 & 5,195.9 & 32.1 \\
\rowcolor{tCbd} & linear & 45 & 45 & 1 & 62.6 & 62.6 & 2,816.7 & 8.6 \\
Multi step & agentic & 42 & 42 & 0 & 123.8 & 123.8 & \vna{undef.} & 27.6 \\
 & linear & 44 & 44 & 0 & 39.0 & 39.0 & \vna{undef.} & 7.9 \\
\rowcolor{tCbd}Database & agentic & 33 & 33 & 33 & 58.0 & 58.0 & 58.0 & 21.1 \\
\rowcolor{tCbd} & linear & 36 & 36 & 0 & 134.1 & 134.1 & \vna{undef.} & 46.7 \\
\bottomrule
\end{tabular}

%% file: generated/tab_ladder_new.tex
\begin{tabular}{@{}llrrrrrrr@{}}
\toprule
Goal type & Workflow & $N$ & $N_c$ & $N_s$ & \EpG{} (J) & \EpCG{} (J) & \EpSG{} (J) & \TpG{} (s) \\
\midrule
\rowcolor{tQhd}\multicolumn{9}{@{}l}{\textbf{GB10, Qwen3 30B, distinct questions}} \\
\rowcolor{tQbd}mmlu & agentic & 50 & 50 & 47 & 3,657.6 & 3,657.6 & 3,891.1 & 61.0 \\
\rowcolor{tQbd} & linear & 50 & 50 & 42 & 209.2 & 209.2 & 249.1 & 3.7 \\
humaneval & agentic & 50 & 40 & 36 & 5,093.0 & 6,366.2 & 7,073.6 & 82.8 \\
 & linear & 50 & 50 & 48 & 250.5 & 250.5 & 260.9 & 4.3 \\
\rowcolor{tQbd}gsm8k & agentic & 50 & 48 & 45 & 874.3 & 910.7 & 971.4 & 14.9 \\
\rowcolor{tQbd} & linear & 50 & 50 & 48 & 194.4 & 194.4 & 202.5 & 3.5 \\
\midrule
\rowcolor{tMhd}\multicolumn{9}{@{}l}{\textbf{GB10, Mistral 7B, distinct GSM8K questions}} \\
\rowcolor{tMbd}gsm8k & agentic & 50 & 33 & 10 & 15,053.7 & 22,808.7 & 75,268.7 & 220.8 \\
\rowcolor{tMbd} & linear & 50 & 50 & 21 & 770.2 & 770.2 & 1,833.8 & 11.9 \\
\midrule
\rowcolor{tQhd}\multicolumn{9}{@{}l}{\textbf{GB10, Qwen3 30B, further goal types}} \\
\rowcolor{tQbd}Hard reasoning & agentic & 20 & 20 & 20 & 497.7 & 497.7 & 497.7 & 10.8 \\
\rowcolor{tQbd} & linear & 20 & 20 & 20 & 270.4 & 270.4 & 270.4 & 5.9 \\
Planning & agentic & 20 & 20 & 20 & 9,007.5 & 9,007.5 & 9,007.5 & 169.1 \\
 & linear & 20 & 20 & 20 & 1,220.3 & 1,220.3 & 1,220.3 & 24.7 \\
\rowcolor{tQbd}Creative writing & agentic & 20 & 20 & 4 & 3,505.6 & 3,505.6 & 17,528.2 & 69.8 \\
\rowcolor{tQbd} & linear & 20 & 20 & 8 & 122.4 & 122.4 & 306.0 & 2.5 \\
Summarization & agentic & 20 & 20 & 20 & 956.1 & 956.1 & 956.1 & 20.9 \\
 & linear & 20 & 20 & 20 & 76.5 & 76.5 & 76.5 & 1.6 \\
\rowcolor{tQbd}Translation (Hindi) & agentic & 20 & 20 & 20 & 269.4 & 269.4 & 269.4 & 6.0 \\
\rowcolor{tQbd} & linear & 20 & 20 & 20 & 81.8 & 81.8 & 81.8 & 1.8 \\
English to Hindi & agentic & 20 & 20 & 19 & 274.4 & 274.4 & 288.8 & 5.8 \\
 & linear & 20 & 20 & 20 & 79.5 & 79.5 & 79.5 & 1.6 \\
\rowcolor{tQbd}Keyword extraction & agentic & 20 & 20 & 20 & 871.3 & 871.3 & 871.3 & 19.0 \\
\rowcolor{tQbd} & linear & 20 & 20 & 20 & 24.7 & 24.7 & 24.7 & 0.5 \\
Literature synthesis & agentic & 20 & 0 & 0 & 14,961.4 & -- & \vna{undef.} & 288.5 \\
 & linear & 20 & 20 & 0 & 808.6 & 808.6 & \vna{undef.} & 16.6 \\
\rowcolor{tQbd}Geography & agentic & 20 & 20 & 20 & 840.4 & 840.4 & 840.4 & 16.4 \\
\rowcolor{tQbd} & linear & 20 & 20 & 20 & 16.4 & 16.4 & 16.4 & 0.2 \\
Methodology comparison & agentic & 20 & 0 & 0 & 31,936.9 & -- & \vna{undef.} & 615.8 \\
 & linear & 20 & 20 & 0 & 924.8 & 924.8 & \vna{undef.} & 18.8 \\
\rowcolor{tQbd}Multilingual summary & agentic & 20 & 20 & 20 & 964.4 & 964.4 & 964.4 & 19.9 \\
\rowcolor{tQbd} & linear & 20 & 20 & 20 & 300.0 & 300.0 & 300.0 & 6.4 \\
Technical blog post & agentic & 20 & 20 & 15 & 2,857.5 & 2,857.5 & 3,810.1 & 54.9 \\
 & linear & 20 & 20 & 10 & 205.6 & 205.6 & 411.2 & 4.1 \\
\rowcolor{tQbd}News summary & agentic & 20 & 20 & 17 & 383.6 & 383.6 & 451.3 & 8.4 \\
\rowcolor{tQbd} & linear & 20 & 20 & 20 & 41.2 & 41.2 & 41.2 & 0.8 \\
Research summary & agentic & 20 & 20 & 20 & 777.7 & 777.7 & 777.7 & 16.9 \\
 & linear & 20 & 20 & 20 & 29.2 & 29.2 & 29.2 & 0.6 \\
\rowcolor{tQbd}Sentiment & agentic & 20 & 20 & 20 & 222.6 & 222.6 & 222.6 & 5.0 \\
\rowcolor{tQbd} & linear & 20 & 20 & 20 & 19.7 & 19.7 & 19.7 & 0.4 \\
Topic classification & agentic & 20 & 20 & 20 & 449.8 & 449.8 & 449.8 & 10.1 \\
 & linear & 20 & 20 & 20 & 11.9 & 11.9 & 11.9 & 0.2 \\
\rowcolor{tQbd}Research abstract & agentic & 20 & 20 & 5 & 1,131.3 & 1,131.3 & 4,525.1 & 20.1 \\
\rowcolor{tQbd} & linear & 20 & 20 & 2 & 174.5 & 174.5 & 1,744.9 & 3.1 \\
\bottomrule
\end{tabular}

%% file: generated/tab_ratio_full.tex
\setlength{\tabcolsep}{2.4pt}\begin{tabular}{@{}llllllll@{}}
\toprule
Goal type & \EpG{} & \EpCG{} & succ. & \EpSG{} (\OOI{}) & time & power & min. \\
\midrule
\rowcolor{tXhd}\multicolumn{8}{@{}l}{\textbf{x86, TinyLlama 1B}} \\
\rowcolor{tXbd}Factual & \vabove{5.05 {\tiny[3.94, 6.53]}} & \vabove{12.67 {\tiny[9.91, 16.83]}} & 0.74 & \vabove{3.76 {\tiny[2.20, 6.35]}} & \vabove{5.13 {\tiny[4.02, 6.60]}} & 0.99 {\tiny[0.95, 1.03]} & 21 \\
Science & \vabove{6.09 {\tiny[4.25, 9.38]}} & \vabove{6.09 {\tiny[4.25, 9.38]}} & \vna{0.17} & \vna{1.02} & \vabove{6.69 {\tiny[4.62, 10.49]}} & \vbelow{0.91 {\tiny[0.83, 0.98]}} & 2 \\
\rowcolor{tXbd}Logic & \vabove{4.68 {\tiny[3.61, 6.55]}} & \vabove{4.68 {\tiny[3.61, 6.55]}} & 0.79 & \vabove{3.68 {\tiny[1.69, 8.19]}} & \vabove{5.01 {\tiny[3.87, 7.02]}} & \vbelow{0.93 {\tiny[0.91, 0.96]}} & 11 \\
Arithmetic & \vabove{3.05 {\tiny[2.61, 3.60]}} & \vabove{6.36 {\tiny[5.41, 7.64]}} & \vna{3.48} & \vna{10.61} & \vabove{3.10 {\tiny[2.68, 3.62]}} & 0.99 {\tiny[0.93, 1.04]} & 6 \\
\rowcolor{tXbd}Multi step & \vabove{7.52 {\tiny[4.92, 13.42]}} & \vabove{7.41 {\tiny[4.83, 13.26]}} & \vna{--} & \vna{0 succ.: 1 vs 0 correct} & \vabove{5.75 {\tiny[3.80, 9.92]}} & \vabove{1.31 {\tiny[1.10, 1.59]}} & 0 \\
Calculator & \vbelow{0.62 {\tiny[0.35, 0.97]}} & \vbelow{0.62 {\tiny[0.35, 0.97]}} & \vna{--} & \vna{0 succ.: 45 vs 0 correct} & \vbelow{0.57 {\tiny[0.33, 0.90]}} & \vabove{1.08 {\tiny[1.00, 1.16]}} & 0 \\
\rowcolor{tXbd}Database & \multicolumn{7}{l}{\vna{not eligible}} \\
Two tools & \multicolumn{7}{l}{\vna{not eligible}} \\
\midrule
\rowcolor{tNhd}\multicolumn{8}{@{}l}{\textbf{x86 client, Llama 3.3 70B via NVIDIA NIM}} \\
\rowcolor{tNbd}Factual & \vabove{2.47 {\tiny[1.36, 5.65]}} & \vabove{2.47 {\tiny[1.36, 5.65]}} & 1.00 & \vabove{2.47 {\tiny[1.36, 5.65]}} & \vabove{2.50 {\tiny[1.42, 5.39]}} & 0.98 {\tiny[0.90, 1.11]} & 30 \\
Science & \vabove{4.24 {\tiny[3.25, 5.70]}} & \vabove{4.24 {\tiny[3.25, 5.70]}} & 1.04 & \vabove{4.40 {\tiny[3.03, 6.52]}} & \vabove{4.03 {\tiny[3.07, 5.46]}} & 1.05 {\tiny[0.95, 1.18]} & 27 \\
\rowcolor{tNbd}Logic & \vabove{8.64 {\tiny[4.92, 16.47]}} & \vabove{8.64 {\tiny[4.92, 16.47]}} & 0.97 & \vabove{8.36 {\tiny[4.80, 15.95]}} & \vabove{7.68 {\tiny[4.99, 12.20]}} & 1.13 {\tiny[0.88, 1.50]} & 29 \\
Commonsense & \vabove{8.31 {\tiny[6.47, 10.98]}} & \vabove{8.31 {\tiny[6.47, 10.98]}} & 1.00 & \vabove{8.31 {\tiny[6.47, 10.98]}} & \vabove{7.50 {\tiny[5.84, 9.87]}} & \vabove{1.11 {\tiny[1.01, 1.21]}} & 30 \\
\rowcolor{tNbd}Arithmetic & \vabove{6.42 {\tiny[2.94, 17.92]}} & \vabove{6.42 {\tiny[2.94, 17.92]}} & 1.02 & \vabove{6.52 {\tiny[2.97, 18.21]}} & \vabove{6.53 {\tiny[4.23, 10.60]}} & 0.98 {\tiny[0.58, 2.13]} & 63 \\
Multi step & \vabove{7.79 {\tiny[3.92, 16.94]}} & \vabove{7.79 {\tiny[3.92, 16.94]}} & 1.00 & \vabove{7.79 {\tiny[3.92, 16.94]}} & \vabove{8.19 {\tiny[4.86, 14.48]}} & 0.95 {\tiny[0.75, 1.37]} & 30 \\
\rowcolor{tNbd}Database & \multicolumn{7}{l}{\vna{not eligible}} \\
Two tools & \multicolumn{7}{l}{\vna{not eligible}} \\
\rowcolor{tNbd}Three tools & \multicolumn{7}{l}{\vna{not eligible}} \\
\midrule
\rowcolor{tGhd}\multicolumn{8}{@{}l}{\textbf{GB10, TinyLlama 1B}} \\
\rowcolor{tGbd}Factual & \vabove{4.29 {\tiny[3.34, 6.07]}} & \vabove{4.29 {\tiny[3.34, 6.07]}} & 0.50 & \vabove{2.15 {\tiny[1.16, 3.90]}} & \vabove{4.26 {\tiny[3.28, 6.14]}} & 1.01 {\tiny[0.98, 1.02]} & 15 \\
Science & \vabove{8.16 {\tiny[5.50, 14.10]}} & \vabove{8.16 {\tiny[5.50, 14.10]}} & \vna{--} & \vna{0 succ.: 14 vs 0 correct} & \vabove{7.46 {\tiny[4.93, 13.40]}} & \vabove{1.09 {\tiny[1.04, 1.12]}} & 0 \\
\rowcolor{tGbd}Logic & \vabove{6.13 {\tiny[4.53, 9.03]}} & \vabove{6.13 {\tiny[4.53, 9.03]}} & 0.71 & \vabove{4.38 {\tiny[1.76, 10.20]}} & \vabove{6.23 {\tiny[4.55, 9.41]}} & 0.98 {\tiny[0.95, 1.01]} & 10 \\
Commonsense & \vabove{5.51 {\tiny[4.27, 7.60]}} & \vabove{5.51 {\tiny[4.27, 7.60]}} & 0.48 & \vabove{2.64 {\tiny[1.25, 4.98]}} & \vabove{5.56 {\tiny[4.24, 7.89]}} & 0.99 {\tiny[0.96, 1.01]} & 12 \\
\rowcolor{tGbd}Arithmetic & \vabove{3.66 {\tiny[2.96, 4.69]}} & \vabove{3.66 {\tiny[2.96, 4.69]}} & \vna{0.67} & \vna{2.44} & \vabove{3.60 {\tiny[2.87, 4.68]}} & 1.02 {\tiny[0.99, 1.04]} & 4 \\
Multi step & \vabove{8.88 {\tiny[6.36, 13.89]}} & \vabove{8.88 {\tiny[6.36, 13.89]}} & \vna{--} & \vna{0 succ.: 0 vs 2 correct} & \vabove{9.64 {\tiny[6.70, 15.83]}} & \vbelow{0.92 {\tiny[0.87, 0.95]}} & 0 \\
\rowcolor{tGbd}Calculator & \vbelow{0.40 {\tiny[0.23, 0.69]}} & \vbelow{0.40 {\tiny[0.23, 0.69]}} & \vna{--} & \vna{0 succ.: 51 vs 0 correct} & \vbelow{0.37 {\tiny[0.19, 0.65]}} & \vabove{1.11 {\tiny[1.03, 1.26]}} & 0 \\
Database & \multicolumn{7}{l}{\vna{not eligible}} \\
\rowcolor{tGbd}Two tools & \multicolumn{7}{l}{\vna{not eligible}} \\
\midrule
\rowcolor{tMhd}\multicolumn{8}{@{}l}{\textbf{GB10, Mistral 7B}} \\
\rowcolor{tMbd}Factual & \vabove{7.67 {\tiny[7.00, 8.43]}} & \vabove{7.67 {\tiny[7.00, 8.43]}} & 1.00 & \vabove{7.67 {\tiny[7.00, 8.43]}} & \vabove{7.54 {\tiny[6.91, 8.25]}} & 1.02 {\tiny[0.99, 1.04]} & 30 \\
Science & \vabove{5.03 {\tiny[4.53, 5.63]}} & \vabove{5.03 {\tiny[4.53, 5.63]}} & 1.00 & \vabove{5.03 {\tiny[4.53, 5.63]}} & \vabove{5.01 {\tiny[4.51, 5.62]}} & 1.00 {\tiny[1.00, 1.01]} & 30 \\
\rowcolor{tMbd}Logic & \vabove{11.39 {\tiny[10.57, 12.31]}} & \vabove{11.39 {\tiny[10.57, 12.31]}} & 1.03 & \vabove{11.78 {\tiny[10.73, 13.09]}} & \vabove{11.05 {\tiny[10.25, 11.98]}} & \vabove{1.03 {\tiny[1.02, 1.04]}} & 29 \\
Commonsense & \vabove{16.17 {\tiny[13.90, 18.42]}} & \vabove{16.17 {\tiny[13.90, 18.42]}} & 1.30 & \vabove{21.09 {\tiny[16.81, 27.87]}} & \vabove{15.47 {\tiny[13.35, 17.53]}} & \vabove{1.05 {\tiny[1.04, 1.06]}} & 23 \\
\rowcolor{tMbd}Arithmetic & \vabove{6.54 {\tiny[5.81, 7.40]}} & \vabove{6.54 {\tiny[5.81, 7.40]}} & 1.12 & \vabove{7.30 {\tiny[6.07, 8.95]}} & \vabove{6.56 {\tiny[5.81, 7.47]}} & 1.00 {\tiny[0.97, 1.02]} & 26 \\
Multi step & \vabove{3.58 {\tiny[3.35, 3.84]}} & \vabove{3.58 {\tiny[3.35, 3.84]}} & 1.07 & \vabove{3.84 {\tiny[3.44, 4.40]}} & \vabove{3.47 {\tiny[3.25, 3.70]}} & \vabove{1.03 {\tiny[1.02, 1.04]}} & 28 \\
\rowcolor{tMbd}Calculator & \vbelow{0.34 {\tiny[0.27, 0.41]}} & \vbelow{0.34 {\tiny[0.27, 0.41]}} & \vna{0.11} & \vna{0.037} & \vbelow{0.32 {\tiny[0.26, 0.40]}} & \vabove{1.04 {\tiny[1.03, 1.05]}} & 3 \\
\midrule
\rowcolor{tThd}\multicolumn{8}{@{}l}{\textbf{GB10, Mistral 7B, tool runs}} \\
\rowcolor{tTbd}Database & \multicolumn{7}{l}{\vna{not eligible}} \\
Two tools & \multicolumn{7}{l}{\vna{not eligible}} \\
\midrule
\rowcolor{tQhd}\multicolumn{8}{@{}l}{\textbf{GB10, Qwen3 30B}} \\
\rowcolor{tQbd}Factual & \vabove{24.31 {\tiny[23.34, 25.32]}} & \vabove{24.31 {\tiny[23.34, 25.32]}} & 1.00 & \vabove{24.31 {\tiny[23.34, 25.32]}} & \vabove{25.44 {\tiny[24.47, 26.44]}} & \vbelow{0.96 {\tiny[0.95, 0.96]}} & 30 \\
Science & \vabove{39.50 {\tiny[35.98, 43.95]}} & \vabove{39.50 {\tiny[35.98, 43.95]}} & 1.00 & \vabove{39.50 {\tiny[35.98, 43.95]}} & \vabove{46.27 {\tiny[41.71, 51.73]}} & \vbelow{0.85 {\tiny[0.84, 0.87]}} & 30 \\
\rowcolor{tQbd}Logic & \vabove{21.84 {\tiny[20.74, 23.01]}} & \vabove{21.84 {\tiny[20.74, 23.01]}} & 1.00 & \vabove{21.84 {\tiny[20.74, 23.01]}} & \vabove{22.04 {\tiny[20.93, 23.21]}} & 0.99 {\tiny[0.98, 1.00]} & 30 \\
Commonsense & \vabove{15.53 {\tiny[15.03, 16.08]}} & \vabove{15.53 {\tiny[15.03, 16.08]}} & 1.00 & \vabove{15.53 {\tiny[15.03, 16.08]}} & \vabove{14.74 {\tiny[14.27, 15.26]}} & \vabove{1.05 {\tiny[1.05, 1.06]}} & 30 \\
\rowcolor{tQbd}Arithmetic & \vabove{14.86 {\tiny[13.89, 15.83]}} & \vabove{14.86 {\tiny[13.89, 15.83]}} & 1.00 & \vabove{14.86 {\tiny[13.89, 15.83]}} & \vabove{15.61 {\tiny[14.62, 16.62]}} & \vbelow{0.95 {\tiny[0.94, 0.96]}} & 30 \\
Multi step & \vabove{8.24 {\tiny[7.79, 8.77]}} & \vabove{8.24 {\tiny[7.79, 8.77]}} & 1.00 & \vabove{8.24 {\tiny[7.79, 8.77]}} & \vabove{8.13 {\tiny[7.72, 8.61]}} & 1.01 {\tiny[0.99, 1.03]} & 30 \\
\rowcolor{tQbd}Calculator & \vabove{1.43 {\tiny[1.34, 1.54]}} & \vabove{1.43 {\tiny[1.34, 1.54]}} & \vna{--} & \vna{0 succ.: 30 vs 0 correct} & \vabove{1.41 {\tiny[1.32, 1.51]}} & \vabove{1.02 {\tiny[1.01, 1.03]}} & 0 \\
Database & \multicolumn{7}{l}{\vna{not eligible}} \\
\rowcolor{tQbd}Two tools & \multicolumn{7}{l}{\vna{not eligible}} \\
\midrule
\rowcolor{tChd}\multicolumn{8}{@{}l}{\textbf{Mac, TinyLlama 1B}} \\
\rowcolor{tCbd}Factual & \vabove{4.24 {\tiny[2.94, 7.40]}} & \vabove{4.24 {\tiny[2.94, 7.40]}} & 0.83 & \vabove{3.52 {\tiny[1.77, 7.64]}} & \vabove{3.74 {\tiny[2.50, 6.43]}} & 1.13 {\tiny[0.86, 1.48]} & 12 \\
Science & \vabove{17.71 {\tiny[8.02, 247.03]}} & \vabove{17.71 {\tiny[8.02, 247.03]}} & \vna{0.20} & \vna{3.54} & \vabove{17.37 {\tiny[7.67, 92.51]}} & 1.02 {\tiny[0.51, 2.77]} & 1 \\
\rowcolor{tCbd}Logic & \vabove{6.00 {\tiny[3.95, 10.94]}} & \vabove{6.00 {\tiny[3.95, 10.94]}} & \vna{0.94} & \vna{5.64} & \vabove{5.82 {\tiny[3.98, 10.42]}} & 1.03 {\tiny[0.76, 1.36]} & 9 \\
Commonsense & \vabove{7.53 {\tiny[4.76, 13.89]}} & \vabove{7.53 {\tiny[4.76, 13.89]}} & \vna{0.21} & \vna{1.56} & \vabove{5.00 {\tiny[3.25, 9.19]}} & \vabove{1.50 {\tiny[1.25, 1.85]}} & 3 \\
\rowcolor{tCbd}Arithmetic & \vabove{3.95 {\tiny[2.93, 5.43]}} & \vabove{3.95 {\tiny[2.93, 5.43]}} & \vna{0.47} & \vna{1.84} & \vabove{3.74 {\tiny[2.98, 4.87]}} & 1.06 {\tiny[0.86, 1.25]} & 1 \\
Multi step & \vabove{3.17 {\tiny[1.85, 6.37]}} & \vabove{3.17 {\tiny[1.85, 6.37]}} & \vna{--} & \vna{0 succ.: 0 vs 0 correct} & \vabove{3.48 {\tiny[2.21, 6.39]}} & 0.91 {\tiny[0.59, 1.35]} & 0 \\
\rowcolor{tCbd}Database & \multicolumn{7}{l}{\vna{not eligible}} \\
\midrule
\rowcolor{tQhd}\multicolumn{8}{@{}l}{\textbf{GB10, Qwen3 30B, distinct questions}} \\
\rowcolor{tQbd}gsm8k & \vabove{4.50 {\tiny[2.76, 7.14]}} & \vabove{4.68 {\tiny[2.76, 7.84]}} & 1.07 & \vabove{4.80 {\tiny[2.83, 8.14]}} & \vabove{4.26 {\tiny[2.70, 6.61]}} & \vabove{1.06 {\tiny[1.02, 1.08]}} & 45 \\
humaneval & \vabove{20.33 {\tiny[13.86, 27.90]}} & \vabove{25.42 {\tiny[15.58, 38.12]}} & 1.33 & \vabove{27.11 {\tiny[16.35, 42.07]}} & \vabove{19.04 {\tiny[13.18, 25.91]}} & \vabove{1.07 {\tiny[1.04, 1.08]}} & 36 \\
\rowcolor{tQbd}mmlu & \vabove{17.48 {\tiny[14.50, 21.96]}} & \vabove{17.48 {\tiny[14.50, 21.96]}} & 0.89 & \vabove{15.62 {\tiny[12.37, 20.36]}} & \vabove{16.62 {\tiny[13.63, 21.14]}} & \vabove{1.05 {\tiny[1.03, 1.07]}} & 42 \\
\midrule
\rowcolor{tMhd}\multicolumn{8}{@{}l}{\textbf{GB10, Mistral 7B, distinct GSM8K questions}} \\
\rowcolor{tMbd}gsm8k & \vabove{19.55 {\tiny[14.91, 24.87]}} & \vabove{29.61 {\tiny[19.44, 44.71]}} & 2.10 & \vabove{41.05 {\tiny[20.89, 90.80]}} & \vabove{18.55 {\tiny[14.30, 23.35]}} & \vabove{1.05 {\tiny[1.02, 1.09]}} & 10 \\
\midrule
\rowcolor{tQhd}\multicolumn{8}{@{}l}{\textbf{GB10, Qwen3 30B, further goal types}} \\
\rowcolor{tQbd}Hard reasoning & \vabove{1.84 {\tiny[1.78, 1.92]}} & \vabove{1.84 {\tiny[1.78, 1.92]}} & 1.00 & \vabove{1.84 {\tiny[1.78, 1.92]}} & \vabove{1.81 {\tiny[1.75, 1.89]}} & \vabove{1.02 {\tiny[1.01, 1.02]}} & 20 \\
Planning & \vabove{7.38 {\tiny[7.04, 7.74]}} & \vabove{7.38 {\tiny[7.04, 7.74]}} & 1.00 & \vabove{7.38 {\tiny[7.04, 7.74]}} & \vabove{6.85 {\tiny[6.54, 7.18]}} & \vabove{1.08 {\tiny[1.07, 1.09]}} & 20 \\
\rowcolor{tQbd}Creative writing & \vabove{28.64 {\tiny[27.83, 29.55]}} & \vabove{28.64 {\tiny[27.83, 29.55]}} & \vna{2.00} & \vna{57.28} & \vabove{28.24 {\tiny[27.56, 29.03]}} & 1.01 {\tiny[1.00, 1.03]} & 4 \\
Summarization & \vabove{12.50 {\tiny[12.07, 12.97]}} & \vabove{12.50 {\tiny[12.07, 12.97]}} & 1.00 & \vabove{12.50 {\tiny[12.07, 12.97]}} & \vabove{12.80 {\tiny[12.35, 13.28]}} & \vbelow{0.98 {\tiny[0.97, 0.99]}} & 20 \\
\rowcolor{tQbd}Translation (Hindi) & \vabove{3.29 {\tiny[2.76, 4.28]}} & \vabove{3.29 {\tiny[2.76, 4.28]}} & 1.00 & \vabove{3.29 {\tiny[2.76, 4.28]}} & \vabove{3.38 {\tiny[2.84, 4.39]}} & \vbelow{0.97 {\tiny[0.96, 0.98]}} & 20 \\
English to Hindi & \vabove{3.45 {\tiny[2.77, 4.42]}} & \vabove{3.45 {\tiny[2.77, 4.42]}} & 1.05 & \vabove{3.63 {\tiny[2.77, 5.05]}} & \vabove{3.54 {\tiny[2.85, 4.57]}} & 0.97 {\tiny[0.95, 1.00]} & 19 \\
\rowcolor{tQbd}Keyword extraction & \vabove{35.33 {\tiny[31.13, 42.39]}} & \vabove{35.33 {\tiny[31.13, 42.39]}} & 1.00 & \vabove{35.33 {\tiny[31.13, 42.39]}} & \vabove{41.00 {\tiny[36.29, 49.25]}} & \vbelow{0.86 {\tiny[0.84, 0.89]}} & 20 \\
Literature synthesis & \multicolumn{7}{l}{\vna{not eligible}} \\
\rowcolor{tQbd}Geography & \vabove{51.21 {\tiny[45.73, 56.38]}} & \vabove{51.21 {\tiny[45.73, 56.38]}} & 1.00 & \vabove{51.21 {\tiny[45.73, 56.38]}} & \vabove{70.66 {\tiny[62.42, 77.87]}} & \vbelow{0.72 {\tiny[0.67, 0.79]}} & 20 \\
Methodology comparison & \multicolumn{7}{l}{\vna{not eligible}} \\
\rowcolor{tQbd}Multilingual summary & \vabove{3.21 {\tiny[3.07, 3.36]}} & \vabove{3.21 {\tiny[3.07, 3.36]}} & 1.00 & \vabove{3.21 {\tiny[3.07, 3.36]}} & \vabove{3.10 {\tiny[2.96, 3.23]}} & \vabove{1.04 {\tiny[1.03, 1.05]}} & 20 \\
Technical blog post & \vabove{13.90 {\tiny[13.30, 14.49]}} & \vabove{13.90 {\tiny[13.30, 14.49]}} & 0.67 & \vabove{9.27 {\tiny[4.93, 15.11]}} & \vabove{13.37 {\tiny[12.91, 13.81]}} & \vabove{1.04 {\tiny[1.02, 1.06]}} & 10 \\
\rowcolor{tQbd}News summary & \vabove{9.31 {\tiny[9.03, 9.62]}} & \vabove{9.31 {\tiny[9.03, 9.62]}} & 1.18 & \vabove{10.95 {\tiny[9.37, 13.74]}} & \vabove{10.09 {\tiny[9.83, 10.34]}} & \vbelow{0.92 {\tiny[0.91, 0.94]}} & 17 \\
Research summary & \vabove{26.65 {\tiny[26.10, 27.32]}} & \vabove{26.65 {\tiny[26.10, 27.32]}} & 1.00 & \vabove{26.65 {\tiny[26.10, 27.32]}} & \vabove{29.99 {\tiny[29.54, 30.47]}} & \vbelow{0.89 {\tiny[0.87, 0.91]}} & 20 \\
\rowcolor{tQbd}Sentiment & \vabove{11.33 {\tiny[9.82, 13.41]}} & \vabove{11.33 {\tiny[9.82, 13.41]}} & 1.00 & \vabove{11.33 {\tiny[9.82, 13.41]}} & \vabove{13.82 {\tiny[11.60, 17.27]}} & \vbelow{0.82 {\tiny[0.78, 0.85]}} & 20 \\
Topic classification & \vabove{37.89 {\tiny[36.63, 39.28]}} & \vabove{37.89 {\tiny[36.63, 39.28]}} & 1.00 & \vabove{37.89 {\tiny[36.63, 39.28]}} & \vabove{54.94 {\tiny[53.61, 56.37]}} & \vbelow{0.69 {\tiny[0.67, 0.71]}} & 20 \\
\rowcolor{tQbd}Research abstract & \vabove{6.48 {\tiny[5.95, 7.05]}} & \vabove{6.48 {\tiny[5.95, 7.05]}} & \vna{0.40} & \vna{2.59} & \vabove{6.50 {\tiny[6.36, 6.66]}} & 1.00 {\tiny[0.92, 1.07]} & 2 \\
\bottomrule
\end{tabular}

%% file: generated/tab_split.tex
\setlength{\tabcolsep}{2.6pt}\begin{tabular}{@{}llrrrrr@{}}
\toprule
Set & & kJ & Not done & Wrong & Unkn. & Correct \\
\midrule
\rowcolor{tXbd}x86 TinyLlama & ag. & 630.7 & \textcolor{fgray}{19\%} & \textcolor{fwrong}{65\%} & \textcolor{fna}{0\%} & \textcolor{fnavy}{\textbf{17\%}} \\
\rowcolor{tXbd} & lin. & 156.5 & \textcolor{fgray}{0\%} & \textcolor{fwrong}{90\%} & \textcolor{fna}{0\%} & \textcolor{fnavy}{\textbf{10\%}} \\
\rowcolor{tNbd}x86 client NIM 70B & ag. & 5.7 & \textcolor{fgray}{0\%} & \textcolor{fwrong}{5\%} & \textcolor{fna}{0\%} & \textcolor{fnavy}{\textbf{95\%}} \\
\rowcolor{tNbd} & lin. & 1.0 & \textcolor{fgray}{0\%} & \textcolor{fwrong}{7\%} & \textcolor{fna}{0\%} & \textcolor{fnavy}{\textbf{93\%}} \\
\rowcolor{tGbd}GB10 TinyLlama & ag. & 219.4 & \textcolor{fgray}{0\%} & \textcolor{fwrong}{77\%} & \textcolor{fna}{0\%} & \textcolor{fnavy}{\textbf{23\%}} \\
\rowcolor{tGbd} & lin. & 44.3 & \textcolor{fgray}{0\%} & \textcolor{fwrong}{90\%} & \textcolor{fna}{0\%} & \textcolor{fnavy}{\textbf{10\%}} \\
\rowcolor{tMbd}GB10 Mistral & ag. & 519.3 & \textcolor{fgray}{0\%} & \textcolor{fwrong}{12\%} & \textcolor{fna}{0\%} & \textcolor{fnavy}{\textbf{88\%}} \\
\rowcolor{tMbd} & lin. & 96.5 & \textcolor{fgray}{0\%} & \textcolor{fwrong}{34\%} & \textcolor{fna}{0\%} & \textcolor{fnavy}{\textbf{66\%}} \\
\rowcolor{tQbd}GB10 Qwen3 & ag. & 368.5 & \textcolor{fgray}{0\%} & \textcolor{fwrong}{0\%} & \textcolor{fna}{0\%} & \textcolor{fnavy}{\textbf{100\%}} \\
\rowcolor{tQbd} & lin. & 31.4 & \textcolor{fgray}{0\%} & \textcolor{fwrong}{26\%} & \textcolor{fna}{0\%} & \textcolor{fnavy}{\textbf{74\%}} \\
\rowcolor{tCbd}Mac TinyLlama & ag. & 101.6 & \textcolor{fgray}{0\%} & \textcolor{fwrong}{75\%} & \textcolor{fna}{0\%} & \textcolor{fnavy}{\textbf{25\%}} \\
\rowcolor{tCbd} & lin. & 17.0 & \textcolor{fgray}{0\%} & \textcolor{fwrong}{73\%} & \textcolor{fna}{0\%} & \textcolor{fnavy}{\textbf{27\%}} \\
\rowcolor{tQbd}GB10 Qwen3 distinct & ag. & 481.2 & \textcolor{fgray}{33\%} & \textcolor{fwrong}{6\%} & \textcolor{fna}{0\%} & \textcolor{fnavy}{\textbf{60\%}} \\
\rowcolor{tQbd} & lin. & 32.7 & \textcolor{fgray}{0\%} & \textcolor{fwrong}{8\%} & \textcolor{fna}{0\%} & \textcolor{fnavy}{\textbf{92\%}} \\
\rowcolor{tMbd}GB10 Mistral distinct & ag. & 752.7 & \textcolor{fgray}{74\%} & \textcolor{fwrong}{21\%} & \textcolor{fna}{0\%} & \textcolor{fnavy}{\textbf{6\%}} \\
\rowcolor{tMbd} & lin. & 38.5 & \textcolor{fgray}{0\%} & \textcolor{fwrong}{62\%} & \textcolor{fna}{0\%} & \textcolor{fnavy}{\textbf{38\%}} \\
\rowcolor{tQbd}GB10 Qwen3 further types & ag. & 460.2 & \textcolor{fgray}{0\%} & \textcolor{fwrong}{19\%} & \textcolor{fna}{0\%} & \textcolor{fnavy}{\textbf{81\%}} \\
\rowcolor{tQbd} & lin. & 53.5 & \textcolor{fgray}{0\%} & \textcolor{fwrong}{13\%} & \textcolor{fna}{0\%} & \textcolor{fnavy}{\textbf{87\%}} \\
\rowcolor{tBbd}BFCL A to E & ag. & 6,723.2 & \textcolor{fgray}{4\%} & \textcolor{fwrong}{14\%} & \textcolor{fna}{0\%} & \textcolor{fnavy}{\textbf{81\%}} \\
\bottomrule
\end{tabular}

%% file: generated/tab_failures.tex
\setlength{\tabcolsep}{3.5pt}\begin{tabular}{@{}lrrrrrr@{}}
\toprule
Goal type & Fail. & Att. & Calls & \EpG{} (J) & Per call (J) & Incr. (J) \\
\midrule
Factual & 0 & 1.0 & 6.0 & 2,911 & 485 &  \\
 & 1 & 1.9 & 11.4 & 5,916 & 519 &  \\
 & 2 & 2.8 & 16.8 & 8,520 & 507 & \textbf{3,116} \\
\rowcolor{fzebra}Science & 0 & 1.0 & 3.0 & 1,144 & 382 &  \\
\rowcolor{fzebra} & 1 & 1.9 & 5.7 & 1,854 & 325 &  \\
\rowcolor{fzebra} & 2 & 2.8 & 8.4 & 3,224 & 384 & \textbf{1,155} \\
Logic & 0 & 1.0 & 7.0 & 3,895 & 556 &  \\
 & 1 & 1.9 & 13.3 & 7,687 & 578 &  \\
 & 2 & 2.8 & 19.6 & 11,469 & 585 & \textbf{4,208} \\
\rowcolor{fzebra}Commonsense & 0 & 1.0 & 7.0 & 9,657 & 1,380 &  \\
\rowcolor{fzebra} & 1 & 1.9 & 13.3 & 18,356 & 1,380 &  \\
\rowcolor{fzebra} & 2 & 2.8 & 19.6 & 25,587 & 1,306 & \textbf{8,850} \\
Arithmetic & 0 & 1.0 & 5.0 & 2,072 & 414 &  \\
 & 1 & 1.9 & 9.5 & 3,435 & 362 &  \\
 & 2 & 2.8 & 14.0 & 5,468 & 391 & \textbf{1,887} \\
\rowcolor{fzebra}Multi step & 0 & 1.0 & 5.0 & 2,966 & 593 &  \\
\rowcolor{fzebra} & 1 & 1.9 & 9.5 & 5,542 & 583 &  \\
\rowcolor{fzebra} & 2 & 2.8 & 14.0 & 8,018 & 573 & \textbf{2,807} \\
Calculator & 0 & 1.0 & 1.0 & 261 & 261 &  \\
 & 1 & 1.9 & 1.9 & 1,324 & 697 &  \\
 & 2 & 2.8 & 2.8 & 2,265 & 809 & \textbf{1,113} \\
\rowcolor{fzebra}Database & 0 & 1.0 & 1.0 & 442 & 442 &  \\
\rowcolor{fzebra} & 1 & 1.9 & 1.9 & 1,535 & 808 &  \\
\rowcolor{fzebra} & 2 & 2.8 & 2.8 & 2,705 & 966 & \textbf{1,257} \\
Two tools & 0 & 1.0 & 1.0 & 1,030 & 1,030 &  \\
 & 1 & 1.9 & 1.9 & 2,499 & 1,316 &  \\
 & 2 & 2.8 & 2.8 & 3,621 & 1,293 & \textbf{1,439} \\
\bottomrule
\end{tabular}

%% file: generated/tab_bgfilter.tex
\setlength{\tabcolsep}{2.5pt}\begin{tabular}{@{}lrrcr@{}}
\toprule
Goal type & \cellcolor{tChd}all runs & \cellcolor{tChd}filtered & min.\ succ. & Change \\
\midrule
\rowcolor{tCbd}Factual & 4.5 {\scriptsize[3.5, 6.1]} & 4.2 {\scriptsize[2.9, 7.4]} & 12\,/\,12 & -5\% \\
Science & 7.6 {\scriptsize[5.3, 12.0]} & 17.7 {\scriptsize[8.0, 247.0]} & \vna{1}\,/\,\vna{1} & +133\% \\
\rowcolor{tCbd}Logic & 4.0 {\scriptsize[3.1, 5.4]} & 6.0 {\scriptsize[3.9, 10.9]} & \vna{9}\,/\,\vna{9} & +50\% \\
Commonsense & 5.6 {\scriptsize[4.3, 7.6]} & 7.5 {\scriptsize[4.8, 13.9]} & \vna{3}\,/\,\vna{3} & +35\% \\
\rowcolor{tCbd}Arithmetic & 3.7 {\scriptsize[2.7, 5.0]} & 4.0 {\scriptsize[2.9, 5.4]} & \vna{1}\,/\,\vna{1} & +7\% \\
Multi step & 3.7 {\scriptsize[2.5, 5.5]} & 3.2 {\scriptsize[1.9, 6.4]} & \vna{0}\,/\,\vna{0} & -14\% \\
\bottomrule
\end{tabular}

%% file: generated/tab_sensitivity.tex
\setlength{\tabcolsep}{3.5pt}\begin{tabular}{@{}lrrrrr@{}}
\toprule
 & \multicolumn{2}{c}{\cellcolor{tXhd}x86} & \multicolumn{3}{c}{\cellcolor{tMhd}GB10} \\
Goal type & clean up & idle & clean up & idle & process (\%) \\
\midrule
Factual & \textcolor{fna}{1.00} & \textcolor{fna}{1.01} & \textbf{0.94} & \textcolor{fna}{0.99} & \vabove{5.3} \\
\rowcolor{fzebra}Science & \textcolor{fna}{1.00} & \textcolor{fna}{1.01} & \textbf{0.91} & \textbf{0.98} & \vabove{5.7} \\
Logic & \textcolor{fna}{1.00} & \textcolor{fna}{1.01} & \textbf{0.93} & \textbf{0.98} & \vabove{6.0} \\
\rowcolor{fzebra}Commonsense & \vna{--} & \vna{--} & \textbf{0.96} & \textbf{0.98} & \vabove{6.1} \\
Arithmetic & \textcolor{fna}{1.00} & \textcolor{fna}{1.01} & \textbf{0.93} & \textbf{0.98} & \vabove{5.2} \\
\rowcolor{fzebra}Multi step & \textcolor{fna}{0.99} & \textcolor{fna}{0.99} & \textbf{0.98} & \textcolor{fna}{0.99} & \vabove{6.1} \\
Hard reasoning & \vna{--} & \vna{--} & \textbf{0.96} & \textcolor{fna}{0.99} & \vabove{5.8} \\
\rowcolor{fzebra}GSM8K (distinct) & \vna{--} & \vna{--} & \textbf{0.91} & \textbf{0.97} & \vabove{4.7} \\
MMLU (distinct) & \vna{--} & \vna{--} & \textbf{0.89} & \textbf{0.96} & \vabove{5.3} \\
\rowcolor{fzebra}HumanEval (distinct) & \vna{--} & \vna{--} & \textbf{0.91} & \textbf{0.96} & \vabove{5.0} \\
Planning & \vna{--} & \vna{--} & \textcolor{fna}{0.99} & \textbf{0.98} & \vabove{5.7} \\
\rowcolor{fzebra}Creative writing & \vna{--} & \vna{--} & \textbf{0.83} & \textbf{0.95} & \vabove{5.6} \\
Summarization & \vna{--} & \vna{--} & \textbf{0.78} & \textbf{0.95} & \vabove{5.6} \\
\rowcolor{fzebra}Translation (Hindi) & \vna{--} & \vna{--} & \textbf{0.83} & \textbf{0.96} & \vabove{6.0} \\
English to Hindi & \vna{--} & \vna{--} & \textbf{0.82} & \textbf{0.96} & \vabove{5.6} \\
\rowcolor{fzebra}Keyword extraction & \vna{--} & \vna{--} & \textbf{0.53} & \textbf{0.87} & \vabove{5.6} \\
Literature synthesis & \vna{--} & \vna{--} & \textbf{0.98} & \textbf{0.98} & \vabove{5.7} \\
\rowcolor{fzebra}Geography & \vna{--} & \vna{--} & \textbf{0.39} & \textbf{0.79} & \vabove{5.6} \\
Methodology comparison & \vna{--} & \vna{--} & \textbf{0.98} & \textbf{0.98} & \vabove{5.6} \\
\rowcolor{fzebra}Multilingual summary & \vna{--} & \vna{--} & \textbf{0.95} & \textbf{0.98} & \vabove{5.5} \\
Technical blog post & \vna{--} & \vna{--} & \textbf{0.89} & \textbf{0.96} & \vabove{5.7} \\
\rowcolor{fzebra}News summary & \vna{--} & \vna{--} & \textbf{0.65} & \textbf{0.91} & \vabove{5.7} \\
Research summary & \vna{--} & \vna{--} & \textbf{0.57} & \textbf{0.88} & \vabove{5.9} \\
\rowcolor{fzebra}Sentiment & \vna{--} & \vna{--} & \textbf{0.54} & \textbf{0.86} & \vabove{6.0} \\
Topic classification & \vna{--} & \vna{--} & \textbf{0.36} & \textbf{0.77} & \vabove{5.8} \\
\rowcolor{fzebra}Research abstract & \vna{--} & \vna{--} & \textbf{0.87} & \textbf{0.97} & \vabove{5.1} \\
\bottomrule
\end{tabular}

%% file: generated/tab_excl.tex
\setlength{\tabcolsep}{4pt}\begin{tabular}{@{}llrp{0.55\textwidth}@{}}
\toprule
Rule & Goal type & Goals & Reason \\
\midrule
\rowcolor{tXhd}\multicolumn{4}{@{}l}{\textbf{x86, TinyLlama 1B}} \\
\rowcolor{tXbd}invalid experiment & all & \textbf{2} & runs marked invalid when collected: test, debugging or interrupted runs \\
measurement fault & Three tools & \textbf{60} & database and API tools were unavailable: every call failed because of a file path error in our harness \\
\rowcolor{tXbd}measurement fault & Hard reasoning & \textbf{38} & run only on x86 (19 goals), so no comparison across machines \\
\midrule
\rowcolor{tMhd}\multicolumn{4}{@{}l}{\textbf{GB10, Mistral 7B}} \\
\rowcolor{tMbd}measurement fault & Database & \textbf{60} & database tool was unavailable: all 120 calls failed because of a file path error in our harness, since fixed \\
measurement fault & Two tools & \textbf{60} & database tool was unavailable: all 120 calls failed because of a file path error in our harness, since fixed \\
\rowcolor{tMbd}measurement fault & Three tools & \textbf{60} & database and API tools were unavailable: every call failed because of a file path error in our harness \\
\midrule
\rowcolor{tBhd}\multicolumn{4}{@{}l}{\textbf{GB10, BFCL}} \\
\rowcolor{tBbd}data integrity & all & \textbf{3} & goal failed but last attempt success \\
data integrity & all & \textbf{1} & last attempt failure but checker pass \\
\midrule
\rowcolor{tNhd}\multicolumn{4}{@{}l}{\textbf{x86 client, Llama 3.3 70B via NVIDIA NIM}} \\
\rowcolor{tNbd}invalid experiment & all & \textbf{18} & runs marked invalid when collected: test, debugging or interrupted runs \\
\midrule
\rowcolor{tThd}\multicolumn{4}{@{}l}{\textbf{GB10, Mistral 7B, tool runs}} \\
\rowcolor{tTbd}measurement fault & Three tools & \textbf{60} & API stub on localhost:8765 was down during the run: no agentic goal completed \\
\midrule
\rowcolor{tChd}\multicolumn{4}{@{}l}{\textbf{Mac, TinyLlama 1B}} \\
\rowcolor{tCbd}measurement fault & Calculator & \textbf{87} & run stopped after 43 of 60 goals: llama.cpp ran out of GPU memory (Metal) \\
measurement fault & Two tools & \textbf{120} & energy recorded as 0 after the GPU memory failure \\
\rowcolor{tCbd}high background load & all & \textbf{354} & busy CPU: other processes used at least 50\% CPU during the run \\
\bottomrule
\end{tabular}